\documentclass{article}

\usepackage[final]{neurips_2019}
\usepackage[numbers]{natbib}
\usepackage{amsthm}

\usepackage{wrapfig}

\newtheorem{theorem}{Theorem}[section]

\newtheorem{proposition}[theorem]{Proposition}
\newtheorem{lemma}[theorem]{Lemma}
\newtheorem{corollary}[theorem]{Corollary}

\usepackage{algorithm}
\usepackage{algpseudocode}
\usepackage{caption}
\usepackage{comment}
\usepackage{dsfont}
\usepackage{graphicx,verbatim}
\usepackage{listings}
\usepackage{subcaption}
\usepackage{todonotes}
\usepackage{amsmath}

\usepackage{tikz}
\usetikzlibrary{arrows}

\newcommand{\eat}[1]{ }

\def\ie{{\it i.e.},~}
\def\eg{{\it e.g.},~}

\def\one{\mathds{1}}

\def\q{\hat{y}^{(\textrm{leaf})}}

\newcommand{\x}{\mathbf{x}}
\newcommand{\y}{\mathbf{y}}
\def\RL{{d}}
\def\Prefix{d_{un}}
\def\Labels{\delta_{un}}
\def\Default{d_{\textrm{\rm split}}}
\def\DefaultLabels{\delta_{\textrm{\rm split}}}
\def\RLB{{D}}
\def\PrefixB{D_{un}}
\def\LabelsB{\Delta_{un}}
\def\DefaultB{D_{\textrm{\rm split}}}
\def\DefaultLabelsB{\Delta_{\textrm{\rm split}}}
\def\Obj{R}
\def\Loss{\ell}
\def\Reg{{\lambda}}

\def\RuleSet{S}
\def\Cap{\textnormal{cap}}
\def\Supp{\textnormal{supp}}

\def\InitialObj{R^0}
\def\InitialRL{d^0}
\def\CurrentObj{{R^c}}
\def\CurrentRL{d^c}
\def\OptimalObj{R^*}
\def\OptimalRL{d^*}
\def\Remaining{\Gamma}
\def\TotalRemaining{{\Gamma_{\textnormal{tot}}}}
\def\StartsWith{\sigma}
\def\StartContains{\phi}
\def\Queue{Q}
\DeclareMathOperator*{\argmin}{argmin}
\DeclareMathOperator*{\argmax}{argmax}

\def\Curiosity{\mathcal{C}}
\def\NCap{N_{\textnormal{cap}}}

\def\leaf{\textrm{leaf}}

\def\one{\mathds{1}}

\newcommand{\nn}{\nonumber}
\newcommand{\be}{\begin{equation}}
\newcommand{\ee}{\end{equation}}
\newcommand{\bea}{\begin{eqnarray}}
\newcommand{\eea}{\end{eqnarray}}

\newcommand{\given}{\,|\,}

\includecomment{kdd}\excludecomment{arxiv}

\usepackage[utf8]{inputenc} 
\usepackage[T1]{fontenc}    
\usepackage{hyperref}       
\usepackage{url}            
\usepackage{booktabs}       
\usepackage{amsfonts}       
\usepackage{nicefrac}       
\usepackage{microtype}      

\title{Optimal Sparse Decision Trees}

\author{%
  Xiyang Hu$^{1}$, Cynthia Rudin$^{2}$, Margo Seltzer$^{3}$\thanks{Authors are listed alphabetically.} \\
  $^{1}$Carnegie Mellon University, \texttt{xiyanghu@cmu.edu} \\
  $^{2}$Duke University, \texttt{cynthia@cs.duke.edu} \\
  $^{3}$The University of British Columbia, \texttt{mseltzer@cs.ubc.ca} \\
}

\begin{document}

\maketitle

\begin{abstract}
  Decision tree algorithms have been among the most popular algorithms for interpretable (transparent) machine learning since the early 1980's. The problem that has plagued decision tree algorithms since their inception is their lack of optimality, or lack of guarantees of closeness to optimality: decision tree algorithms are often greedy or myopic, and sometimes produce unquestionably suboptimal models. Hardness of decision tree optimization is both a theoretical and practical obstacle, and even careful mathematical programming approaches have not been able to solve these problems efficiently. This work introduces the first practical algorithm for optimal decision trees for binary variables. The algorithm is a co-design of analytical bounds that reduce the search space and modern systems techniques, including data structures and a custom bit-vector library. Our experiments highlight advantages in scalability, speed, and proof of optimality.
  The code is available at \url{https://github.com/xiyanghu/OSDT}.
\end{abstract}

\section{Introduction}

Interpretable machine learning has been growing in importance as society has begun to realize the dangers of using black box models for high stakes decisions: complications with confounding have haunted our medical machine learning models \cite{Zech2018}, bad predictions from black boxes have announced to millions of people that their dangerous levels of air pollution were safe \cite{McGough2018}, high-stakes credit risk decisions are being made without proper justification, and black box risk predictions have been wreaking havoc with the perception of fairness of our criminal justice system \cite{flores16}. 
In all of these applications -- medical imaging, pollution modeling, recidivism risk, credit scoring -- accurate interpretable models have been created (by the Center for Disease Control and Prevention, Arnold Foundation, and others). 
However, such interpretable-yet-accurate models are not generally easy to construct. If we want people to replace their black box models with interpretable models, the tools to build these interpretable models must first exist. 

Decision trees are one of the leading forms of interpretable models. Despite several attempts over the last several decades to improve the optimality of decision tree algorithms, the CART \cite{Breiman84} and C4.5 \cite{Quinlan93} decision tree algorithms (and other greedy tree-growing variants) have remained as dominant methods in practice.  CART and C4.5  grow decision trees from the top down without backtracking, which means that  if a suboptimal split was introduced near the top of the tree, the algorithm could spend many extra splits trying to undo the mistake it made at the top, leading to less-accurate and less-interpretable trees. Problems with greedy splitting and pruning have been known since the early 1990's, when mathematical programming tools had started to be used for creating optimal binary-split decision trees \cite{Bennett92,Bennett96optimaldecision}, in a line of work \cite{bertsimas2017optimal, molerooptimal, MenickellyGKS18, nijssen2007mining} until the present \cite{verwer2019learning}. However, these techniques use all-purpose optimization toolboxes and tend not to scale to realistically-sized problems unless simplified to trees of a specific form. Other works \cite{Klivans06} make overly strong assumptions (e.g., independence of all variables) to ensure optimal trees are produced using greedy algorithms. 

We produce optimal sparse decision trees taking a different approach than mathematical programming, greedy methods, or brute force. We find optimal trees according to a regularized loss function that balances accuracy and the number of leaves. Our algorithm is computationally efficient due to a collection of analytical bounds to perform massive pruning of the search space. Our implementation uses specialized data structures to store intermediate computations and symmetries, a bit-vector library to evaluate decision trees more quickly, fast search policies, and computational reuse. 
Despite the hardness of finding optimal solutions, our algorithm is able to locate optimal trees and prove optimality (or closeness of optimality) in reasonable amounts of time for datasets of the sizes used in the criminal justice system (tens of thousands or millions of observations, tens of features).


Because we find provably optimal trees, our experiments show where previous studies have claimed to produce optimal models yet failed; we show specific cases where this happens. 
We test our method on benchmark data sets, as well as criminal recidivism and credit risk data sets;
%
these are two of the high-stakes decision problems where interpretability is needed most in AI systems.
%
We provide ablation experiments to show which of our techniques is most influential at reducing computation for various datasets. As a result of this analysis, we are able to pinpoint possible future paths to improvement for scalability and computational speed.
Our contributions are: (1) The first practical optimal binary-variable decision tree algorithm to achieve solutions for nontrivial problems. (2) A series of analytical bounds to reduce the search space. (3) Algorithmic use of a tree representation using only its leaves. (4) Implementation speedups saving 97\% run time. (5) We present the first optimal sparse binary split trees ever published for the COMPAS and FICO datasets.

The code and the supplementary materials are available at \url{https://github.com/xiyanghu/OSDT}.

\section{Related Work}
\label{sec:related-work}
Optimal decision trees have a quite a long history \cite{Bennett92}, so we focus on closely related techniques. 
There are efficient algorithms that claim to generate optimal sparse trees, but do not optimally balance the criteria of optimality and sparsity; instead they pre-specify the topology of the tree (\ie they know \textit{a priori} exactly what the structure of the splits and leaves are, even though they do not know which variables are split) and only find the optimal tree of the given topology \cite{MenickellyGKS18}. 
This is not the problem we address, as we do not 
know the topology of the optimal tree in advance. The most successful algorithm of this variety is BinOCT  \cite{verwer2019learning}, which searches for a 
complete binary tree of a given depth;
 we discuss BinOCT shortly. Some exploration of learning optimal decision trees is based on boolean satisfiability (SAT) \cite{narodytska2018learning}, but again, this work looks only for the optimal tree of a given number of nodes.
The DL8 algorithm \cite{nijssen2007mining} optimizes a ranking function to find a decision
tree under constraints of size, depth, accuracy and leaves. DL8 creates trees from the bottom up, meaning that trees are assembled out of all possible leaves, which are itemsets pre-mined from the data \cite[similarly to][]{AngelinoEtAl18}. DL8 does not have publicly available source code, and its authors warn about running out of memory when storing all partially-constructed trees.
Some works  
consider oblique trees \cite{molerooptimal}, where splits involve several variables; oblique trees are not addressed here, as they can be less interpretable.

The most recent mathematical programming algorithms are OCT  \cite{bertsimas2017optimal} and BinOCT  \cite{verwer2019learning}. Example figures from the OCT paper \cite{bertsimas2017optimal} show decision trees that are clearly suboptimal. However, as the code was not made public, the work in the OCT paper \cite{bertsimas2017optimal} is not easily reproducible, so it is not clear where the problem occurred. We discuss this in Section \S\ref{sec:experiments}.  Verwer and Zhang's mathematical programming formulation for BinOCT is much faster~\cite{verwer2019learning}, and their experiments indicate that BinOCT outperforms OCT, but since BinOCT is constrained to create complete binary trees of a given depth rather than optimally sparse trees, it sometimes creates unnecessary leaves in order to complete a tree at a given depth, as we show in Section~\S\ref{sec:experiments}.
BinOCT solves a dramatically easier problem than the method introduced in this work. As it turns out, the search space of perfect binary trees of a given depth is much smaller than that of binary trees with the same number of leaves. For instance, the number of different unlabeled binary trees with 8 leaves is $Catalan(7)=429$, but the number of unlabeled perfect binary trees with 8 leaves is only 1. In our setting, we penalize (but do not fix) the number of leaves, which means that our search space contains all trees, though we can bound the maximum number of leaves based on the size of the regularization parameter. Therefore, our search space is much larger than that of BinOCT.

Our work builds upon the CORELS algorithm
~\cite{AngelinoLaAlSeRu17-kdd,AngelinoEtAl18,Larus-Stone18-sysml} and its predecessors~\cite{LethamRuMcMa15,YangRuSe16}, which create optimal decision lists (rule lists). Applying those ideas to decision trees is nontrivial. The rule list optimization problem is much easier, since the rules are pre-mined (there is no mining of rules in our decision tree optimization). Rule list optimization involves selecting an optimal subset and an optimal permutation of the rules in each subset. Decision tree optimization involves considering every possible split of every possible variable and every possible shape and size of tree. This is an exponentially harder optimization problem with a huge number of symmetries to consider. In addition, in CORELS, the maximum number of clauses per rule is set to be $c=2$. For a data set with $p$ binary features, there would be $D = p+\binom{p}{2}$ rules in total, and the number of distinct rule lists with $d_{r}$ rules is $P(D, d_{r})$, where~$P(m, k)$ is the number of $k$-permutations of~$m$. Therefore, the search space of CORELS is $\sum_{d_{r}=1}^{d-1} P(D, d_{r})$. But, for a full binary tree with depth $d_t$ and data with $p$ binary features, the number of distinct trees is:
\begin{align}
N_{d_t} =& \sum_{n_{0}=1}^{1} \sum_{n_{1}=1}^{2{n_0}} \dots \sum_{n_{d_{t}-1}=1}^{2{n_{d_{t}-2}}} p \times \binom{2{n_0}}{n_1}(p-1)^{n_1} \times \dots
&\times \binom{2{n_{d_t-2}}}{n_{d_t-1}}(p-(d_t-1))^{n_{d_t-1}},
\label{eq:number-tree-depth}
\end{align}
and the search space of decision trees up to depth $d$ is $\sum_{d_t=1}^{d} N_{d_t}$.
Table~\ref{tab:search-sapce} shows how the search spaces of rule lists and decision trees grow as the tree depth increases. The search space of the trees is massive compared to that of the rule lists. 
\begin{table}
\centering
Search Space of CORELS and Decision Trees \\
\vspace{1mm}
\begin{tabular}{l | c | c | c | c }
\hline
 & \multicolumn{2}{ c |}{$p=10$} & \multicolumn{2}{ c }{$p=20$}\\
\hline
$d$ & Rule Lists & Trees &  Rule Lists & Trees \\
\hline
1 & $5.500\times 10^1$ & $1.000\times 10^1$& $2.100\times 10^2$ & $2.000\times 10^1$ \\
2 & $3.025\times 10^3$ & $1.000\times 10^3$ & $4.410\times 10^4$ & $8.000\times 10^3$ \\
3 & $1.604\times 10^5$ & $5.329\times 10^6$ & $9.173\times 10^6$ & $9.411\times 10^8$ \\
4 & $8.345\times 10^6$ & $5.609\times 10^{13}$ & $1.898\times 10^9$ & $8.358\times 10^{18}$ \\
5 & $4.257\times 10^8$ & ``Inf'' & $3.911\times 10^{11}$ & ``Inf'' \\
\end{tabular}
\caption{Search spaces of rule lists and decision trees with number of variables $p=10,20$ and depth $d=1,2,3,4,5$. The search space of the trees explodes in comparison.\vspace*{-20pt}
}
\label{tab:search-sapce}
\end{table}

Applying techniques from rule lists to decision trees necessitated new designs for the data structures, splitting mechanisms and bounds. 
An important difference between rule lists and trees is that during the growth of rule lists, we add only one new rule to the list each time, but for the growth of trees, we need to split existing leaves and add a new \textit{pair} of leaves for each. This leads to several bounds that are quite different from those in CORELS, \ie Theorem~\ref{thm:incre-min-capture-correct}, Theorem~\ref{thm:min-capture-correct} and Corollary~\ref{thm:permutation}, which consider a pair of leaves rather than a single leaf. In this paper, we introduce bounds only for the case of one split at a time; however, in our implementation, we can split more than one leaf at a time, and the bounds are adapted accordingly. 

\section{Optimal Sparse Decision Trees (OSDT)}
\label{sec:framework}


We focus on binary classification, although
it is possible to generalize 
this framework to multiclass settings.
We denote training data as~${\{(x_n, y_n)\}_{n=1}^N}$,
where ${x_n \in \{0, 1\}^M}$ are binary features and ${y_n \in \{0,1\}}$ are labels.
Let~${\x = \{x_n\}_{n=1}^N}$ and~${\y = \{y_n\}_{n=1}^N}$,
and let~${x_{n,m}}$ denote the $m$-th feature of~$x_n$.
For a decision tree, its leaves are conjunctions of predicates. Their order does not matter in evaluating the accuracy of the tree, and a tree grows only at its leaves.
Thus, within our algorithm, we represent a tree as a collection of leaves.
A leaf set ${\RL = (p_1, p_2, \dots, p_H)}$ of length~${H \ge 0}$
is an ${H}$-tuple containing~$H$ distinct leaves, 
where $p_k$ is the classification rule of the path from the root to leaf $k$.
Here, $p_k$ is a Boolean assertion, which evaluates to either true or false for each datum~$x_n$ indicating whether it is classified by leaf $k$. Here, $\hat{y}^{(\leaf)}_k$ is the label for all points so classified.
%
%
%




We explore the search space by considering which leaves of the tree can be beneficially split.
The leaf set $\RL$ $=$ $(p_1, p_2, \dots,$ $p_K, p_{K+1}, $ $\dots,  p_H)$ is the $H$-leaf tree, where the first $K$ leaves may not to be split, and the remaining $H-K$ leaves can be split. 
We alternately represent this leaf set as
$\RL = (\Prefix, $
$\Labels, \Default, \DefaultLabels, K, H)$,
where ${\Prefix = (p_1, \dots, p_K)}$ are the unchanged leaves of $\RL$,
${\Labels = }$ ${(\q_1, \dots, \q_K) \in \{0, 1\}^K}$ are the predicted labels of leaves $\Prefix$, ${\Default = (p_{K+1}, \dots, p_H)}$ are the leaves we are going to split, and ${\DefaultLabels = (\q_{K+1}, \dots, }$ ${\q_H) \in \{0,}$ ${ 1\}^{H-K}}$ are
the predicted labels of leaves~$\Default$. 
We call $\Prefix $ a $K$-prefix of~$\RL$, which means its leaves are a size-$K$ unchanged subset of $(p_1, \dots, p_K, \dots, p_H)$.
If we have a new prefix~$\Prefix'$, which is a superset of $\Prefix$, \ie $\Prefix' \supseteq \Prefix$, then we say $\Prefix'$ starts with~$\Prefix$.
We define~$\StartsWith(\RL)$ to be all descendents of $\RL$:
\begin{equation}
\StartsWith(\RL) =
\{(\Prefix', \Labels', \Default', \DefaultLabels', K', H_{\RL'}) : \Prefix' \supseteq  \Prefix, \RL' \supset  \RL \}.
\label{eq:starts-with}
\end{equation}
If we have two trees $\RL = (\Prefix, \Labels, \Default, \DefaultLabels, K, H)$ and $\RL' = (\Prefix', \Labels', \Default', \DefaultLabels', K', H')$, where $H'=H+1, \RL' \supset  \RL, \Prefix' \supseteq  \Prefix$, \ie $\RL'$ contains one more leaf than $\RL$ and $\Prefix'$ starts with~$\Prefix$, then we define $\RL'$ to be a child of~$\RL$ and $\RL$ to be a parent of~$\RL'$.

Note that 
two trees with identical leaf sets, but 
different assignments to $\Prefix$ and $\Default$, are \emph{different} trees. Further, a child tree can \emph{only} be generated through splitting 
leaves of its parent tree within $\Default$.

A tree~$\RL$ classifies datum~$x_n$ by providing the label prediction $\q_{k}$
of the leaf whose $p_k$ is true for~$x_n$. Here, the leaf label $\q_{k}$ is the majority label of data captured by the leaf $k$.
If~$p_k$ evaluates to true for~$x_n$, we say the leaf~$k$ of leaf set~$\Prefix$ \textit{captures}~$x_n$ .
In our notation, all the data captured by a prefix's leaves are also captured by the prefix itself.

Let~$\beta$ be a set of leaves.
We define~${\Cap(x_n, \beta) = 1}$ if a leaf in~$\beta$
captures datum~$x_n$, and~0 otherwise.
For example, let~$\RL$ and~$\RL'$ be leaf sets such that~$\RL' \supset \RL$, then~$\RL'$ captures all the data that~$\RL$ captures:
\begin{arxiv}
\begin{align*}
\{x_n: \Cap(x_n, \RL)\} \subseteq \{x_n: \Cap(x_n, \RL')\}.
\end{align*}
\end{arxiv}
\begin{kdd}
${\{x_n: \Cap(x_n, \RL)\} \subseteq \{x_n: \Cap(x_n, \RL')\}}$.
\end{kdd}

%
The normalized support of $\beta$, denoted  ${\Supp(\beta, \x)}$, is the fraction of data captured by $\beta$:
\begin{align}
\Supp(\beta, \x) = \frac{1}{N} \sum_{n=1}^N \Cap(x_n, \beta).
\label{eq:support}
\end{align}

%

\subsection{Objective Function}
\label{sec:objective}

For a tree ${\RL = (\Prefix, \Labels, \Default, \DefaultLabels, K, H_{\RL})}$, we define its objective function as a combination of the misclassification error and a sparsity penalty on the number of leaves:
\begin{align}
\Obj(\RL, \x, \y) = \Loss(\RL, \x, \y) + \Reg H_{\RL}.
\label{eq:objective}
\end{align}
$\Obj(\RL,\x,\y)$ is a regularized empirical risk.
The loss~$\Loss(\RL, \x, \y)$ is the misclassification error of~$\RL$, \ie the fraction of training data with incorrectly predicted labels. $H_{\RL}$ is the number of leaves in the tree $\RL$. $\Reg H_{\RL}$ is a regularization term that penalizes bigger trees. Statistical learning theory provides guarantees for this problem; minimizing the loss subject to a (soft or hard) constraint on model size leads to a low upper bound on test error from the Occham's Razor Bound. 
%
%

\subsection{Optimization Framework}
\label{sec:optimization}

We minimize the objective function based on a branch-and-bound framework. We propose a series of specialized bounds that work together to eliminate a large part of the search space. These bounds are discussed in detail in the following paragraphs. Proofs are in the supplementary materials.

Some of our bounds could be adapted directly from CORELS~\cite{AngelinoEtAl18}, namely these two:\\
\textbf{(Hierarchical objective lower bound)} Lower bounds of a parent tree also hold for every child tree of that parent
(\S\ref{sec:hierarchical}, Theorem~\ref{thm:bound}).
\textbf{(Equivalent points bound)} For a given dataset, if there are multiple samples with exactly the same features but different labels, then no matter how we build our classifier, we will always make mistakes. The lower bound on the number of mistakes is therefore the number of such samples with minority class labels
(\S\ref{sec:identical}, Theorem~\ref{thm:identical}).

Some of our bounds adapt from CORELS~\cite{AngelinoLaAlSeRu17-kdd} with minor changes:
\textbf{(Objective lower bound with one-step lookahead)} 
With respect to the number of leaves, if a tree does not achieve enough accuracy,
we can prune all child trees of it
(\S\ref{sec:hierarchical}, Lemma~\ref{lemma:lookahead}).
\textbf{(A priori bound on the number of leaves)} For an optimal decision tree, we provide an \emph{a priori} upper bound on the maximum number of leaves
(\S\ref{sec:ub-prefix-length}, Theorem~\ref{thm:ub-prefix-specific}).
\textbf{(Lower bound on node support)} For an optimal tree, the support traversing through each internal node must be at least $2\Reg$
%
(\S\ref{sec:lb-support}, Theorem~\ref{thm:min-capture}).

Some of our bounds are distinct from CORELS, because they are only relevant to trees and not to lists:
\textbf{(Lower bound on incremental classification accuracy)} Each split must result in sufficient reduction of the loss. Thus, if the loss reduction is less than or equal to the regularization term, we should still split, and we have to further split at least one of the new child leaves to search for the optimal tree
(\S\ref{sec:lb-support}, Theorem~\ref{thm:incre-min-capture-correct}).
\textbf{(Leaf permutation bound)} We need to consider only one permutation of leaves in a tree;
we do not need to consider other permutations (explained in \S\ref{sec:permutation}, Corollary~\ref{thm:permutation}).
\textbf{(Leaf accurate support bound)} 
%
For each leaf in an optimal decision tree, the number of correctly classified samples must be above a threshold.
(\S\ref{sec:lb-support}, Theorem~\ref{thm:min-capture-correct}).
The supplement contains an additional set of bounds on the number of remaining tree evaluations.

\subsection{Hierarchical Objective Lower Bound}
\label{sec:hierarchical}

The loss can be decomposed
\begin{arxiv}
in~\eqref{eq:objective}
\end{arxiv}
into two parts corresponding to the unchanged leaves and the leaves to be split:
$\Loss(\RL, \x, \y) 
\equiv \Loss_p(\Prefix, \Labels, \x, \y) + \Loss_q(\Default, \DefaultLabels, \x, \y),$
where $\Prefix$ $=$ $(p_1, \dots, p_K)$, $\Labels$ $=$ $(\q_1, \dots, \q_K)$, $\Default$ $=$ $(p_{K+1}, $ $\dots, p_{H_{\RL}})$ and $\DefaultLabels$ $=$ $(\q_{K+1}, \dots, \q_{H_{\RL}})$;
$\Loss_p(\Prefix, \Labels, \x, \y) =
\frac{1}{N} \sum_{n=1}^N \sum_{k=1}^K \Cap(x_n, p_k) \wedge \one [ \q_k \neq y_n ]$
is the proportion of data in the unchanged leaves that are misclassified, and
$\Loss_p(\Default, \DefaultLabels, \x, \y) =
\frac{1}{N} \sum_{n=1}^N \sum_{k=K+1}^{H_{\RL}} \Cap(x_n, p_k) \wedge \nn 
\one [ \q_k \neq y_n ]$
is the proportion of data in the leaves we are going to split that are misclassified.
We define a lower bound~$b(\Prefix, \x, \y)$ on the objective by leaving out the latter loss,
\begin{align}
b(\Prefix, \x, \y) \equiv \Loss_p(\Prefix, \Labels, \x, \y) + \Reg H_{\RL} \le \Obj(\RL, \x, \y),
\label{eq:lower-bound}
\end{align}
where the leaves $\Prefix$ are kept and the leaves $\Default$ are going to be split.
Here,
\begin{arxiv}
as in Theorem~\ref{thm:bound},
\end{arxiv}
$b(\Prefix, \x, \y)$ gives a lower bound on the objective of
\emph{any} child tree of $\RL$. 

\begin{theorem}[Hierarchical objective lower bound]
\begin{arxiv}
Define~${b(\Prefix, \x, \y)}$
\end{arxiv}
\begin{kdd}
Define $b(\Prefix, \x, \y)$ $=$ $\Loss_p(\Prefix, \Labels, \x, \y) + \Reg H_{\RL}$,
\end{kdd}
as in~\eqref{eq:lower-bound}.
Define $\StartsWith(\RL)$ to be the set of all $\RL$'s child trees
whose unchanged leaves contain~$\Prefix$, as in~\eqref{eq:starts-with}.
For tree ${\RL = }$ ${(\Prefix, \Labels, \Default, \DefaultLabels, K, H_{\RL})}$ 
with unchanged leaves~$\Prefix$, let
${\RL' = (\Prefix', \Labels', \Default', \DefaultLabels', K', H_{\RL'})}$ $\in \StartsWith(\RL)$
be any child tree such that its unchanged leaves~$\Prefix'$ contain~$\Prefix$
and ${K' \ge K, H_{\RL'} \ge H_{\RL}}$, then ${b(\Prefix, \x, \y) \le}$ ${\Obj(\RL', \x, \y)}$.
\label{thm:bound}
\end{theorem}

\begin{arxiv}
\begin{proof}
Let ${\Prefix = (p_1, \dots, p_K)}$ and ${\Labels = (\q_1, \dots, \q_K)}$;
let ${\Prefix' = (p_1, \dots, p_K, p_{K+1}, \dots, p_{K'})}$
and ${\Labels' = (\q_1, \dots, \q_K, \q_{K+1}, \dots, \q_{K'})}$.
Notice that~$\Prefix'$ yields the same mistakes as~$\Prefix$,
and possibly additional mistakes:
\begin{align}
&\Loss_p(\Prefix', \Labels', \x, \y)
= \frac{1}{N} \sum_{n=1}^N  \sum_{k=1}^{K'} \Cap(x_n, p_k) \wedge \one [ \q_k \neq y_n ] \nn \\
&= \frac{1}{N} \sum_{n=1}^N \left( \sum_{k=1}^K \Cap(x_n, p_k) \wedge \one [ \q_k \neq y_n ]
+ \sum_{k=K+1}^{K'} \Cap(x_n, p_k) \wedge \one [ \q_k \neq y_n ] \right) \nn \\
&=\Loss_p(\Prefix, \Labels, \x, \y)
+ \frac{1}{N} \sum_{n=1}^N \sum_{k=K+1}^{K'} \Cap(x_n, p_k) \wedge \one [ \q_k \neq y_n ]
\ge \Loss_p(\Prefix, \Labels, \x, \y),
\label{eq:prefix-loss}
\end{align}
It follows that
\begin{align}
b(\Prefix, \x, \y) &= \Loss_p(\Prefix, \Labels, \x, \y) + \Reg K \nn \\
&\le  \Loss_p(\Prefix', \Labels', \x, \y) + \Reg K' = b(\Prefix', \x, \y)
\le \Obj(\RL', \x, \y).
\label{eq:prefix-lb}
\end{align}
\end{proof}
\end{arxiv}

Consider a sequence of trees, where each tree is the parent of the following tree. In this case, the lower bounds of these trees increase monotonically, which is amenable to branch-and-bound. We illustrate our framework in Algorithm~\ref{alg:branch-and-bound} in Supplement \ref{appendix:alg}.
According to Theorem~\ref{thm:bound}, we can hierarchically prune the search space.
During the execution of the algorithm, we cache the current best (smallest) objective~$\CurrentObj$, which is dynamic and monotonically decreasing.
In this process, when we generate a tree whose unchanged leaves~$\Prefix$ correspond to a lower bound satifying ${b(\Prefix, \x, \y) \ge \CurrentObj}$, according to Theorem~\ref{thm:bound}, we do not need to consider \emph{any} child tree~${\RL' \in \StartsWith(\RL)}$ of this tree whose $\Prefix'$ contains~$\Prefix$.

Based on Theorem~\ref{thm:bound}, we describe a consequence in Lemma~\ref{lemma:lookahead}.

\begin{lemma}[Objective lower bound with one-step lookahead]
\label{lemma:lookahead}
Let~$\RL$ be a $H_{\RL}$-leaf tree with a $K$-leaf prefix
and let~$\CurrentObj$ be the current best objective.
If ${b(\Prefix, \x, \y) + \Reg \ge \CurrentObj}$,
then for any child tree ${\RL' \in \StartsWith(\RL)}$, 
its prefix~$\Prefix'$ starts with~$\Prefix$ and~${K' > K, H_{\RL'}>H_{\RL}}$,
and it follows that ${\Obj(\RL', \x, \y) \ge \CurrentObj}$.
\end{lemma}

\begin{arxiv}
\begin{proof}
By the definition of the lower bound~\eqref{eq:lower-bound},
which includes the penalty for longer prefixes,
\begin{align}
\Obj(\Prefix', \x, y) \ge b(\Prefix', \x, \y) &= \Loss_p(\Prefix', \Labels', \x, \y) + \Reg H' \nn \\
&= \Loss_p(\Prefix', \Labels', \x, \y) + \Reg H + \Reg (H' - H) \nn \\
&= b(\Prefix, \x, \y) + \Reg (H' - H)
\ge b(\Prefix, \x, \y) + \Reg \ge \CurrentObj.
\end{align}
\end{proof}
\end{arxiv}

This bound tends to be very powerful in practice in pruning the search space, because it states that
\textit{even though we might have a tree with unchanged leaves $\Prefix$
whose lower bound ${b(\Prefix, \x, \y) \le \CurrentObj}$,
if
${b(\Prefix, \x, \y) + \Reg }$ ${\ge \CurrentObj}$, we can still prune
all of its child trees.} 

\subsection{Lower Bounds on Node Support and Classification Accuracy}
\label{sec:lb-support}

We provide three lower bounds on the fraction of correctly classified data and the normalized support of leaves in any optimal tree. All of them depend on $\Reg$.
\begin{theorem}[Lower bound on node support]
\label{thm:min-capture}
~Let ${\OptimalRL = (\Prefix, }$ ${\Labels, \Default, \DefaultLabels, K,}$ ${ H_{\RL^*})}$
be any optimal tree with objective~$\OptimalObj$, \ie
${\OptimalRL}$ ${ \in \argmin_\RL \Obj(\RL, \x, \y)}$.
For an optimal tree, the support traversing through each internal node must be at least $2\Reg$. That is, 
for each child leaf pair $p_k, p_{k+1}$ of a split,
the sum of normalized supports of~$p_k, p_{k+1}$ should be no less than twice the regularization parameter, \ie $2\lambda$,
\begin{equation}
2\Reg \le \Supp(p_k, \x)+\Supp(p_{k+1}, \x).
\label{eq:min-capture}
\end{equation}
\end{theorem}

\begin{arxiv}
\begin{proof}
Let ${\OptimalRL = (\Prefix, \Labels, \Default, \DefaultLabels, K, H)}$ be an optimal
tree with leaves ${(p_1, \dots, p_H)}$
and labels ${(\q_1, \dots, \q_H)}$.
Consider the tree ${\RL = (\Prefix', \Labels', \Default', \DefaultLabels', K', H')}$
derived from~$\OptimalRL$ by deleting a pair of leaves ${p_i \rightarrow \q_i, p_{i+1} \rightarrow \q_{i+1}}$ and adding the their parent leaf $p_{j}\rightarrow \q_{j}$,
therefore ${\Prefix' = (p_1, \dots, p_{i-1}, p_{i+2}, \dots, p_H, p_{j})}$
and ${\Labels' = (\q_1, \dots, \q_{i-1},}$ ${\q_{i+2}, \dots, \q_H, \q_{j})}$.

The largest possible discrepancy between~$\OptimalRL$ and~$\RL$ would occur
if~$\OptimalRL$ correctly classified all the data captured by~$p_i, p_{i+1}$,
while~$\RL$ misclassified half of these data.
This gives an upper bound:
\begin{align}
\Obj(\RL, \x, \y) = \Loss(\RL, \x, \y) + \Reg (H - 1)
&\le \Loss(\OptimalRL, \x, \y) + \Supp(p_i, \x) + \Supp(p_{i+1}, \x) \nn \\ &~~~ -\frac{1}{2}[\Supp(p_i, \x) + \Supp(p_{i+1}, \x)] + \Reg(H - 1) \nn \\
&= \Obj(\OptimalRL, \x, \y) + \frac{1}{2}[\Supp(p_i, \x) + \Supp(p_{i+1}, \x)] - \Reg \nn \\
&= \OptimalObj + \frac{1}{2}[\Supp(p_i, \x) + \Supp(p_{i+1}, \x)] - \Reg
\label{eq:ub-i}
\end{align}
where~$\Supp(p_i, \x), \Supp(p_i, \x)$ is the normalized support of~$p_i, p_{i+1}$, defined in~\eqref{eq:support},
and the regularization `bonus' comes from the fact that~$\RL$
is one rule shorter than~$\OptimalRL$.

At the same time, we must have ${\OptimalObj \le \Obj(\RL, \x, \y)}$ for~$\OptimalRL$ to be optimal.
Combining this with~\eqref{eq:ub-i} and rearranging gives~\eqref{eq:min-capture},
therefore twice of the regularization parameter~$2\Reg$ provides a lower bound
on the support of a pair of antecedent~$p_i, p_{i+1}$ in an optimal tree~$\OptimalRL$.
\end{proof}
\end{arxiv}

Therefore, for a tree~$\RL$, if any of its internal nodes capture less than
a fraction~$2\Reg$ of the samples, it cannot be an optimal tree, even if~${b(\Prefix, \x, \y) < \OptimalObj}$. None of its child trees would be an optimal tree either. Thus, after evaluating $\RL$, we can prune tree~$\RL$.
%
%

\begin{theorem}[Lower bound on incremental classification accuracy]
\label{thm:incre-min-capture-correct}
Let ${\OptimalRL = (\Prefix, \Labels,}$ ${ \Default, \DefaultLabels, K, H_{\RL^*})}$
be any optimal tree with objective~$\OptimalObj$, \ie
${\OptimalRL \in \argmin_\RL \Obj(\RL, \x, \y)}$.
Let $\OptimalRL$ have leaves ${\Prefix = (p_1, \dots, p_{H_{\RL^*}})}$
and labels ${\Labels = (\q_1, \dots, \q_{H_{\RL^*}})}$.
%
For each leaf pair $p_k, p_{k+1}$ with corresponding labels $\q_k, \q_{k+1}$ in $\OptimalRL$ and their parent node (the leaf in the parent tree) $p_{j}$ and its label $\q_{j}$, define $a_k$ to be the incremental classification accuracy of splitting $p_j$ to get $p_k,p_{k+1}$:
{\scriptsize
\begin{align}
a_k &\equiv \frac{1}{N} \sum_{n=1}^N \{
  \Cap(x_n, p_k) \wedge \one [ \q_k = y_n ] 
  + \Cap(x_n, p_{k+1}) \wedge \one [ \q_{k+1} = y_n ]
  - \Cap(x_n, p_j) \wedge \one [ \q_j = y_n ] \}.
\label{eq:rule-correct}
\end{align}
}%
In this case, $\Reg$ provides a lower bound, $\Reg \le a_k$.
\end{theorem}
\begin{arxiv}
In that case, the regularization parameter~$\Reg$ provides a lower bound on~$a_k$:
\begin{align}
\Reg \le a_k.
\label{eq:min-capture-correct}
\end{align}
\end{arxiv}

\begin{arxiv}
\begin{proof}
As in Theorem~\ref{thm:min-capture},
let ${\RL = (\Prefix', \Labels', \Default', \DefaultLabels', K', H')}$
be the tree derived from~$\OptimalRL$ by deleting a pair of leaves ${p_i \rightarrow \q_i, p_{i+1} \rightarrow \q_{i+1}}$ and adding the their parent leaf $p_{j}\rightarrow \q_{j}$.
The discrepancy between~$\OptimalRL$ and~$\RL$ is the discrepancy between $p_i, p_{i+1}$ and $p_j$:
\begin{align}
\Loss(\RL, \x, \y) - \Loss(\OptimalRL, \x, \y) = a_i, \nn
\end{align}
where we defined~$a_i$ in~\eqref{eq:rule-correct}.
Relating this bound to the objectives of~$\RL$ and~$\OptimalRL$ gives
\begin{align}
\Obj(\RL, \x, \y) = \Loss(\RL, \x, \y) + \Reg (K - 1)
&= \Loss(\OptimalRL, \x, \y) + a_i + \Reg(K - 1) \nn \\
&= \Obj(\OptimalRL, \x, \y) + a_i - \Reg \nn \\
&= \OptimalObj + a_i - \Reg.
\label{eq:ub-ii}
\end{align}
Combining~\eqref{eq:ub-ii} with the requirement
${\OptimalObj \le \Obj(\RL, \x, \y)}$ gives the bound~${\Reg \le a_i}$.
\end{proof}
\end{arxiv}

Thus, when we split a leaf of the parent tree, if the incremental fraction of data that are correctly classified after this split is less than a fraction~$\Reg$, we need to further split at least one of the two child leaves to search for the optimal tree.
%
%
Thus, we apply Theorem~\ref{thm:min-capture} when we split the leaves. We need only split
 leaves whose normalized supports are no less than~$2\Reg$.
We apply Theorem~\ref{thm:incre-min-capture-correct} when constructing the trees. For every new split, we check the incremental accuracy 
for this split. If it is less than $\Reg$, we further split at least one of the two child leaves.
Both Theorem~\ref{thm:min-capture} and Theorem~\ref{thm:incre-min-capture-correct} are bounds for pairs of leaves. We 
give a bound on a single leaf's classification accuracy in Theorem~\ref{thm:min-capture-correct}.

\begin{theorem}[Lower bound on classification accuracy]
\label{thm:min-capture-correct}
Let ${\OptimalRL = (\Prefix, \Labels, \Default, \DefaultLabels, K, H_{\RL^*})}$
be any optimal tree with objective~$\OptimalObj$, \ie
${\OptimalRL \in \argmin_\RL \Obj(\RL, \x, \y)}$.
For each leaf $(p_k, \hat{y}^{(\leaf)}_k)$ in $\OptimalRL$, the fraction of correctly classified data in leaf $k$ should be no less than $\Reg$,
\begin{align}
&\Reg \le \frac{1}{N}\sum_{n=1}^N \Cap(x_n,p_k)\wedge \one[\hat{y}^{(\leaf)}_k = y_n].
\label{eq:incre-min-capture-correct}
\end{align}
\end{theorem}

\begin{arxiv}
\begin{proof}
As in Theorem~\ref{thm:incre-min-capture-correct},
let ${\RL = (\Prefix', \Labels', \Default', \DefaultLabels', K', H')}$
be the tree derived from~$\OptimalRL$ by deleting a pair of leaves ${p_i \rightarrow \q_i, p_{i+1} \rightarrow \q_{i+1}}$ and adding the their parent leaf $p_{j}\rightarrow \q_{j}$.
The discrepancy between~$\OptimalRL$ and~$\RL$ is the discrepancy between $p_i, p_{i+1}$ and $p_j$:
\begin{align}
\Loss(\RL, \x, \y) - \Loss(\OptimalRL, \x, \y) = a_i, \nn
\end{align}
where we defined~$a_i$ in~\eqref{eq:rule-correct}.
According to Theorem~\ref{thm:incre-min-capture-correct},
\begin{align}
\Reg \le & a_i, \nn \\
\Reg \le & \frac{1}{N} \sum_{n=1}^N \{
  \Cap(x_n, p_i) \wedge \one [ \q_i = y_n ] + \Cap(x_n, p_{i+1}) \wedge \one [ \q_{i+1} = y_n ]  \nn \\
  &~~~ - \Cap(x_n, p_j) \wedge \one [ \q_j = y_n ] \}.
\end{align}
For any leaf $j$ and its two child leaves $i,i+1$, we always have
\begin{align}
  \sum_{n=1}^N \Cap(x_n, p_i) \wedge \one [ \q_i = y_n ] \le \sum_{n=1}^N \Cap(x_n, p_j) \wedge \one [ \q_j = y_n ], \nn \\
  \sum_{n=1}^N \Cap(x_n, p_{i+1}) \wedge \one [ \q_{i+1} = y_n ] \le \sum_{n=1}^N \Cap(x_n, p_j) \wedge \one [ \q_j = y_n ]
\end{align}
which indicates that $a_i \le \frac{1}{N}\sum_{n=1}^N \Cap(x_n,p_i)\wedge \one[\q_i = y_n]$ and $a_i \le \frac{1}{N}\sum_{n=1}^N \Cap(x_n,p_{i+1})\wedge \one[\q_{i+1} = y_n]$. Therefore,
\begin{align}
  &\Reg \le \frac{1}{N}\sum_{n=1}^N \Cap(x_n,p_i)\wedge \one[\q_i = y_n], \nn \\
  &\Reg \le \frac{1}{N}\sum_{n=1}^N \Cap(x_n,p_{i+1})\wedge \one[\q_{i+1} = y_n].
\end{align}
\end{proof}
\end{arxiv}
%
\textit{Thus, in a leaf we consider extending by splitting on a particular feature, if that proposed split leads to less than $\Reg$ correctly classified data going to either side of the split, then this split can be excluded, and we can exclude that feature anywhere further down the tree extending that leaf.}


\begin{arxiv} 
\subsection{Upper Bound on Antecedent Support}
\label{sec:ub-support}

In the previous section~(\S\ref{sec:lb-support}), we proved lower bounds on
antecedent support; in Appendix~\ref{appendix:ub-supp},
we give an upper bound on antecedent support.
Specifically, Theorem~\ref{thm:ub-support} shows that an antecedent's
support in a rule list cannot be too similar to the set of data not
captured by preceding antecedents in the rule list.
In particular, Theorem~\ref{thm:ub-support} implies that we should
only mine rules with normalized support less than or equal to ${1 - \Reg}$;
we need not mine rules with a larger fraction of observations.
Note that we do not otherwise use this bound in our implementation,
because we did not observe a meaningful benefit in preliminary experiments.
\end{arxiv}  

\begin{arxiv} 
\subsection{Antecedent Rejection and its Propagation}
\label{sec:reject}

In this section, we demonstrate further consequences of
our lower~(\S\ref{sec:lb-support}) and upper
bounds (\S\ref{sec:ub-support}) on antecedent support,
under a unified framework we refer to as antecedent rejection.
Let ${\Prefix = (p_1, \dots, p_K)}$ be a prefix,
and let~$p_k$ be an antecedent in~$\Prefix$.
Define~$p_k$ to have insufficient support in~$\Prefix$
if it does not obey the bound in~\eqref{eq:min-capture}
of Theorem~\ref{thm:min-capture}.
Define~$p_k$ to have insufficient accurate support in~$\Prefix$
if it does not obey the bound in~\eqref{eq:min-capture-correct}
of Theorem~\ref{thm:min-capture-correct}.
Define~$p_k$ to have excessive support in~$\Prefix$
if it does not obey the bound in~\eqref{eq:ub-support}
of Theorem~\ref{thm:ub-support} (Appendix~\ref{appendix:ub-supp}).
If~$p_k$ in the context of~$\Prefix$ has insufficient support,
insufficient accurate support, or excessive support,
let us say that prefix~$\Prefix$ rejects antecedent~$p_K$.
Next, in Theorem~\ref{thm:reject}, we describe large classes of
related rule lists whose prefixes all reject the same antecedent.

\begin{theorem}[Antecedent rejection propagates]
\label{thm:reject}
For any prefix ${\Prefix = (p_1, \dots, p_K)}$,
let $\StartContains(\Prefix)$ denote the set of all
prefixes~$\Prefix'$ such that
the set of all antecedents in~$\Prefix$ is a subset of
the set of all antecedents in~$\Prefix'$, \ie
\begin{align}
\StartContains(\Prefix) =
\{\Prefix' = (p'_1, \dots, p'_{K'})
~s.t.~ \{p_k : p_k \in \Prefix \} \subseteq
\{p'_\kappa : p'_\kappa \in \Prefix'\}, K' \ge K \}.
\label{eq:start-contains}
\end{align}
Let ${\RL = (\Prefix, \Labels, \Default, K)}$ be a rule list
with prefix ${\Prefix = (p_1, \dots, p_{K-1}, p_{K})}$,
such that~$\Prefix$ rejects its last antecedent~$p_{K}$,
either because~$p_{K}$ in the context of~$\Prefix$ has
insufficient support, insufficient accurate support,
or excessive support.
Let ${\Prefix^{K-1} = (p_1, \dots, p_{K-1})}$ be the
first~${K - 1}$ antecedents of~$\Prefix$.
Let ${\RLB = (\PrefixB, \LabelsB, \DefaultB, \kappa)}$
be any rule list with prefix
${\PrefixB = (P_1, \dots, P_{K'-1},}$ ${P_{K'}, \dots, P_{\kappa})}$
such that~$\PrefixB$ starts with ${\PrefixB^{K'-1} =}$
${(P_1, \dots, P_{K'-1}) \in}$ ${\StartContains(\Prefix^{K-1})}$
and antecedent ${P_{K'} = p_{K}}$.
It follows that prefix~$\PrefixB$ rejects~$P_{K'}$
for the same reason that~$\Prefix$ rejects~$p_{K}$,
and furthermore, $\RLB$~cannot be optimal, \ie
${\RLB \notin \argmin_{\RL^\dagger} \Obj(\RL^\dagger, \x, \y)}$.
\end{theorem}

\begin{proof}
Combine Proposition~\ref{prop:min-capture}, Proposition~\ref{prop:min-capture-correct},
and Proposition~\ref{prop:ub-support}.
The first two are found below, and the last in Appendix~\ref{appendix:ub-supp}.
\end{proof}

Theorem~\ref{thm:reject} implies potentially significant
computational savings.
We know from Theorems~\ref{thm:min-capture},
\ref{thm:min-capture-correct}, and~\ref{thm:ub-support}
that during branch-and-bound execution, if we ever encounter a
prefix ${\Prefix = (p_1, \dots, p_{K-1}, p_K)}$ that rejects its
last antecedent~$p_K$, then we can prune~$\Prefix$.
By Theorem~\ref{thm:reject}, we can also prune \emph{any} prefix~$\Prefix'$
whose antecedents contains the set of antecedents in~$\Prefix$,
in almost any order, with the constraint that all antecedents
in ${\{p_1, \dots, p_{K-1}\}}$ precede~$p_K$.
These latter antecedents are also rejected
directly by the bounds in Theorems~\ref{thm:min-capture},
\ref{thm:min-capture-correct}, and~\ref{thm:ub-support};
this is how our implementation works in practice.
In a preliminary implementation (not shown), we maintained additional
data structures to support the direct use of Theorem~\ref{thm:reject}.
We leave the design of efficient data structures for this task as future work.

\begin{proposition}[Insufficient antecedent support propagates]
\label{prop:min-capture}
First define~$\StartContains(\Prefix)$ as in~\eqref{eq:start-contains},
and let ${\Prefix = (p_1, \dots, p_{K-1}, p_{K})}$ be a prefix,
such that its last antecedent~$p_{K}$ has insufficient support,
\ie the opposite of the bound in~\eqref{eq:min-capture}:
${\Supp(p_K, \x \given \Prefix) < \Reg}$.
Let ${\Prefix^{K-1} =}$ ${(p_1, \dots, p_{K-1})}$,
and let ${\RLB = (\PrefixB, \LabelsB, \DefaultB, \kappa)}$
be any rule list with prefix
${\PrefixB = (P_1, \dots, P_{K'-1},}$ ${P_{K'}, \dots, P_{\kappa})}$,
such that~$\PrefixB$ starts with ${\PrefixB^{K'-1} =}$
${(P_1, \dots, P_{K'-1}) \in \StartContains(\Prefix^{K-1})}$
and~${P_{K'} = p_{K}}$.
It follows that~$P_{K'}$ has insufficient support in
prefix~$\PrefixB$, and furthermore, $\RLB$~cannot be optimal,
\ie ${\RLB \notin \argmin_{\RL} \Obj(\RL, \x, \y)}$.
\end{proposition}

\begin{proof}
The support of~$p_K$ in~$\Prefix$ depends only on the
set of antecedents in ${\Prefix^{K} = (p_1, \dots, p_{K})}$:
\begin{align}
\Supp(p_K, \x \given \Prefix)
= \frac{1}{N} \sum_{n=1}^N \Cap(x_n, p_K \given \Prefix)
&= \frac{1}{N} \sum_{n=1}^N \left( \neg\, \Cap(x_n, \Prefix^{K-1}) \right)
  \wedge \Cap(x_n, p_K) \nn \\
&= \frac{1}{N} \sum_{n=1}^N \left( \bigwedge_{k=1}^{K-1} \neg\, \Cap(x_n, p_k) \right)
  \wedge \Cap(x_n, p_K)
< \Reg, \nn
\end{align}
and the support of~$P_{K'}$ in~$\PrefixB$ depends only on
the set of antecedents in ${\PrefixB^{K'} =}$ ${(P_1, \dots, P_{K'})}$:
\begin{align}
\Supp(P_{K'}, \x \given \PrefixB)
= \frac{1}{N} \sum_{n=1}^N \Cap(x_n, P_{K'} \given \PrefixB) 
&= \frac{1}{N} \sum_{n=1}^N \left( \bigwedge_{k=1}^{K'-1} \neg\, \Cap(x_n, P_k) \right)
   \wedge \Cap(x_n, P_{K'}) \nn \\
&\le \frac{1}{N} \sum_{n=1}^N \left( \bigwedge_{k=1}^{K-1} \neg\, \Cap(x_n, p_k) \right)
  \wedge \Cap(x_n, P_{K'}) \nn \\
&= \frac{1}{N} \sum_{n=1}^N \left( \bigwedge_{k=1}^{K-1} \neg\, \Cap(x_n, p_k) \right)
  \wedge \Cap(x_n, p_{K}) \nn \\
&= \Supp(p_K, \x \given \Prefix) < \Reg.
\label{ineq:supp}
\end{align}
The first inequality reflects the condition that
${\PrefixB^{K'-1} \in \StartContains(\Prefix^{K-1})}$,
which implies that the set of antecedents in~$\PrefixB^{K'-1}$
contains the set of antecedents in~$\Prefix^{K-1}$,
and the next equality reflects the fact that~${P_{K'} = p_K}$.
Thus,~$P_K'$ has insufficient support in prefix~$\PrefixB$,
therefore by Theorem~\ref{thm:min-capture}, $\RLB$~cannot be optimal,
\ie ${\RLB \notin \argmin_{\RL} \Obj(\RL, \x, \y)}$.
\end{proof}

\begin{proposition}[Insufficient accurate antecedent support propagates]
\label{prop:min-capture-correct}
~Let~$\StartContains(\Prefix)$ denote the set of all
prefixes~$\Prefix'$ such that
the set of all antecedents in~$\Prefix$ is a subset of
the set of all antecedents in~$\Prefix'$,
as in~\eqref{eq:start-contains}.
Let ${\RL = (\Prefix, \Labels, \Default, K)}$ be a rule list
with prefix ${\Prefix = (p_1, \dots, p_{K})}$
and labels ${\Labels = (q_1, \dots, q_{K})}$, such that
the last antecedent~$p_{K}$ has insufficient accurate support,
\ie the opposite of the bound in~\eqref{eq:min-capture-correct}:
\begin{align}
\frac{1}{N} \sum_{n=1}^N \Cap(x_n, p_K \given \Prefix) \wedge \one [ q_K = y_n ]
< \Reg. \nn
\end{align}
Let ${\Prefix^{K-1} = (p_1, \dots, p_{K-1})}$
and let ${\RLB = (\PrefixB, \LabelsB, \DefaultB, \kappa)}$
be any rule list with prefix ${\PrefixB =}$ ${(P_1, \dots, P_{\kappa})}$
and labels ${\LabelsB = (Q_1, \dots, Q_{\kappa})}$,
such that~$\PrefixB$ starts with ${\PrefixB^{K'-1} =}$
${(P_1, \dots, P_{K'-1})}$ ${\in \StartContains(\Prefix^{K-1})}$
and ${P_{K'} = p_{K}}$.
It follows that~$P_{K'}$ has insufficient accurate support in
prefix~$\PrefixB$, and furthermore,
${\RLB \notin \argmin_{\RL^\dagger} \Obj(\RL^\dagger, \x, \y)}$.
\end{proposition}

\begin{proof}
The accurate support of~$P_{K'}$ in~$\PrefixB$ is insufficient:
\begin{align}
\frac{1}{N} \sum_{n=1}^N \Cap(x_n, P_{K'} &\given \PrefixB) \wedge \one [ Q_{K'} = y_n ] \nn \\
&= \frac{1}{N} \sum_{n=1}^N \left( \bigwedge_{k=1}^{K'-1} \neg\, \Cap(x_n, P_k) \right)
   \wedge \Cap(x_n, P_{K'}) \wedge \one [ Q_{K'} = y_n ] \nn \\
&\le \frac{1}{N} \sum_{n=1}^N \left( \bigwedge_{k=1}^{K-1} \neg\, \Cap(x_n, p_k) \right)
   \wedge \Cap(x_n, P_{K'}) \wedge \one [ Q_{K'} = y_n ] \nn \\
&= \frac{1}{N} \sum_{n=1}^N \left( \bigwedge_{k=1}^{K-1} \neg\, \Cap(x_n, p_k) \right)
   \wedge \Cap(x_n, p_K) \wedge \one [ Q_{K'} = y_n ] \nn \\
&= \frac{1}{N} \sum_{n=1}^N \Cap(x_n, p_K \given \Prefix) \wedge \one [ Q_{K'} = y_n ] \nn \\
&\le \frac{1}{N} \sum_{n=1}^N \Cap(x_n, p_K \given \Prefix) \wedge \one [ q_{K} = y_n ]
< \Reg. \nn
\end{align}
The first inequality reflects the condition that
${\PrefixB^{K'-1} \in \StartContains(\Prefix^{K-1})}$,
the next equality reflects the fact that~${P_{K'} = p_K}$.
For the following equality, notice that~$Q_{K'}$ is the majority
class label of data captured by~$P_{K'}$ in~$\PrefixB$, and~$q_K$
is the majority class label of data captured by~$P_K$ in~$\Prefix$,
and recall from~\eqref{ineq:supp} that
${\Supp(P_{K'}, \x \given \PrefixB) \le \Supp(p_{K}, \x \given \Prefix)}$.
By Theorem~\ref{thm:min-capture-correct},
${\RLB \notin \argmin_{\RL^\dagger} \Obj(\RL^\dagger, \x, \y)}$.
\end{proof}

Propositions~\ref{prop:min-capture} and~\ref{prop:min-capture-correct},
combined with Proposition~\ref{prop:ub-support} (Appendix~\ref{appendix:ub-supp}),
constitute the proof of Theorem~\ref{thm:reject}.
\end{arxiv} 

\subsection{Equivalent Points Bound}
\label{sec:identical}
When multiple observations captured by a leaf in~$\Default$
have identical features but opposite labels, then no tree, including those that extend $\Default$, can correctly classify all of these observations. The number of misclassifications must be at least the minority label of the equivalent points. This bound has been a very powerful tool for our algorithm, but requires substantial notation and is thus presented in Appendix \ref{appendix:eqsu}.

%

\begin{arxiv} 
For data set~${\{(x_n, y_n)\}_{n=1}^N}$ and a set of features
${\{s_m\}_{m=1}^M}$,
we define a set of samples to be equivalent if they have exactly the same feature values, \ie ${x_i \neq x_j}$ are equivalent~if
$\frac{1}{M} \sum_{m=1}^M \one [ \Cap(x_i, s_m) = \Cap(x_j, s_m) ] = 1$. 
Note that a data set consists of multiple sets of equivalent points;
let~${\{e_u\}_{u=1}^U}$ enumerate these sets.
For each observation~$x_i$, it belongs to a equivalent points set~$e_u$.
We denote the fraction of data with the minority label in set~$e_u$ as~$\theta(e_u)$, \eg let
\begin{align}
{e_u = \{x_n : \forall m \in [M],\, \one [ \Cap(x_n, s_m) = \Cap(x_i, s_m) ] \}}, \nn
\end{align}

${e_u = \{x_n : \forall m \in [M],\, }$ ${\one [ \Cap(x_n, s_m) = \Cap(x_i, s_m) ] \}}$,
and let~$q_u$ be the minority class label among points in~$e_u$, then
\begin{align}
\theta(e_u) = \frac{1}{N} \sum_{n=1}^N \one [ x_n \in e_u ]\, \one [ y_n = q_u ].
\label{eq:theta}
\end{align}
We can combine the equivalent points bound with other bounds to get a tighter lower bound on the objective function. As the experimental results demonstrate in \S\ref{sec:experiments}, there is sometimes a substantial reduction of the search space after incorporating the equivalent points bound.
We propose a general equivalent points bound in Proposition~\ref{prop:identical}. We incorporate it into our framework by proposing the specific equivalent points bound in Theorem~\ref{thm:identical}.
\begin{proposition}[General equivalent points bound]
\label{prop:identical}
Let ${\RL = (\Prefix, \Labels, \Default, \DefaultLabels, K, H)}$ be a tree, then
$\Obj(\RL, \x, \y) \ge \sum_{u=1}^U \theta(e_u) + \Reg H$. 
\end{proposition}

\begin{proof}
Recall that the objective is ${\Obj(\RL, \x, \y) = \Loss(\RL, \x, \y) + \Reg H}$,
where the misclassification error~${\Loss(\RL, \x, \y)}$ is given by
\begin{align}
\Loss(\RL, \x, \y)
&= \frac{1}{N} \sum_{n=1}^N \sum_{k=1}^K \Cap(x_n, p_k) \wedge \one [\q_k \neq y_n]. \nn
\end{align}
Any particular tree uses a specific leaf, and therefore a single class label,
to classify all points within a set of equivalent points.
Thus, for a set of equivalent points~$u$, the tree~$\RL$ correctly classifies either
points that have the majority class label, or points that have the minority class label.
It follows that~$\RL$ misclassifies a number of points in~$u$ at least as great as
the number of points with the minority class label.
To translate this into a lower bound on~$\Loss(\RL, \x, \y)$,
we first sum over all sets of equivalent points, and then for each such set,
count differences between class labels and the minority class label of the set,
instead of counting mistakes:
\begin{align}
\Loss(\RL, \x, \y) \nn
&= \frac{1}{N} \sum_{u=1}^U \sum_{n=1}^N \sum_{k=1}^K \Cap(x_n, p_k) \wedge \one [\q_k \neq y_n] \wedge
   \one [x_n \in e_u]  \nn \\
&\ge \frac{1}{N} \sum_{u=1}^U \sum_{n=1}^N \sum_{k=1}^K \Cap(x_n, p_k) \wedge \one [q_u = y_n] \wedge
   \one [x_n \in e_u]  \nn \\
\label{eq:lb-equiv-pts}
\end{align}
Next, because every datum must be captured by a leaf in the tree~$\RL$, $\sum_{k=1}^K \Cap(x_n, p_k)=1$.
\begin{align}
\Loss(\RL, \x, \y) &= \frac{1}{N} \sum_{u=1}^U \sum_{n=1}^N \sum_{k=1}^K \Cap(x_n, p_k) \wedge \one [q_u = y_n] \wedge
   \one [x_n \in e_u] \nn \\
&= \frac{1}{N} \sum_{u=1}^U \sum_{n=1}^N \one [ x_n \in e_u ]\, \one [ y_n = q_u ]
= \sum_{u=1}^U \theta(e_u), \nn
\end{align}
where the final equality applies the definition of~$\theta(e_u)$ in~\eqref{eq:theta}.
Therefore, ${\Obj(\RL, \x, \y) =}$ ${\Loss(\RL, \x, \y) + \Reg K}$ ${\ge \sum_{u=1}^U \theta(e_u) + \Reg K}$.
\end{proof}

In our lower bound~${b(\Prefix, \x, \y)}$
in~\eqref{eq:lower-bound}, we omitted the misclassification errors of leaves we are going to split~$\Loss_0(\Default, \DefaultLabels,$ $\x, \y)$
from the objective~${\Obj(\RL, \x, \y)}$.
Incorporating the equivalent points bound in Theorem~\ref{thm:identical} gives a tighter bound on our objective because we now have a tighter lower bound on the misclassification errors of leaves we are going to split,
${0 \le b_0(\Default, \x, \y) \le}$ $\Loss_0(\Default, \DefaultLabels, \x, \y)$.

\begin{theorem}[Equivalent points bound]
\label{thm:identical}
Let~$\RL$ be a tree with leaves~$\Prefix, \Default$
and lower bound ${b(\Prefix, \x, \y)}$,
then for any tree~${\RL' \in \StartsWith(\RL)}$
whose prefix~$\Prefix' \supseteq \Prefix$,
\begin{align}
\Obj(\RL', \x, \y) &\ge b(\Prefix, \x, \y) + b_0(\Default, \x, \y),\;\; \textrm{ where }
\label{eq:identical}\\
b_0(\Default, \x, \y) &= \frac{1}{N} \sum_{u=1}^U \sum_{n=1}^N
    \Cap(x_n, \Default) \wedge\one [ x_n \in e_u ]\, \one [ y_n = q_u ].
\label{eq:lb-b0}
\end{align}
\end{theorem}
\end{arxiv}
\begin{arxiv}
\begin{proof}
See Appendix~\ref{appendix:equiv-pts} for the proof of Theorem~\ref{thm:identical}.
\end{proof}
\end{arxiv}  

\begin{arxiv}
\section{Implementation}
\label{sec:implementation}
\end{arxiv}
\begin{arxiv}
\section{Incremental Computation}
\label{sec:incremental}
\end{arxiv}

\subsection{Incremental Computation}
\label{sec:incremental}
Much of our implementation effort revolves around exploiting incremental computation, designing data structures and ordering of the worklist. Together, these ideas save $>$97\% execution time. We provide the details of our implementation in the supplement.

\begin{arxiv}
We provide the details of the incremental computation in this section. And in the next section~\S\ref{sec:implementation}, we further provide the data structures we implement in practice to cache those values for the incremental computation.

For the hierarchical branch-and-bound execution, we start from an empty tree $\RL = ((),(),(),\DefaultLabels,0,0)$, which is the first tree evaluated in the algorithm. For the empty tree, it does not contain any unchanged leaves and has a length of zero. Thus, its lower bound is $b((),\x,\y)=0$. Since all the data are captured by the empty leaf, the objective of the empty tree is the proportion of minority class data, \ie $\Obj{(d,\x,\y)}=\Loss_q((), \DefaultLabels, \x, \y)$.

Based on the initial objective and the lower bound we get from the first tree, \ie the empty tree, we further derive their incremental expressions.
Let ${\RL = (\Prefix, \Labels, \Default, \DefaultLabels, K, H)}$ and
${\RL' = (\Prefix', \Labels', \Default', \DefaultLabels', K', H')}$
be decision trees such that $\RL$ is the parent of $\RL'$ and ${\Prefix =}$ ${(p_1, \dots, p_K)} \subset {\Prefix' = (p_1, \dots, p_K, \dots, p_{K'})}$.
Let ${\Labels = (\q_1, \dots, \q_H)}$ and
${\Labels' = (\q_1, \dots,}$ ${\q_{H'})}$ be the corresponding labels.
It follows from the hierarchical structure of Algorithm~\ref{alg:branch-and-bound} that we must have already evaluated both the objective and the lower bound of~$\RL$ if we are evaluating its child~$\RL'$. Thus, in the evaluation of~$\RL'$, we would like to reuse these computations as much as possible in our algorithm.
Now, we express the objective lower bound of~$\RL'$ incrementally,
with respect to the objective lower bound of its parent~$\RL$:
\begin{align}
b(\Prefix', \x, \y)
  &= \Loss_p(\Prefix', \Labels', \x, \y) + \Reg (H') \nn \\
&= \frac{1}{N} \sum_{n=1}^N \sum_{k=1}^{K'} \Cap(x_n, p_k)
  \wedge \one [ \q_k \neq y_n ] + \Reg H' \label{eq:non-inc-lb} \\
&= \Loss_p(\Prefix, \Labels, \x, \y) + \Reg H'
  + \frac{1}{N} \sum_{n=1}^N \sum_{k=K+1}^{K'} \Cap(x_n, p_k) \wedge \one [\q_k \neq y_n ] \nn \\
&= b(\Prefix, \x, \y) + \Reg (H'-H)
  + \frac{1}{N} \sum_{n=1}^N \sum_{k=K+1}^{K'} \Cap(x_n, p_k) \wedge \one [\q_k \neq y_n ].
\label{eq:inc-lb}
\end{align}
Therefore, when computing $b(\Prefix', \x, \y)$, we can reuse the quantity of the lower bound $b(\Prefix, \x, \y)$ if we cache it in a data structure. Moreover, in practice, we cache the misclassification error of each leaf, because we can view leaves as set of clauses which we can reuse in different trees.
In the same way, we can compute the objective of~$\RL'$ incrementally with respect to its lower bound:
\begin{align}
\Obj(\RL', \x, \y) &=  \Loss_p(\Prefix', \Labels', \x, \y) +
  \Loss_0(\Default', \DefaultLabels', \x, \y) + \Reg H' \nn \\
&= b(\Prefix', \x, \y) + \Loss_0(\Default', \DefaultLabels', \x, \y) \nn \\
&= b(\Prefix', \x, \y) + \frac{1}{N}\sum_{n=1}^N \sum_{k=K'+1}^{H'} \Cap(x_n, p_k) \wedge
  \one [\q_k \neq y_n].
\label{eq:inc-obj}
\end{align}
The incremental computation in~\eqref{eq:inc-obj} reuses $b(\Prefix', \x, \y)$,
which we have already computed in~\eqref{eq:inc-lb}.
Note that we could alternatively compute~$\Obj(\RL', \x, \y)$ incrementally from $\Obj(\RL, \x, \y)$. However, doing so is a somewhat more complex, because we have to subtract the misclassification error of those split leaves of $\RL$ from $\Obj(\RL, \x, \y)$ and then add that of those new leaves of $\RL'$ to $\Obj(\RL, \x, \y)$.

\begin{algorithm}[t!]
  \caption{Incremental branch-and-bound for learning trees, for simplicity, from a cold start.}
\label{alg:incremental}
\begin{algorithmic}
\normalsize
\State \textbf{Input:} Objective function~$\Obj(\RL, \x, \y)$,
objective lower bound~${b(\Prefix, \x, \y)}$,
set of features ${\RuleSet = \{s_m\}_{m=1}^M}$,
training data~$(\x, \y) = {\{(x_n, y_n)\}_{n=1}^N}$,
regularization parameter~$\Reg$
\State \textbf{Output:} Provably optimal tree~$\OptimalRL$ with minimum objective~$\OptimalObj$ \\

\State $\CurrentRL \gets ((), (), (), \DefaultLabels, 0, 0)$ \Comment{Initialize current best tree with the empty tree}
\State $\CurrentObj \gets \Obj(\CurrentRL, \x, \y)$ \Comment{Initialize current best objective}
\State $Q \gets $ queue$(\,[\CurrentRL]\,)$ \Comment{Initialize queue with the empty tree}
\State $C \gets $ cache$(\,[\,(\,\CurrentRL\,, 0\,)\,]\,)$ \Comment{Initialize cache with the empty tree and its objective lower bound}
\While {$Q$ not empty} \Comment{Optimization complete when the queue is empty}
	\State $(\Prefix, \Labels, \Default, \DefaultLabels, K, H) \gets Q$.pop(\,) \Comment{Remove a length-$H$ tree~$\RL$ from the queue}
        \State $b(\Prefix, \x, \y) \gets C$.find$(\Prefix)$ \Comment{Look up $\Prefix$'s lower bound in the cache}
        \For {every possible combination of features to split $\Default$}
            \State split $\Default$ and get new leaves $d_{new}$
            \For {each possible subset $\Default'$ of $d_{new}$}
                \State $\Prefix'=\Prefix+d_{new}-\Default'$
                \State $\RL'=(\Prefix', \Labels', \Default', \DefaultLabels', K', H')$  \Comment{\textbf{Branch}: Generate child tree $\RL'$}
                \State $\Delta_{un}=\frac{1}{N} \sum_{n=1}^N \sum_{k=K+1}^{K'} \Cap(x_n, p_k) \wedge \one [\q_k \neq y_n ]$
                \State $b(\Prefix', \x, \y) \gets b(\Prefix, \x, \y) + \Reg (H'-H)~ + \Delta_{un}$  
                \If {$b(\Prefix', \x, \y) < \CurrentObj$} \Comment{\textbf{Bound}: Apply bound from Theorem~\ref{thm:bound}}
                \State $\Delta_R = \frac{1}{N}\sum_{n=1}^N \sum_{k=K'+1}^{H'} \Cap(x_n, p_k) \wedge
  \one [\q_k \neq y_n]$
                    \State $\Obj(\RL', \x, \y) \gets b(\Prefix', \x, \y)~ + \Delta_R$  
                    \If {$\Obj(\RL', \x, \y) < \CurrentObj$}
                        \State $(\CurrentRL, \CurrentObj) \gets (\RL', \Obj(\RL', \x, \y))$ \Comment{Update current best tree and objective}
                    \EndIf
                    \State $Q$.push$(\RL')$ \Comment{Add $\RL'$ to the queue}
                    \State $C$.insert$(\RL', b(\Prefix', \x, \y))$ \Comment{Add $\RL'$ and its lower bound to the cache}
                \EndIf
            \EndFor
        \EndFor
\EndWhile
\State $(\OptimalRL, \OptimalObj) \gets (\CurrentRL, \CurrentObj)$ \Comment{Identify provably optimal tree and objective}
\end{algorithmic}
\end{algorithm}

In Algorithm~\ref{alg:incremental}, we present an incremental branch-and-bound procedure with the incremental computations of the objective lower bound~\eqref{eq:inc-lb} and objective~\eqref{eq:inc-obj}, and we use a cache to store each tree and its lower bound.
In addition, we reorganizes the structure
of Algorithm~\ref{alg:branch-and-bound} in Algorithm~\ref{alg:incremental} to apply the hierarchical objective lower bound from Theorem~\ref{thm:bound} to child trees instead of the parent tree.
This could shrink the size of the queue since we do not need to push those trees which do not satisfy Theorem~\ref{thm:bound} into the queue. And therefore, this could further save the execution time especially when we use priority queue to order entries as described in section~\S\ref{sec:implementation}.

\end{arxiv}

\section{Experiments}
\label{sec:experiments}
\begin{figure*}[t!]
    \centering
    \begin{subfigure}[b]{0.245\textwidth}
        \centering
        \includegraphics[width=0.245\textwidth, trim={0mm 12mm 0mm 0mm}, width=\textwidth]{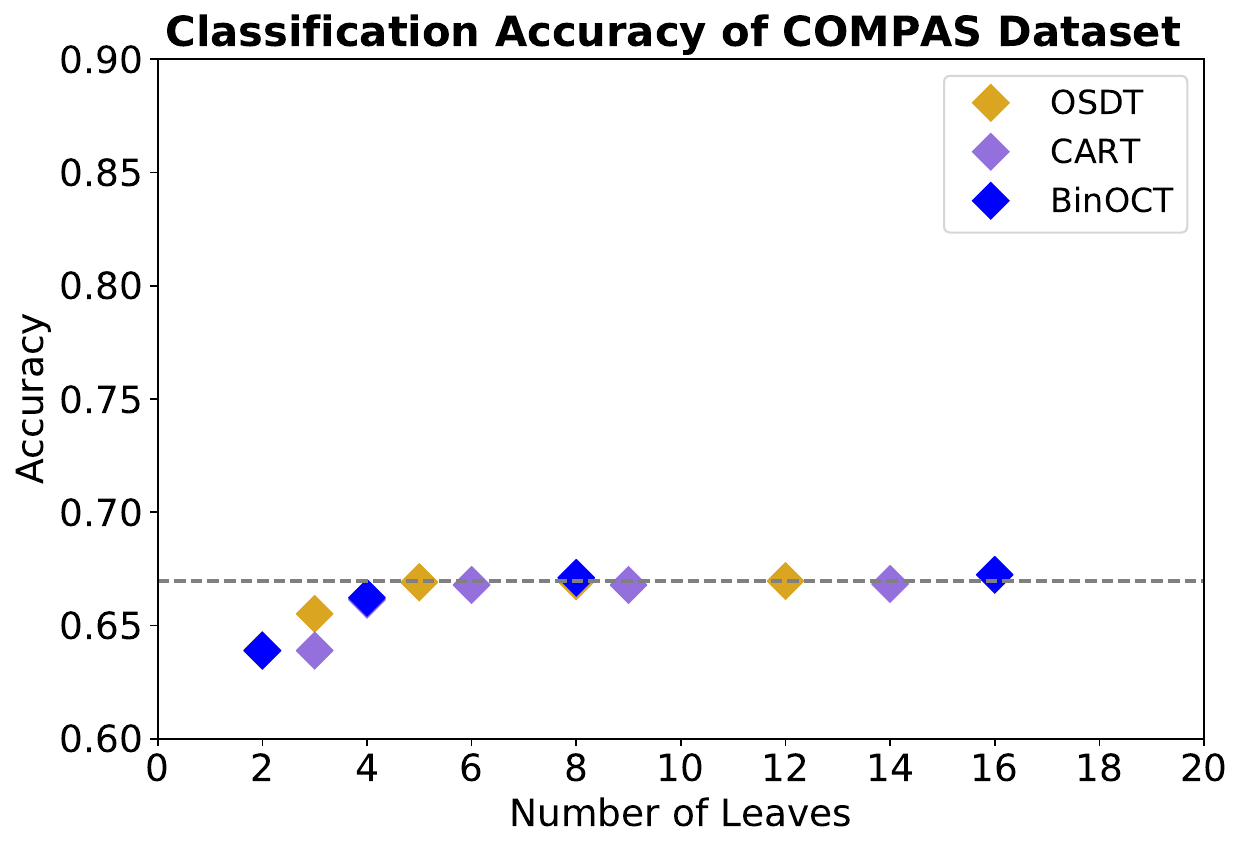}
        \label{fig:COMPAS}
    \end{subfigure}
    \begin{subfigure}[b]{0.245\textwidth}  
        \centering 
        \includegraphics[width=0.245\textwidth,trim={0mm 12mm 0mm 0mm}, width=\textwidth]{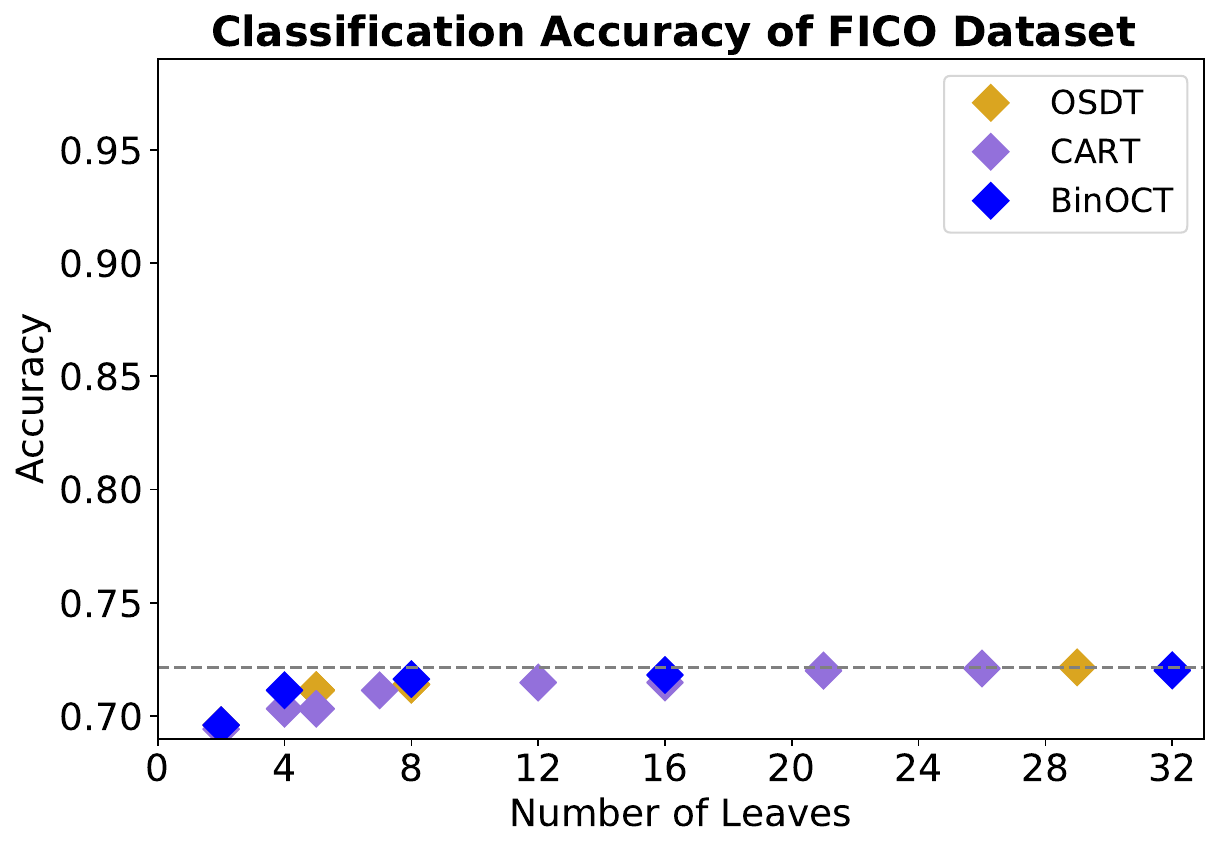}
        \label{fig:FICO}
    \end{subfigure}
    \begin{subfigure}[b]{0.245\textwidth}   
        \centering 
        \includegraphics[width=0.245\textwidth,trim={0mm 12mm 0mm 0mm}, width=\textwidth]{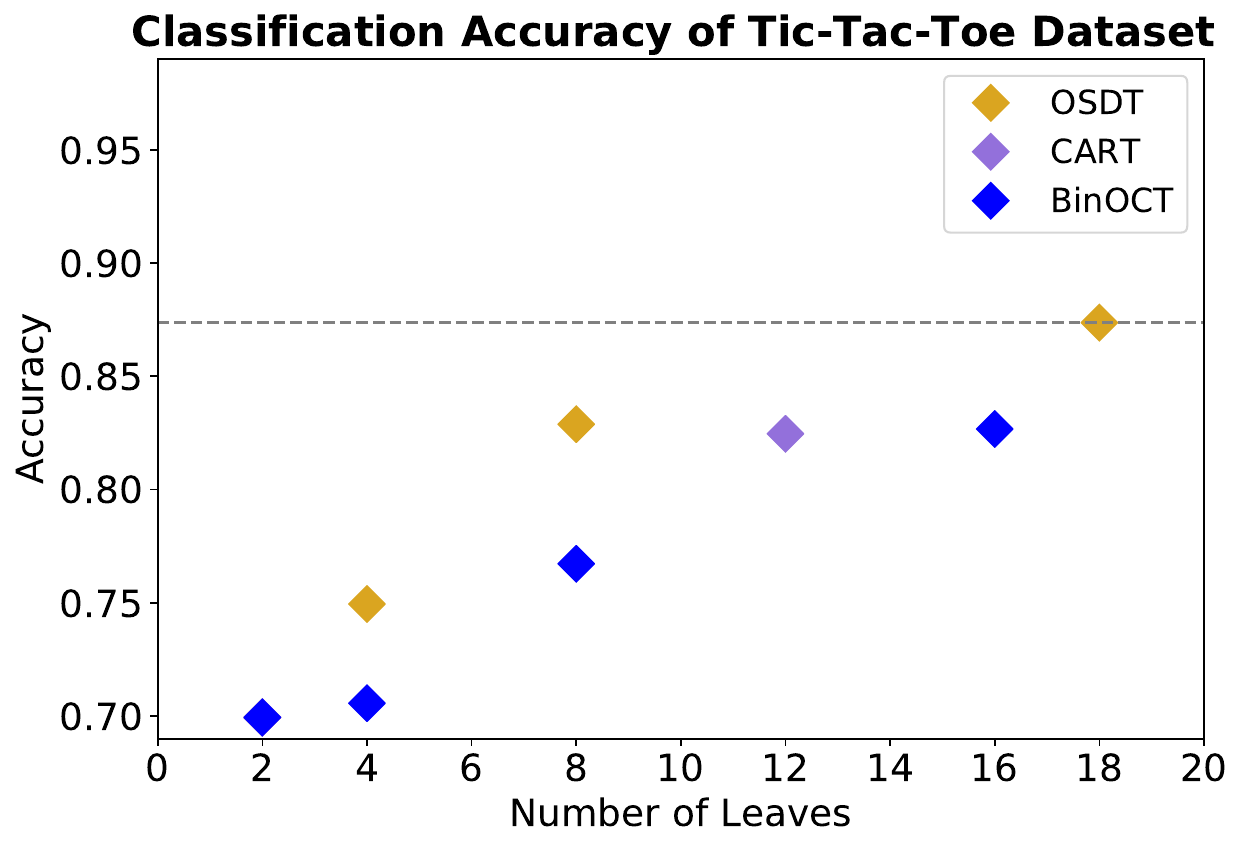}
        \label{fig:tictactoe}
    \end{subfigure}
    \begin{subfigure}[b]{0.245\textwidth}   
        \centering 
         \includegraphics[width=0.245\textwidth,trim={0mm 12mm 0mm 0mm}, width=\textwidth]{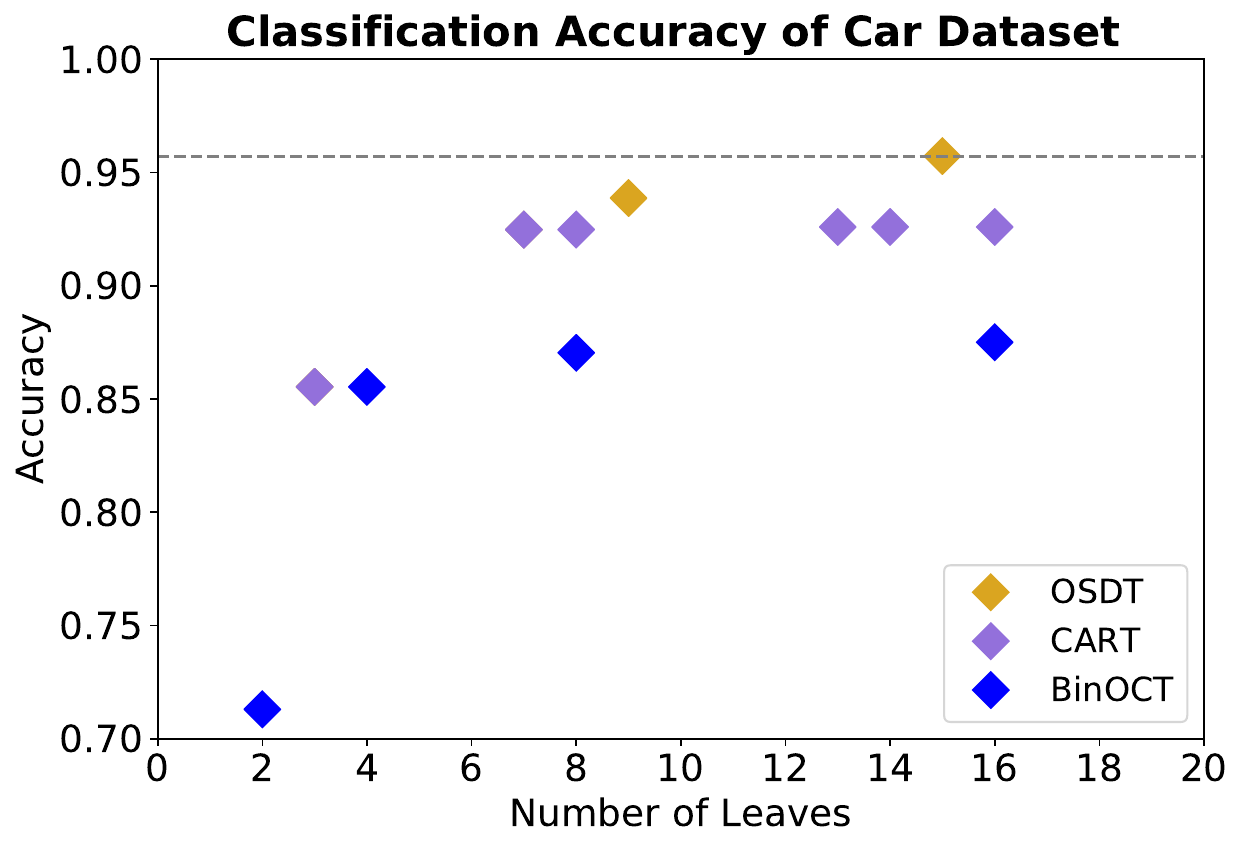}
        \label{fig:Car}
    \end{subfigure}
    \vskip\baselineskip
    \begin{subfigure}[b]{0.245\textwidth}   
        \centering 
        \includegraphics[width=0.245\textwidth,trim={0mm 12mm 0mm 15mm}, width=\textwidth]{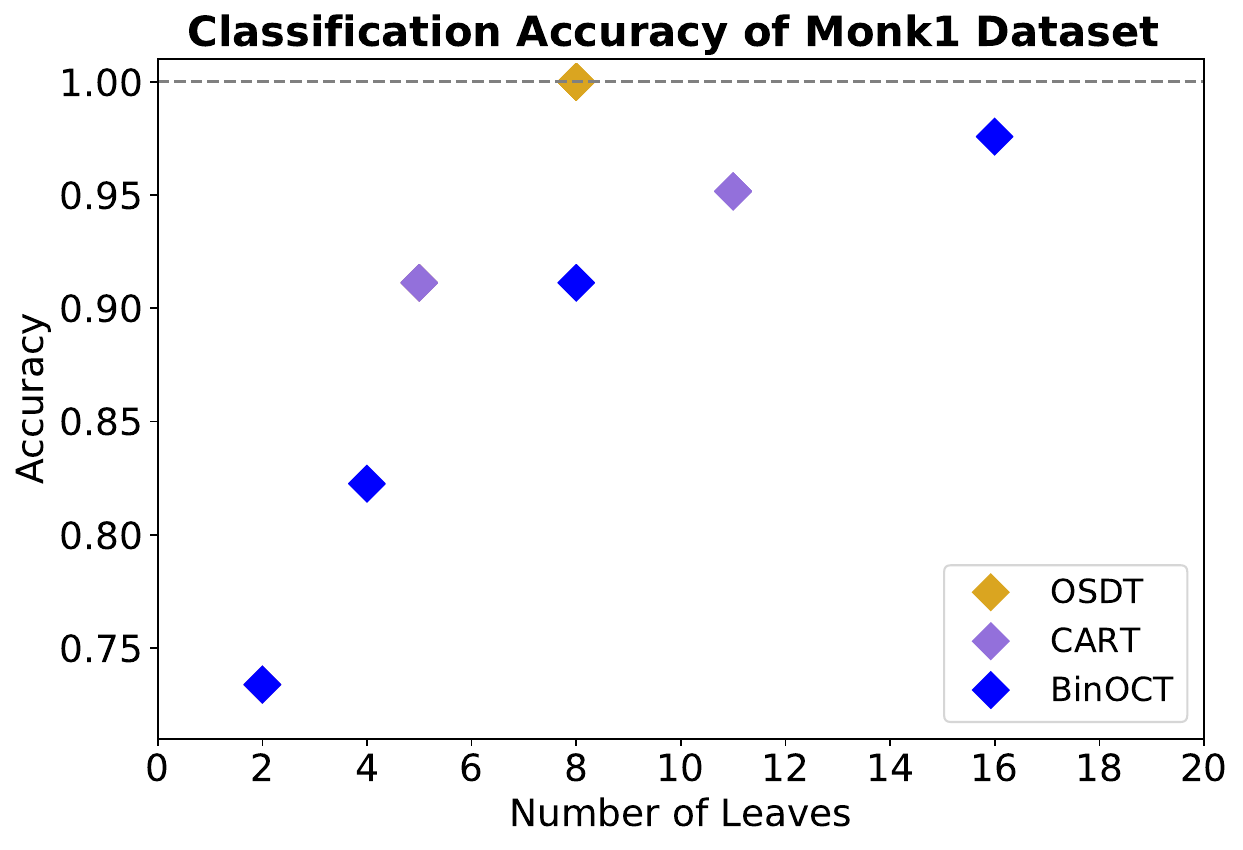}
        \label{fig:Monk1}
    \end{subfigure}
    \begin{subfigure}[b]{0.245\textwidth}   
        \centering 
         \includegraphics[width=0.245\textwidth,trim={0mm 12mm 0mm 15mm}, width=\textwidth]{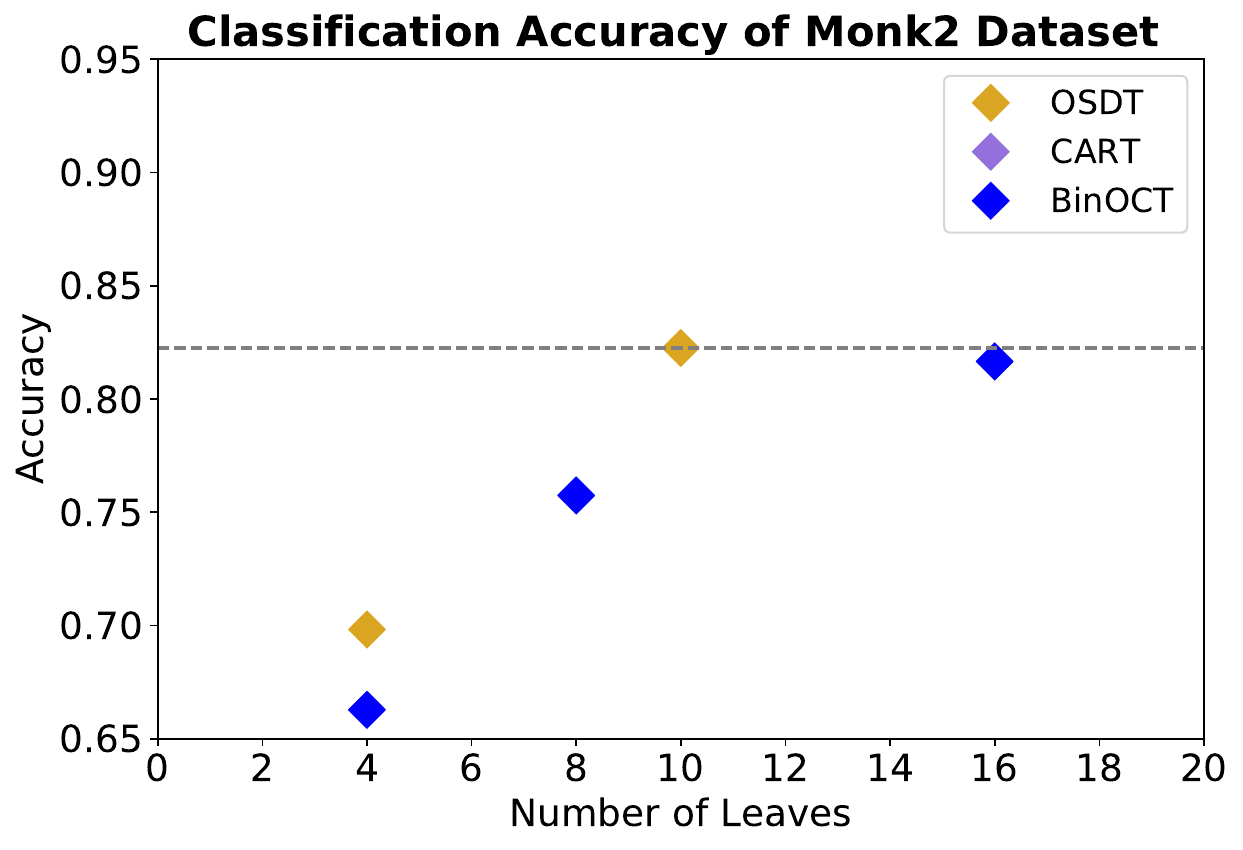}
        \label{fig:Monk2}
    \end{subfigure}
    \begin{subfigure}[b]{0.245\textwidth}   
        \centering 
        \includegraphics[width=0.245\textwidth,trim={0mm 12mm 0mm 15mm}, width=\textwidth]{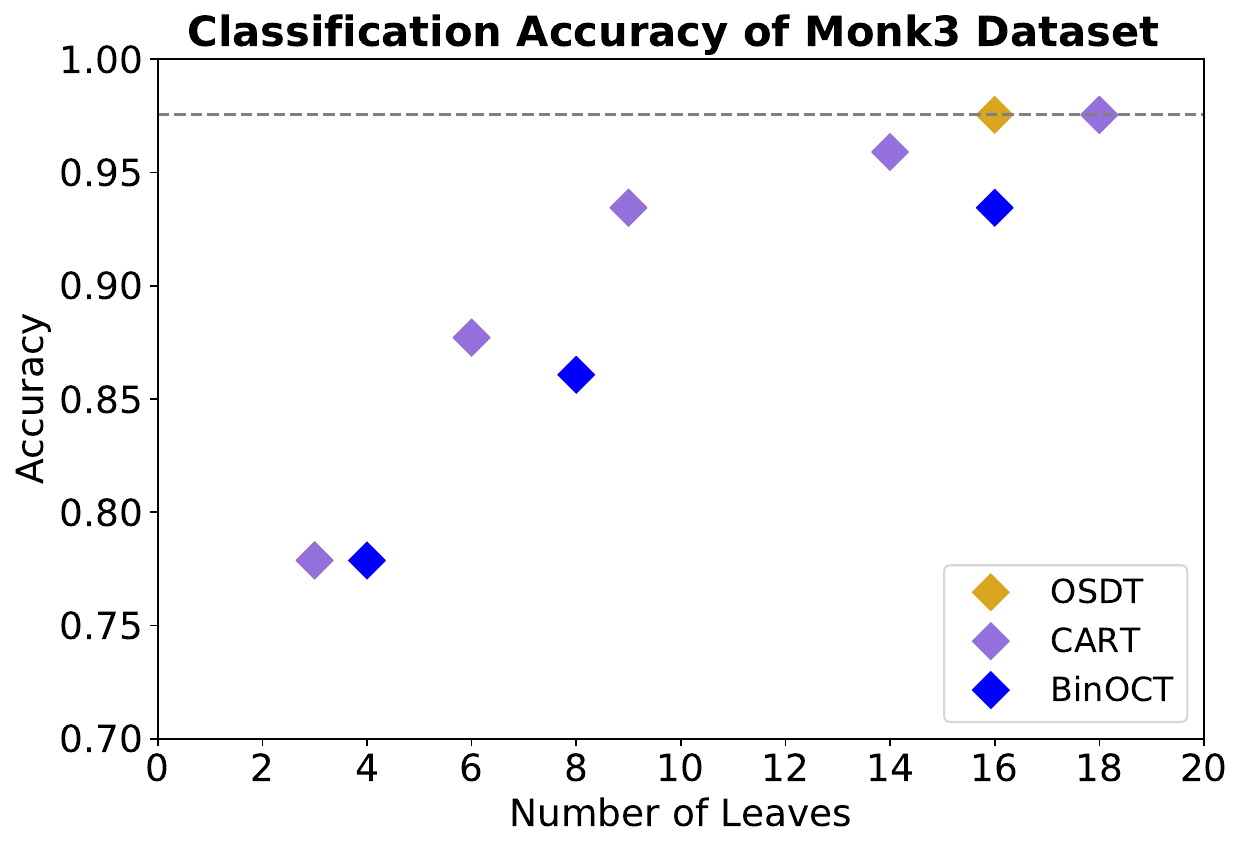}
        \label{fig:Monk3}
    \end{subfigure}
    \vskip -2mm
    \caption{Training accuracy of OSDT, CART, BinOCT on different datasets (time limit: 30 minutes). Horizontal lines indicate the accuracy of the best OSDT tree. On most datasets, all trees of BinOCT and CART are below this line.%
    } 
    \label{fig:train-accuracy}
    \vskip -5mm
\end{figure*}

We address the following questions through experimental analysis: 
(1)~\textrm{Do existing methods achieve optimal solutions, and if not, how far are they from optimal?}
(2)~\textrm{How fast does our algorithm converge given the hardness of the problem it is solving?} 
(3)~\textrm{How much does each of the bounds contribute to the performance of our algorithm?}
(4)~\textrm{What do optimal trees look like?}

The results of the per-bound performance and memory improvement experiment (Table~\ref{tab:ablation} in the supplement) were run on a $m5a.4xlarge$ instance of AWS's Elastic Compute Cloud (EC2). The instance has 16 2.5GHz virtual CPUs (although we run single-threaded on a single core) and 64 GB of RAM. All other results were run on a personal laptop with a 2.4GHz i5-8259U processor and 16GB of RAM.

We used 7 datasets: Five of them are from the UCI Machine Learning Repository \cite{Dua:2017}, (Tic Tac Toe, Car Evaluation, Monk1, Monk2, Monk3). The other two datasets are the ProPublica recidivism data set~\cite{LarsonMaKiAn16} and the Fair Isaac (FICO) credit risk dataset \cite{competition}. We predict which individuals are arrested within two years of release (${N = 7,215}$) on the recidivism data set and whether an individual will default on a loan for the FICO dataset.

\textit{Accuracy and optimality:}
 We tested the accuracy of our algorithm against baseline methods CART and BinOCT~\cite{verwer2019learning}. BinOCT is the most recent publicly available method for learning optimal classification trees and was shown to outperform other previous methods. As far as we know, there is no public code for most of the other relevant baselines, including  \cite{bertsimas2017optimal, molerooptimal, MenickellyGKS18}. One of these methods, OCT \cite{bertsimas2017optimal}, reports that CART often dominates their performance (see Fig. 4 and Fig. 5 in their paper). Our models can never be worse than CART's models even if we stop early, because in our implementation, we use the objective value of CART's solution as a warm start to the objective value of the current best. 
%
Figure~\ref{fig:train-accuracy} shows the training accuracy on each dataset. The time limits for both BinOCT and our algorithm are set to be 30 minutes. 

\begin{wrapfigure}{r}{0.5\textwidth}
  \centering
    \begin{subfigure}[b]{0.245\textwidth}
        \centering
        \includegraphics[trim={5mm 5mm 0mm 15mm}, width=\textwidth]{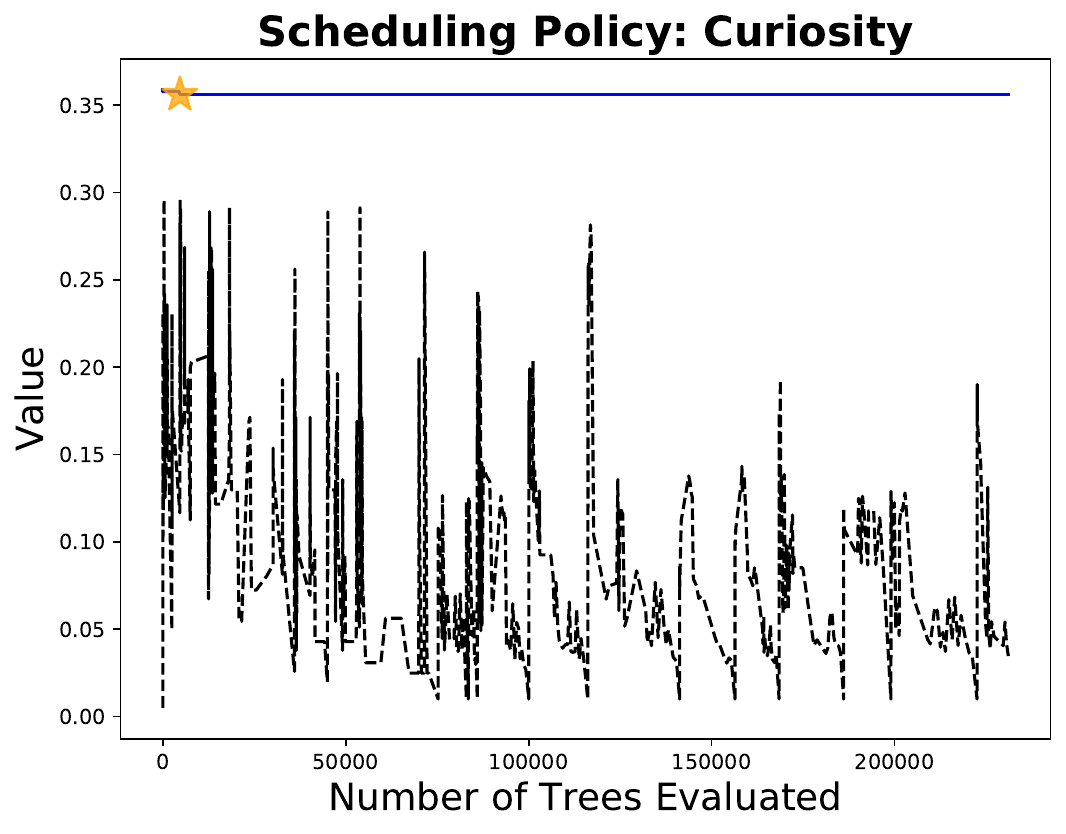}
        \label{fig:trace-COMPAS}
    \end{subfigure}
    \begin{subfigure}[b]{0.245\textwidth}  
        \centering 
        \includegraphics[trim={5mm 5mm 0mm 15mm}, width=\textwidth]{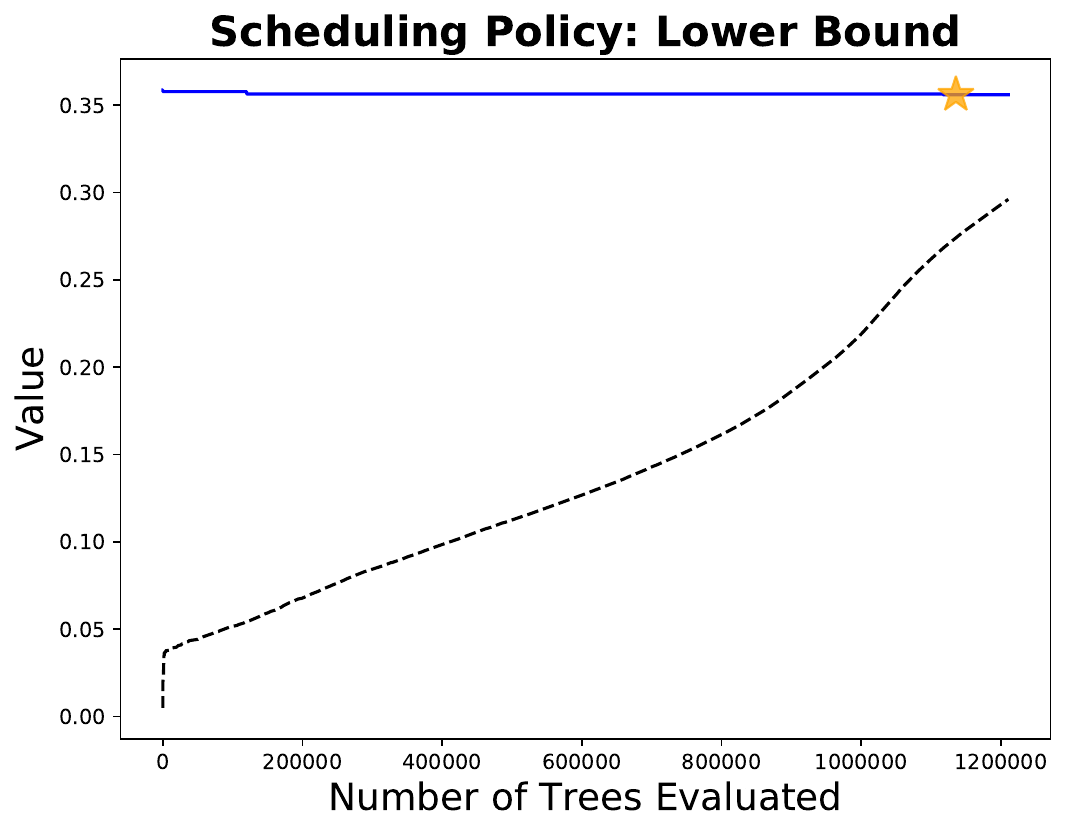}
        \label{fig:trace-FICO}
    \end{subfigure}
    \vskip -2mm
\caption{Example OSDT execution traces (COMPAS data, $\lambda=0.005$). Lines are the objective value and dashes are the lower bound for OSDT. For each scheduling policy, the time to optimum and optimal objective value are marked with a star.}
\label{fig:trace}
\end{wrapfigure}

\textit{Main results:} 
(i) We can now evaluate how close to optimal other methods are (and they are often close to optimal or optimal). (ii) Sometimes, the baselines are \textit{not} optimal. Recall that BinOCT searches only for the optimal tree \textit{given the topology of the complete binary tree of a certain depth}. This restriction on the topology massively reduces the search space so that BinOCT runs quickly, but in exchange, it misses optimal sparse solutions that our method finds.  (iii) Our method is \textit{fast}. Our method runs on only one thread (we have not yet parallelized it) whereas BinOCT is highly optimized; it makes use of eight threads. 
Even with BinOCT's 8-thread parallelism, our method is competitive.

\begin{wrapfigure}{r}{0.5\textwidth}
  \begin{subfigure}[b]{0.245\textwidth}
        \centering
        \includegraphics[trim={5mm 5mm 0mm 10mm}, width=\textwidth]{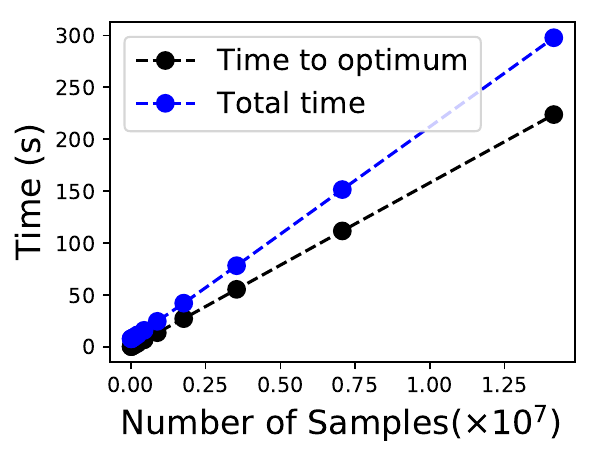}
        \caption[]%
        {{\footnotesize This is based on all the 12 features}}
        \label{fig:scal_ndata}
    \end{subfigure}
    \begin{subfigure}[b]{0.24\textwidth}  
        \centering 
        \includegraphics[trim={5mm 15mm 0mm 10mm}, width=\textwidth]{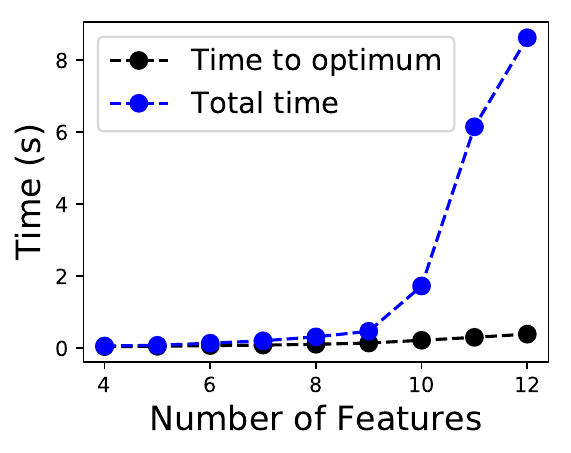}
        \label{fig:scal_nfeature.}
        \caption[]%
        {{\footnotesize The 4 features are those in Figure~\ref{fig:tree-compas-5}}}
    \end{subfigure}
  \vskip -2mm
\caption{Scalability with respect to number of samples and number of features using (multiples of) the ProPublica data set. (${\Reg = 0.005}$).
Note that all these executions include the 4 features of the optimal tree, and the data size are increased by duplicating the whole data set multiple times.
}
\label{fig:scalability}
\end{wrapfigure}

\textit{Convergence:}
Figure~\ref{fig:trace} illustrates the behavior of OSDT for the ProPublica COMPAS dataset with $\Reg = 0.005$, for two different scheduling policies (curiosity and lower bound, see supplement).
The charts show how the
current best objective value~$\CurrentObj$
and the lower bound~$b(\Prefix, \x, \y)$ vary as the algorithm progresses.
When we schedule using the lower bound, the lower bounds of evaluated trees increase monotonically, and OSDT certifies optimality only when the value of the lower bound becomes large enough that we can prune the remaining search space or when the queue is empty, whichever is reached earlier.
Using curiosity, OSDT finds the optimal tree much more quickly than when using the lower bound.

\textit{Scalability:}
Figure~\ref{fig:scalability} shows the scalability of OSDT with respect to the number of samples and the number of features. Runtime can theoretically grow exponentially with the number of features. However, as we add extra features that differ from those in the optimal tree, we can reach the optimum more quickly, because we are able to prune the search space more efficiently as the number of extra features grows. For example, with 4 features, it spends about 75\% of the runtime to reach the optimum; with 12 features, it takes about 5\% of the runtime to reach the optimum.

\textit{Ablation experiments:} Appendix \ref{appendix:ablation} shows that the lookahead and equivalent points bounds are, by far, the most significant of our bounds, reducing time to optimum by at least two orders of magnitude and reducing memory consumption by more than one order of magnitude.

\begin{arxiv}
\textit{Algorithm optimizations:}
Next, we evaluate how much each of our bounds contributes to OSDT's performance and what effect the scheduling metric has on execution.
Table~\ref{tab:ablation} provides experimental statistics of total execution time, time to optimum, total number of trees evaluated,
number of trees evaluated to optimum, and memory consumption on the recidivism data set.
The first row is the full OSDT implementation, and the others are variants each of which removes a specific bound.
While all the optimizations reduce the search space, the lookahead and equivalent points bounds are, by far, the most significant, reducing time to optimum by at least two orders of magnitude and reducing memory consumption by more than one order of magnitude.
In our experiment, although the scheduling policy has a smaller effect, it is still significant -- curiosity is a factor of two faster than the objective function and consumes 25\% of the memory consumed when using the objective function. All other scheduling policies, \ie the lower bound and the entropy, are significantly worse.

\begin{table}[t!]
\centering
Per-bound performance improvement (ProPublica data set) \\
\vspace{1mm}
\begin{tabular}{l | c | c | c }
& Total time & Slow- & Time to \\
Algorithm variant & (s) & down & optimum (s) \\
\hline
All bounds & 14.75 & --- & 1.09 \\
No support bound & 16.69 & 1.13$\times$ & 1.18 \\
No incremental accuracy bound & 29.42 & 1.99$\times$ & 1.27 \\
No accuracy bound & 31.72 & 2.15$\times$ & 1.44 \\
No lookahead bound & 31479 & 2134$\times$ & 186 \\
No equivalent points bound & $>$8226 & $>$557$\times$ & --- \\
\hline
\end{tabular}
\begin{tabular}{l | c | c | c}
\hline
 & Total \#trees & \#trees &  Mem \\
Algorithm variant & evaluated & to optimum &~ (GB) \\
\hline
All bounds & 241306 & 18195 & .08 \\
No support bound & 278868 & 21232 & .08 \\
No incremental accuracy bound & 548618 & 24682 & .08 \\
No accuracy bound & 482013 & 23075 & .09 \\
No lookahead bound & 283712499 & 3078445 & 10 \\
No equivalent points bound & $>$41000000 & --- & $>$64 \\
\end{tabular}
\vspace{2mm}
\caption{Per-bound performance improvement, for the ProPublica data set
(${\Reg = 0.005}$, cold start, using curiosity).
The columns report the total execution time,
time to optimum, total number of trees evaluated,
number of trees evaluated to optimum,
and memory consumption.
The first row shows our algorithm with all bounds; subsequent rows show variants
that each remove a specific bound 
(one bound at a time, not cumulative).
%
All rows except the last one represent a complete execution, \ie until the queue is empty.
For the last row (`No equivalent points bound'), the algorithm was terminated after running out of the memory (about $\sim$64GB RAM). \vspace*{-20pt}
%
%
}
\label{tab:ablation}
\end{table}

\begin{table}[t!]
\centering
Performance of different priority metrics (ProPublica data set) \\
\vspace{1mm}
\begin{tabular}{l | c | c | c }
& Total time & Slow- & Time to \\
Priority metric & (s) & down & optimum (s) \\
\hline
Curiosity & 96 & --- & 25 \\
Objective & 130 & 1.35$\times$ & 57 \\
Lower bound & 186 & 1.93$\times$ & 175 \\
Entropy & 671 & 6.98$\times$ & 582 \\
Gini & 646 & 6.72$\times$ & 551 \\
\hline
\end{tabular}

\begin{tabular}{l | c | c | c}
\hline
 & Total \#trees & \#trees &  Memory Con-\\
Priority metric & evaluated & to optimum &~ sumption (GB) \\
\hline
Curiosity & 241306 & 18195 & .08 \\
Objective & 635449 & 423627 & .30 \\
Lower bound & 1714011 & 1713709 & .78 \\
Entropy & 5381467 & 5332413 & 4.90 \\
Gini & 5381001 & 5331925 & 4.90 \\
\end{tabular}
\vspace{4mm}
\caption{Performance of different priority policies, for the ProPublica data set
(${\Reg = 0.005}$, cold start).
The columns report the total execution time, factor slower,
time to optimum, total number of trees evaluated,
number of trees evaluated to optimum,
and memory consumption.
The first row shows our algorithm with curiosity (\ref{eq:curio}) as the priority queue metric; subsequent rows show variants
that use other policies.
All rows represent complete executions that certify optimality.
%
%
}
\vspace{4mm}
\label{tab:scheduling}
\end{table}
\end{arxiv}

\textit{Trees:}
We provide illustrations of the trees produced by OSDT and the baseline methods in Figures~\ref{fig:tree-compas-5}, \ref{fig:tree-tictactoe} and \ref{fig:tree-monk1}. OSDT generates trees of any shape, and our objective penalizes trees with more leaves, thus it never introduces splits that produce a pair of leaves with the same label. 
In contrast, BinOCT trees are always complete binary trees of a given depth.
This limitation on the tree shape can prevent BinOCT from finding the globally optimal tree.
In fact, BinOCT often produces useless splits, leading to trees with more leaves than necessary to achieve the same accuracy.

\begin{figure}[t]
\centering
    \begin{subfigure}{\linewidth}
        \centering
        \includegraphics[trim={30mm 30mm 30mm 8mm}, width=0.25\linewidth]{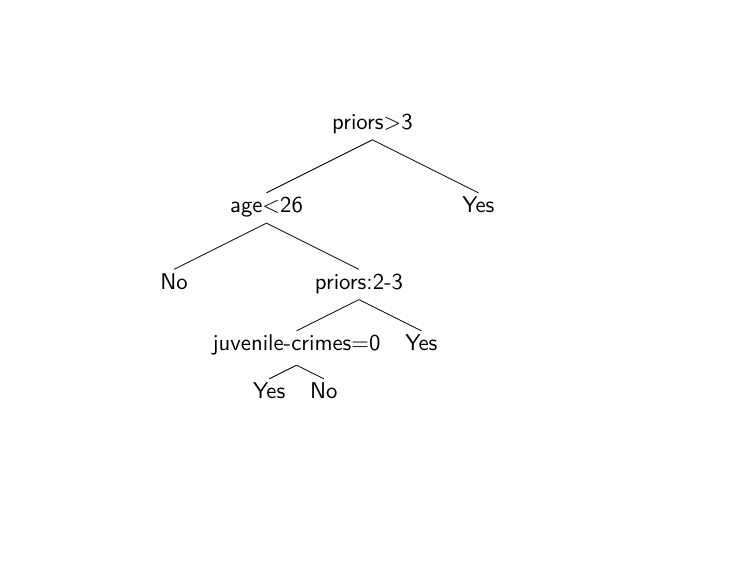}
    \end{subfigure}
\caption{An optimal decision tree generated by OSDT on the COMPAS dataset. ($\Reg = 0.005$, accuracy: 66.90\%) 
}
\label{fig:tree-compas-5}
\vskip -5mm
\end{figure}


\begin{figure}[t!]
\centering
    \begin{subfigure}[b]{0.4\textwidth}
        \centering
        \includegraphics[trim={0mm 0mm 0mm 0mm}, width=\textwidth]{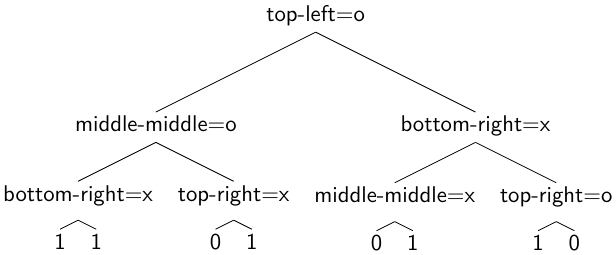}
        \caption[]%
        {{\footnotesize BinOCT (accuracy: 76.722\%)}}
        \label{fig:tree-tictactoe-binoct}
    \end{subfigure}
    \begin{subfigure}[b]{0.3\textwidth} 
        \centering 
        \includegraphics[trim={0mm 0mm 0mm 5mm}, width=\textwidth]{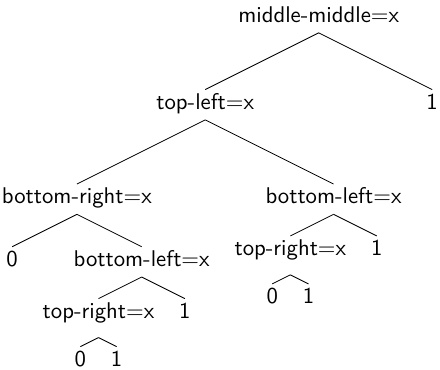}
        \caption[]%
        {{\footnotesize OSDT (accuracy: 82.881\%)}}
        \label{fig:tree-tictactoe-osdt}
    \end{subfigure}
    \vskip -2mm
\caption{Eight-leaf decision trees generated by BinOCT and OSDT on the Tic-Tac-Toe data. Trees of BinOCT must be complete binary trees, while OSDT can generate binary trees of any shape.
}
\label{fig:tree-tictactoe}
\vskip -3mm
\end{figure}

\begin{figure}[t!]
\centering

    \centering
    \begin{subfigure}[b]{0.35\textwidth}
        \centering
        \includegraphics[trim={0mm 0mm 0mm 3mm}, width=\textwidth]{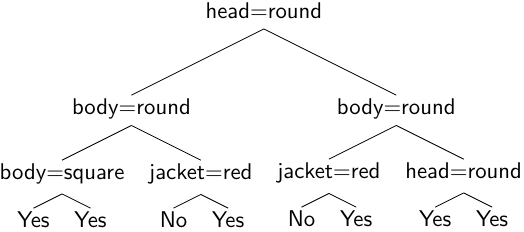}
        \caption[]%
        {{\footnotesize BinOCT (accuracy: 91.129\%)}}
        \label{fig:tree-monk1-binoct}
    \end{subfigure}
    \begin{subfigure}[b]{0.35\textwidth} 
        \centering 
        \includegraphics[trim={0mm 0mm 0mm 3mm}, width=\textwidth]{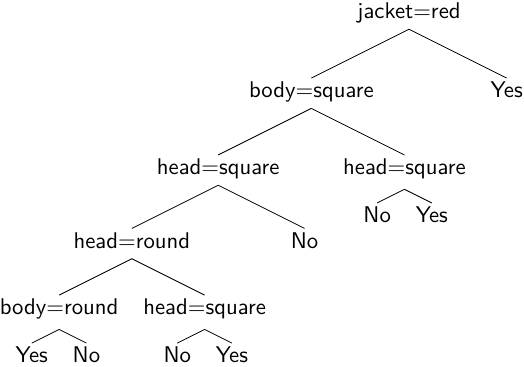}
        \caption[]%
        {{\footnotesize OSDT (accuracy: 100\%)}}
        \label{fig:tree-monk1-osdt}
    \end{subfigure}
    \vskip -2mm
\caption{Decision trees generated by BinOCT and OSDT on the Monk1 dataset. The tree generated by BinOCT includes two useless splits (the left and right splits), while OSDT can avoid this problem. BinOCT is 91\% accurate, OSDT is 100\% accurate. 
}
\label{fig:tree-monk1}
\vskip -5mm
\end{figure}

\textit{Additional experiments:} It is well-established that simpler models such as small decision trees generalize well; a set of cross-validation experiments is in the supplement demonstrating this.


\textbf{Conclusion:}
Our work shows the possibility of optimal (or provably near-optimal) sparse decision trees. It is the first work to balance the accuracy and the number of leaves optimally in a practical amount of time.
We have reason to believe this framework can be extended to much larger datasets. Theorem~\ref{thm:similar} identifies a key mechanism for scaling these algorithms up. It suggests a bound stating that highly correlated features can substitute for each other, leading to similar model accuracies. Applications of this bound allow for the elimination of features throughout the entire execution, allowing for more aggressive pruning. Our experience to date shows that by supporting such bounds with the right data structures can potentially lead to dramatic increases in performance and scalability. 

\clearpage

\vspace{-1mm}
\bibliographystyle{abbrv}
\small
\bibliography{refs}

\appendix 

\newpage
\begin{center}
    {\Large{\textbf{Optimal Sparse Decision Trees: Supplementary Material}}}
\end{center}

\section{Branch and Bound Algorithm}
\label{appendix:alg}
Algorithm \ref{alg:branch-and-bound} shows the structure of our approach.
\begin{algorithm}[b!]
\caption{Branch-and-bound for learning optimal decision trees. }
\label{alg:branch-and-bound}
\begin{algorithmic}
\normalsize
\State \textbf{Input:} Objective function $\Obj(\RL, \x, \y)$,
objective lower bound ${b(\Prefix, \x, \y)}$,
set of features ${S = \{s_m\}_{m=1}^M}$,
training data $(\x, \y) = {\{(x_n, y_n)\}_{n=1}^N}$,
initial best known tree~$\InitialRL$ with objective
${\InitialObj = \Obj(\InitialRL, \x, \y)}$;
$\InitialRL$ could be obtained as output from another (approximate) algorithm,
otherwise, $(\InitialRL, \InitialObj) = (\text{null}, 1)$ provides reasonable default values. The initial value of $\DefaultLabels$ is the majority label of the whole dataset.
\State \textbf{Output:} Provably optimal decision tree~$\OptimalRL$ with minimum objective~$\OptimalObj$ \\

\State $(\CurrentRL, \CurrentObj) \gets (\InitialRL, \InitialObj)$ \Comment{Initialize best tree and objective}
\State $Q \gets $ queue$(\,[\,((), (), (), \DefaultLabels, 0, 0)\,]\,)$ \Comment{Initialize queue with empty tree}
\While {$Q$ not empty} \Comment{Stop when queue is empty}
	\State $d=(\Prefix, \Labels, \Default, \DefaultLabels, K, H) \gets Q$.pop(\,) \Comment{Remove tree~$\RL$ from the queue}
	\If {$b(\Prefix, \x, \y) < \CurrentObj$} \Comment{\textbf{Bound}: Apply Theorem~\ref{thm:bound}}
        \State $\Obj \gets \Obj(\RL, \x, \y)$ \Comment{Compute objective of tree~$\RL$}
        \If {$\Obj < \CurrentObj$} \Comment{Update best tree and objective}
            \State $(\CurrentRL, \CurrentObj) \gets (\RL, \Obj)$
        \EndIf
        \For {every possible combination of features to split $\Default$}
        \State \Comment{\textbf{Branch}: Enqueue~$\Prefix$'s children}
            \State split $\Default$ and get new leaves $d_{\textrm{new}}$
            \For {each possible subset $\Default'$ of $d_{\textrm{new}}$}
                \State $ \Prefix'=\Prefix \cup (d_{\textrm{new}}\setminus \Default')$
                \State $Q$.push$(\,(\Prefix', \Labels', \Default', \DefaultLabels', K', H')\,)$
            \EndFor
        \EndFor

    \EndIf
\EndWhile
\State $(\OptimalRL, \OptimalObj) \gets (\CurrentRL, \CurrentObj)$ \Comment{Identify provably optimal solution}
\end{algorithmic}
\end{algorithm}

\section{Equivalent Points Bound}
\label{sec:identical}
When multiple observations captured by a leaf in~$\Default$
have identical features but opposite labels, then no tree, including those that extend $\Default$, can correctly classify all of these observations. The number of misclassifications must be at least the minority label of the equivalent points.
%

For data set~${\{(x_n, y_n)\}_{n=1}^N}$ and a set of features
${\{s_m\}_{m=1}^M}$,
we define a set of samples to be equivalent if they have exactly the same feature values, \ie ${(x_i, y_i)}$ and ${(x_j, y_j)}$ are equivalent~if
$\frac{1}{M} \sum_{m=1}^M \one [ \Cap(x_i, s_m) = \Cap(x_j, s_m) ] = 1$. 
Note that a data set consists of multiple sets of equivalent points;
let~${\{e_u\}_{u=1}^U}$ enumerate these sets.
For each observation~$x_i$, it belongs to a equivalent points set~$e_u$.
We denote the fraction of data with the minority label in set~$e_u$ as~$\theta(e_u)$, \eg let
\begin{arxiv}
\begin{align}
{e_u = \{x_n : \forall m \in [M],\, \one [ \Cap(x_n, s_m) = \Cap(x_i, s_m) ] \}}, \nn
\end{align}
\end{arxiv}
\begin{kdd}
${e_u = \{x_n : \forall m \in [M],\, }$ ${\one [ \Cap(x_n, s_m) = \Cap(x_i, s_m) ] \}}$,
\end{kdd}
and let~$q_u$ be the minority class label among points in~$e_u$, then
\begin{align}
\theta(e_u) = \frac{1}{N} \sum_{n=1}^N \one [ x_n \in e_u ]\, \one [ y_n = q_u ].
\label{eq:theta}
\end{align}
We can combine the equivalent points bound with other bounds to get a tighter lower bound on the objective function. As the experimental results demonstrate in \S\ref{sec:experiments}, there is sometimes a substantial reduction of the search space after incorporating the equivalent points bound.
We propose a general equivalent points bound in Proposition~\ref{prop:identical}. We incorporate it into our framework by proposing the specific equivalent points bound in Theorem~\ref{thm:identical}.
\begin{proposition}[General equivalent points bound]
\label{prop:identical}
Let ${\RL = (\Prefix, \Labels, \Default, \DefaultLabels, K, H)}$ be a tree, then
$\Obj(\RL, \x, \y) \ge \sum_{u=1}^U \theta(e_u) + \Reg H$. 
\end{proposition}

\begin{arxiv}
\begin{proof}
Recall that the objective is ${\Obj(\RL, \x, \y) = \Loss(\RL, \x, \y) + \Reg H}$,
where the misclassification error~${\Loss(\RL, \x, \y)}$ is given by
\begin{align}
\Loss(\RL, \x, \y)
&= \frac{1}{N} \sum_{n=1}^N \sum_{k=1}^K \Cap(x_n, p_k) \wedge \one [\q_k \neq y_n]. \nn
\end{align}
Any particular tree uses a specific leaf, and therefore a single class label,
to classify all points within a set of equivalent points.
Thus, for a set of equivalent points~$u$, the tree~$\RL$ correctly classifies either
points that have the majority class label, or points that have the minority class label.
It follows that~$\RL$ misclassifies a number of points in~$u$ at least as great as
the number of points with the minority class label.
To translate this into a lower bound on~$\Loss(\RL, \x, \y)$,
we first sum over all sets of equivalent points, and then for each such set,
count differences between class labels and the minority class label of the set,
instead of counting mistakes:
\begin{align}
\Loss(\RL, \x, \y) \nn
&= \frac{1}{N} \sum_{u=1}^U \sum_{n=1}^N \sum_{k=1}^K \Cap(x_n, p_k) \wedge \one [\q_k \neq y_n] \wedge
   \one [x_n \in e_u]  \nn \\
&\ge \frac{1}{N} \sum_{u=1}^U \sum_{n=1}^N \sum_{k=1}^K \Cap(x_n, p_k) \wedge \one [q_u = y_n] \wedge
   \one [x_n \in e_u]  \nn \\
\label{eq:lb-equiv-pts}
\end{align}
Next, because every datum must be captured by a leaf in the tree~$\RL$, $\sum_{k=1}^K \Cap(x_n, p_k)=1$.
\begin{align}
\Loss(\RL, \x, \y) &= \frac{1}{N} \sum_{u=1}^U \sum_{n=1}^N \sum_{k=1}^K \Cap(x_n, p_k) \wedge \one [q_u = y_n] \wedge
   \one [x_n \in e_u] \nn \\
&= \frac{1}{N} \sum_{u=1}^U \sum_{n=1}^N \one [ x_n \in e_u ]\, \one [ y_n = q_u ]
= \sum_{u=1}^U \theta(e_u), \nn
\end{align}
where the final equality applies the definition of~$\theta(e_u)$ in~\eqref{eq:theta}.
Therefore, ${\Obj(\RL, \x, \y) =}$ ${\Loss(\RL, \x, \y) + \Reg K}$ ${\ge \sum_{u=1}^U \theta(e_u) + \Reg K}$.
\end{proof}
\end{arxiv}

Recall that in our lower bound~${b(\Prefix, \x, \y)}$
in~\eqref{eq:lower-bound}, we leave out the misclassification errors of leaves we are going to split~$\Loss_0(\Default, \DefaultLabels,$ $\x, \y)$
from the objective~${\Obj(\RL, \x, \y)}$.
Incorporating the equivalent points bound in Theorem~\ref{thm:identical}, we obtain a tighter bound on our objective because we now have a tighter lower bound on the misclassification errors of leaves we are going to split,
${0 \le b_0(\Default, \x, \y) \le}$ $\Loss_0(\Default, \DefaultLabels, \x, \y)$.

\begin{theorem}[Equivalent points bound]
\label{thm:identical}
Let~$\RL$ be a tree with leaves~$\Prefix, \Default$
and lower bound ${b(\Prefix, \x, \y)}$,
then for any tree~${\RL' \in \StartsWith(\RL)}$
whose prefix~$\Prefix' \supseteq \Prefix$,
\begin{align}
\Obj(\RL', \x, \y) \ge b(\Prefix, \x, \y) + b_0(\Default, \x, \y),\;\; \textrm{ where }
\label{eq:identical}
\end{align}
\begin{align}
b_0(\Default, \x, \y) &= \frac{1}{N} \sum_{u=1}^U \sum_{n=1}^N
    \Cap(x_n, \Default) \wedge\one [ x_n \in e_u ]\, \one [ y_n = q_u ].
\label{eq:lb-b0}
\end{align}
\end{theorem}

\begin{arxiv}
\begin{proof}
See Appendix~\ref{appendix:equiv-pts} for the proof of Theorem~\ref{thm:identical}.
\end{proof}
\end{arxiv}

\section{Upper Bounds on Number of Leaves}
\label{sec:ub-prefix-length}

\begin{arxiv}
In this section, we propose several upper bounds on the number of leaves:
\begin{itemize}
\item The simplest upper bound on the number of leaves is given by the total
number of leaves of the perfect binary tree. (Proposition~\ref{prop:trivial-length})
\item The current best objective~$\CurrentObj$ implies an upper bound
on the number of leaves. (Theorem~\ref{thm:ub-prefix-length})
\item For intuition, we state a version of the above bound that is valid
at the start of execution. (Corollary~\ref{cor:ub-prefix-length})
\item By considering specific families of trees,
we can obtain tighter bounds on the number of leaves. (Theorem~\ref{thm:ub-prefix-specific})
\end{itemize}
In the next section (\S\ref{sec:ub-size}), we use these results
to derive corresponding upper bounds on the number of
prefix evaluations made by Algorithm~\ref{alg:branch-and-bound}.

\begin{proposition}[Trivial upper bound on the number of leaves]
\label{prop:trivial-length}
Consider a state space of all trees formed from
a set of~$M$ features,
and let~$L(\RL)$ be the number of leaves of tree~$\RL$.
$2^M$ provides an upper bound on the  number of leaves of
any optimal tree
${\OptimalRL \in \argmin_\RL \Obj(\RL, \x, \y)}$,
\ie ${L(\RL) \le 2^M}$.
\end{proposition}

\begin{proof}
Perfect binary tree of depth $M$ contains $2^M$ leaves.
\end{proof}
\end{arxiv}

During the branch-and-bound execution, the current best objective~$\CurrentObj$
implies an upper bound on the maximum number of leaves for those trees we still need to consider.
\begin{theorem}[Upper bound on the number of leaves]
\label{thm:ub-prefix-length}
For a dataset with $M$ features, consider a state space of all trees.
Let~$L(\RL)$ be the number of leaves of tree~$\RL$
and let~$\CurrentObj$ be the current best objective.
For all optimal trees ${\OptimalRL \in \argmin_\RL \Obj(\RL, \x, \y)}$
\begin{arxiv}
\begin{align}
L(\OptimalRL) \le \min \left(\left\lfloor \frac{\CurrentObj}{\Reg} \right\rfloor, 2^M \right),
\label{eq:max-length}
\end{align}
\end{arxiv}
\begin{kdd}
\begin{align}
L(\OptimalRL) \le \min \left(\left\lfloor \CurrentObj / \Reg \right\rfloor, 2^M \right),
\label{eq:max-length}
\end{align}
\end{kdd}
where~$\Reg$ is the regularization parameter.
\begin{arxiv}
Furthermore, if~$\CurrentRL$ is a tree with
objective ${\Obj(\CurrentRL, \x, \y) }$ ${= \CurrentObj}$,
size~$H$, and zero misclassification error,
then for every optimal tree
${\OptimalRL \in \argmin_\RL}$ ${ \Obj(\RL, \x, \y)}$,
if ${\CurrentRL \in \argmin_d \Obj(\RL, \x, \y)}$,
then ${L(\OptimalRL) \le H}$,
or otherwise if ${\CurrentRL \notin \argmin_d \Obj(\RL, \x, \y)}$,
then ${L(\OptimalRL) \le H - 1}$.
\end{arxiv}
\end{theorem}

\begin{arxiv}
\begin{proof}
For an optimal tree~$\OptimalRL$ with objective~$\OptimalObj$,
\begin{align}
\Reg L(\OptimalRL) \le \OptimalObj = \Obj(\OptimalRL, \x, \y)
= \Loss(\OptimalRL, \x, \y) + \Reg L(\OptimalRL)
\le \CurrentObj. \nn
\end{align}
The maximum possible number of leaves for~$\OptimalRL$ occurs
when~$\Loss(\OptimalRL, \x, \y)$ is minimized;
combining with Proposition~\ref{prop:trivial-length}
gives bound~\eqref{eq:max-length}.

For the rest of the proof,
let~${K^* = L(\OptimalRL)}$ be the number of leaves of~$\OptimalRL$.
If the current best tree~$\CurrentRL$ has zero
misclassification error, then
\begin{align}
\Reg H^* \leq \Loss(\OptimalRL, \x, \y) + \Reg H^* = \Obj(\OptimalRL, \x, \y)
\le \CurrentObj = \Obj(\CurrentRL, \x, \y) = \Reg H, \nn
\end{align}
and thus ${H^* \leq H}$.
If the current best tree is suboptimal,
\ie ${\CurrentRL \notin \argmin_\RL \Obj(\RL, \x, \y)}$, then
\begin{align}
\Reg H^* \leq \Loss(\OptimalRL, \x, \y) + \Reg H^* = \Obj(\OptimalRL, \x, \y)
< \CurrentObj = \Obj(\CurrentRL, \x, \y) = \Reg H, \nn
\end{align}
in which case ${H^* < H}$, \ie ${H^* \leq H-1}$, since $H$ is an integer.
\end{proof}

The latter part of Theorem~\ref{thm:ub-prefix-length} tells us that
if we only need to identify a single instance of an optimal tree
${\OptimalRL \in \argmin_\RL \Obj(\RL, \x, \y)}$, and we encounter a perfect
$H$-leaf tree with zero misclassification error, then we can prune all
tree of number of leaves~$H$ or greater.

\end{arxiv}

\begin{corollary}[A priori upper bound on the number of leaves]
\label{cor:ub-prefix-length}
\begin{arxiv}
Let~$L(\RL)$ be the number of leaves of tree~$\RL$.
\end{arxiv}
For all optimal trees ${\OptimalRL \in \argmin_\RL \Obj(\RL, \x, \y)}$,
\begin{arxiv}
\begin{align}
L(\OptimalRL) \le \min \left( \left\lfloor \frac{1}{2\Reg} \right\rfloor, 2^M \right).
\label{eq:max-length-trivial}
\end{align}
\end{arxiv}
\begin{kdd}
\begin{align}
L(\OptimalRL) \le \min \left( \left\lfloor 1 / 2\Reg \right\rfloor, 2^M \right).
\label{eq:max-length-trivial}
\end{align}
\end{kdd}
\end{corollary}

\begin{arxiv}
\begin{proof}
Let ${d = ((), (), (), (), 0, 0)}$ be the empty tree;
it has objective ${\Obj(\RL, \x, \y) = \Loss(\RL, \x, \y) \le}$ ${1/2}$,
which gives an upper bound on~$\CurrentObj$.
Combining with~\eqref{eq:max-length}
and Proposition~\ref{prop:trivial-length}
gives~\eqref{eq:max-length-trivial}.
\end{proof}
\end{arxiv}

For any particular tree~$\RL$ with unchanged leaves~$\Prefix$, we can obtain potentially tighter
upper bounds on the number of leaves for
all its child trees whose unchanged leaves include~$\Prefix$.

\begin{theorem}[Parent-specific upper bound on the number of leaves]
\label{thm:ub-prefix-specific}
Let ${\RL = (\Prefix, }$ ${\Labels, \Default, \DefaultLabels, K, H)}$ be a tree, let
${\RL' = }$ ${(\Prefix', \Labels', \Default', \DefaultLabels', K', H') \in \StartsWith(\RL)}$
be any child tree such that $\Prefix' \supseteq \Prefix$,
and let $\CurrentObj$ be the current best objective.
If~$\Prefix'$ has lower bound~${b(\Prefix', \x, }$ ${\y) < \CurrentObj}$, then
\begin{align}
H' < \min \left( H + \left\lfloor \frac{\CurrentObj - b(\Prefix, \x, \y)}{\Reg} \right\rfloor, 2^M \right).
\label{eq:max-length-prefix}
\end{align}
\end{theorem}

\begin{arxiv}
\begin{proof}
First, note that~${H' \ge H}$.
Now recall from~\eqref{eq:prefix-lb} that
\begin{align}
b(\Prefix, \x, \y) = \Loss_p(\Prefix, \Labels, \x, \y) + \Reg H
\le \Loss_p(\Prefix', \Labels', \x, \y) + \Reg H' = b(\Prefix', \x, \y), \nn
\end{align}
and from~\eqref{eq:prefix-loss} that
${\Loss_p(\Prefix, \Labels, \x, \y) \le \Loss_p(\Prefix', \Labels', \x, \y)}$.
Combining these bounds and rearranging gives
\begin{align}
b(\Prefix', \x, \y) &= \Loss_p(\Prefix', \Labels', \x, \y) + \Reg H + \Reg(H' - H) \nn \\
&\ge \Loss_p(\Prefix, \Labels, \x, \y) + \Reg H + \Reg(H' - H)
= b(\Prefix, \x, \y) + \Reg (H' - H).
\label{eq:length-diff}
\end{align}
Combining~\eqref{eq:length-diff} with~${b(\Prefix', \x, \y) < \CurrentObj}$
and Proposition~\ref{prop:trivial-length} gives~\eqref{eq:max-length-prefix}.
\end{proof}
\end{arxiv}

Theorem~\ref{thm:ub-prefix-specific} can be viewed as a generalization
of the one-step lookahead bound (Lemma~\ref{lemma:lookahead}). This is because we can view \eqref{eq:max-length-prefix} as a bound on ${H' - H}$,
which provides an upper bound on the number of remaining splits we may need, based on the best tree we already have evaluated.
\begin{arxiv}
Notice that when~${\RL = ((), (), (), (), 0, 0)}$ is the empty tree,
this bound replicates~\eqref{eq:max-length}, since ${b(\Prefix, \x, \y) = 0}$.
\end{arxiv}

\section{Upper Bounds on Number of Tree Evaluations}
\label{sec:ub-size}

In this section, based on the upper bounds on the number of leaves
from~\S\ref{sec:ub-prefix-length}, we give corresponding
upper bounds on the number of tree evaluations made by
Algorithm~\ref{alg:branch-and-bound}.
First, in Theorem~\ref{thm:remaining-eval-fine},
based on information about the state of
Algorithm~\ref{alg:branch-and-bound}'s execution, we calculate, for any given execution state,
upper bounds on the number of additional tree evaluations needed for the execution to complete.
We define the number of \emph{remaining tree evaluations} as the number of
trees that are currently in, or will be inserted into, the queue.
We evaluate the number of tree evaluations based on current execution information of the current best objective $\CurrentObj$ and the trees in the queue $\Queue$ of Algorithm~\ref{alg:branch-and-bound}.

\begin{arxiv}
We use Theorem~\ref{thm:remaining-eval-fine} in some of our empirical results
(\S\ref{sec:experiments}, Figure~\ref{fig:objective}) to help illustrate
the dramatic impact of certain algorithm optimizations.
The execution trace of this upper bound on remaining tree evaluations
complements the execution traces of other quantities,
\eg that of the current best objective~$\CurrentObj$.
After presenting Theorem~\ref{thm:remaining-eval-fine}, we also give two
weaker propositions that provide useful intuition.
In particular, Proposition~\ref{prop:remaining-eval-coarse} is a practical
approximation to Theorem~\ref{thm:remaining-eval-fine} that is significantly
easier to compute; we use it in our implementation as a metric of
execution progress that we display to the user.
\end{arxiv}

\begin{theorem}[Upper bound on number of remaining tree evaluations]
\label{thm:remaining-eval-fine}
~Consider the state space of all possible leaves formed from a set of~$M$ features,
and consider Algorithm~\ref{alg:branch-and-bound} at a particular instant
during execution.
Denote the current best objective as $\CurrentObj$, the queue as $\Queue$,
and the size of prefix~$\Prefix$ as $L(\Prefix)$.
Denoting the number of remaining
prefix evaluations as ${\Remaining(\CurrentObj, \Queue)}$, the bound is:
\begin{align}
\Remaining(\CurrentObj, \Queue)
\le \sum_{\Prefix \in Q} \sum_{k=0}^{f(\Prefix)} \frac{(3^M - L(\Prefix))!}{(3^M - L(\Prefix) - k)!},
\label{eq:remaining}
\end{align}
\begin{arxiv}
where
\begin{align}
f(\Prefix) = \min \left( \left\lfloor
  \frac{\CurrentObj - b(\Prefix, \x, \y)}{\Reg} \right\rfloor, 3^M - L(\Prefix)\right). \nn
\end{align}
\end{arxiv}
\begin{kdd}
\begin{align}
\text{where} \quad f(\Prefix) = \min \left( \left\lfloor
  \frac{\CurrentObj - b(\Prefix, \x, \y)}{\Reg} \right\rfloor, 3^M - L(\Prefix)\right).
\end{align}
\end{kdd}
\end{theorem}

\begin{arxiv}
Next, we propose a Corollary~\ref{thm:ub-total-eval} by leaving out the information of the execution state. It is strictly weaker than
Theorem~\ref{thm:remaining-eval-fine}. 
It
\end{arxiv}
\begin{kdd}
The corollary below
\end{kdd}
is a na\"ive upper bound on
the total number of tree evaluations during the process of
Algorithm~\ref{alg:branch-and-bound}'s execution.
It does not use algorithm execution state to
bound the size of the search space like Theorem~\ref{thm:remaining-eval-fine}, and it relies only on the number of features and the regularization parameter~$\Reg$.

\begin{corollary}[Upper bound on the total number of tree evaluations]
\label{thm:ub-total-eval}
~Define $\TotalRemaining(\RuleSet)$ to be the total number of trees
evaluated by Algorithm~\ref{alg:branch-and-bound}, given the state space of
all possible leaves formed from a set~$\RuleSet$ of~$M$ features.
For any set~$\RuleSet$ of all leaves formed of $M$ features,
\begin{align}
\TotalRemaining(\RuleSet) \le \sum_{k=0}^K \frac{3^M!}{(3^M - k)!}, \text{where } {K = \min(\lfloor 1/2 \Reg \rfloor, 2^M)}.\nn
\end{align}
\end{corollary}

\begin{arxiv}

Our next upper bound is strictly tighter than the bound in
Corollary~\ref{thm:ub-total-eval}.
Like Theorem~\ref{thm:remaining-eval-fine}, it uses the
current best objective and information about
the lengths of prefixes in the queue to constrain
the lengths of prefixes in the remaining search space.
However, Proposition~\ref{prop:remaining-eval-coarse}
is weaker than Theorem~\ref{thm:remaining-eval-fine} because
it leverages only coarse-grained information from the queue.
Specifically, Theorem~\ref{thm:remaining-eval-fine} is
strictly tighter because it additionally incorporates
prefix-specific objective lower bound information from
prefixes in the queue, which further constrains
the lengths of prefixes in the remaining search space.

\begin{proposition}[Coarse-grained upper bound on remaining prefix evaluations] \hfill
\label{prop:remaining-eval-coarse}
Consider a state space of all possible leaves formed from a set of~$M$ features,
and consider Algorithm~\ref{alg:branch-and-bound} at a particular instant
during execution.
Let~$\CurrentObj$ be the current best objective, let~$\Queue$ be the queue,
and let~$L(\Prefix)$ be the length of prefix~$\Prefix$.
Let~$\Queue_j$ be the number of prefixes of length~$j$ in~$\Queue$,
\begin{align}
\Queue_j = \big | \{ \Prefix : L(\Prefix) = j, \Prefix \in \Queue \} \big | \nn
\end{align}
and let~${J = \argmax_{\Prefix \in \Queue} L(\Prefix)}$
be the length of the longest prefix in~$\Queue$.
Define~${\Remaining(\CurrentObj, \Queue)}$ to be the number of remaining
prefix evaluations, then
\begin{align}
\Remaining(\CurrentObj, \Queue)
\le \sum_{j=1}^J \Queue_j \left( \sum_{k=0}^{K-j} \frac{(3^M-j)!}{(3^M-j - k)!} \right), \nn
\end{align}
where~${K = \min(\lfloor \CurrentObj / \Reg \rfloor, 2^M)}$.
\end{proposition}

\begin{proof}
The number of remaining prefix evaluations is equal to the number of
prefixes that are currently in or will be inserted into queue~$\Queue$.
For any such remaining prefix~$\Prefix$,
Theorem~\ref{thm:ub-prefix-length} gives an upper bound on its length;
define~$K$ to be this bound:
${L(\Prefix) \le \min(\lfloor \CurrentObj / \Reg \rfloor, 2^M) \equiv K}$.
For any prefix~$\Prefix$ in queue~$\Queue$ with length~${L(\Prefix) = j}$,
the maximum number of prefixes that start with~$\Prefix$
and remain to be evaluated is:
\begin{align}
\sum_{k=0}^{K-j} P(3^M-j, k) = \sum_{k=0}^{K-j} \frac{(3^M-j)!}{(3^M-j - k)!}, \nn
\end{align}
where~${P(T, k)}$ denotes the number of $k$-permutations of~$T$.
This gives an upper bound on the number of remaining prefix evaluations:
\begin{align}
\Remaining(\CurrentObj, \Queue)
\le \sum_{j=0}^J \Queue_j \left( \sum_{k=0}^{K-j} P(3^M-j, k) \right)
= \sum_{j=0}^J \Queue_j \left( \sum_{k=0}^{K-j} \frac{(3^M-j)!}{(3^M-j - k)!} \right). \nn
\end{align}
\end{proof}
\end{arxiv}

\section{Permutation Bound}
\label{sec:permutation}

If two trees are composed of the same leaves,
\ie they contain the same conjunctions of features up to a permutation,
then they classify all the data in the same way and their child trees are also permutations of each other.
Therefore, if we already have all children from one permutation of a tree, then there is no benefit to
considering child trees generated from a different permutation.

\begin{corollary}[Leaf Permutation bound]
\label{thm:permutation}
Let~$\pi$ be any permutation of ${\{1, \dots, H\}}$,
Let ${\RL = }$ ${(\Prefix, \Labels, \Default, \DefaultLabels, K, H)}$
and ${\RLB = }$ ${(\PrefixB, \LabelsB, \DefaultB, \DefaultLabelsB, K, H)}$
denote trees with leaves ${(p_1, \dots, p_H)}$
and ${\PrefixB = (p_{\pi(1)}, \dots, p_{\pi(H)})}$,
respectively, \ie the leaves in~$\RLB$
correspond to a permutation of the leaves in~$\RL$.
Then the objective lower bounds of~$\RL$ and~$\RLB$
are the same and their child trees correspond to permutations of each other.
\end{corollary}

\begin{arxiv}
\begin{proof}
This is intuitive because we can view a tree as a set of leaves and a leaf as a set of clauses, where the order of the leaves does not matter.
\end{proof}
\end{arxiv}

Therefore, if two trees have the same leaves,
up to a permutation, according to Corollary~\ref{thm:permutation}, either of them can be pruned.
We call this 
symmetry-aware pruning. In Section ~\S\ref{sec:permutation-counting}, we demonstrate how this helps to save  computation.

\begin{arxiv}
\subsection{Upper Bound on Tree Evaluations with Symmetry-aware Pruning}
\end{arxiv}
\begin{kdd}
\subsection{Upper bound on tree evaluations with symmetry-aware pruning}
\end{kdd}
\label{sec:permutation-counting}

Here we give an upper bound on the total number of tree
evaluations based on symmetry-aware
pruning~(\S\ref{sec:permutation}).
For every subset of~$K$ leaves, there are $K!$ leaf sets equivalent up to permutation. Thus, symmetry-aware pruning dramatically reduces the search space by considering only one of them.
This effects the execution of Algorithm~\ref{alg:branch-and-bound}'s breadth-first search. With symmetry-aware pruning, when it evaluates trees of size~$K$, for each set of trees equivalent up to a permutation, it keeps only a single tree.

\begin{theorem}[\fontdimen2\font=0.6ex Upper bound on tree evaluations with symmetry-aware pruning]
\label{thm:ub-symmetry}
Consider a state space of all trees formed from a set~$\RuleSet$
of~$3^M$ leaves where $M$ is the number of features (the 3 options correspond to having the feature's value be 1, having its value be 0, or not including the feature), and consider the branch-and-bound algorithm with
symmetry-aware pruning.
Define $\TotalRemaining(\RuleSet)$ to be the total number of prefixes evaluated.
For any set~$\RuleSet$ of $3^M$ leaves,
\begin{align}
\TotalRemaining(\RuleSet)
\le  1 + \sum_{k=1}^K N_k+C(M, k)-P(M, k), 
\label{eq:ub-symmetry}
\end{align}
where ${K = \min(\lfloor 1 / 2 \Reg \rfloor, 2^M)}$, $N_k$ is defined in \eqref{eq:number-tree-depth}.
\end{theorem}

\begin{proof}
By Corollary~\ref{cor:ub-prefix-length},
${K \equiv \min(\lfloor 1 / 2 \Reg \rfloor, 2^M)}$
gives an upper bound on the number of leaves of any optimal tree.
The algorithm begins by evaluating the empty tree,
followed by~$M$ trees of depth~${k=1}$,
then~${N_2 = \sum_{n_{0}=1}^{1} \sum_{n_{1}=1}^{2^{n_0}} M \times \binom{2^{n_0}}{n_1}(M-1)^{n_1} }$ trees of depth~${k=2}$.
Before proceeding to length~${k=3}$, we keep only~${N_2+}$ ${C(M, 2) - P(M, 2)}$
trees of depth~${k=2}$, where~$N_k$ is defined in \eqref{eq:number-tree-depth}, ${P(M, k)}$ denotes the
number of $k$-permutations of~$M$ and ${C(M, k)}$ denotes the
number of $k$-combinations of~$M$.
Now, the number of length~${k=3}$ prefixes we evaluate is~${N_3+C(M, 3)-P(M, 3)}$.
Propagating this forward gives \eqref{eq:ub-symmetry}.
\end{proof}

Pruning based on permutation symmetries thus yields significant
computational savings of $\sum_{k=1}^K P(M, k)-C(M, k)$.
For example, when $M=10$ and $K=5$, the number reduced due to symmetry-aware pruning is about $35463$. If $M=20$ and $K=10$, the number of evaluations is reduced by about $7.36891\times 10^{11}$.

\section{Similar Support Bound}
\label{sec:similar}

Here, we present the similar support bound to deal with similar trees. Let us say we are given two trees that are the same except that one internal node is split by a different feature, where this second feature is similar to the first tree's feature. By the similar support bound, if we know that one of these trees and all its child trees are worse (beyond a margin) than the current best tree, we can prune the other one and all of its child trees.  

\begin{theorem}[Similar support bound]
\label{thm:similar}
Define $\RL=(\Prefix, \Labels, \Default, \DefaultLabels, K, H)$ and $D=(\PrefixB, \LabelsB, \DefaultB, \DefaultLabelsB, K, H)$ to be two trees which are exactly the same but one internal node split by different features. Let $f_1, f_2$ be the features used to split that node in $\RL$ and $D$ respectively. Let $t_1,t_2$ be the left subtree and the right subtree under the node $f_1$ in $\RL$, and let $T_1,T_2$ be the left subtree and the right subtree under the node $f_2$ in $D$. Denote the normalized support
of data captured by only one of $t_1$ and $T_1$ as~$\omega$, \ie
\begin{align}
\omega \equiv \frac{1}{N} \sum_{n=1}^N
  [\neg\, \Cap(x_n, t_1)
  \wedge \Cap(x_n, T_1)
  +\Cap(x_n, t_1)
  \wedge \neg\, \Cap(x_n, T_1)]. 
\label{eq:omega}
\end{align}
The difference between the two trees' objectives is bounded by $\omega$ as the following:
\begin{align}
\omega \ge \Obj(\RL, \x, \y)
- \Obj(D, \x, \y) \ge -\omega,
\end{align}
where~$\Obj(\RL, \x, \y)$ is objective of~$\RL$ and $\Obj(D, \x, \y)$ is the
objective of~$D$. 
Then, we have  
\begin{align}
\omega \ge 
\min_{\RL^\dagger \in \StartsWith(\RL)} \Obj(\RL^\dagger, \x, \y) -
\min_{\RLB^\dagger \in \StartsWith(\PrefixB)} \Obj(\RLB^\dagger, \x, \y)
& \ge -\omega.
\label{eq:similar}
\end{align}
\end{theorem}

\begin{proof}
The difference between the objectives of $\RL$ and $\RLB$ is maximized when one of them correctly classifies all the data corresponding to $\omega$ but the other misclassifies all of them. Therefore,
\begin{align}
\omega \ge \Obj(\RL, \x, \y)
- \Obj(\RLB, \x, \y) \ge -\omega. \nn
\end{align}
Let $\RL^*$ be the best child tree of $\RL$, $\ie \Obj(\RL^*, \x, \y)=\min_{\RL^\dagger \in \StartsWith(\RL)} \Obj(\RL^\dagger, \x, \y)$, and let $\RLB'\in \StartsWith(\PrefixB)$ be its counterpart which is exactly the same but one internal node split by a different feature. Because $\Obj(\RLB', \x, \y) \ge \min_{\RLB^\dagger \in \StartsWith(\PrefixB)} \Obj(\RLB^\dagger, \x, \y)$, 
\begin{align}
    \min_{\RL^\dagger \in \StartsWith(\RL)} \Obj(\RL^\dagger, \x, \y) &= \Obj(\RL^*, \x, \y) \ge \Obj(\RLB', \x, \y)-\omega \nn \\ 
    &\ge \min_{\RLB^\dagger \in \StartsWith(\PrefixB)} \Obj(\RLB^\dagger, \x, \y)-\omega .
\end{align}
Similarly, $\min_{\RLB^\dagger \in \StartsWith(\PrefixB)} \Obj(\RLB^\dagger, \x, \y)+\omega \ge \min_{\RL^\dagger \in \StartsWith(\RL)} \Obj(\RL^\dagger, \x, \y)$.
\end{proof}

With the similar support bound, for two trees $\RL$ and $\RLB$ as above, if we already know $\RLB$ and all its child trees cannot achieve a better tree than the current optimal one, and in particular, we assume we know that $\min_{\RLB^\dagger \in \StartsWith(\PrefixB)} \Obj(\RLB^\dagger, \x, \y) \ge \CurrentObj+\omega$, we can then also prune $\RL$ and all of its child trees, because 
\begin{align}
\min_{\RL^\dagger \in \StartsWith(\RL)} \Obj(\RL^\dagger, \x, \y)
\ge \min_{\RLB^\dagger \in \StartsWith(\PrefixB)} \Obj(\RLB^\dagger, \x, \y)-\omega
\ge \CurrentObj.
\end{align}

\section{Implementation}
\label{sec:implementation-supp}

We implement a series of data structures designed to support incremental computation.

\subsection{Data Structure of Leaf and Tree}
\label{sec:tree}

First, we store bounds and intermediate results for both full trees and the individual leaves to support the incremental computation of the lower bound and the objective.
As a global statistic, we maintain the best (minimum) observed value of the objective function and the corresponding tree.
As each leaf in a tree represents a set of clauses, each leaf stores a bit-vector representing the set of samples captured by that clause and the prediction accuracy for those samples.
From these values in the leaves, we can efficiently compute both the value of the objective for an entire tree and new leaf values for children formed from splitting a leaf.

Specifically, for the data structure of leaf $l$, we store:
\begin{itemize}
\item A set of clauses defining the leaf.
\item A binary vector of length $N$ (number of data points) indicating whether or not each point is captured by the leaf.
\item The number of points captured by the leaf.
\item A binary vector of length $M$ (number of features) indicating the set of dead features. In a leaf, a feature is dead if Theorem~\ref{thm:min-capture-correct} does not hold.
\item The lower bound on the leaf misclassification error $b_0(l,\x,\y)$, which is defined in~\eqref{eq:lb-b0}
\item The label of the leaf.
\item The loss of the leaf.
\item A boolean indicating whether the leaf is dead. A leaf is dead if Theorem~\ref{thm:min-capture} does not hold;  we never split a dead leaf.
\end{itemize}

We store additional information for entire trees:
\begin{itemize}
\item A set of leaves in the tree.
\item The objective.
\item The lower bound of the objective.
\item A binary vector of length $n_l$ (number of leaves) indicating whether the leaf can be split, that is, this vector records split leaves $\Default$ and unchanged leaves $\Prefix$. The unchanged leaves of a tree are also unchanged leaves in its child trees.
\end{itemize}

\subsection{Queue}
\label{sec:queue}

Second, we use a priority queue to order the exploration of the search space. The queue serves as a worklist, with each entry in the queue corresponding to a tree.
When an entry is removed from the queue, we use it to generate child trees, incrementally computing the information for the child trees.
The ordering of the worklist represents a scheduling policy.
We evaluated both structural orderings, e.g., breadth first search and depth first search, and metric-based orderings, e.g., objective function, and its lower bound.
Each metric produces a different scheduling policy.
We achieve the best performance in runtime and memory consumption using the curiosity metric from CORELS~\cite{AngelinoEtAl18}, which is the objective's lower bound, divided by the normalized support of its unchanged leaves. For example, relative to using the objective, curiosity reduces runtime by a factor of two and memory consumption by a factor of four.


\subsection{Symmetry-aware Map}
\label{sec:pmap}

Third, to support symmetry-aware pruning from Corollary~\ref{thm:permutation}, we introduce two symmetry aware maps -- a LeafCache and a TreeCache. The LeafCache ensures that we only compute values for a particular leave once; the TreeCache ensures we do not create trees equivalent to those we have already explored.

A leaf is a set of clauses, each of which corresponds to an attribute and the
value (0,1) of that attribute.
As the leaves of a decision tree are mutually exclusive, the data captured by each leaf is
insensitive to the order of the leaf's clauses.
We encode leaves in a canonical order (sorted by attribute indexes) and use that canonical order as the key into the LeafCache.
Each entry in the LeafCache represents all permutations of a set of clauses.
We use a Python dictionary to map these keys to the leaf and its cached values.
Before creating a leaf object, we first check if we already have that leaf in our map.
If not,  we create the leaf and insert it into the map.
Otherwise, the permutation already exists, so we use the cached copy in the tree we are constructing.

Next, we implement the permutation bound (Corollary~\ref{thm:permutation}) using the TreeCache.
The TreeCache contains encodings of all the trees we have evaluated.
Like we did for the clauses in the LeafCache, we introduce a canonical order over the 
leaves in a tree and use that as the key to the TreeCache.
If our algorithm produces a tree that is a permutation of a tree we have already 
evaluated, we need not evaluate it again.
Before evaluating a tree, we look it up in the cache. 
If it's in the cache, we do nothing; if it is not, we compute the bounds for the tree and insert it into the cache.

\subsection{Execution}
\label{sec:exec}

Now, we illustrate how these data structures support execution of our algorithm.
We initialize the algorithm with the current best objective~$\CurrentObj$ and tree~$\CurrentRL$.
For unexplored trees in the queue, the scheduling policy selects the next tree~$\RL$
to split; we keep removing elements from the queue until the queue is empty.
Then, for every possible combination of features to split $\Default$,
we construct a new tree~$\RL'$ with incremental calculation of the lower bound~$b(\Prefix', \x, \y)$ and the objective~$\Obj(\RL', \x, \y)$.
If we achieve a better objective~$\Obj(\RL', \x, \y)$, \ie less than the current best objective~$\CurrentObj$,
we update~$\CurrentObj$ and~$\CurrentRL$.
If the lower bound of the new tree~$\RL'$, combined with the equivalent points bound (Theorem~\ref{thm:identical}) and the lookahead bound (Theorem~\ref{lemma:lookahead}), is less than the current best objective,
then we push it into the queue.
Otherwise, 
according to the hierarchical lower bound (Theorem~\ref{thm:bound}),
no child of~$\RL'$ could possibly have an objective
better than~$\CurrentObj$, which means we do not push~$\RL'$ queue.
When there are no more trees to explore, \ie the queue is empty, we have finished the search of the whole space and output the (provably) optimal tree.
\begin{arxiv}
\subsection{Custom Scheduling Policies}
\label{sec:scheduling}

For a tree with unchanged leaves $\Default$, we adapt the definition of curiosity from CORELS:
\begin{align}
\Curiosity(\Prefix, \x, \y)
&= \left( \frac{N}{\NCap} \right) \biggl(\Loss_p(\Prefix, \Labels, \x, \y) + \Reg K \biggr) \nn \\
&= \left( \frac{1}{N} \sum_{n=1}^N \Cap(x_n, \Prefix) \right)^{-1} b(\Prefix, \x, \y) \nn \\
&= \frac{b(\Prefix, \x, \y)}{\Supp(\Prefix, \x)},
\label{eq:curio}
\end{align}
where $\NCap$ is the number of observations captured by~$\Prefix$, \ie
\begin{align}
\NCap \equiv \sum_{n=1}^N \Cap(x_n, \Prefix).
\label{eq:num-cap}
\end{align}
That is, the curiosity of a tree~$\RL$ is actually its objective lower bound,
scaled by the inverse of its unchanged leaves' normalized support.
For two trees with the same lower bound, curiosity gives higher priority to
the one whose unchanged leaves capture more data.
Intuitively, we treat trees that extend
the tree with smaller support of unchanged leaves as having more `potential' to make mistakes,
which is a well-motivated scheduling strategy for the queue.
In practice, as we illustrate in the experiments~(\S\ref{sec:experiments}), curiosity significantly improves the performance of the priority queue.

\end{arxiv}

\section{Proof of Theorems}

\subsection{Proof of Theorem~\ref{thm:ub-prefix-length}}

\begin{proof}
For an optimal tree~$\OptimalRL$ with objective~$\OptimalObj$,
\begin{align}
\Reg L(\OptimalRL) \le \OptimalObj = \Obj(\OptimalRL, \x, \y)
= \Loss(\OptimalRL, \x, \y) + \Reg L(\OptimalRL)
\le \CurrentObj. \nn
\end{align}
The maximum possible number of leaves for~$\OptimalRL$ occurs
when $\Loss(\OptimalRL, \x,$ $ \y)$ is minimized;
therefore this
gives bound~\eqref{eq:max-length}.

For the rest of the proof,
let~${H^* = L(\OptimalRL)}$ be the number of leaves of~$\OptimalRL$.
If the current best tree~$\CurrentRL$ has zero
misclassification error, then
\begin{align}
\Reg H^* \leq \Loss(\OptimalRL, \x, \y) + \Reg H^* = \Obj(\OptimalRL, \x, \y)
\le \CurrentObj = \Obj(\CurrentRL, \x, \y) = \Reg H, \nn
\end{align}
and thus ${H^* \leq H}$.
If the current best tree is suboptimal,
\ie ${\CurrentRL \notin }$ ${\argmin_\RL \Obj(\RL, \x, \y)}$, then
\begin{align}
\Reg H^* \leq \Loss(\OptimalRL, \x, \y) + \Reg H^* = \Obj(\OptimalRL, \x, \y)
< \CurrentObj = \Obj(\CurrentRL, \x, \y) = \Reg H, \nn
\end{align}
in which case ${H^* < H}$, \ie ${H^* \leq H-1}$, since $H$ is an integer.
\end{proof}



\subsection{Proof of Theorem~\ref{thm:ub-prefix-specific}}

\begin{proof}
First, note that~${H' \ge H}$.
Now recall that
\begin{align}
b(\Prefix, \x, \y) &= \Loss_p(\Prefix, \Labels, \x, \y) + \Reg H \nn \\
&\le \Loss_p(\Prefix', \Labels', \x, \y) + \Reg H' = b(\Prefix', \x, \y), \nn
\end{align}
and that
${\Loss_p(\Prefix, \Labels, \x, \y) \le \Loss_p(\Prefix', \Labels', \x, \y)}$.
Combining these bounds and rearranging gives
\begin{align}
b(\Prefix', \x, \y) &= \Loss_p(\Prefix', \Labels', \x, \y) + \Reg H + \Reg(H' - H) \nn \\
&\ge \Loss_p(\Prefix, \Labels, \x, \y) + \Reg H + \Reg(H' - H) \nn \\
&= b(\Prefix, \x, \y) + \Reg (H' - H).
\label{eq:length-diff}
\end{align}
Combining~\eqref{eq:length-diff} with~${b(\Prefix', \x, \y) < \CurrentObj}$ gives~\eqref{eq:max-length-prefix}.
\end{proof}

\subsection{Proof of Theorem~\ref{thm:remaining-eval-fine}}

\begin{proof}
The number of remaining tree evaluations is equal to the number of
trees that are currently in or will be inserted into queue~$\Queue$.
For any such tree with unchanged leaves~$\Prefix$, Theorem~\ref{thm:ub-prefix-specific}
gives an upper bound on the number of leaves of a tree with unchanged leaves~$\Prefix'$
that contains~$\Prefix$:
\begin{align}
L(\Prefix') \le \min \left( L(\Prefix) + \left\lfloor \frac{\CurrentObj - b(\Prefix, \x, \y)}{\Reg} \right\rfloor, 2^M \right)
\equiv U(\Prefix). \nn
\end{align}
This gives an upper bound on the remaining tree evaluations:
\begin{align}
\Remaining(\CurrentObj, \Queue)
\le \sum_{\Prefix \in Q} \sum_{k=0}^{U(\Prefix) - L(\Prefix)} P(3^M - L(\Prefix), k) \\
= \sum_{\Prefix \in Q} \sum_{k=0}^{f(\Prefix)} \frac{(3^M - L(\Prefix))!}{(3^M - L(\Prefix) - k)!}\,, \nn
\end{align}
where~$P(m, k)$ denotes the number of $k$-permutations of~$m$.
\end{proof}

\subsection{Proof of Proposition~\ref{thm:ub-total-eval}}

\begin{proof}
By Corollary~\ref{cor:ub-prefix-length},
${K \equiv \min(\lfloor 1 / 2 \Reg \rfloor, 2^M)}$
gives an upper bound on the number of leaves of any optimal tree.
Since we can think of our problem as finding the optimal
selection and permutation of~$k$ out of~$3^M$ leaves,
over all~${k \le K}$,
\begin{align}
\TotalRemaining(\RuleSet) \le 1 + \sum_{k=1}^K P(3^M, k)
= \sum_{k=0}^K \frac{3^M!}{(3^M - k)!}. \nn
\end{align}
\end{proof}

\subsection{Proof of Theorem~\ref{thm:min-capture}}

\begin{proof}
Let ${\OptimalRL = (\Prefix, \Labels, \Default, \DefaultLabels, K, H)}$ be an optimal
tree with leaves ${(p_1, \dots, p_H)}$
and labels ${(\q_1, \dots, \q_H)}$.
Consider the tree ${\RL = (\Prefix', \Labels', \Default', \DefaultLabels', K', H')}$
derived from~$\OptimalRL$ by deleting a pair of sibling leaves ${p_i \rightarrow \q_i, p_{i+1} \rightarrow \q_{i+1}}$ and adding their parent leaf $p_{j}\rightarrow \q_{j}$,
therefore ${\Prefix' = (p_1, \dots, p_{i-1}, }$ ${p_{i+2}, \dots, p_H, p_{j})}$
and ${\Labels' = (\q_1, \dots, \q_{i-1},\q_{i+2}, \dots, \q_H, }$ ${\q_{j})}$.

When~$\RL$ misclassifies half of the data captured by~$p_i, p_{i+1}$, while $\OptimalRL$ correctly classifies them all, the difference between~$\RL$ and~$\OptimalRL$ would be maximized,
which provides an upper bound:
\begin{align}
\Obj(\RL, \x, \y) &= \Loss(\RL, \x, \y) + \Reg (H - 1) \nn \\
&\le \Loss(\OptimalRL, \x, \y) + \Supp(p_i, \x) + \Supp(p_{i+1}, \x) \nn \\ &~~~ -\frac{1}{2}[\Supp(p_i, \x) + \Supp(p_{i+1}, \x)] + \Reg(H - 1) \nn \\
&= \Obj(\OptimalRL, \x, \y) + \frac{1}{2}[\Supp(p_i, \x) + \Supp(p_{i+1}, \x)] - \Reg \nn \\
&= \OptimalObj + \frac{1}{2}[\Supp(p_i, \x) + \Supp(p_{i+1}, \x)] - \Reg
\label{eq:ub-i}
\end{align}
where~$\Supp(p_i, \x), \Supp(p_i, \x)$ is the normalized support of~$p_i, p_{i+1}$, defined in~\eqref{eq:support},
and the regularization `bonus' comes from the fact that~$\OptimalRL$ has one more leaf than~$\RL$.

Because $\OptimalRL$ is the optimal tree, we have ${\OptimalObj \le \Obj(\RL, \x, \y)}$, which combined with~\eqref{eq:ub-i} leads to~\eqref{eq:min-capture}. Therefore, for each child leaf pair $p_k, p_{k+1}$ of a split,
the sum of normalized supports of~$p_k, p_{k+1}$ should be no less than twice the regularization parameter, \ie $2\lambda$.
\end{proof}

\subsection{Proof of Theorem~\ref{thm:incre-min-capture-correct}}

\begin{proof}
Let ${\RL = (\Prefix', \Labels', \Default', \DefaultLabels', K',}$ ${ H')}$
be the tree derived from~$\OptimalRL$ by deleting a pair of leaves, ${p_i}$ and $p_{i+1}$, and adding the their parent leaf, $p_{j}$.
The discrepancy between~$\OptimalRL$ and~$\RL$ is the discrepancy between $(p_i, p_{i+1})$ and $p_j$:
$\Loss(\RL, \x, \y) - \Loss(\OptimalRL, \x, \y) = a_i,$ \nn
where $a_i$ is defined in~\eqref{eq:rule-correct}.
Therefore,
\begin{eqnarray}
\Obj(\RL, \x, \y) &=& \Loss(\RL, \x, \y) + \Reg (K - 1)
= \Loss(\OptimalRL, \x, \y) + a_i + \Reg(K - 1) \nn \\
\lefteqn{= \Obj(\OptimalRL, \x, \y) + a_i - \Reg \nn 
= \OptimalObj + a_i - \Reg.
\label{eq:ub-ii}}
\end{eqnarray}
This combined with
${\OptimalObj \le \Obj(\RL, \x, \y)}$ leads to~${\Reg \le a_i}$.
\end{proof}

\subsection{Proof of Theorem~\ref{thm:min-capture-correct}}

\begin{proof}
Let ${\RL = (\Prefix', \Labels', \Default', \DefaultLabels', K', }$ ${H')}$
be the tree derived from~$\OptimalRL$ by deleting a pair of leaves, ${p_i}$ with label $\q_i$ and $p_{i+1}$ with label $\q_{i+1}$, and adding the their parent leaf $p_{j}$ with label $\q_{j}$.
The discrepancy between~$\OptimalRL$ and~$\RL$ is the discrepancy between $p_i, p_{i+1}$ and $p_j$:
$\Loss(\RL, \x, \y) - \Loss(\OptimalRL, \x, \y) = a_i, \nn$
where we defined~$a_i$ in~\eqref{eq:rule-correct}.
According to Theorem~\ref{thm:incre-min-capture-correct}, $\Reg \le  a_i$ and
\begin{align}
\Reg \le & \frac{1}{N} \sum_{n=1}^N \{
  \Cap(x_n, p_i) \wedge \one [ \q_i = y_n ] \nn \\
  & + \Cap(x_n, p_{i+1}) \wedge \one [ \q_{i+1} = y_n ]  \nn \\
  & - \Cap(x_n, p_j) \wedge \one [ \q_j = y_n ] \}.
\end{align}
For any leaf $j$ and its two child leaves $i,i+1$, we always have
\begin{align*}
 \sum_{n=1}^N \Cap(x_n, p_i) \wedge \one [ \q_i = y_n ] \le 
   \sum_{n=1}^N \Cap(x_n, p_j) \wedge \one [ \q_j = y_n ], \nn \\
  \sum_{n=1}^N \Cap(x_n, p_{i+1}) \wedge \one [ \q_{i+1} = y_n ] \le 
  \sum_{n=1}^N \Cap(x_n, p_j) \wedge \one [ \q_j = y_n ]
\end{align*}
which indicates that $a_i \le \frac{1}{N}\sum_{n=1}^N \Cap(x_n,p_i)\wedge \one[\q_i = y_n]$ and $a_i \le \frac{1}{N}\sum_{n=1}^N \Cap(x_n,p_{i+1})\wedge \one[\q_{i+1} = y_n]$. Therefore,
\begin{align*}
  &\Reg \le \frac{1}{N}\sum_{n=1}^N \Cap(x_n,p_i)\wedge \one[\q_i = y_n], \nn \\
  &\Reg \le \frac{1}{N}\sum_{n=1}^N \Cap(x_n,p_{i+1})\wedge \one[\q_{i+1} = y_n].
\end{align*}
\end{proof}


%
%
%
%

\subsection{Proof of Proposition~\ref{prop:identical}}

\begin{proof}
Recall that the objective is ${\Obj(\RL, \x, \y) = \Loss(\RL, \x, \y) + \Reg H}$,
where the misclassification error~${\Loss(\RL, \x, \y)}$ is given by
\begin{align}
\Loss(\RL, \x, \y)
&= \frac{1}{N} \sum_{n=1}^N \sum_{k=1}^K \Cap(x_n, p_k) \wedge \one [\q_k \neq y_n]. \nn
\end{align}
Any particular tree uses a specific leaf, and therefore a single class label,
to classify all points within a set of equivalent points.
Thus, for a set of equivalent points~$u$, the tree~$\RL$ correctly classifies either
points that have the majority class label, or points that have the minority class label.
It follows that~$\RL$ misclassifies a number of points in~$u$ at least as great as
the number of points with the minority class label.
To translate this into a lower bound on~$\Loss(\RL, \x, \y)$,
we first sum over all sets of equivalent points, and then for each such set,
count differences between class labels and the minority class label of the set,
instead of counting mistakes:
\begin{align}
& \Loss(\RL, \x, \y) \nn \\
&= \frac{1}{N} \sum_{u=1}^U \sum_{n=1}^N \sum_{k=1}^K \Cap(x_n, p_k) \wedge \one [\q_k \neq y_n] \wedge
   \one [x_n \in e_u]  \nn \\
&\ge \frac{1}{N} \sum_{u=1}^U \sum_{n=1}^N \sum_{k=1}^K \Cap(x_n, p_k) \wedge \one [q_u = y_n] \wedge
   \one [x_n \in e_u].  \nn 
\end{align}
Next, because every datum must be captured by a leaf in the tree~$\RL$, $\sum_{k=1}^K \Cap(x_n, p_k)=1$.
\begin{align}
\Loss(\RL, \x, \y) &\ge \frac{1}{N} \sum_{u=1}^U \sum_{n=1}^N \one [ x_n \in e_u ]\, \one [ y_n = q_u ]
= \sum_{u=1}^U \theta(e_u), \nn
\end{align}
where the final equality applies the definition of~$\theta(e_u)$ in~\eqref{eq:theta}.
Therefore, ${\Obj(\RL, \x, \y) =}$ ${\Loss(\RL, \x, \y) + \Reg K}$ ${\ge \sum_{u=1}^U \theta(e_u) + \Reg K}$.
\end{proof}

\section{Ablation Experiments}\label{appendix:ablation}
We evaluate how much each of our bounds contributes to OSDT's performance and what effect the scheduling metric has on execution.
Table~\ref{tab:ablation} provides experimental statistics of total execution time, time to optimum, total number of trees evaluated,
number of trees evaluated to optimum, and memory consumption on the recidivism data set.
The first row is the full OSDT implementation, and the others are variants, each of which removes a specific bound.
While all the optimizations reduce the search space, the lookahead and equivalent points bounds are, by far, the most significant, reducing time to optimum by at least two orders of magnitude and reducing memory consumption by more than one order of magnitude.
In our experiment, although the scheduling policy has a smaller effect, it is still significant -- curiosity is a factor of two faster than the objective function and consumes 25\% of the memory consumed when using the objective function for scheduling. All other scheduling policies, \ie the lower bound and the entropy, are significantly worse.

\begin{table}[t!]
\centering
Per-bound performance improvement (ProPublica data set) \\
\vspace{1mm}
\begin{tabular}{l | c | c | c }
& Total time & Slow- & Time to \\
Algorithm variant & (s) & down & optimum (s) \\
\hline
All bounds & 14.70 & --- & 1.01 \\
No support bound & 17.11 & 1.16$\times$ & 1.09 \\
No incremental accuracy bound & 30.16 & 2.05$\times$ & 1.13 \\
No accuracy bound & 31.83 & 2.17$\times$ & 1.23 \\
No lookahead bound & 31721 & 2157.89$\times$ & 187.18 \\
No equivalent points bound & $>$12475 & $>$848$\times$ & --- \\
\hline
\end{tabular}
\begin{tabular}{l | c | c | c}
\hline
 & Total \#trees & \#trees &  Mem \\
Algorithm variant & evaluated & to optimum &~ (GB) \\
\hline
All bounds & 232402 & 16001 & .08 \\
No support bound & 279763 & 18880 & .08 \\
No incremental accuracy bound & 546402 & 21686 & .08 \\
No accuracy bound & 475691 & 19676 & .09 \\
No lookahead bound & 284651888 & 3078274 & 10 \\
No equivalent points bound & $>$77000000 & --- & $>$64 \\
\end{tabular}
\vspace{2mm}
\caption{Per-bound performance improvement, for the ProPublica data set
(${\Reg = 0.005}$, cold start, using curiosity).
The columns report the total execution time,
time to optimum, total number of trees evaluated,
number of trees evaluated to optimum,
and memory consumption.
The first row shows our algorithm with all bounds; subsequent rows show variants
that each remove a specific bound 
(one bound at a time, not cumulative).
%
All rows except the last one represent a complete execution, \ie until the queue is empty.
For the last row (`No equivalent points bound'), the algorithm was terminated after running out of the memory (about $\sim$64GB RAM). \vspace*{-10pt}
%
%
}
\label{tab:ablation}
\end{table}

\section{Regularized BinOCT}
\label{sec:rbinoct}

Since BinOCT always produces complete binary trees of given depth, we add a regularization term to the objective function of BinOCT. In this way, regularized BinOCT (RBinOCT) can generate the same trees as OSDT. Following the notation of \cite{verwer2019learning}, we provide the formulation of regularized BinOCT:

\begin{align}
\text{min} \sum_{l,c} e_{l,c}+\Reg \sum_{l}\alpha_l \textrm{ s.t. }
\label{eq:binoct-obj}
\end{align}
\vskip -5mm
\begin{align}
&\forall n\;\; \sum_{f} f_{n,f}=1 \nn \\ 
&\forall r\;\; \sum_{l} l_{r,l}=1 \nn \\
&\forall l\;\; \sum_{c} p_{l,c}=1 \nn \\
&\forall {n,f,b\in bin(f)}\;\;\; M\cdot f_{n,f}+\sum_{r\in lr(b)}\sum_{l\in ll(n)}l_{r,l}+\sum_{t\in tl(b)}M\cdot t_{n,t}-\sum_{t\in tl(b)}M \le M \nn \\
&\forall {n,f,b\in bin(f)}\;\;\; M'\cdot f_{n,f}+\sum_{r\in rr(b)}\sum_{l\in rl(n)}l_{r,l}-\sum_{t\in tl(b)}M'\cdot t_{n,t} \le M' \nn \\
&\forall  {n,f}\;\;\; M''\cdot f_{n,f}+\sum_{\text{max}_{t(f)}<f(r)}\sum_{l\in ll(n)}l_{r,l}+\sum_{f(r)<\text{min}_{t(f)}}\sum_{l\in rl(n)}l_{r,l}\le M" \nn \\
&\forall\; {l,c}\;\; \sum_{r:C_r=c}l_{r,l}-M'''\cdot p_l \le e_{l,c} \nn \\
&\forall \; {l} \sum_{r}l_{r,l}\le R\cdot \alpha_l,
\label{eq:binoct-newcons}
\end{align}
where $1\leq n \leq N$, $1\leq f \leq F$, $1\leq r\leq R$, $1\leq l \leq L$, $1\leq c\leq C$. Variables $\alpha_l$, $f_{n,f}$, $t_{n,t}$, $p_{l,c}$ are binary, and $e_{l,c}$ and $l_{r,l}$ are continuous (see Table \ref{tab:rbinoct-notation}).
Compared with BinOCT, we add a penalization term $\Reg \sum_{l}\alpha_l$ to the objective function \eqref{eq:binoct-obj} and a new constraint \eqref{eq:binoct-newcons},
where $\Reg$ is the same as that of OSDT, and $\alpha_l=1$ if leaf $l$ is not empty and $\alpha_l=0$ if leaf $l$ is empty. All the rest of the constraints are the same as those of BinOCT. We use the same notation as in the original BinOCT formulation \cite{verwer2019learning}.

\begin{table}[]
\centering
\vspace{1mm}
\begin{tabular}{l | l | l }
\hline
Notation & Type & Definition \\
\hline
$n$ & index & internal (non-leaf) node in the tree, $1\le n\le N$\\
$l$ & index & leaf of the tree, $1\le l\le L$\\
$r$ & index & row in the training data, $1\le r \le R$\\
$f$ & index & feature in the training data, $1\le f\le F$\\
$c$ & index & class in the training data, $1\le c \le C$\\
\hline
$bin(f)$ & set & feature $f$'s binary encoding ranges\\
$lr(b)$ & set & rows with values in $b$'s lower range, $b\in bin(f)$\\
$ur(b)$ & set & rows with values in $b$'s upper range\\
$tl(b)$ & set & $t_{n,t}$ variables for $b$'s range\\
$ll(n)$ & set & node $n$'s leaves under the left branch\\
$rl(n)$ & set & node $n$'s leaves under the right branch\\
\hline
$K$ & constant & the tree's depth\\
$N=2^K-1$ & constant & the number of internal nodes (not leaves)\\
$L=2^K$ & constant & the number of leaf nodes\\
$F$ & constant & the number of features\\
$C$ & constant & the number of classes\\
$R$ & constant & the number of training data rows\\
$T$ & constant & the total number of threshold values\\
$T_f$ & constant & the number of threshold values for feature $f$\\
$T_{max}$ & constant & maximum of $T_f$ over all features $f$\\
$V^f_{r}$ & constant & feature $f$'s value in training data row $r$\\
$C_r$ & constant & class value in training data row $r$\\
$\min_t(f)$ & constant & feature $f$'s minimum threshold value\\
$\max_t(f)$ & constant & feature $f$'s maximum threshold value\\
$M$ & constant & minimized big-M value\\
\hline
$f_{n,f}$ & binary & node $n$'s selected feature $f$\\
$t_{n,t}$ & binary & node $n$'s selected threshold $t$\\
$p_{l,c}$ & binary & leaf $l$'s selected prediction class $c$\\
$\alpha_{l}$ & binary & $\alpha_{l}=1$ if leaf $l$ is not empty\\
\hline
$e_{l,c}$ & continuous & error for rows with class $c$ in leaf $l$\\
$l_{r,l}$ & continuous & row $r$ reaches leaf $l$\\
$\lambda$ & continuous & the regularization parameter\\
\hline
\end{tabular}
\caption{Summary of the notations used in RBinOCT.}
\vspace{5mm}
\label{tab:rbinoct-notation}
\end{table}

Figure \ref{fig:tree-regularized} shows the trees generated by regularized BinOCT and OSDT when using the same regularization parameter $\Reg=0.007$. Although the two algorithms produce the same optimal trees, regularized BinOCT is much slower than OSDT. In our experiment, it took only 3.390 seconds to run the OSDT algorithm to optimality, while the regularized BinOCT algorithm had not finished running after 1 hour.

\begin{figure}[t!]
    \centering
    \begin{subfigure}[b]{0.3\textwidth}
        \centering
        \includegraphics[trim={0mm 3mm 0mm 3mm}, width=\textwidth]{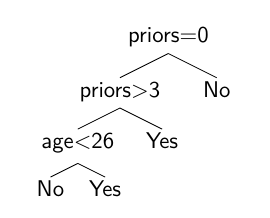}
        \caption[]%
        {{\footnotesize RBinOCT \\ (accuracy: 66.223\%)}}
    \end{subfigure}
    \begin{subfigure}[b]{0.3\textwidth} 
        \centering 
        \includegraphics[trim={0mm 3mm 0mm 3mm}, width=\textwidth]{figs/tree-compas-rbinoct.pdf}
        \caption[]%
        {{\footnotesize OSDT \\ (accuracy: 66.223\%)}}
    \end{subfigure}
\caption{The decision trees generated by Regularized BinOCT and OSDT on the Monk1 dataset. The two trees are exactly the same, but regularized BinOCT is much slower in producing this tree.
}
\label{fig:tree-regularized}
\end{figure}

In Figure \ref{fig:compare_traces}, we provide execution traces of OSDT and RBinOCT. OSDT converges much faster than RBinOCT. For some datasets, \ie FICO and Monk1, the total execution time of RBinOCT was several times longer than that of OSDT.

\begin{figure*}[t!]
    \centering
    \begin{subfigure}[b]{0.48\textwidth}
        \centering
        \includegraphics[trim={0mm 10mm 0mm 12mm}, width=\textwidth]{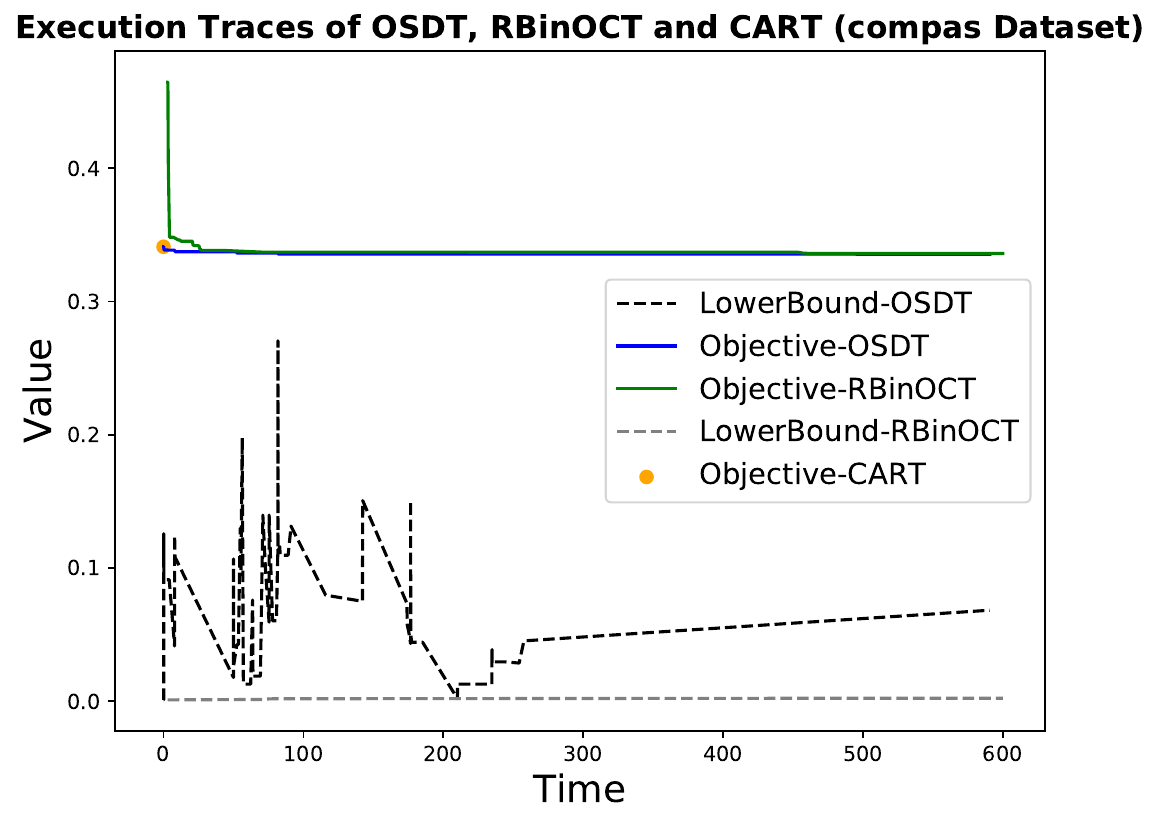}
        \label{fig:COMPAS-trace}
    \end{subfigure}
    \begin{subfigure}[b]{0.47\textwidth}  
        \centering 
        \includegraphics[trim={0mm 10mm 0mm 12mm}, width=\textwidth]{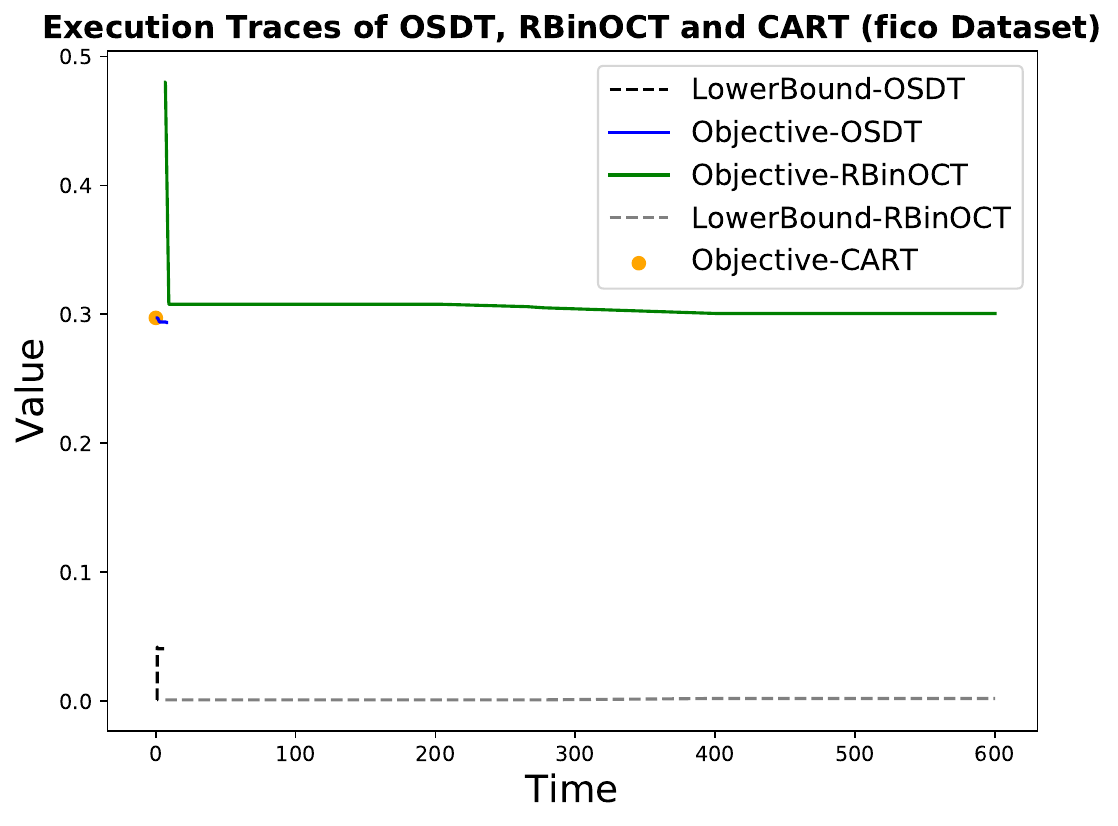}
        \label{fig:FICO-trace}
    \end{subfigure}
    \vskip\baselineskip
    \begin{subfigure}[b]{0.48\textwidth}   
        \centering 
        \includegraphics[trim={0mm 10mm 0mm 12mm}, width=\textwidth]{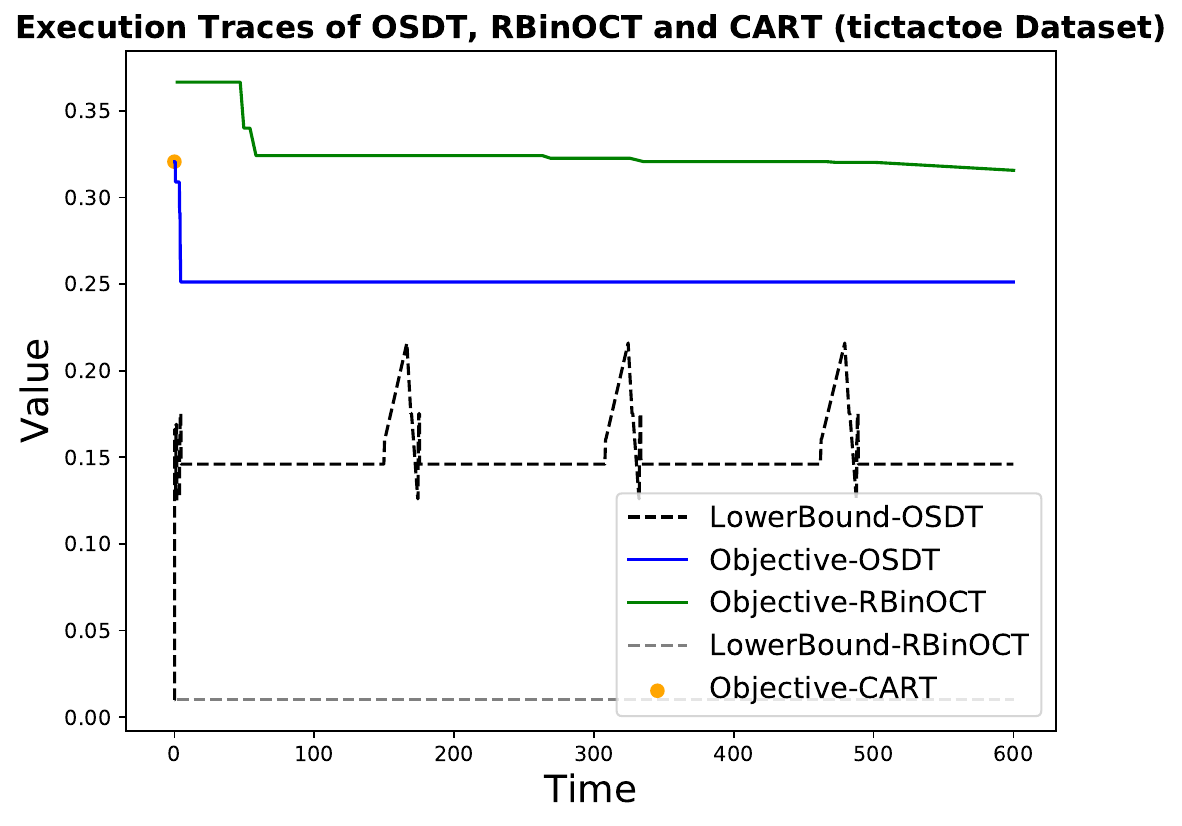}
        \label{fig:tictactoe-trace}
    \end{subfigure}
    \begin{subfigure}[b]{0.465\textwidth}   
        \centering 
         \includegraphics[trim={0mm 10mm 0mm 12mm}, width=\textwidth]{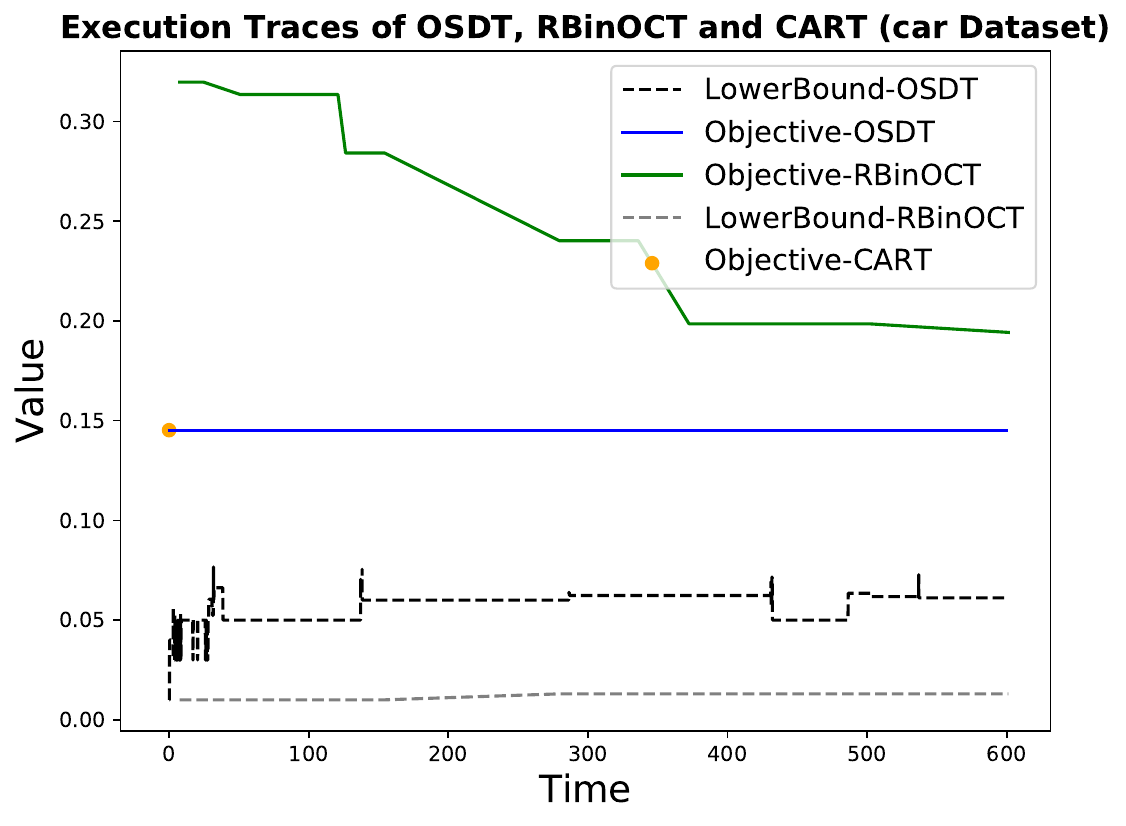}
        \label{fig:Car-trace}
    \end{subfigure}
    \vskip\baselineskip
    \begin{subfigure}[b]{0.48\textwidth}   
        \centering 
        \includegraphics[trim={0mm 10mm 0mm 12mm}, width=\textwidth]{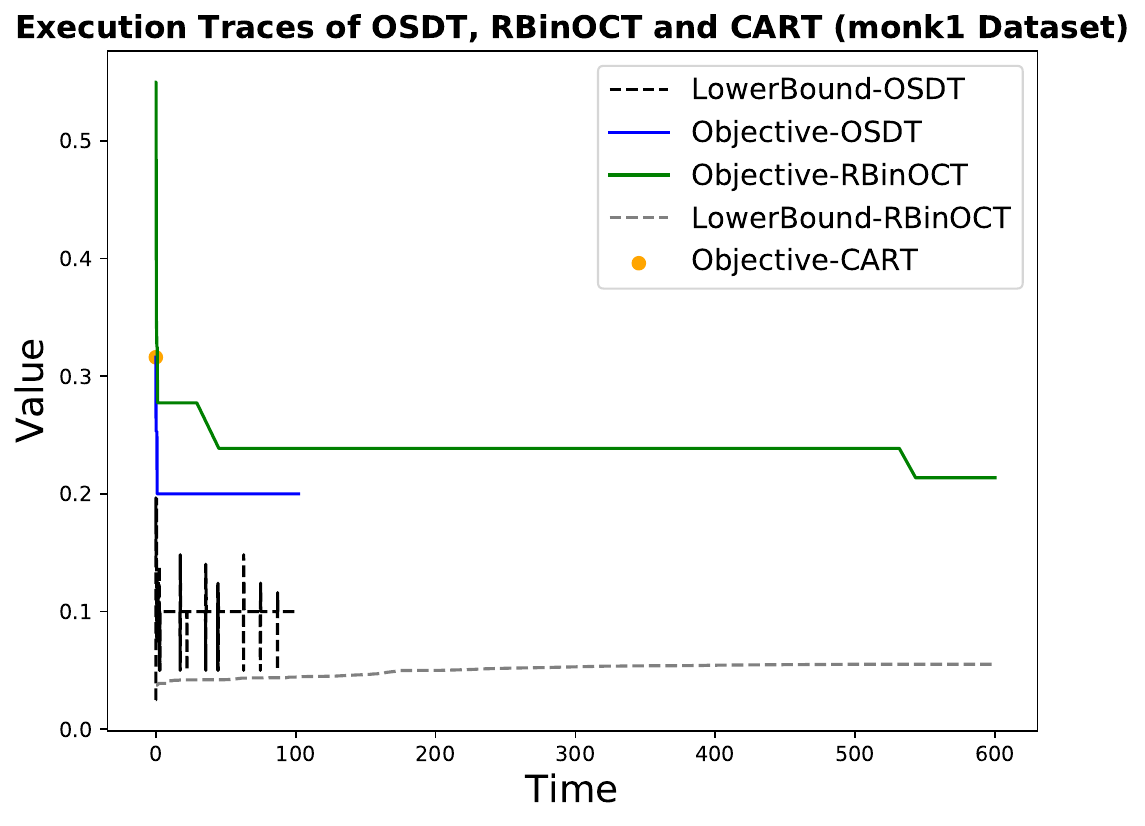}
        \label{fig:Monk1-trace}
    \end{subfigure}
    \begin{subfigure}[b]{0.48\textwidth}   
        \centering 
         \includegraphics[trim={0mm 10mm 0mm 12mm}, width=\textwidth]{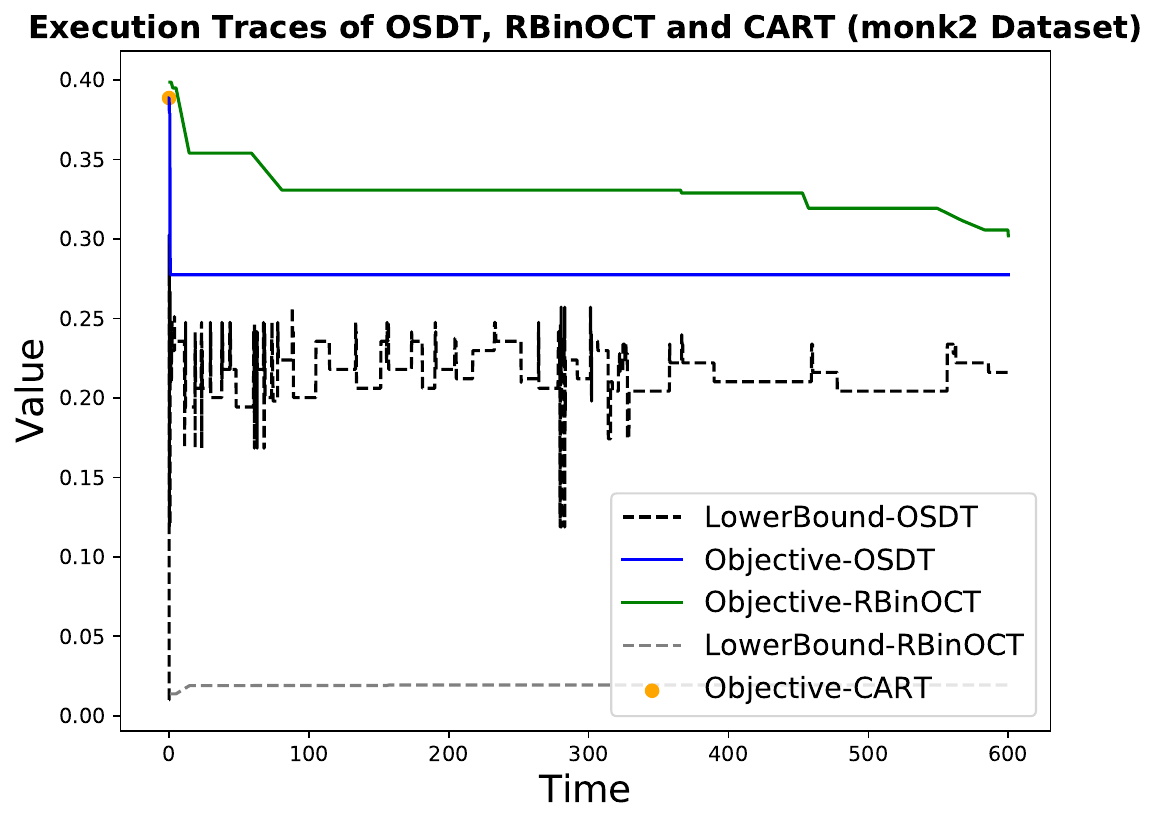}
        \label{fig:Monk2-trace}
    \end{subfigure}
    \vskip\baselineskip
    \begin{subfigure}[b]{0.48\textwidth}   
        \centering 
        \includegraphics[trim={0mm 10mm 0mm 12mm}, width=\textwidth]{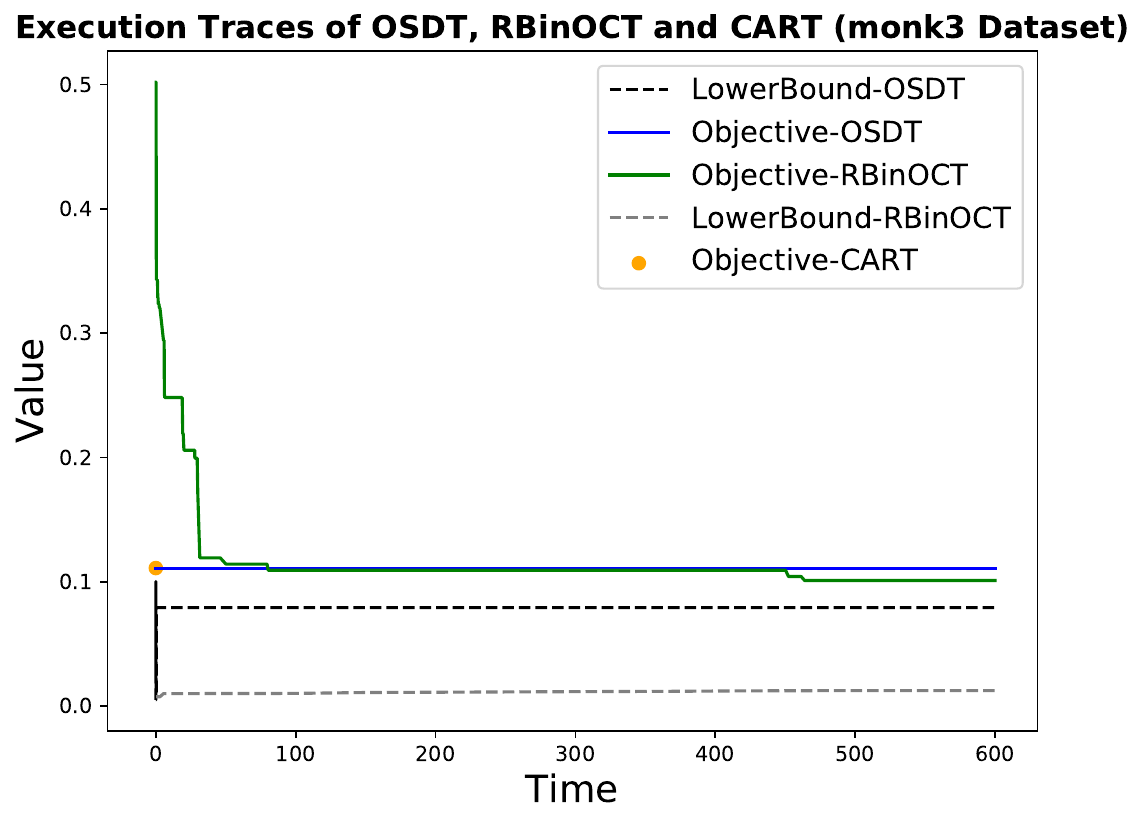}
        \label{fig:Monk3-trace}
    \end{subfigure}
    \vskip -3mm
    \caption{Execution traces of OSDT and regularized BinOCT. {OSDT converges much more quickly than RBinOCT.} %
    } 
    \label{fig:compare_traces}
\end{figure*}

\begin{arxiv}

\begin{table}[t!]
\centering
Performance of different priority metrics (ProPublica data set) \\
\vspace{1mm}
\begin{tabular}{l | c | c | c }
& Total time & Slow- & Time to \\
Priority metric & (s) & down & optimum (s) \\
\hline
Curiosity & 96 & --- & 25 \\
Objective & 130 & 1.35$\times$ & 57 \\
Lower bound & 186 & 1.93$\times$ & 175 \\
Entropy & 671 & 6.98$\times$ & 582 \\
Gini & 646 & 6.72$\times$ & 551 \\
\hline
\end{tabular}

\begin{tabular}{l | c | c | c}
\hline
 & Total \#trees & \#trees &  Memory Con-\\
Priority metric & evaluated & to optimum &~ sumption (GB) \\
\hline
Curiosity & 241306 & 18195 & .08 \\
Objective & 635449 & 423627 & .30 \\
Lower bound & 1714011 & 1713709 & .78 \\
Entropy & 5381467 & 5332413 & 4.90 \\
Gini & 5381001 & 5331925 & 4.90 \\
\end{tabular}
\vspace{4mm}
\caption{Performance of different priority policies, for the ProPublica data set
(${\Reg = 0.005}$, cold start).
The columns report the total execution time, factor slower,
time to optimum, total number of trees evaluated,
number of trees evaluated to optimum,
and memory consumption.
The first row shows our algorithm with curiosity (\ref{eq:curio}) as the priority queue metric; subsequent rows show variants
that use other policies.
All rows represent complete executions that certify optimality.
%
%
}
\vspace{4mm}
\label{tab:scheduling}
\end{table}
\end{arxiv}

\section{CART}\label{appendix:cart}

We provide the trees generated by CART on COMPAS, Tic-Tac-Toe, and Monk1 datasets. With the same number of leaves as their OSDT counterparts (Figure~\ref{fig:tree-compas-5}, Figure~\ref{fig:tree-monk1}, Figure~\ref{fig:tree-tictactoe}), these CART trees perform much worse than the OSDT ones.

\begin{figure}[t!]
\centering
    \centering
    \begin{subfigure}[b]{0.4\textwidth} 
        \centering 
        \includegraphics[trim={0mm 20mm 30mm 0mm}, width=\textwidth]{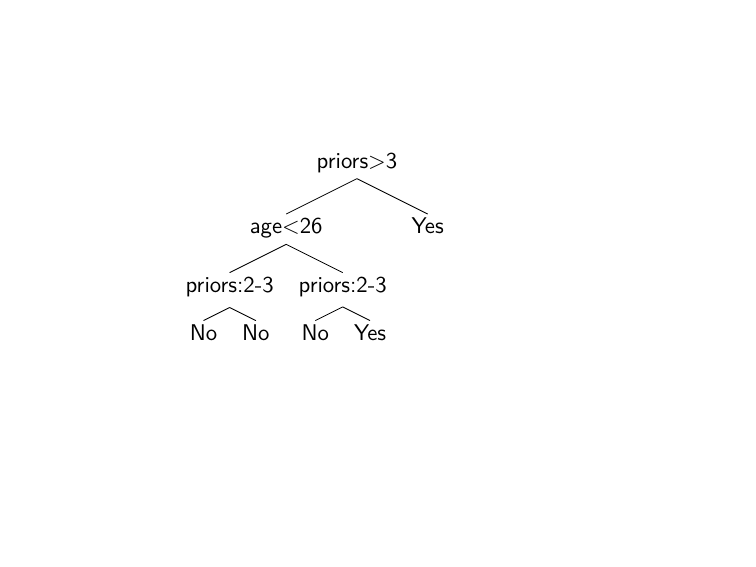}
        \caption[]%
        {{\footnotesize COMPAS \\ (accuracy: 66.135\%)}}
        \label{fig:tree-compas-cart}
    \end{subfigure}
    \begin{subfigure}[b]{0.4\textwidth}
        \centering
        \includegraphics[trim={6mm 15mm 10mm 0mm}, width=\textwidth]{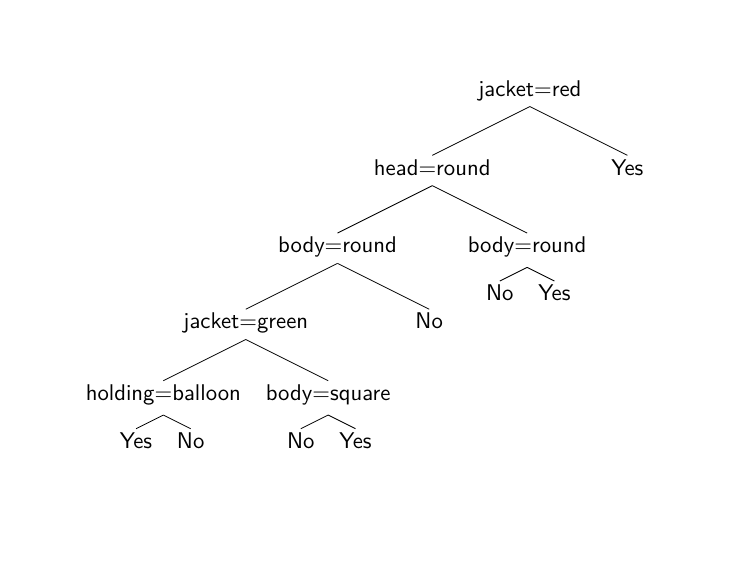}
        \caption[]%
        {{\footnotesize MONK1 \\ (accuracy: 91.935\%)}}
        \label{fig:tree-monk1-cart}
    \end{subfigure}
    \begin{subfigure}[b]{0.35\textwidth} 
        \centering 
        \includegraphics[trim={0mm 15mm 30mm 0mm}, width=\textwidth]{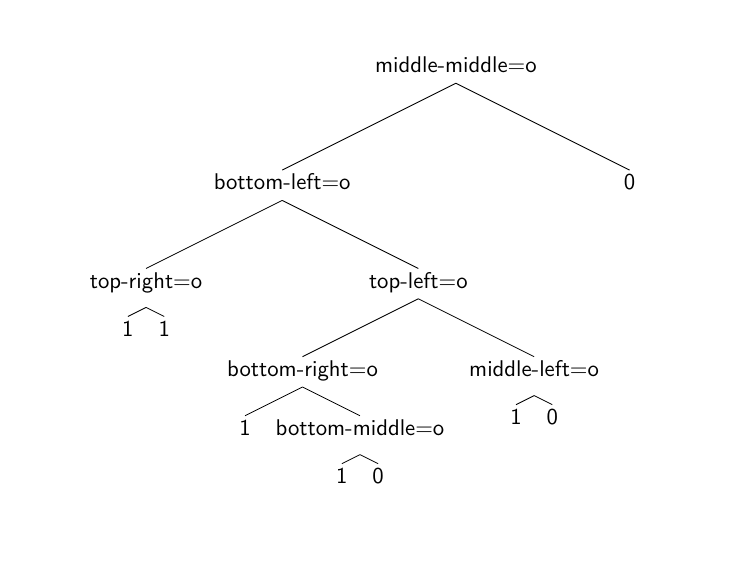}
        \caption[]%
        {{\footnotesize Tic-Tac-Toe \\ (accuracy: 76.513\%)}}
        \label{fig:tree-tictactoe-cart}
    \end{subfigure}
\caption{Decision trees generated by CART on COMPAS, Tic-Tac-Toe, and Monk1 datasets. They are all inferior to these trees produced by OSDT as shown in Section \ref{sec:experiments}.}
\label{fig:cart}
\vskip -5mm
\end{figure}

\section{Cross-validation Experiments}\label{appendix:cross-valid}
The difference between training and test error is probabilistic and depends on the number of observations in both training and test sets, as well as the complexity of the model class. The best learning-theoretic bound on test error occurs when the training error is minimized for each model class, which in our case, is the maximum number of leaves in the tree (the level of sparsity). By adjusting the regularization parameter throughout its full range, OSDT will find the most accurate tree for each given sparsity level. Figures \ref{fig:cv-compas}-\ref{fig:cv-monk3} show the training and test results for each of the datasets and for each fold. As indicated by theory, higher training accuacy for the same level of sparsity tends to yield higher test accuracy in general, but not always. There are some cases, like the car dataset, where OSDT's almost-uniformly-higher training accuracy leads to higher test accuracy, and other cases where all methods perform the same. In the case where all methods perform the same, OSDT provides a certificate of optimality showing that no better training performance is possible for the same level of sparsity.

\begin{figure*}[t!]
    \centering
    \begin{subfigure}[b]{0.4\textwidth}
        \centering
        \includegraphics[trim={0mm 12mm 0mm 15mm}, width=\textwidth]{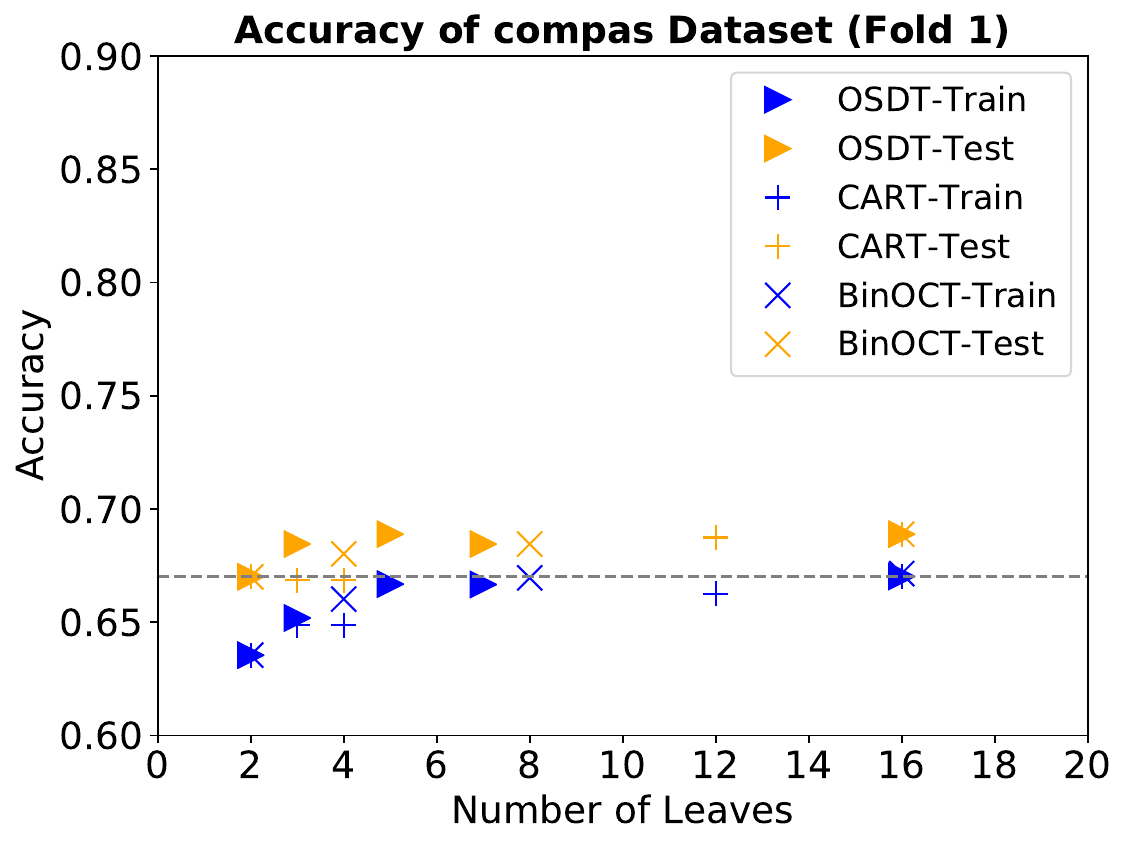}
        \label{fig:COMPAS1}
    \end{subfigure}
    \begin{subfigure}[b]{0.4\textwidth}
        \centering
        \includegraphics[trim={0mm 12mm 0mm 15mm}, width=\textwidth]{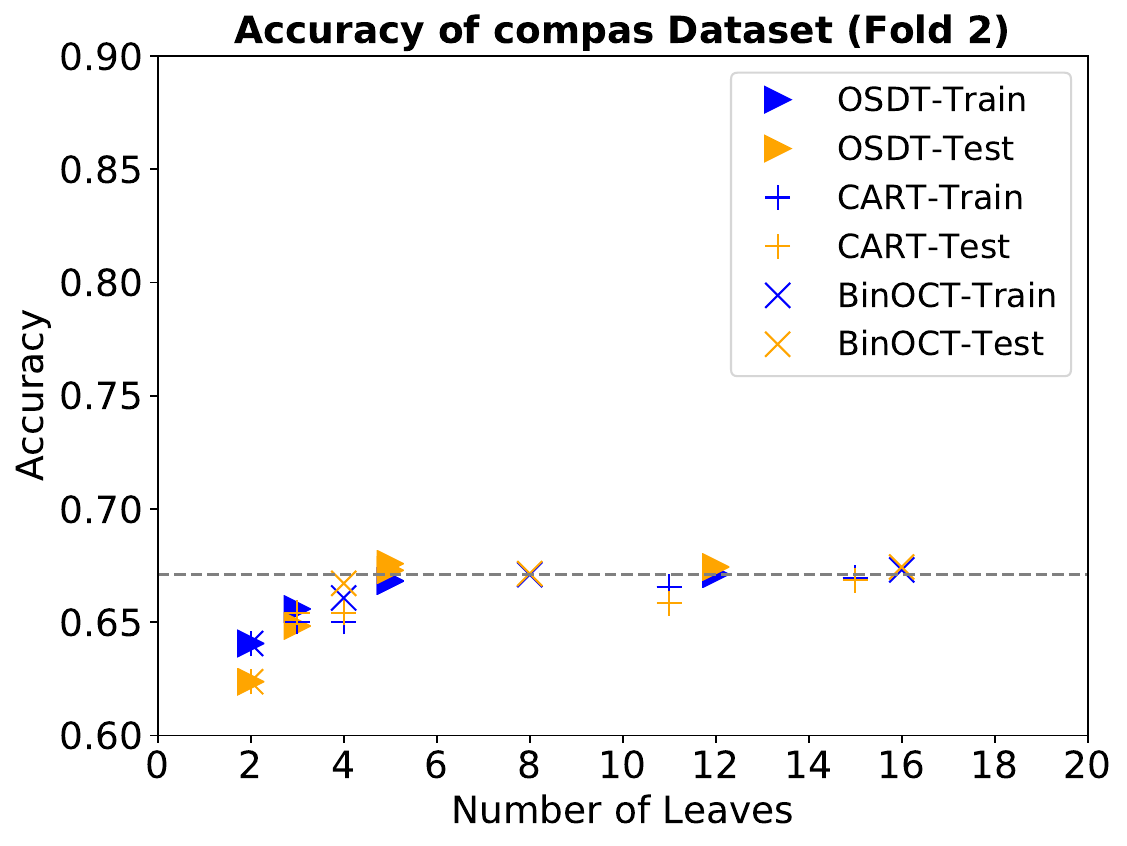}
        \label{fig:COMPAS2}
    \end{subfigure}
    \vskip\baselineskip
    \begin{subfigure}[b]{0.4\textwidth}
        \centering
        \includegraphics[trim={0mm 12mm 0mm 15mm}, width=\textwidth]{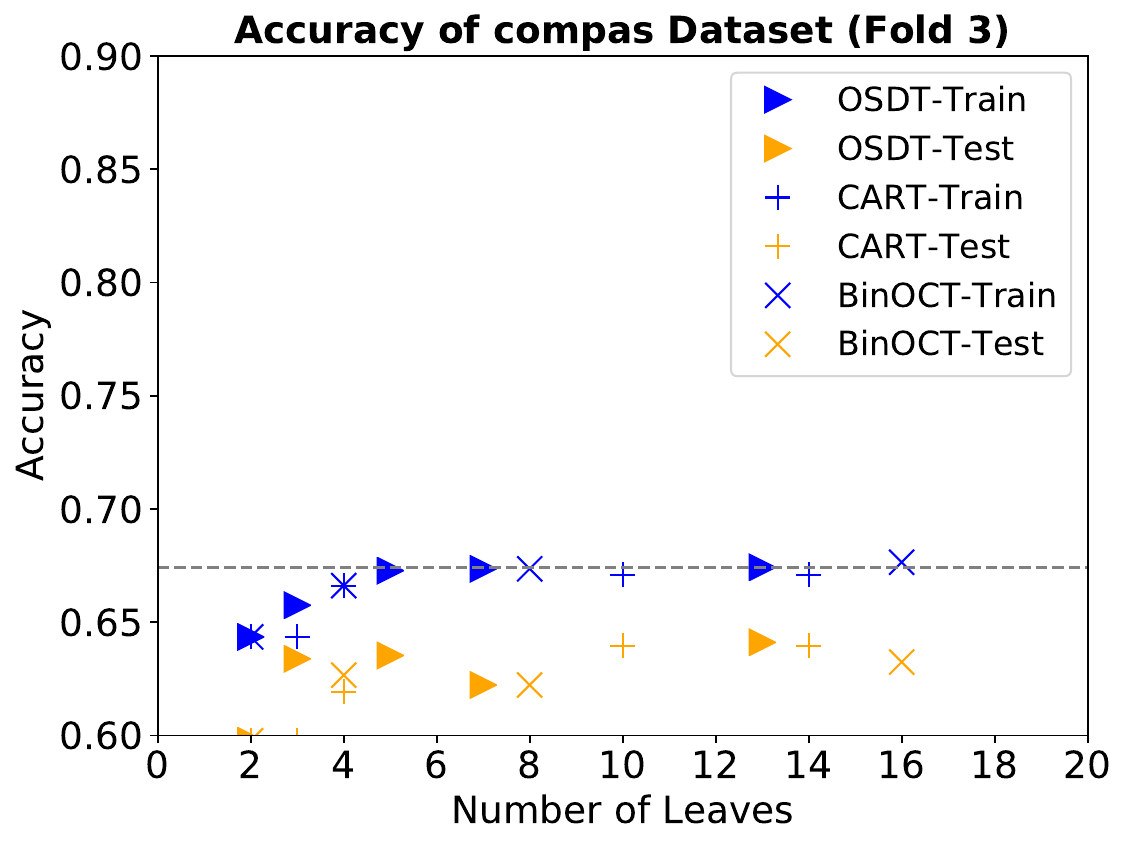}
        \label{fig:COMPAS3}
    \end{subfigure}
    \begin{subfigure}[b]{0.4\textwidth}
        \centering
        \includegraphics[trim={0mm 12mm 0mm 15mm}, width=\textwidth]{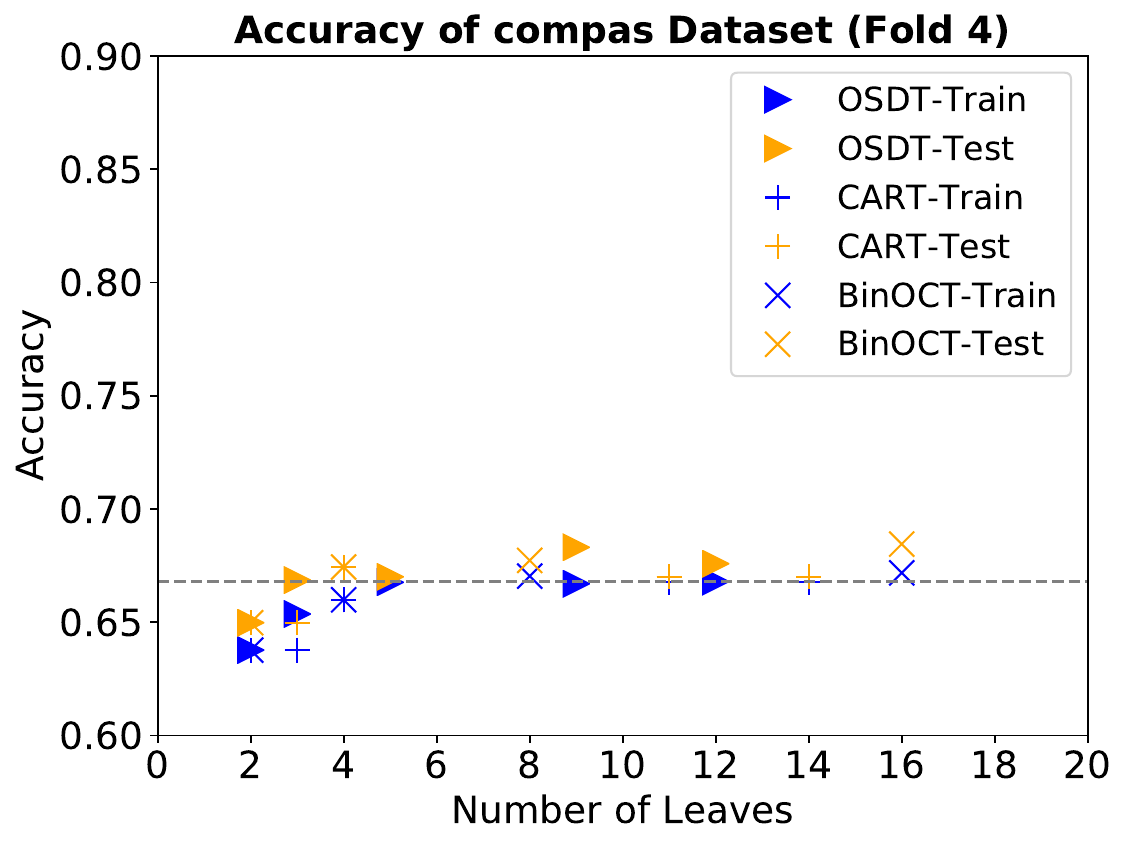}
        \label{fig:COMPAS4}
    \end{subfigure}
    \vskip\baselineskip
    \begin{subfigure}[b]{0.4\textwidth}
        \centering
        \includegraphics[trim={0mm 12mm 0mm 15mm}, width=\textwidth]{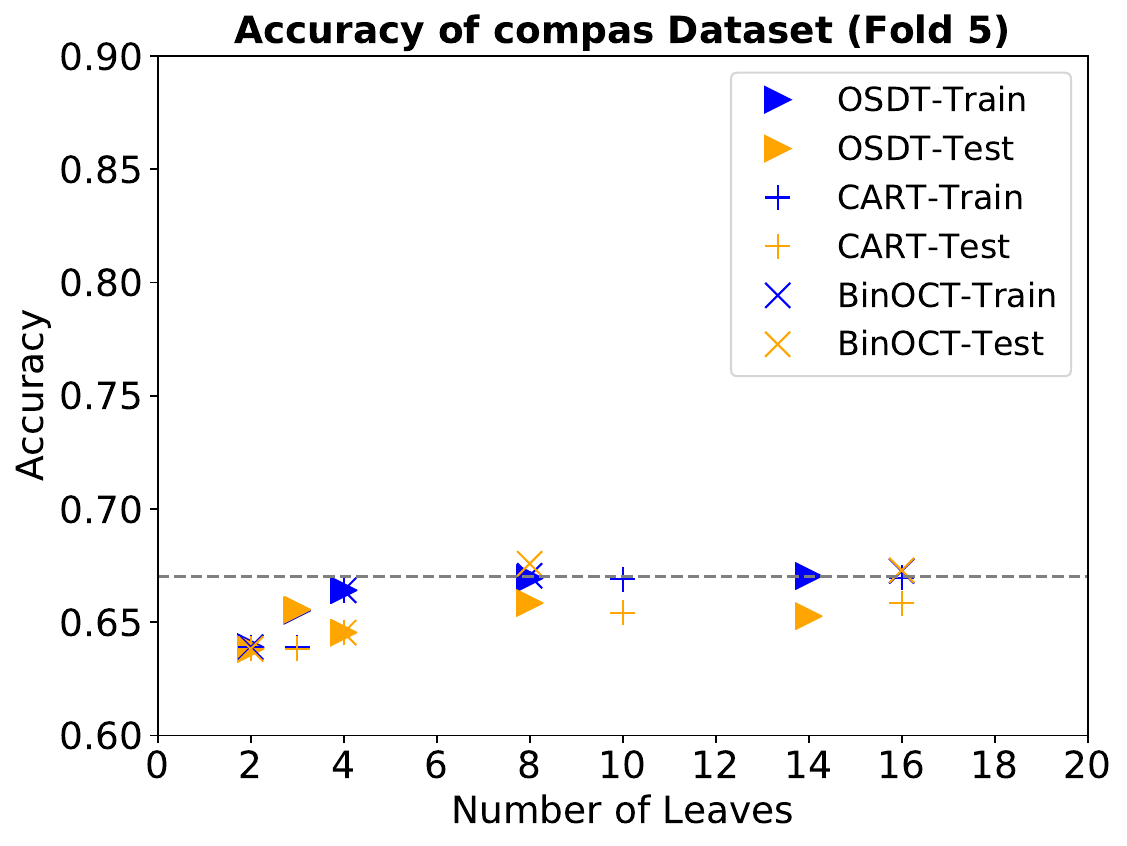}
        \label{fig:COMPAS5}
    \end{subfigure}
    \begin{subfigure}[b]{0.4\textwidth}
        \centering
        \includegraphics[trim={0mm 12mm 0mm 15mm}, width=\textwidth]{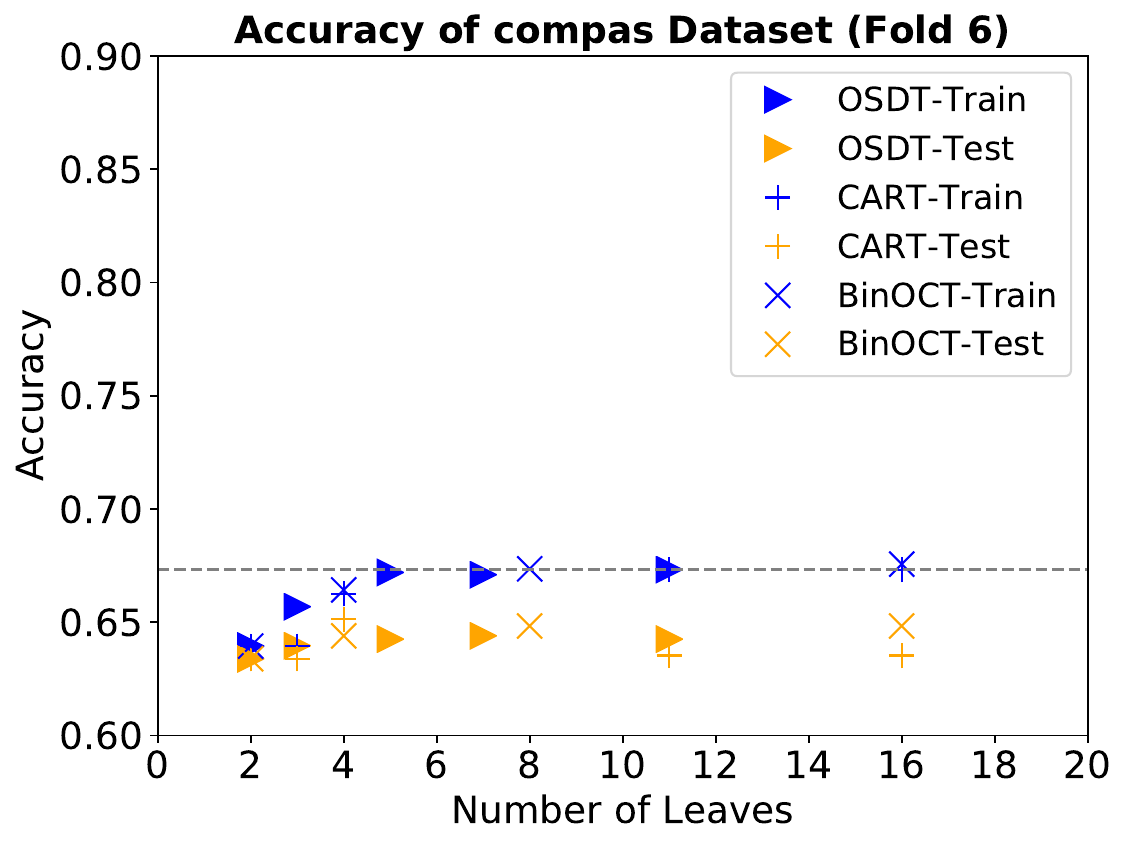}
        \label{fig:COMPAS6}
    \end{subfigure}
    \vskip\baselineskip
    \begin{subfigure}[b]{0.4\textwidth}
        \centering
        \includegraphics[trim={0mm 12mm 0mm 15mm}, width=\textwidth]{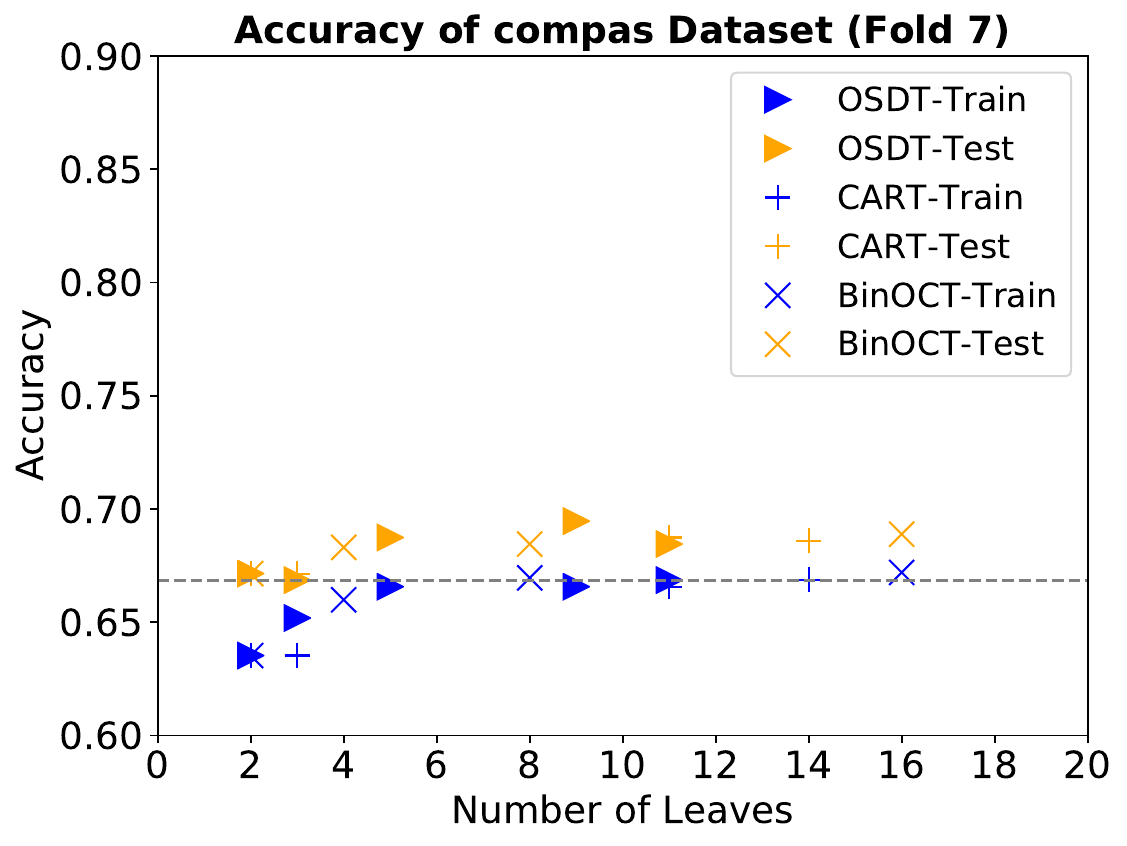}
        \label{fig:COMPAS7}
    \end{subfigure}
    \begin{subfigure}[b]{0.4\textwidth}
        \centering
        \includegraphics[trim={0mm 12mm 0mm 15mm}, width=\textwidth]{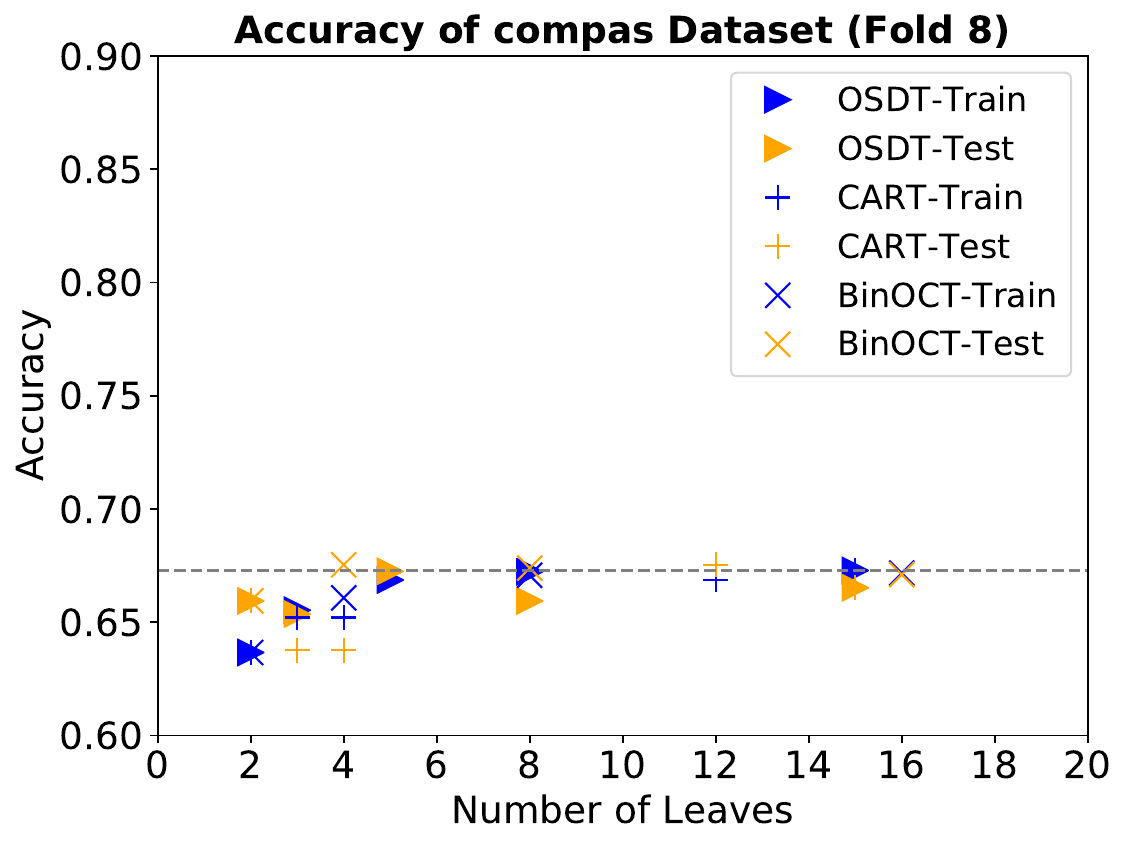}
        \label{fig:COMPAS8}
    \end{subfigure}
    \vskip\baselineskip
    \begin{subfigure}[b]{0.4\textwidth}
        \centering
        \includegraphics[trim={0mm 12mm 0mm 15mm}, width=\textwidth]{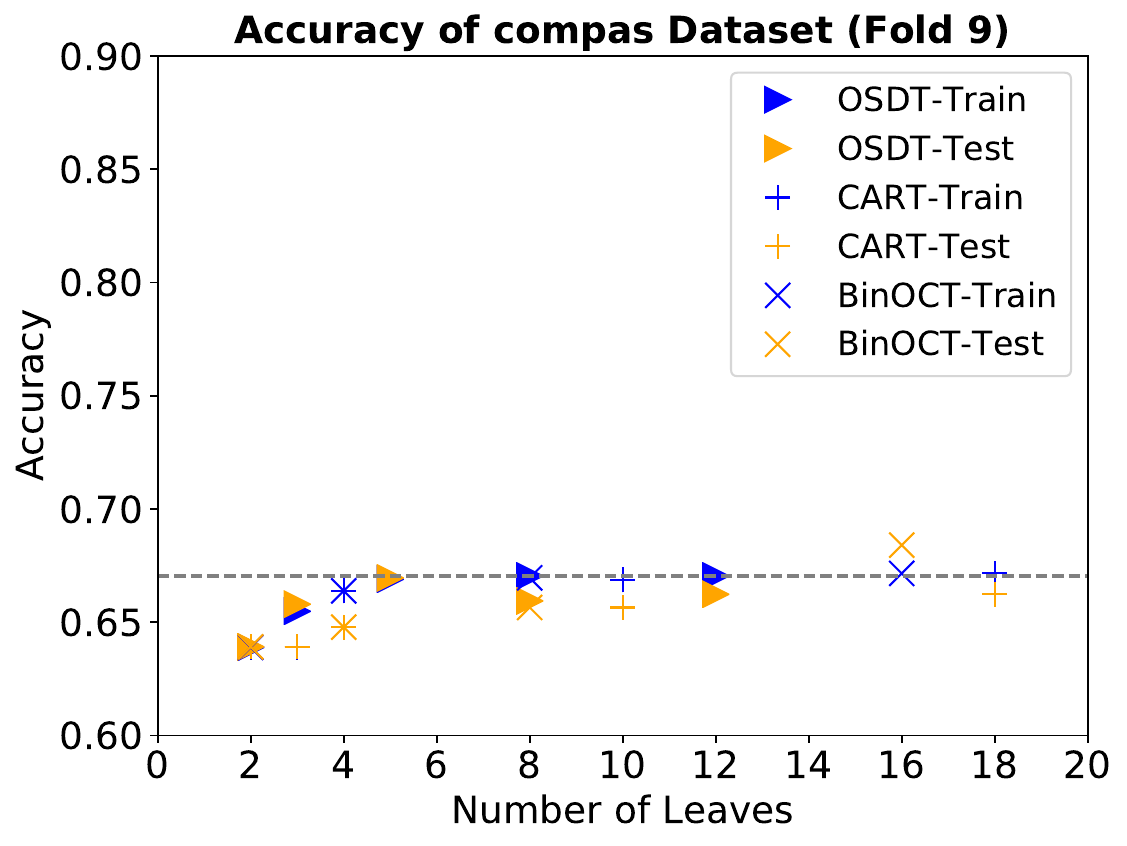}
        \label{fig:COMPAS9}
    \end{subfigure}
    \begin{subfigure}[b]{0.4\textwidth}
        \centering
        \includegraphics[trim={0mm 12mm 0mm 15mm}, width=\textwidth]{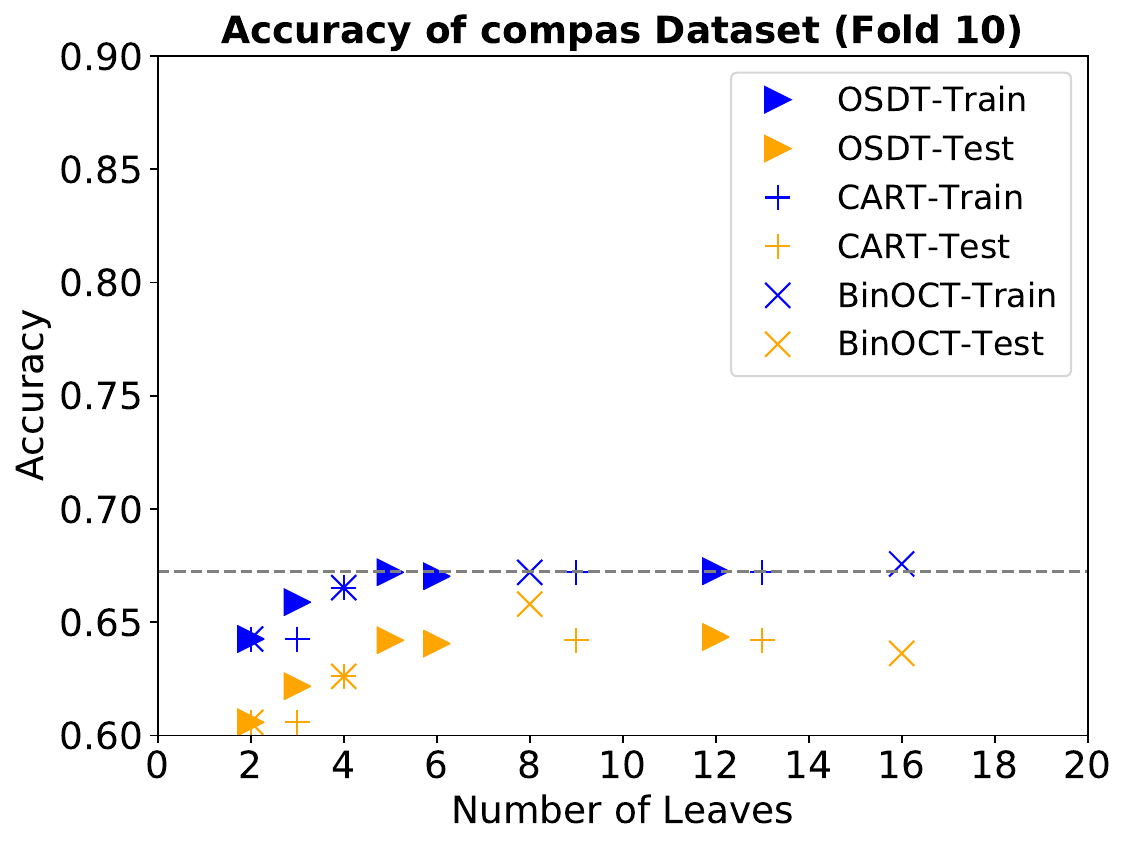}
        \label{fig:COMPAS10}
    \end{subfigure}
    \vskip\baselineskip
    \caption{10-fold cross-validation experiment results of OSDT, CART, BinOCT on COMPAS dataset. Horizontal lines indicate the accuracy of the best OSDT tree in training. %
    } 
    \label{fig:cv-compas}
\end{figure*}

\begin{figure*}[t!]
    \centering
    \begin{subfigure}[b]{0.4\textwidth}
        \centering
        \includegraphics[trim={0mm 12mm 0mm 15mm}, width=\textwidth]{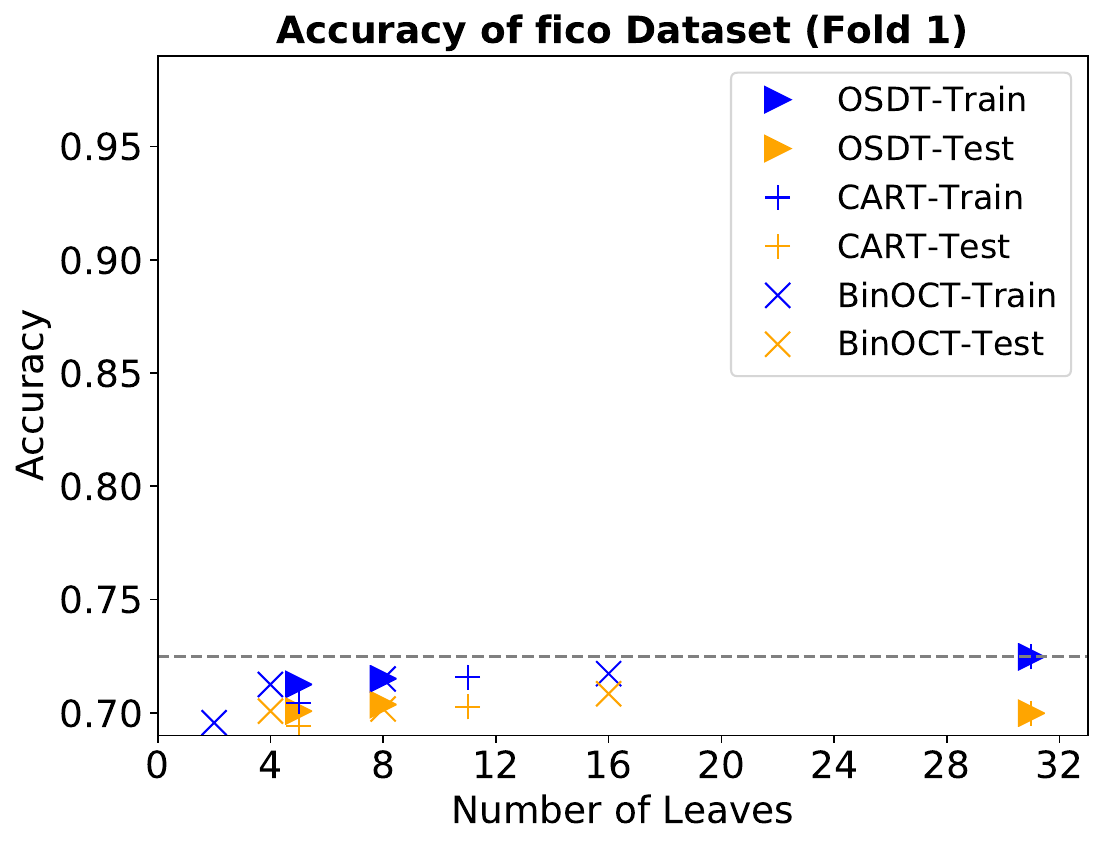}
        \label{fig:fico1}
    \end{subfigure}
    \begin{subfigure}[b]{0.4\textwidth}
        \centering
        \includegraphics[trim={0mm 12mm 0mm 15mm}, width=\textwidth]{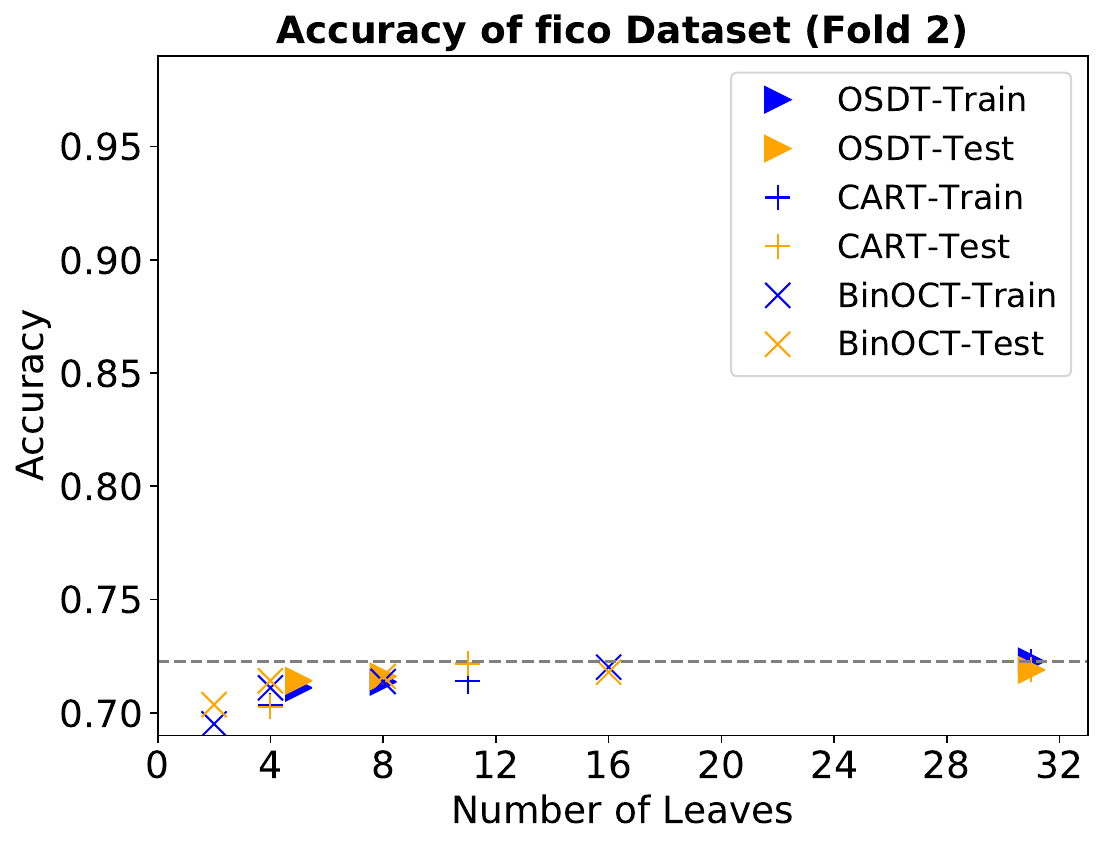}
        \label{fig:fico2}
    \end{subfigure}
    \vskip\baselineskip
    \begin{subfigure}[b]{0.4\textwidth}
        \centering
        \includegraphics[trim={0mm 12mm 0mm 15mm}, width=\textwidth]{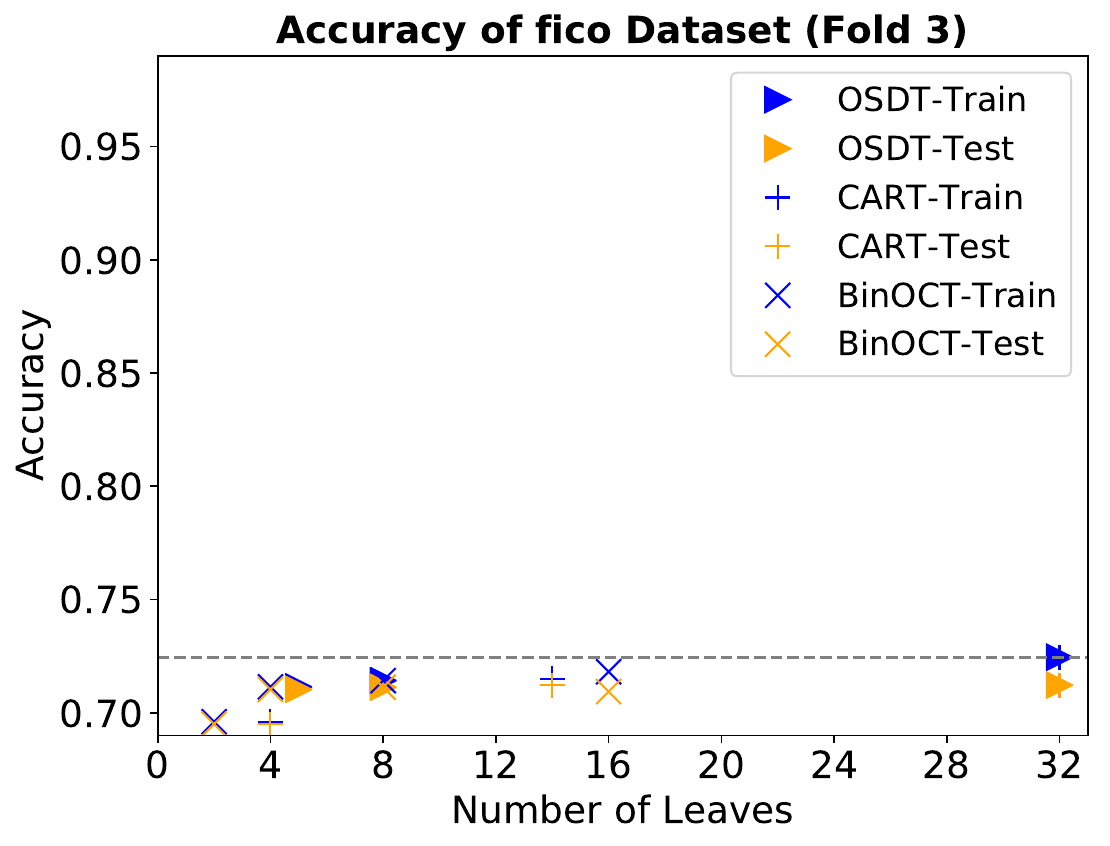}
        \label{fig:fico3}
    \end{subfigure}
    \begin{subfigure}[b]{0.4\textwidth}
        \centering
        \includegraphics[trim={0mm 12mm 0mm 15mm}, width=\textwidth]{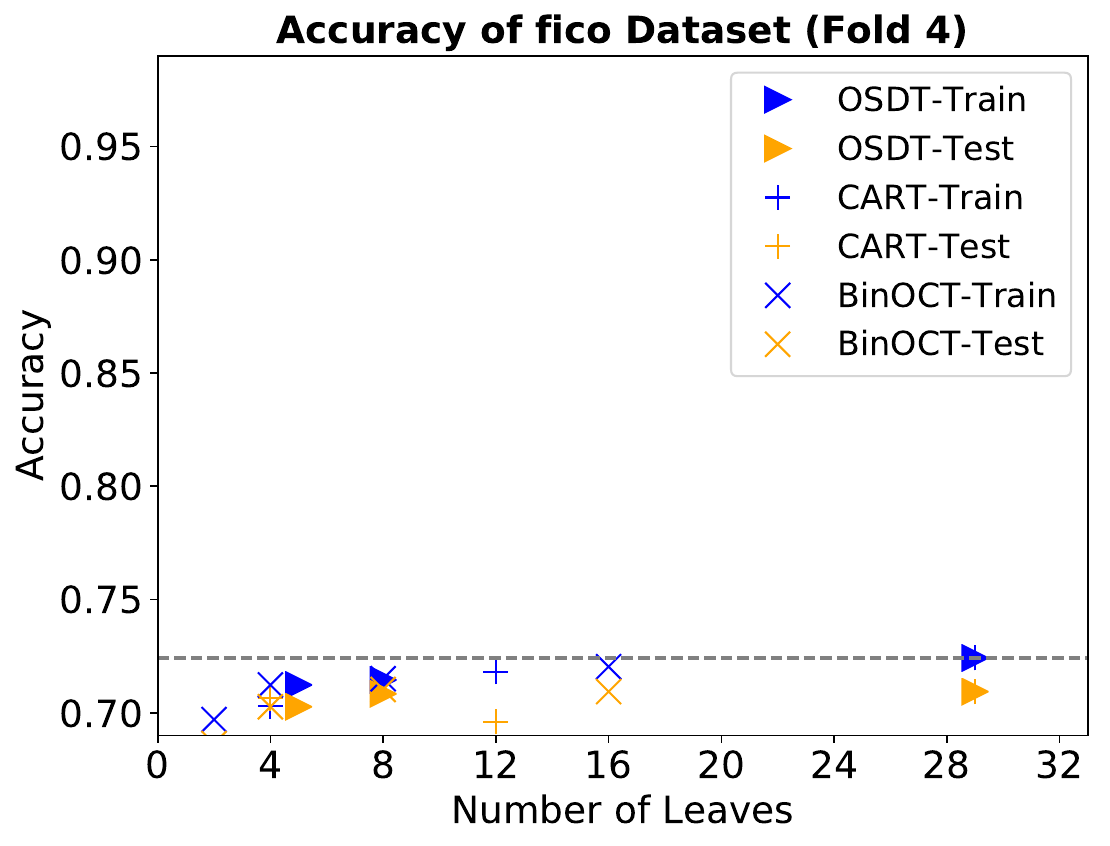}
        \label{fig:fico4}
    \end{subfigure}
    \vskip\baselineskip
    \begin{subfigure}[b]{0.4\textwidth}
        \centering
        \includegraphics[trim={0mm 12mm 0mm 15mm}, width=\textwidth]{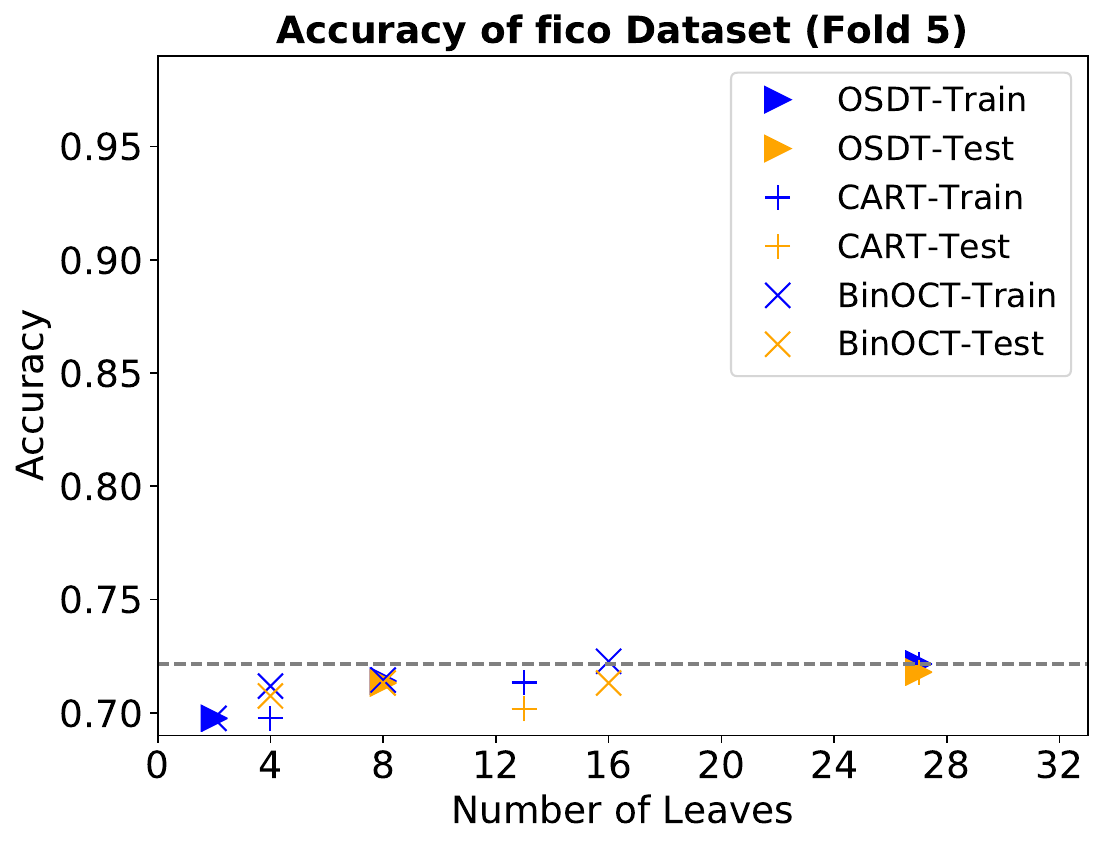}
        \label{fig:fico5}
    \end{subfigure}
    \begin{subfigure}[b]{0.4\textwidth}
        \centering
        \includegraphics[trim={0mm 12mm 0mm 15mm}, width=\textwidth]{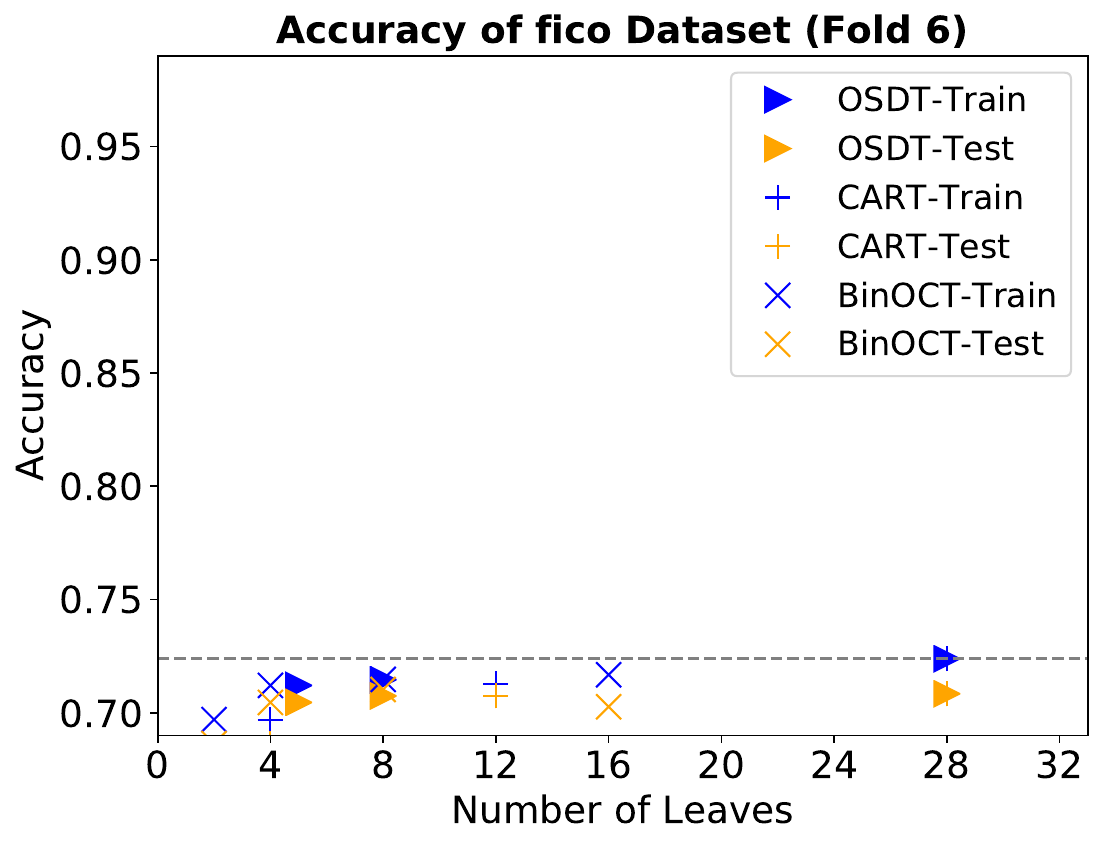}
        \label{fig:fico6}
    \end{subfigure}
    \vskip\baselineskip
    \begin{subfigure}[b]{0.4\textwidth}
        \centering
        \includegraphics[trim={0mm 12mm 0mm 15mm}, width=\textwidth]{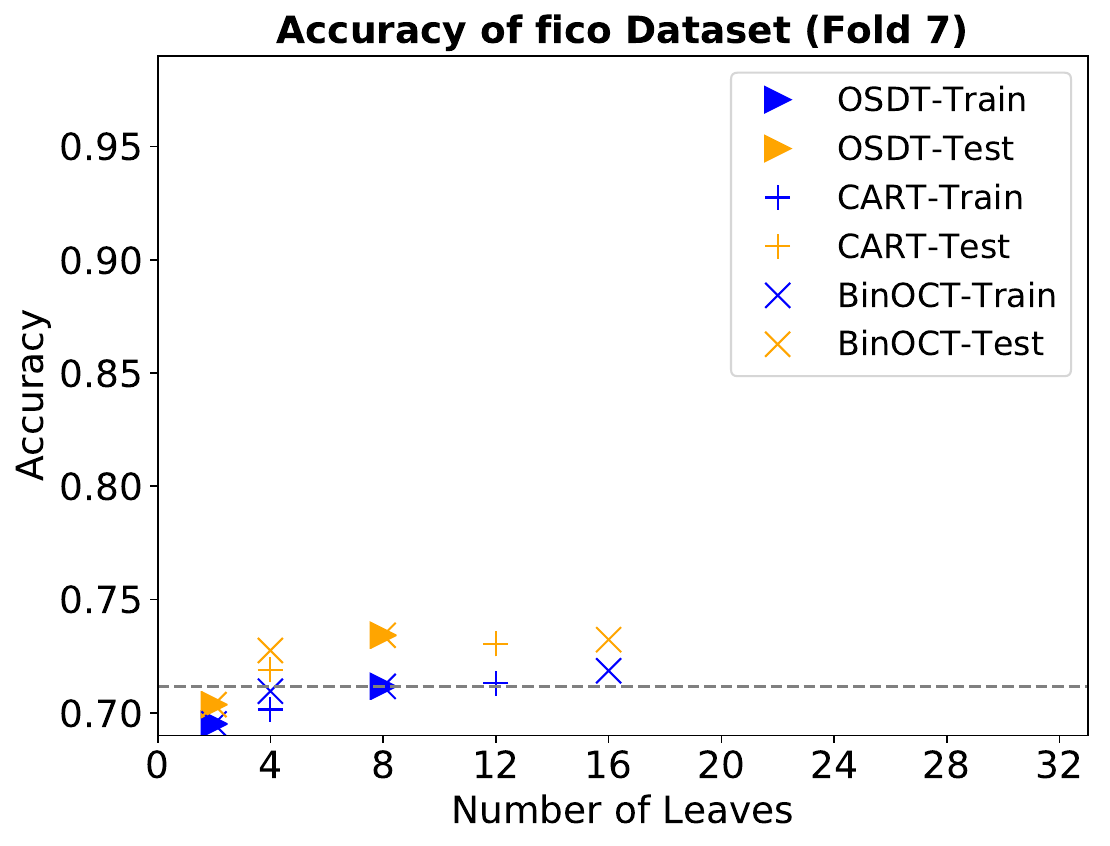}
        \label{fig:fico7}
    \end{subfigure}
    \begin{subfigure}[b]{0.4\textwidth}
        \centering
        \includegraphics[trim={0mm 12mm 0mm 15mm}, width=\textwidth]{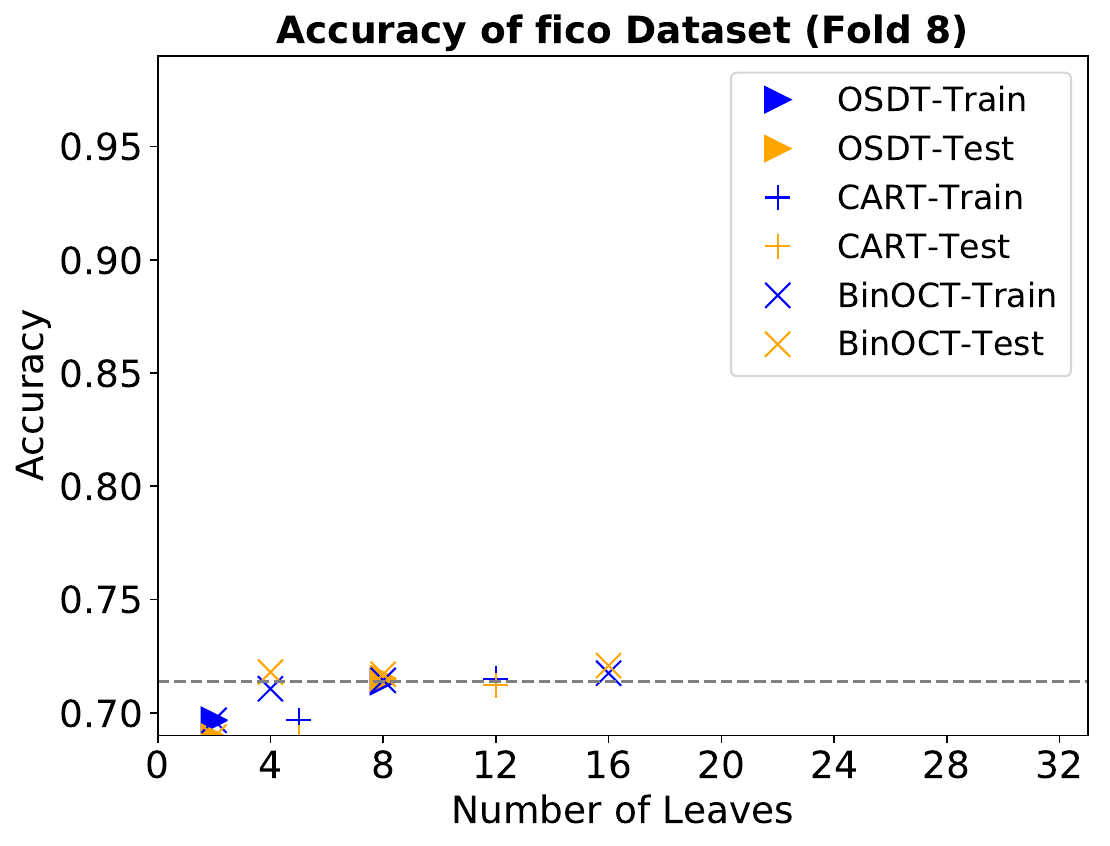}
        \label{fig:fico8}
    \end{subfigure}
    \vskip\baselineskip
    \begin{subfigure}[b]{0.4\textwidth}
        \centering
        \includegraphics[trim={0mm 12mm 0mm 15mm}, width=\textwidth]{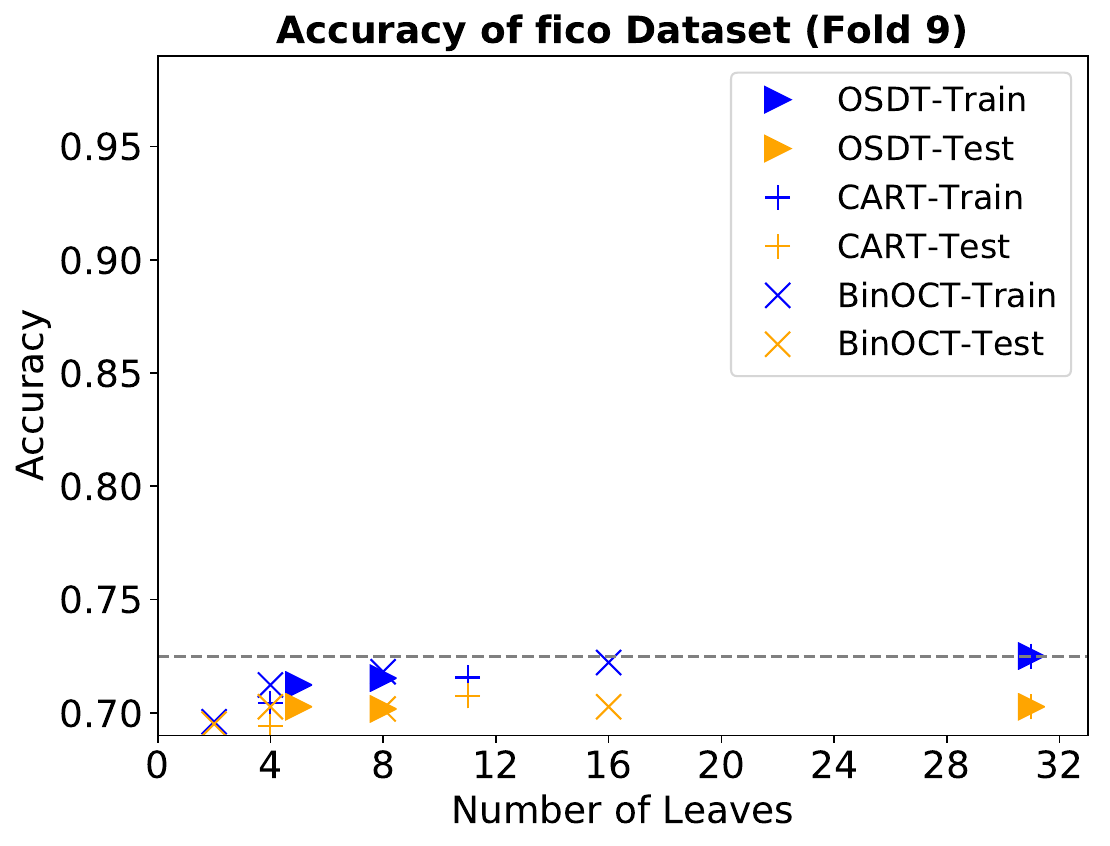}
        \label{fig:fico9}
    \end{subfigure}
    \begin{subfigure}[b]{0.4\textwidth}
        \centering
        \includegraphics[trim={0mm 12mm 0mm 15mm}, width=\textwidth]{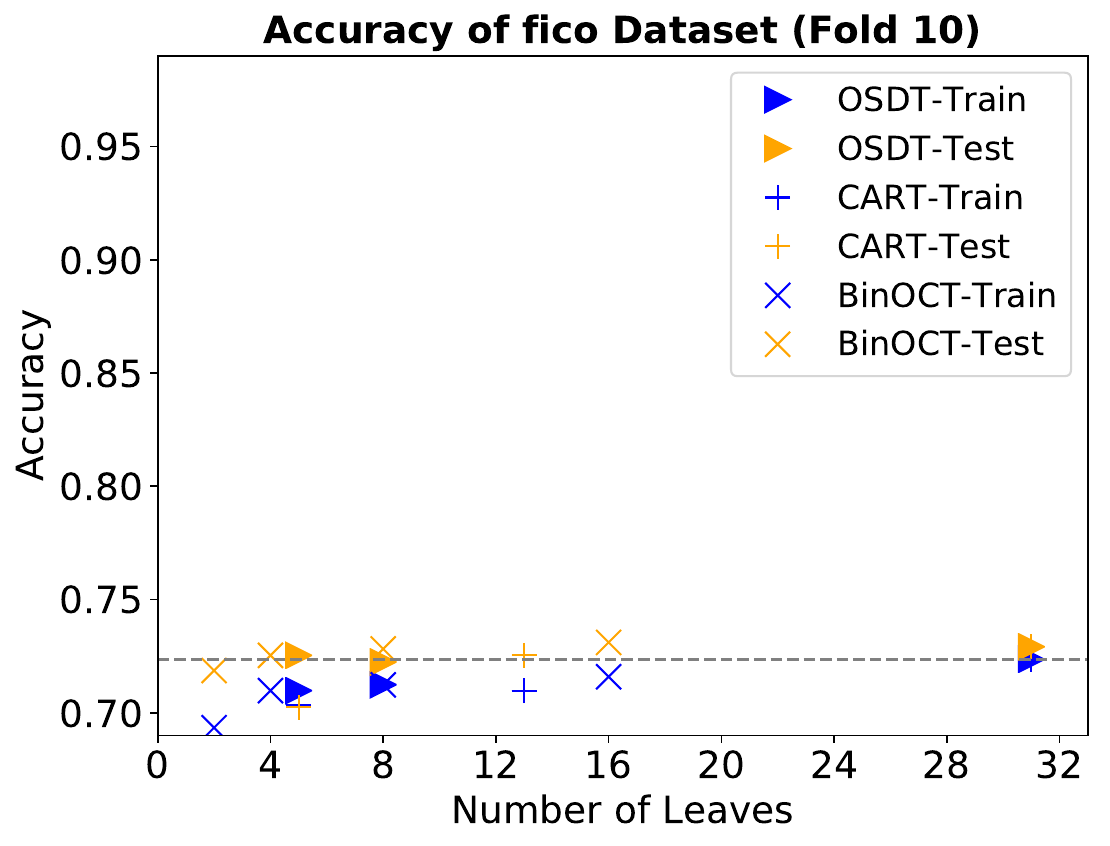}
        \label{fig:fico10}
    \end{subfigure}
    \vskip\baselineskip
    \caption{10-fold cross-validation experiment results of OSDT, CART, BinOCT on FICO dataset. Horizontal lines indicate the accuracy of the best OSDT tree in training. %
    }
    \label{fig:cv-fico}
\end{figure*}

\begin{figure*}[t!]
    \centering
    \begin{subfigure}[b]{0.4\textwidth}
        \centering
        \includegraphics[trim={0mm 12mm 0mm 15mm}, width=\textwidth]{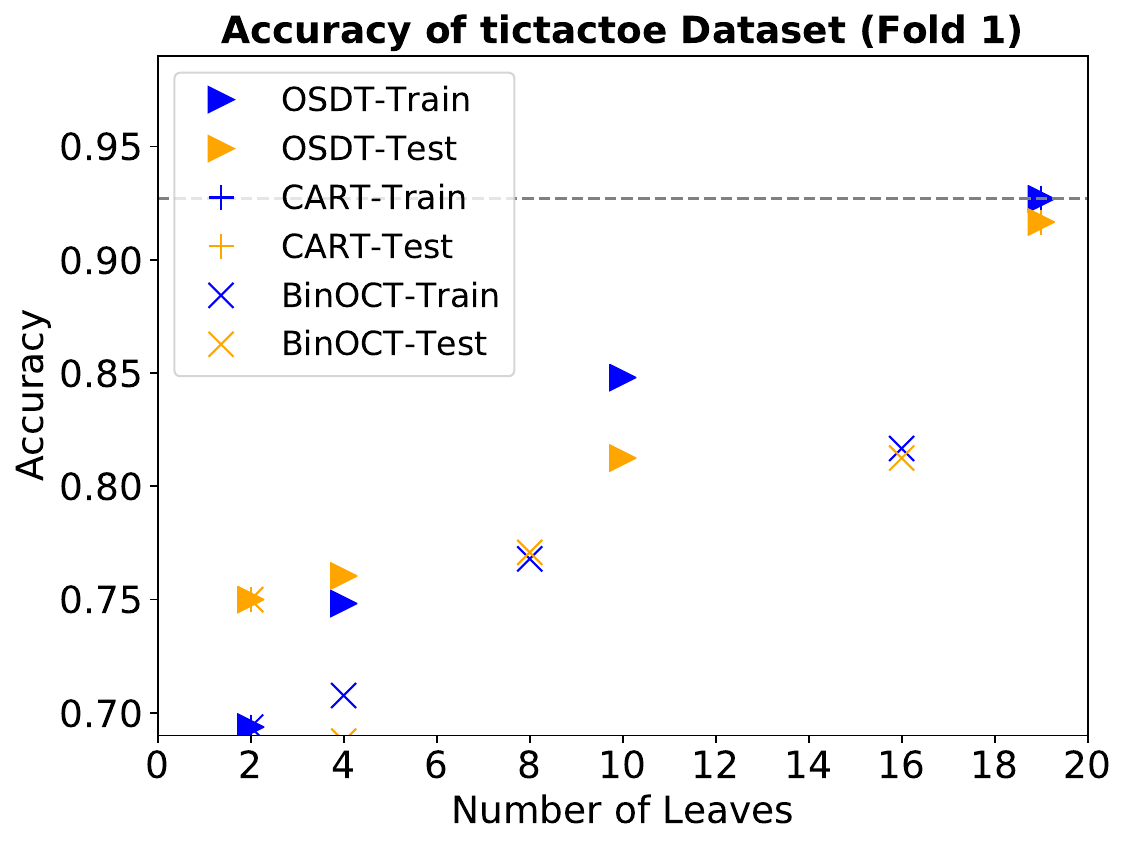}
        \label{fig:tictactoe1}
    \end{subfigure}
    \begin{subfigure}[b]{0.4\textwidth}
        \centering
        \includegraphics[trim={0mm 12mm 0mm 15mm}, width=\textwidth]{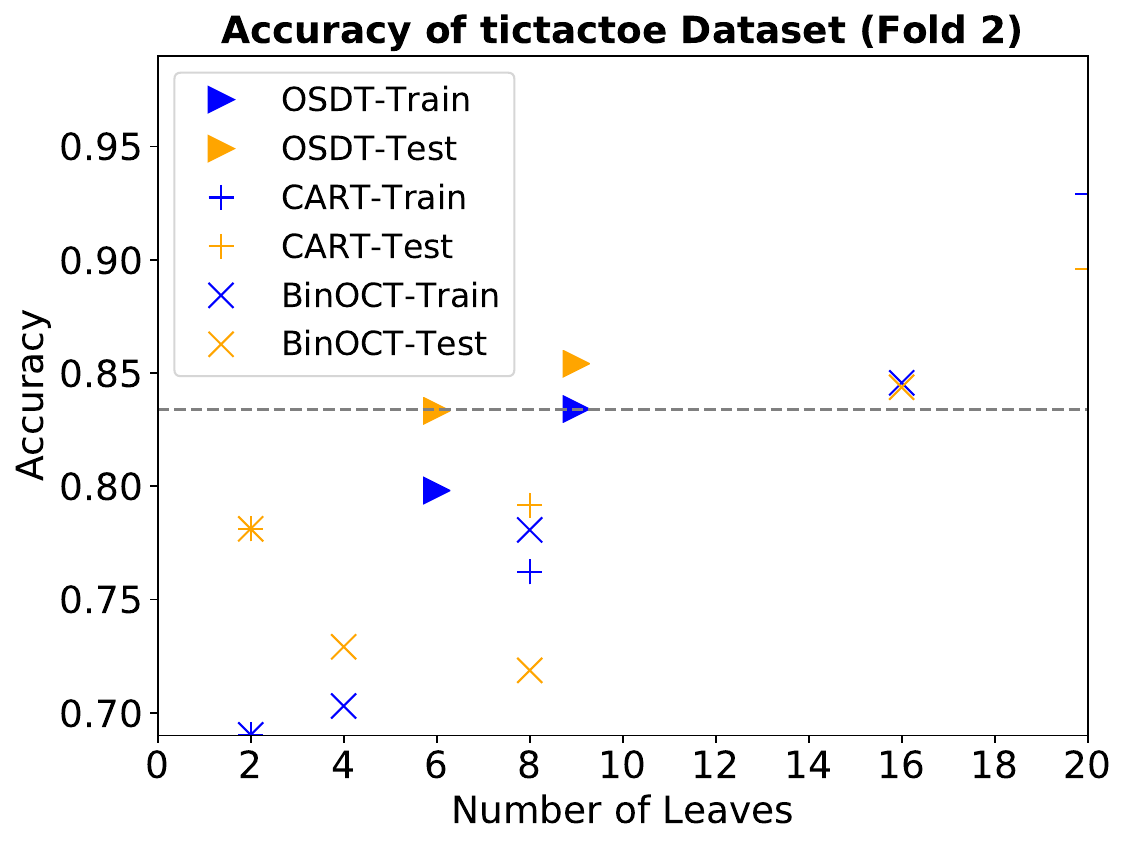}
        \label{fig:tictactoe2}
    \end{subfigure}
    \vskip\baselineskip
    \begin{subfigure}[b]{0.4\textwidth}
        \centering
        \includegraphics[trim={0mm 12mm 0mm 15mm}, width=\textwidth]{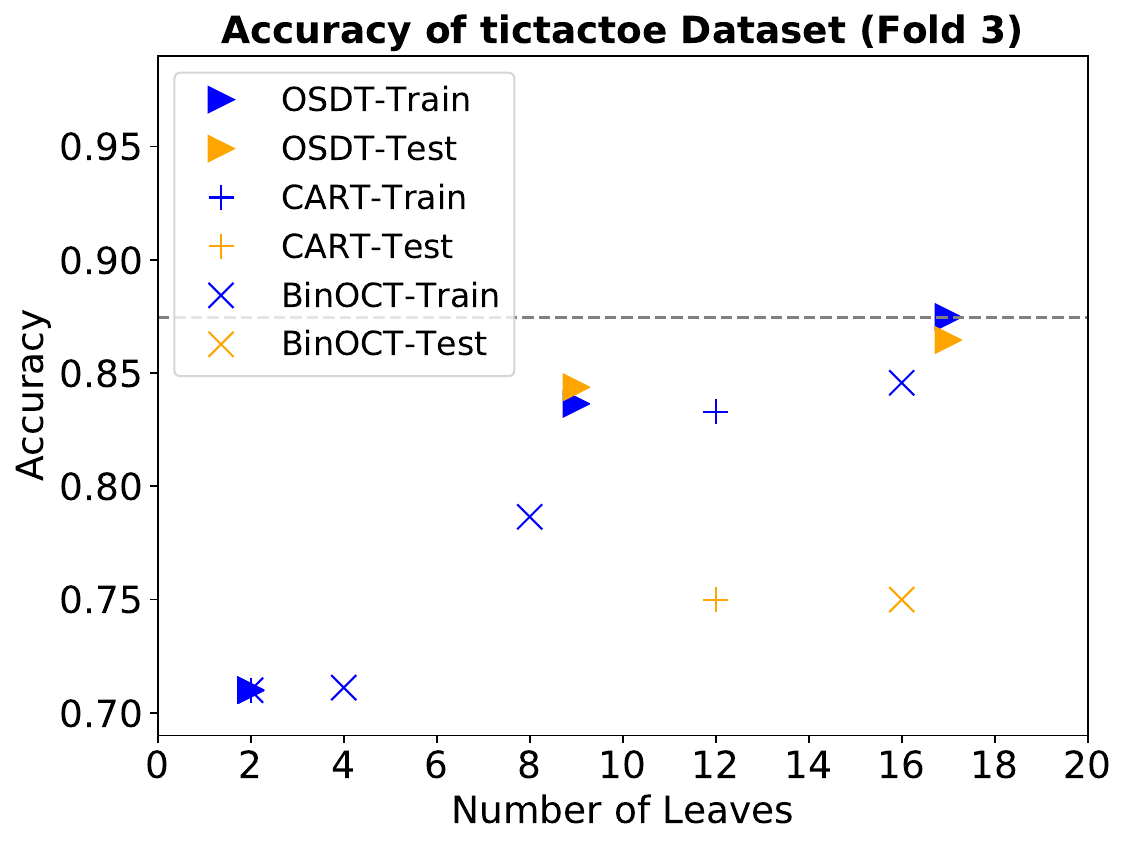}
        \label{fig:tictactoe3}
    \end{subfigure}
    \begin{subfigure}[b]{0.4\textwidth}
        \centering
        \includegraphics[trim={0mm 12mm 0mm 15mm}, width=\textwidth]{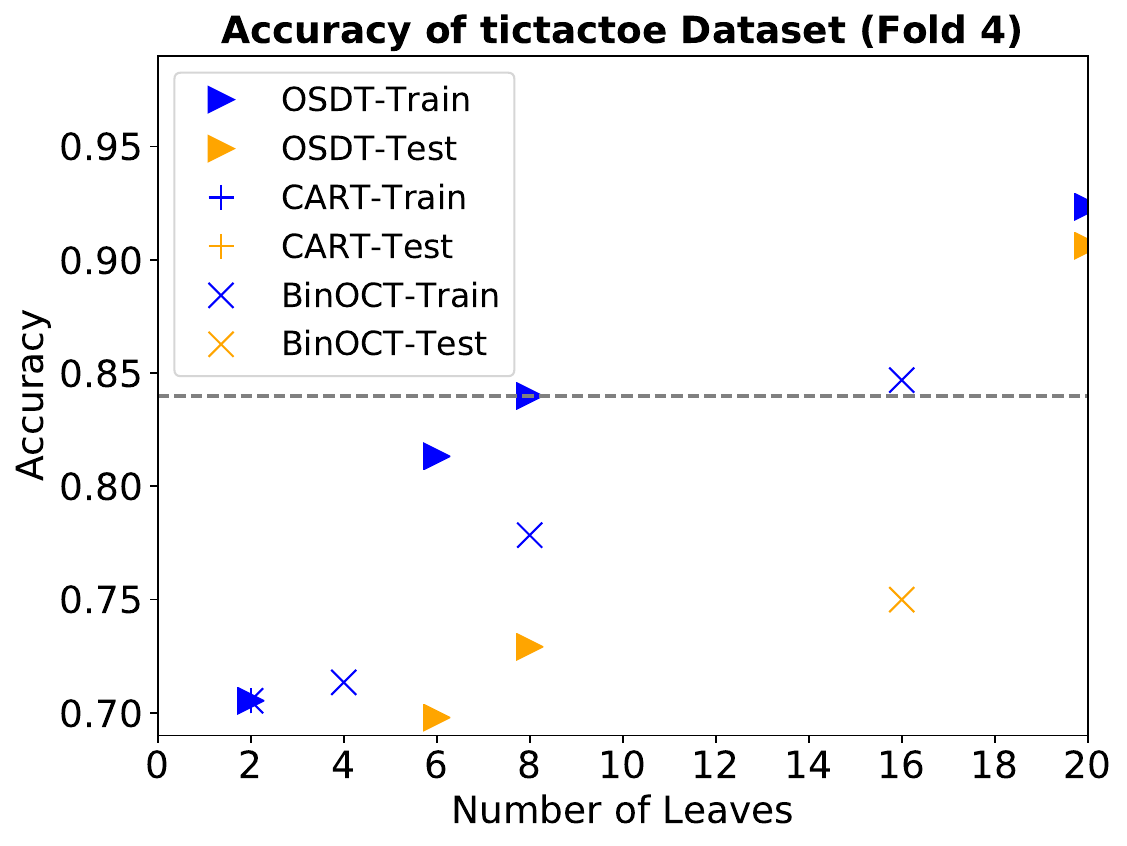}
        \label{fig:tictactoe4}
    \end{subfigure}
    \vskip\baselineskip
    \begin{subfigure}[b]{0.4\textwidth}
        \centering
        \includegraphics[trim={0mm 12mm 0mm 15mm}, width=\textwidth]{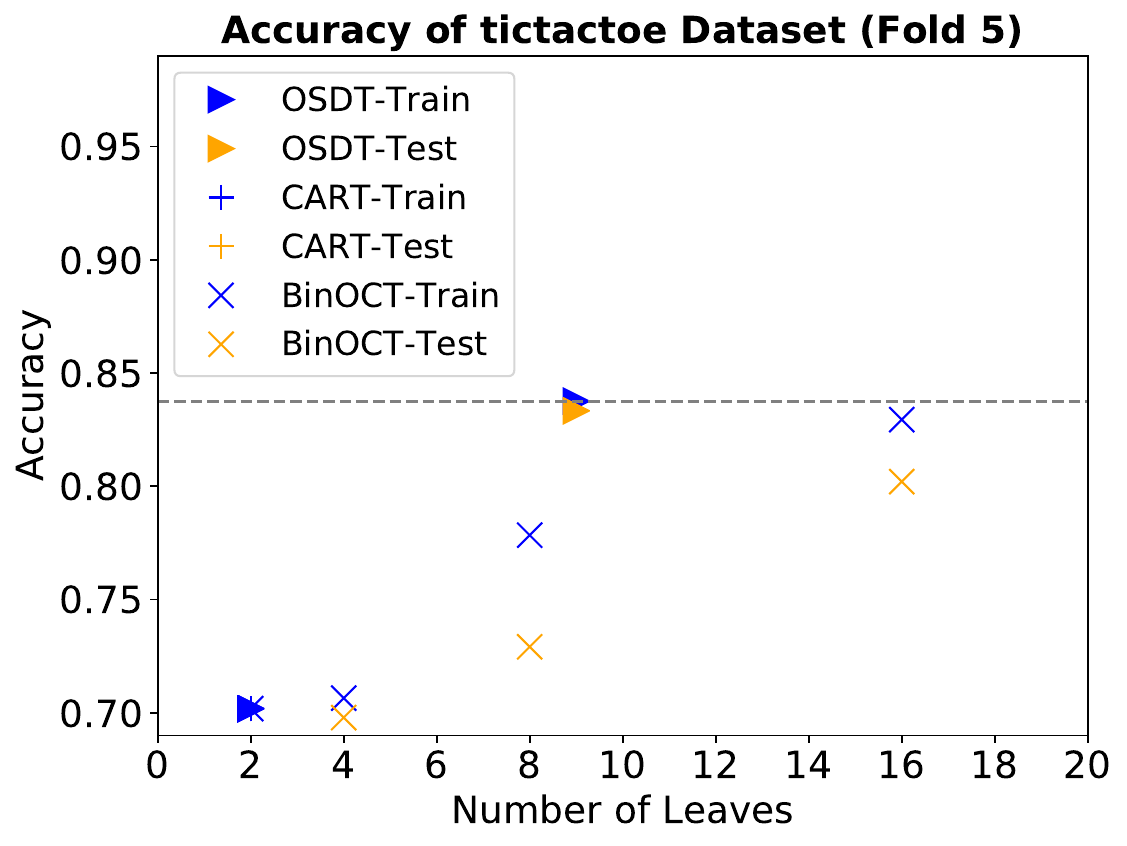}
        \label{fig:tictactoe5}
    \end{subfigure}
    \begin{subfigure}[b]{0.4\textwidth}
        \centering
        \includegraphics[trim={0mm 12mm 0mm 15mm}, width=\textwidth]{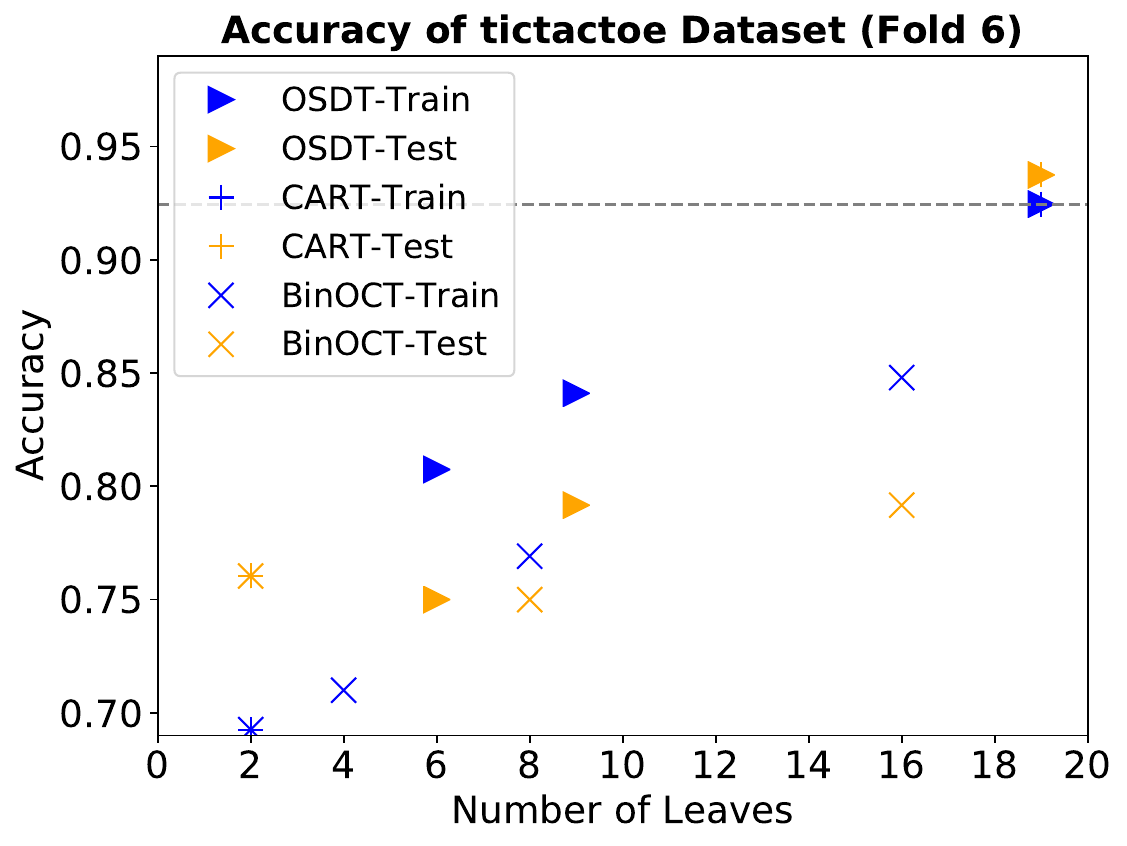}
        \label{fig:tictactoe6}
    \end{subfigure}
    \vskip\baselineskip
    \begin{subfigure}[b]{0.4\textwidth}
        \centering
        \includegraphics[trim={0mm 12mm 0mm 15mm}, width=\textwidth]{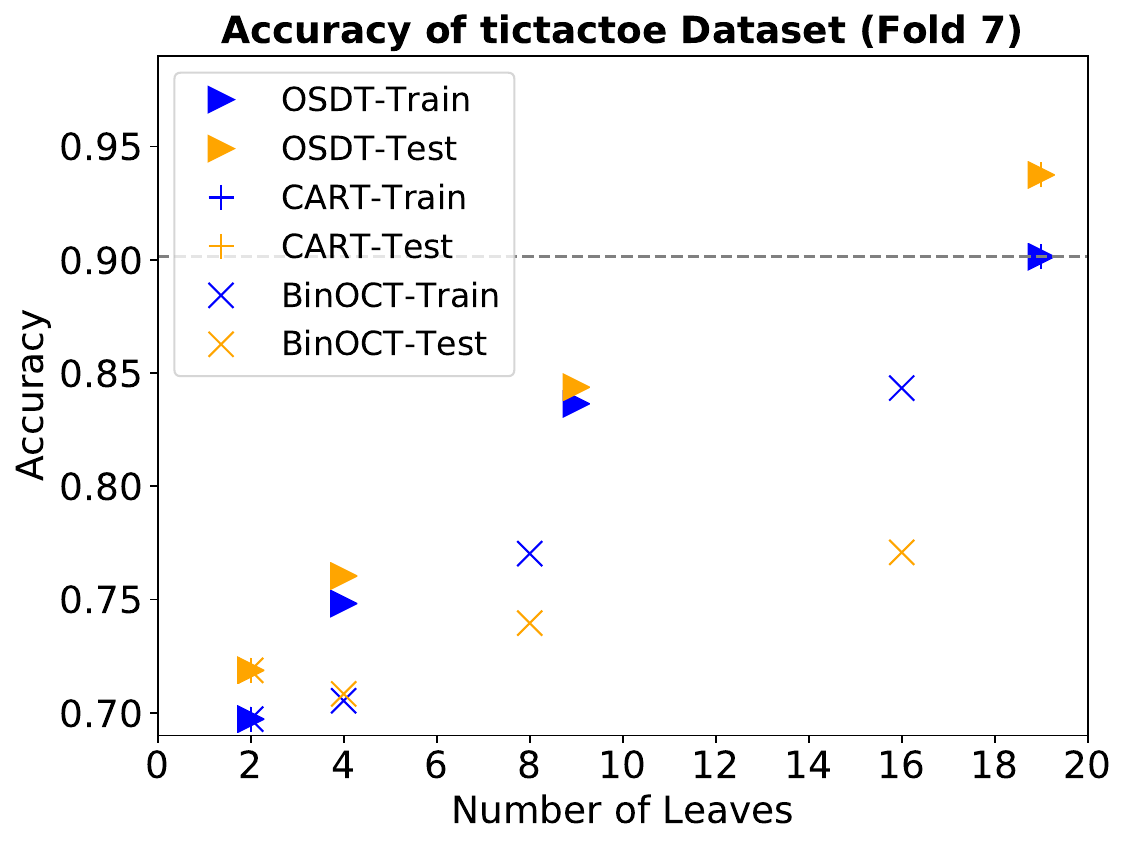}
        \label{fig:tictactoe7}
    \end{subfigure}
    \begin{subfigure}[b]{0.4\textwidth}
        \centering
        \includegraphics[trim={0mm 12mm 0mm 15mm}, width=\textwidth]{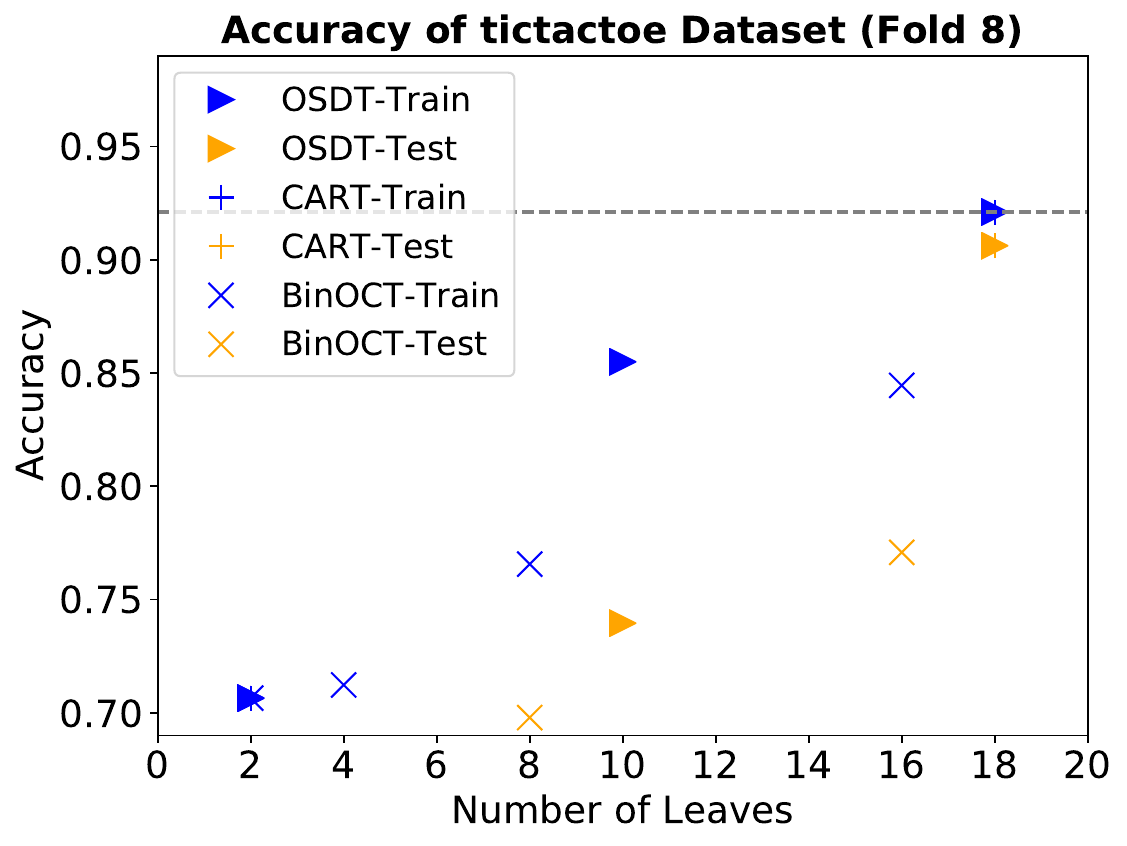}
        \label{fig:tictactoe8}
    \end{subfigure}
    \vskip\baselineskip
    \begin{subfigure}[b]{0.4\textwidth}
        \centering
        \includegraphics[trim={0mm 12mm 0mm 15mm}, width=\textwidth]{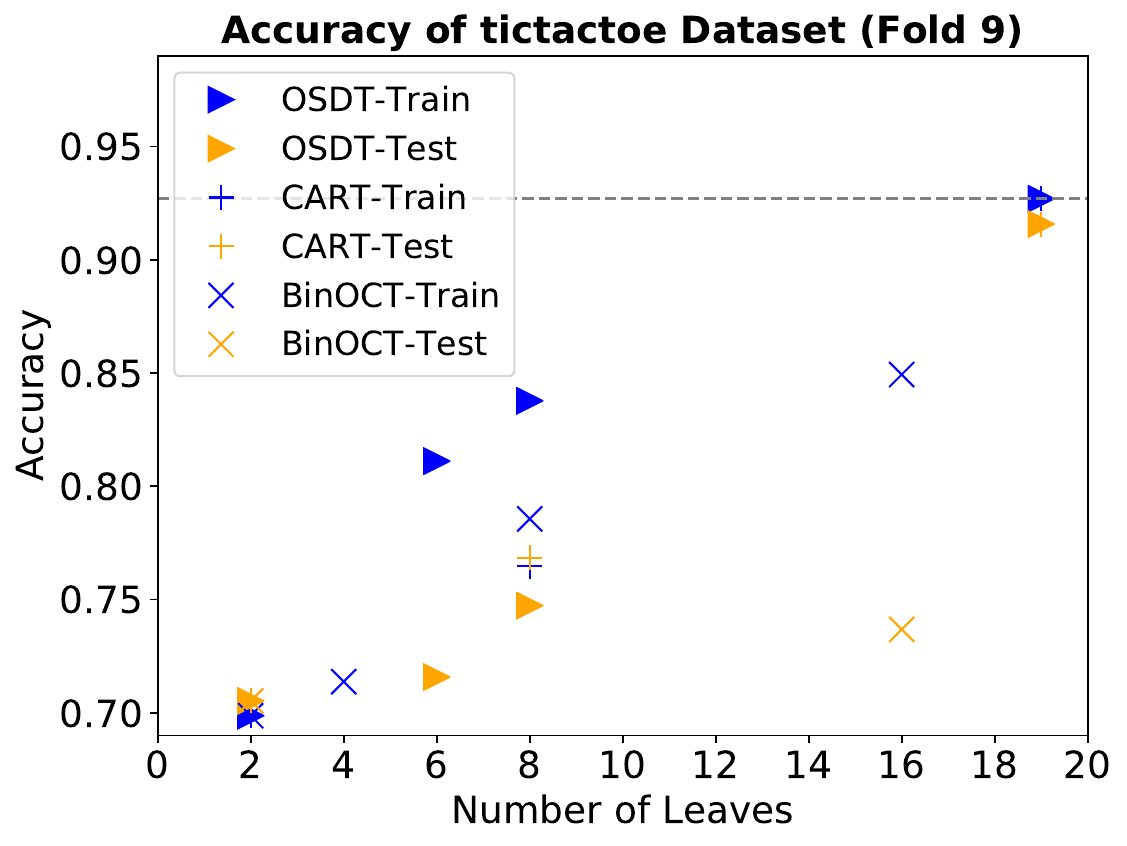}
        \label{fig:tictactoe9}
    \end{subfigure}
    \begin{subfigure}[b]{0.4\textwidth}
        \centering
        \includegraphics[trim={0mm 12mm 0mm 15mm}, width=\textwidth]{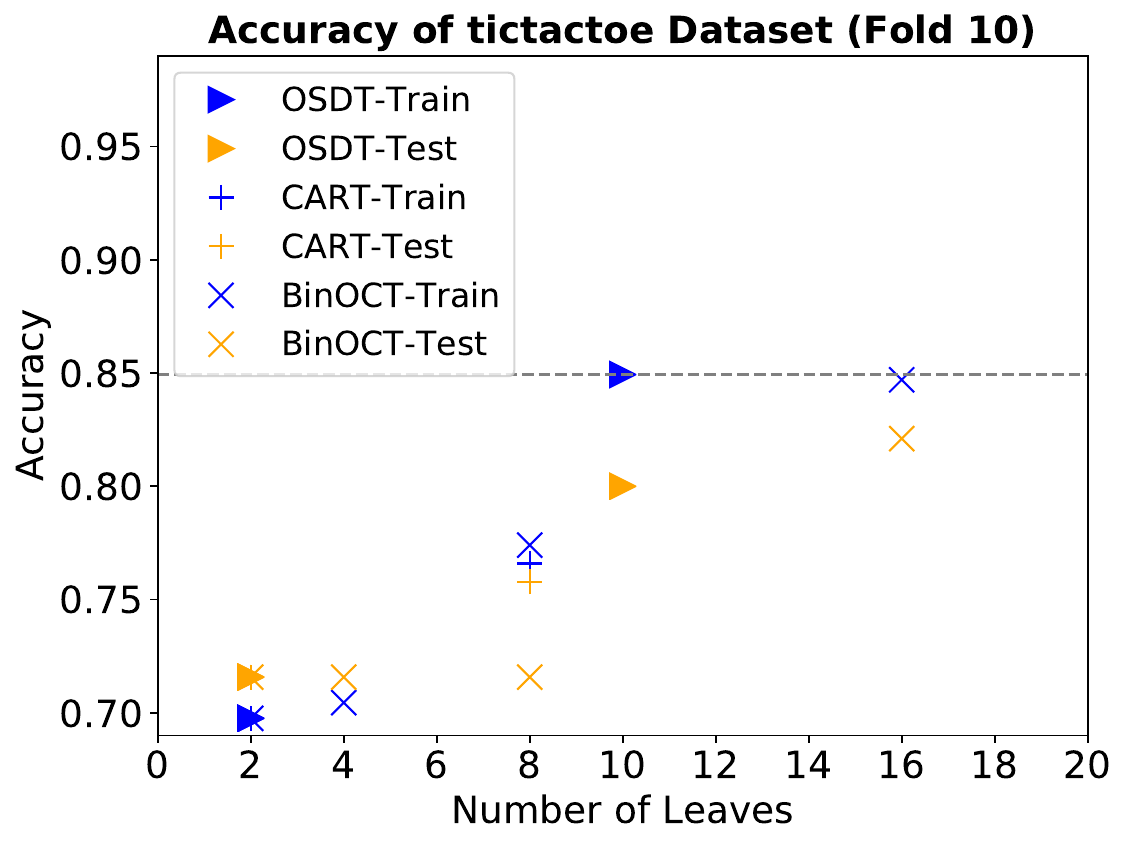}
        \label{fig:tictactoe10}
    \end{subfigure}
    \vskip\baselineskip
    \caption{10-fold cross-validation experiment results of OSDT, CART, BinOCT on Tic-Tac-Toe dataset. Horizontal lines indicate the accuracy of the best OSDT tree in training. %
    }
    \label{fig:cv-tictactoe}
\end{figure*}

\begin{figure*}[t!]
    \centering
    \begin{subfigure}[b]{0.4\textwidth}
        \centering
        \includegraphics[trim={0mm 12mm 0mm 15mm}, width=\textwidth]{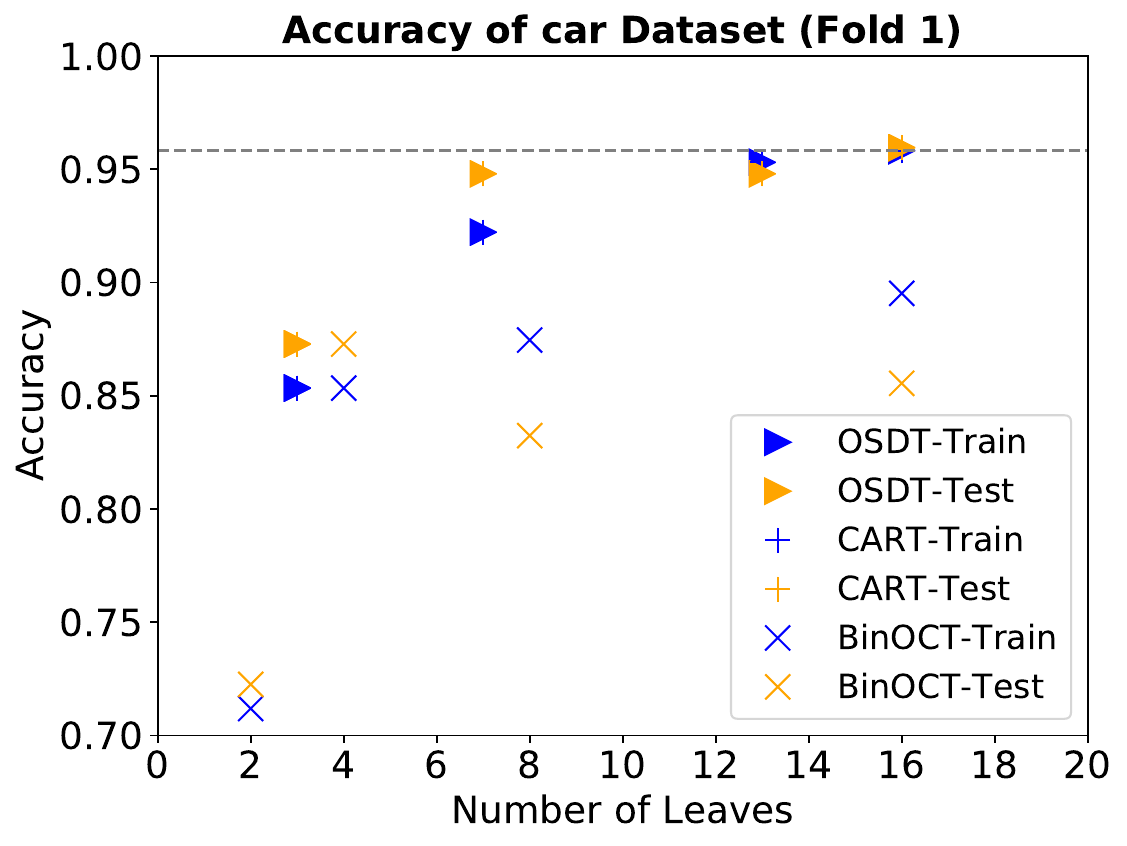}
        \label{fig:car1}
    \end{subfigure}
    \begin{subfigure}[b]{0.4\textwidth}
        \centering
        \includegraphics[trim={0mm 12mm 0mm 15mm}, width=\textwidth]{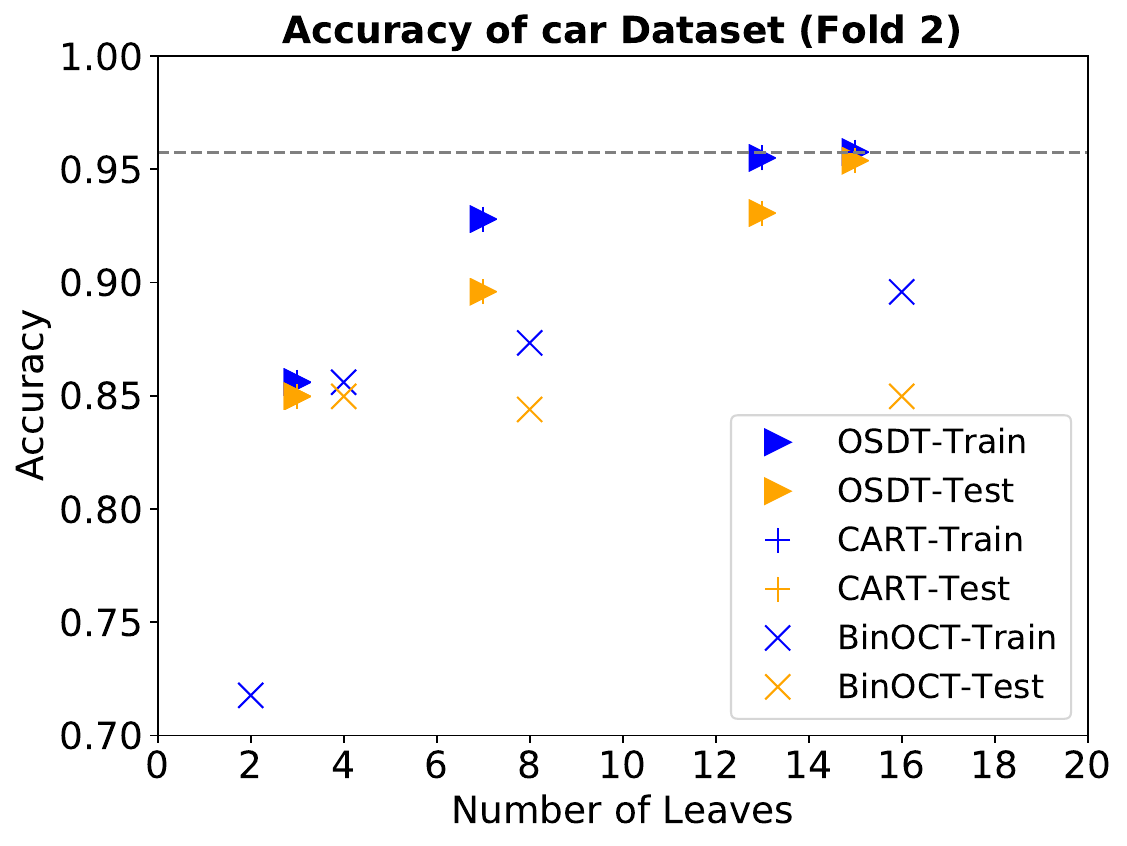}
        \label{fig:car2}
    \end{subfigure}
    \vskip\baselineskip
    \begin{subfigure}[b]{0.4\textwidth}
        \centering
        \includegraphics[trim={0mm 12mm 0mm 15mm}, width=\textwidth]{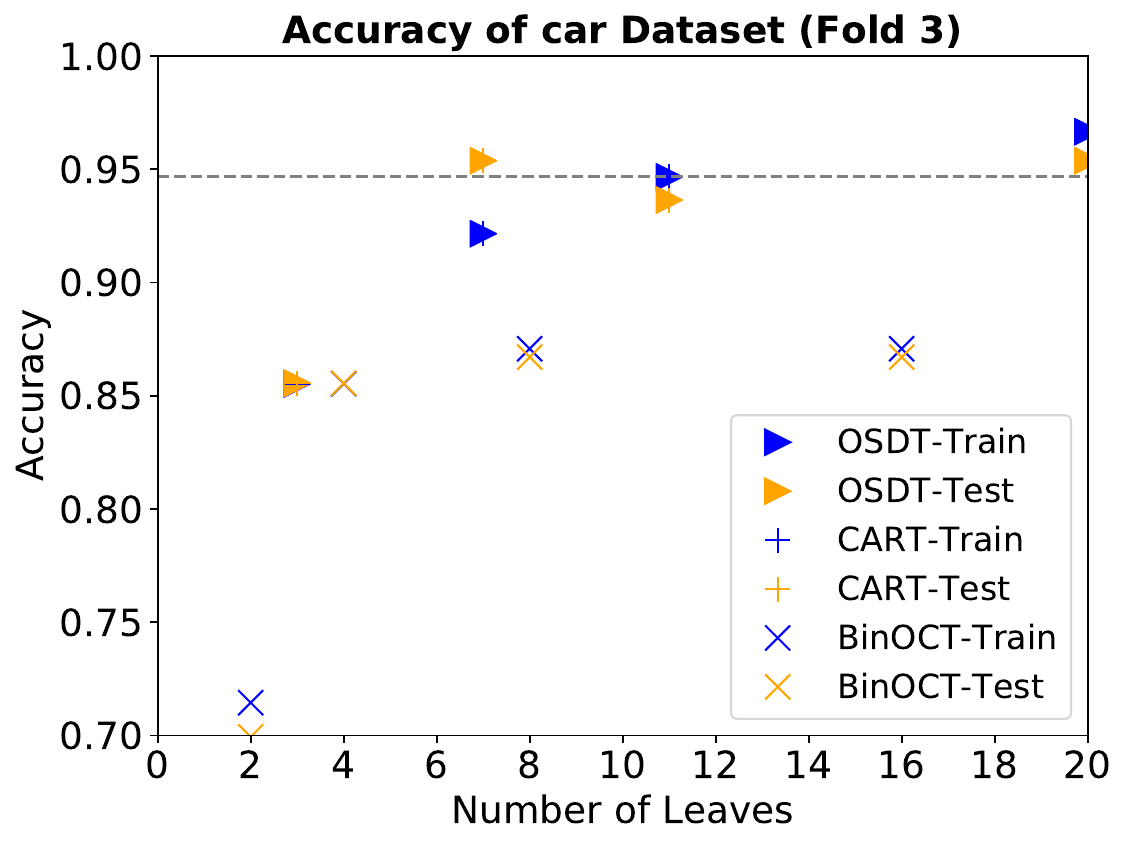}
        \label{fig:car3}
    \end{subfigure}
    \begin{subfigure}[b]{0.4\textwidth}
        \centering
        \includegraphics[trim={0mm 12mm 0mm 15mm}, width=\textwidth]{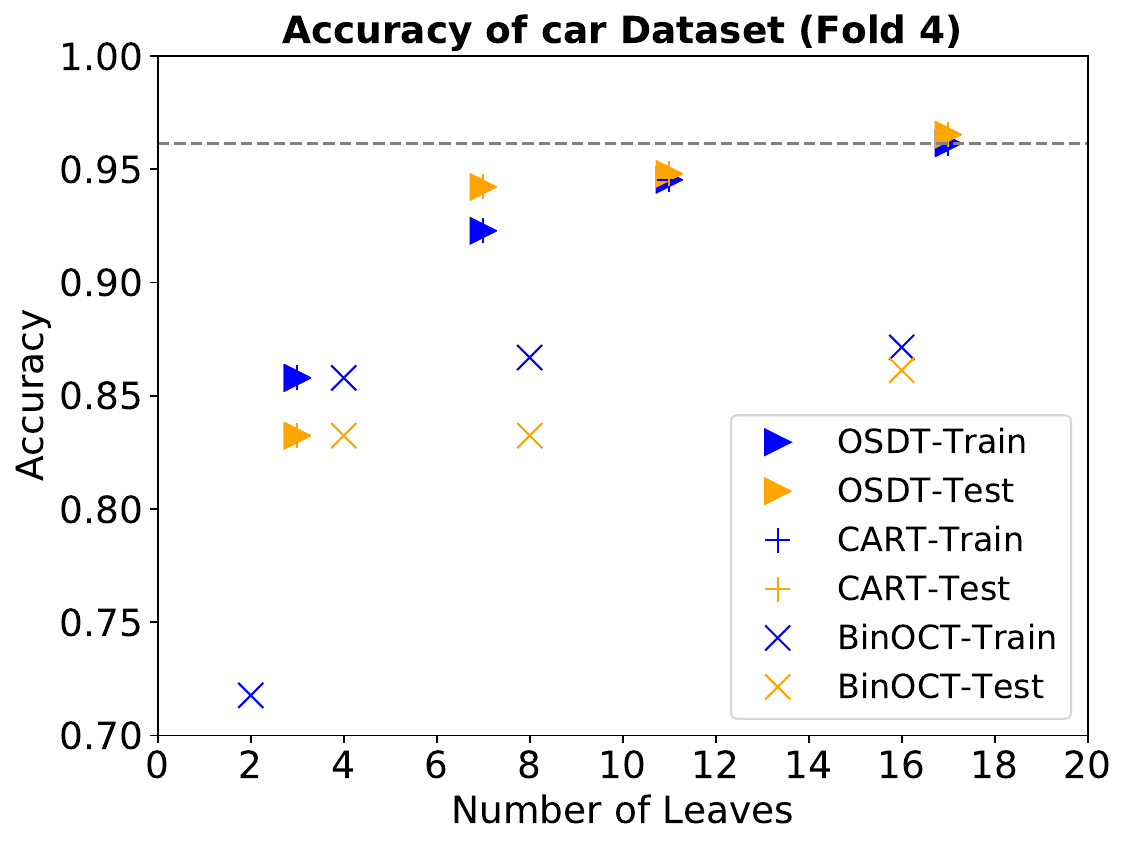}
        \label{fig:car4}
    \end{subfigure}
    \vskip\baselineskip
    \begin{subfigure}[b]{0.4\textwidth}
        \centering
        \includegraphics[trim={0mm 12mm 0mm 15mm}, width=\textwidth]{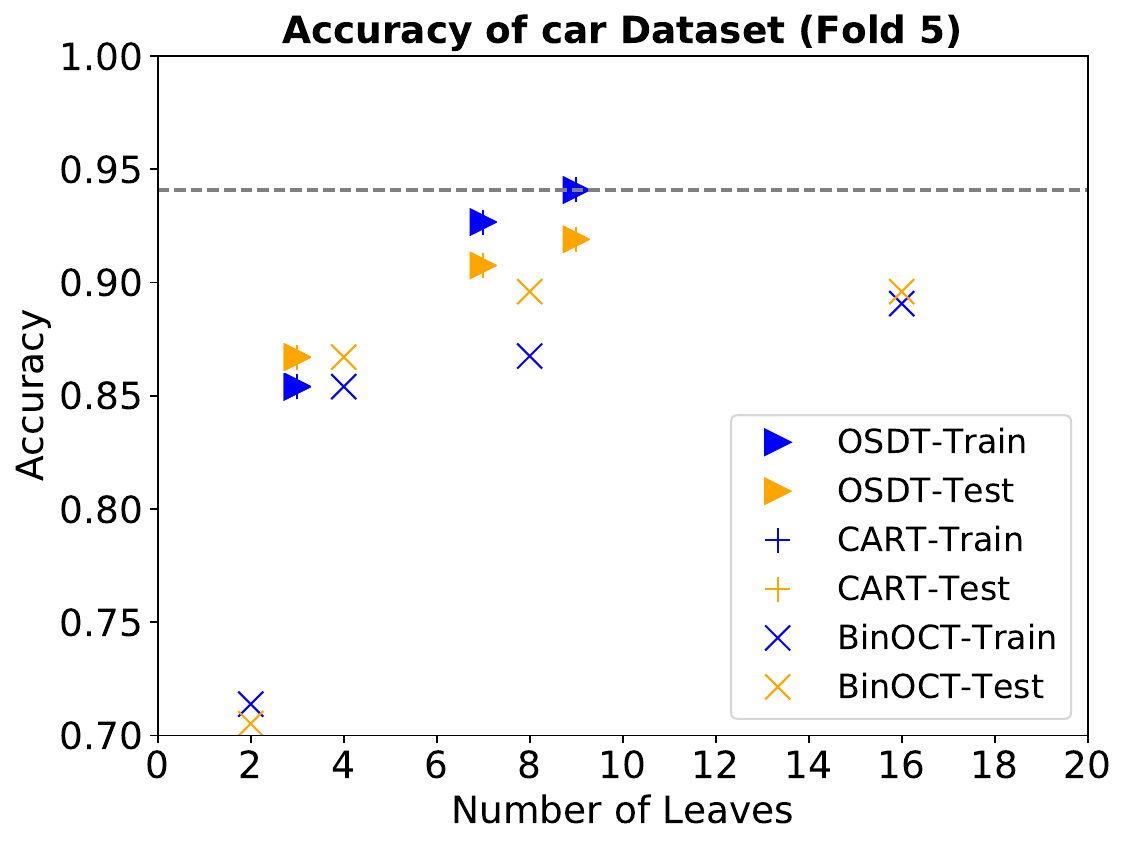}
        \label{fig:car5}
    \end{subfigure}
    \begin{subfigure}[b]{0.4\textwidth}
        \centering
        \includegraphics[trim={0mm 12mm 0mm 15mm}, width=\textwidth]{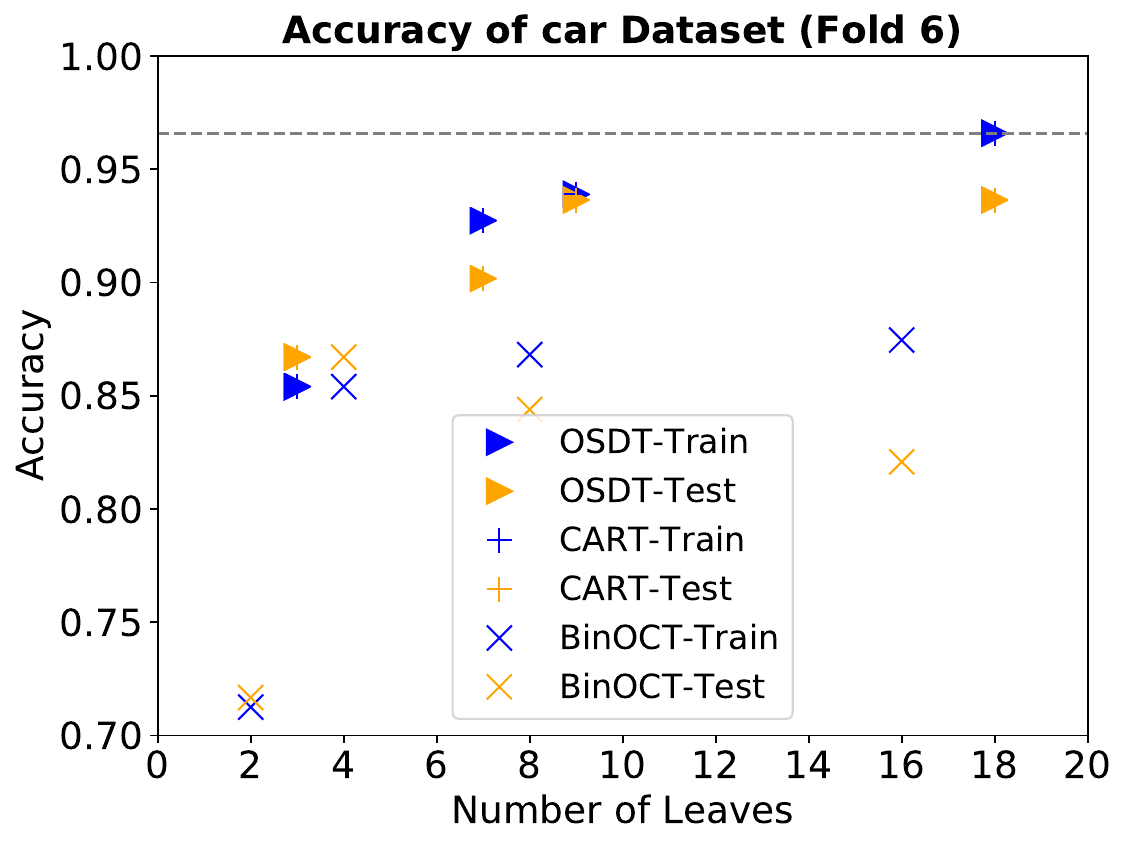}
        \label{fig:car6}
    \end{subfigure}
    \vskip\baselineskip
    \begin{subfigure}[b]{0.4\textwidth}
        \centering
        \includegraphics[trim={0mm 12mm 0mm 15mm}, width=\textwidth]{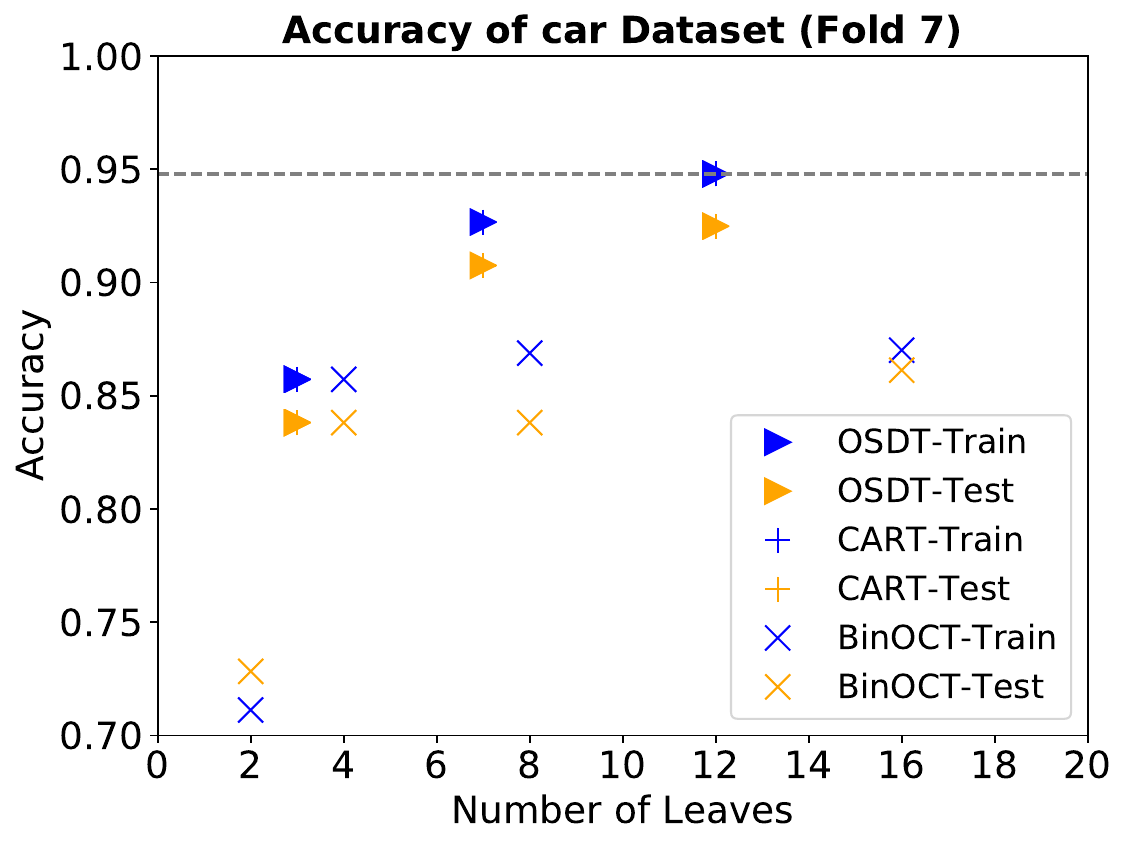}
        \label{fig:car7}
    \end{subfigure}
    \begin{subfigure}[b]{0.4\textwidth}
        \centering
        \includegraphics[trim={0mm 12mm 0mm 15mm}, width=\textwidth]{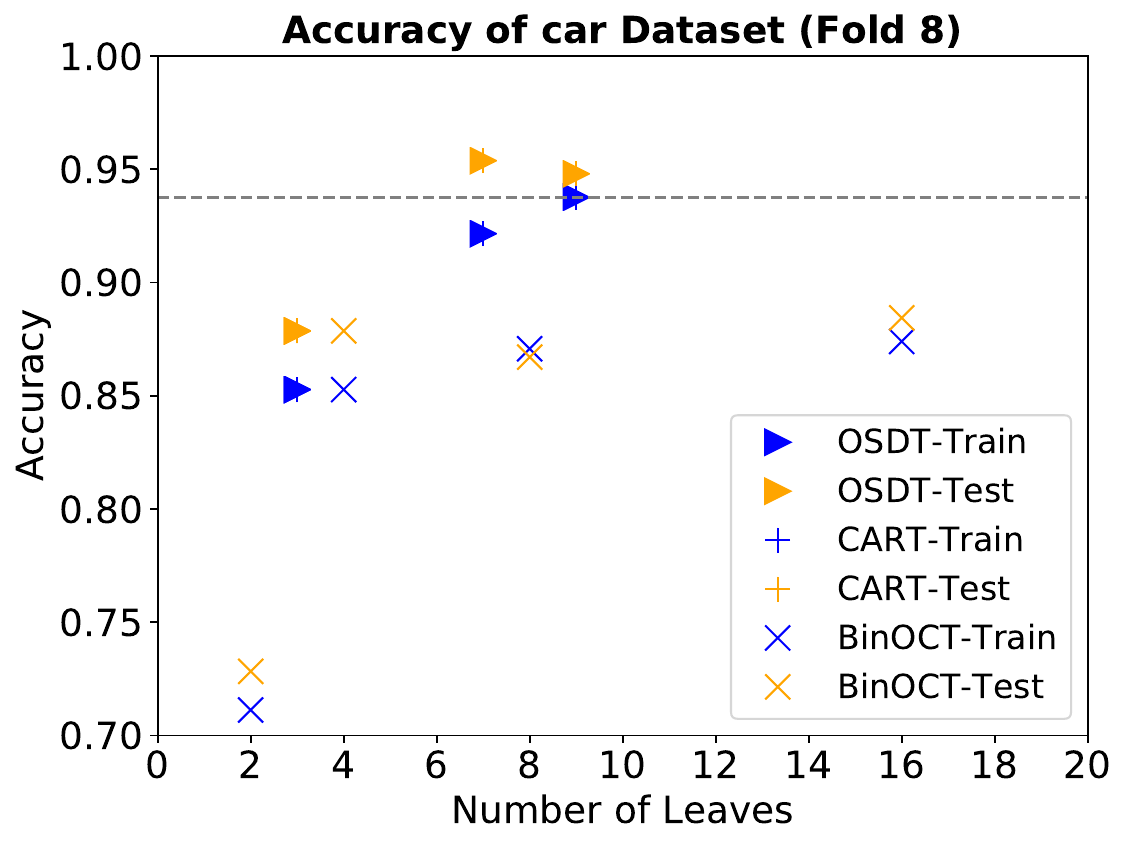}
        \label{fig:car8}
    \end{subfigure}
    \vskip\baselineskip
    \begin{subfigure}[b]{0.4\textwidth}
        \centering
        \includegraphics[trim={0mm 12mm 0mm 15mm}, width=\textwidth]{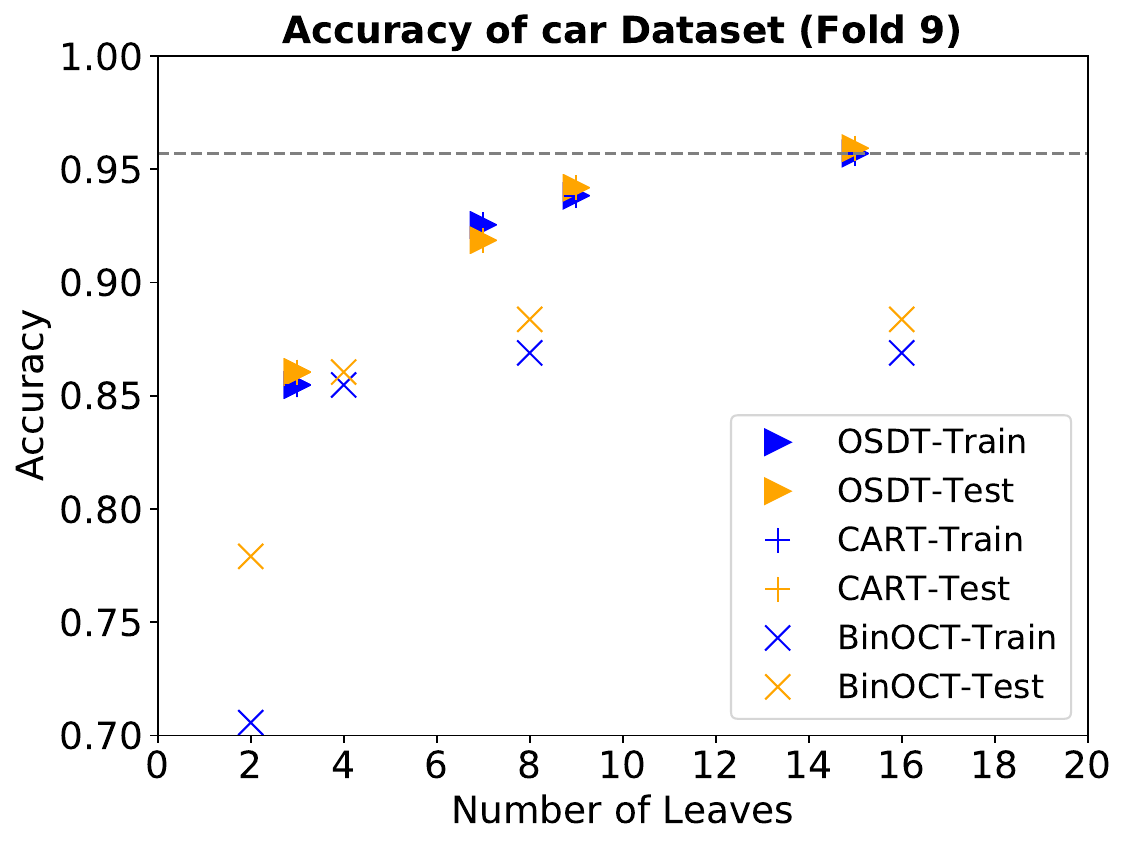}
        \label{fig:car9}
    \end{subfigure}
    \begin{subfigure}[b]{0.4\textwidth}
        \centering
        \includegraphics[trim={0mm 12mm 0mm 15mm}, width=\textwidth]{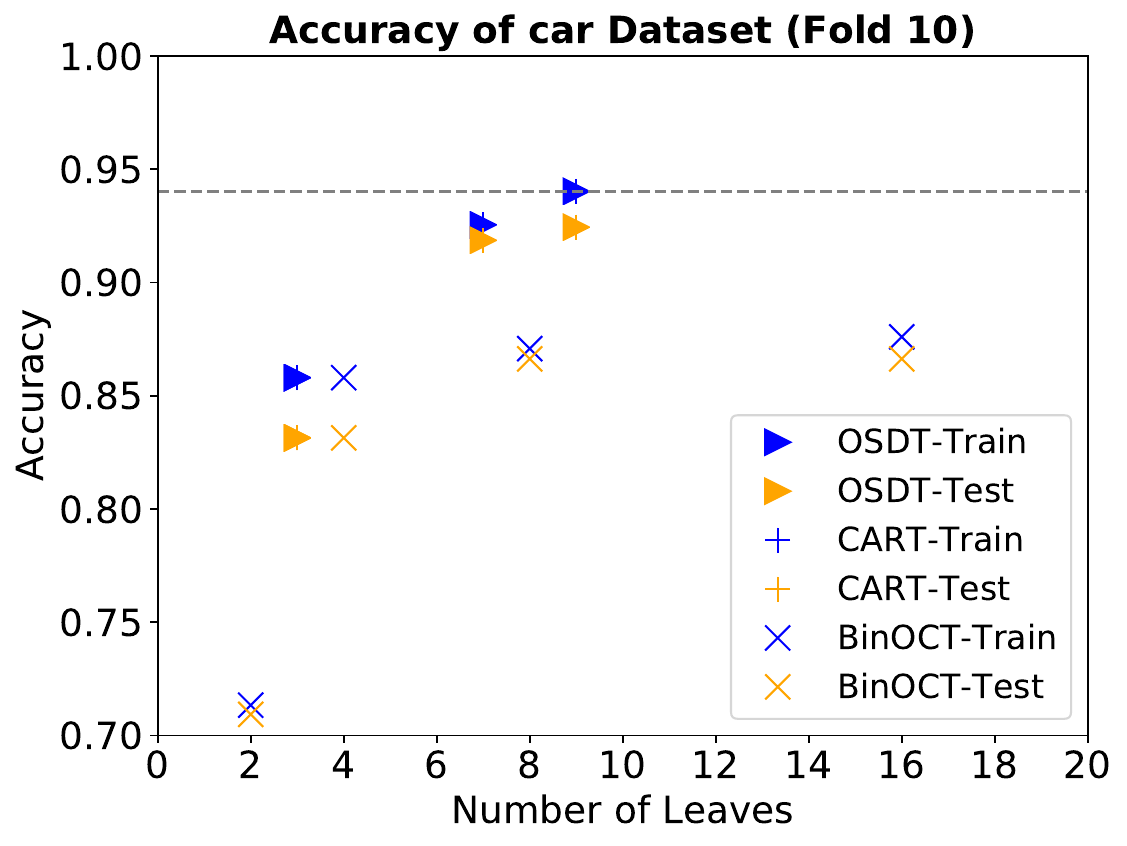}
        \label{fig:car10}
    \end{subfigure}
    \vskip\baselineskip
    \caption{10-fold cross-validation experiment results of OSDT, CART, BinOCT on car dataset. Horizontal lines indicate the accuracy of the best OSDT tree in training. %
    }
    \label{fig:cv-car}
\end{figure*}

\begin{figure*}[t!]
    \centering
    \begin{subfigure}[b]{0.4\textwidth}
        \centering
        \includegraphics[trim={0mm 12mm 0mm 15mm}, width=\textwidth]{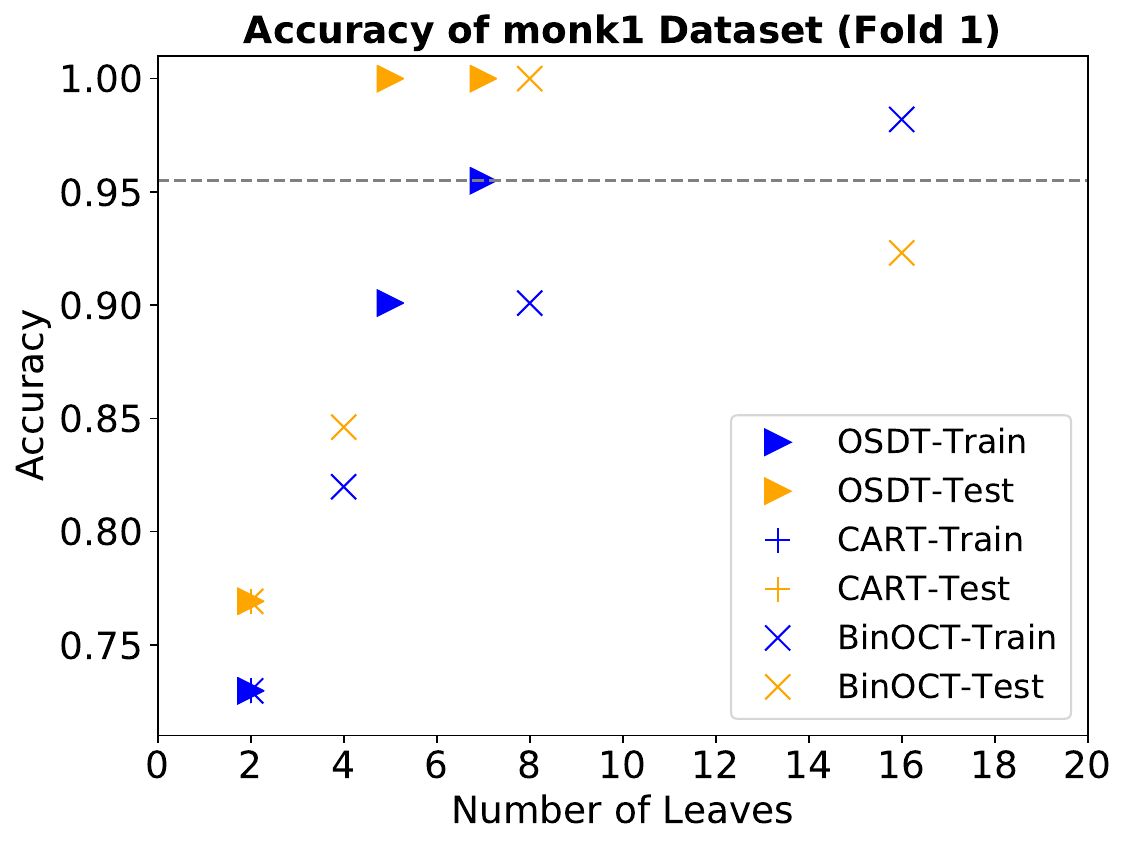}
        \label{fig:monk11}
    \end{subfigure}
    \begin{subfigure}[b]{0.4\textwidth}
        \centering
        \includegraphics[trim={0mm 12mm 0mm 15mm}, width=\textwidth]{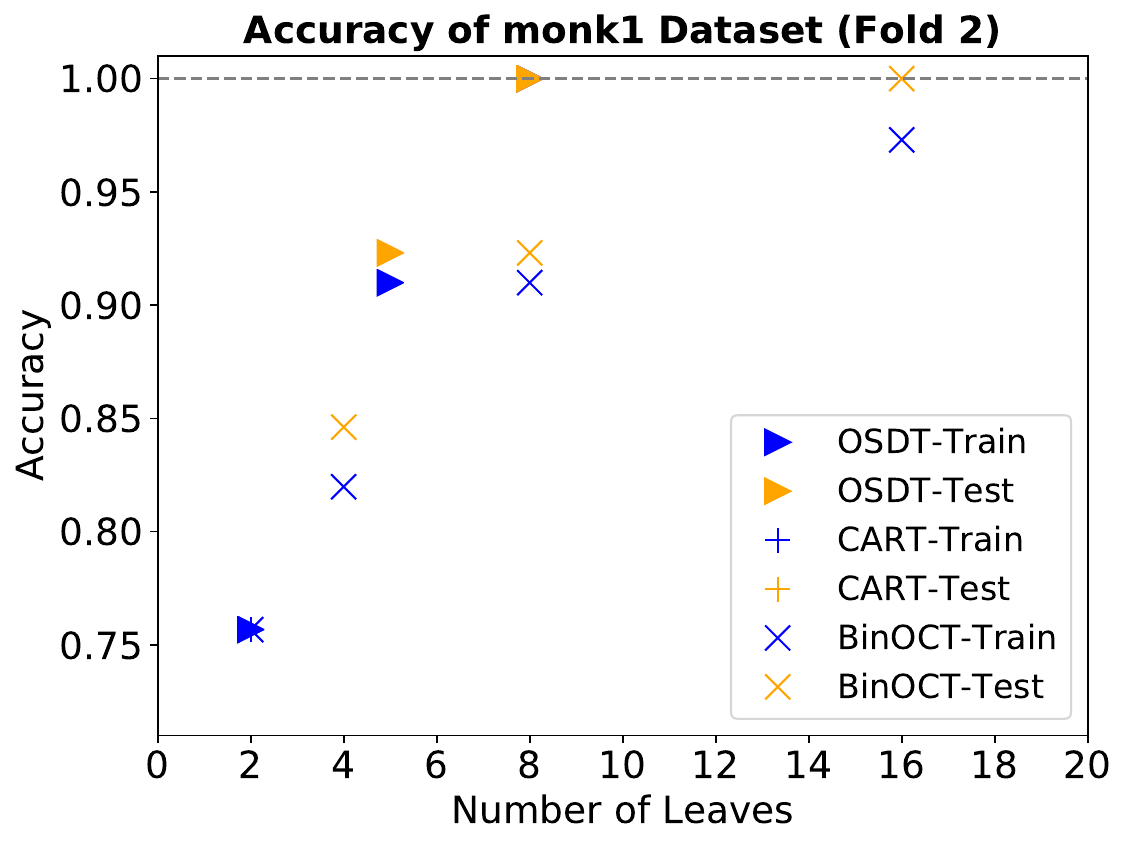}
        \label{fig:monk12}
    \end{subfigure}
    \vskip\baselineskip
    \begin{subfigure}[b]{0.4\textwidth}
        \centering
        \includegraphics[trim={0mm 12mm 0mm 15mm}, width=\textwidth]{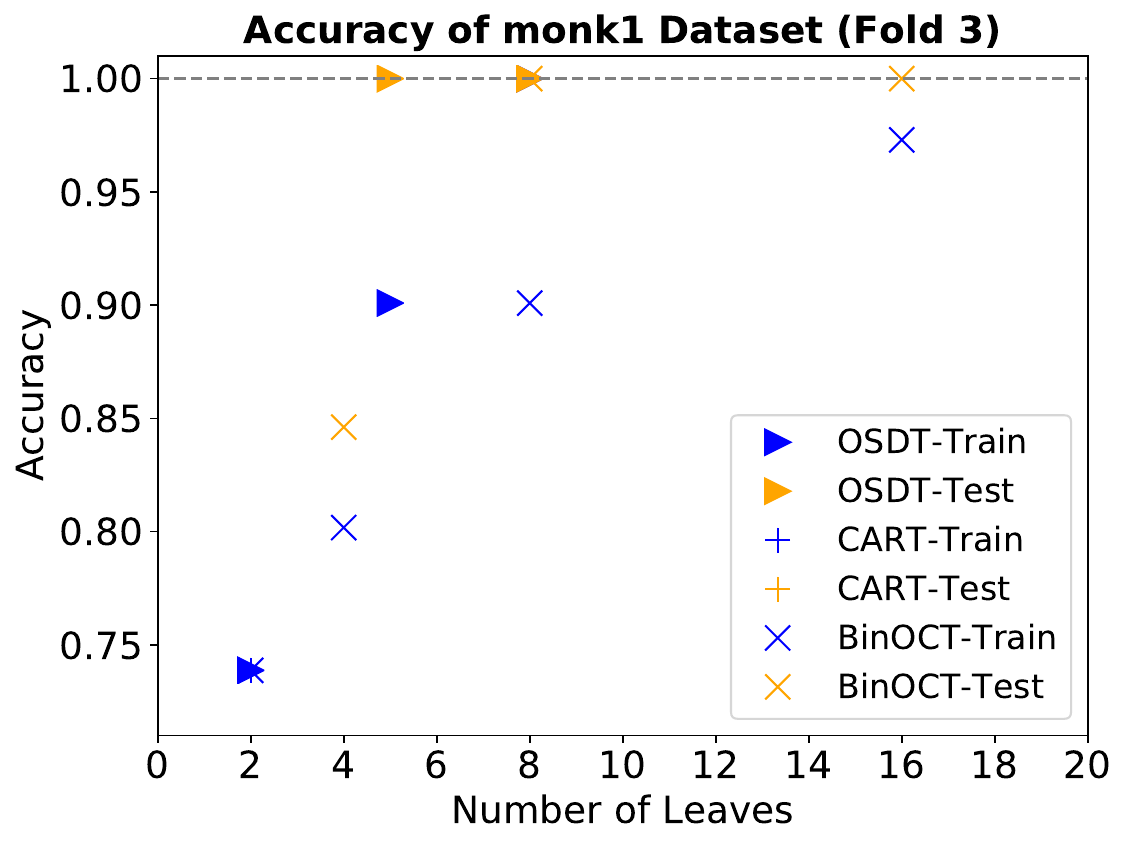}
        \label{fig:monk13}
    \end{subfigure}
    \begin{subfigure}[b]{0.4\textwidth}
        \centering
        \includegraphics[trim={0mm 12mm 0mm 15mm}, width=\textwidth]{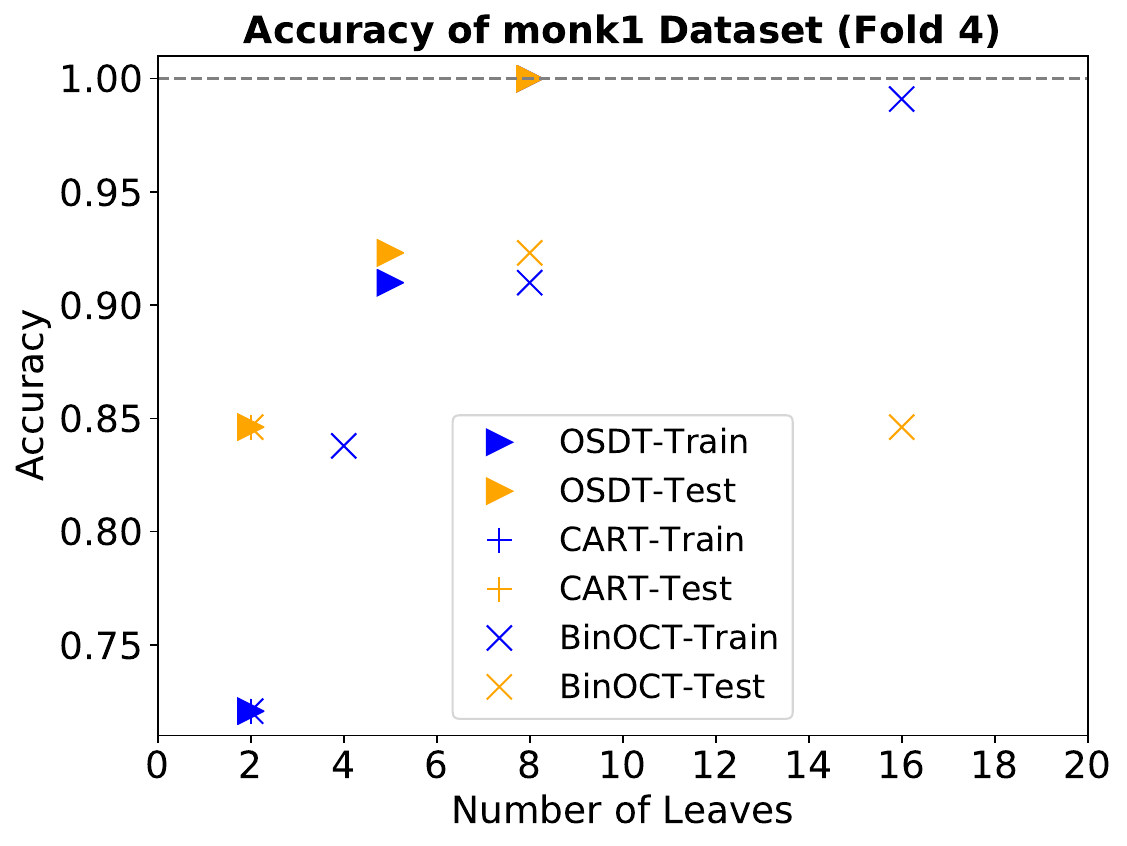}
        \label{fig:monk14}
    \end{subfigure}
    \vskip\baselineskip
    \begin{subfigure}[b]{0.4\textwidth}
        \centering
        \includegraphics[trim={0mm 12mm 0mm 15mm}, width=\textwidth]{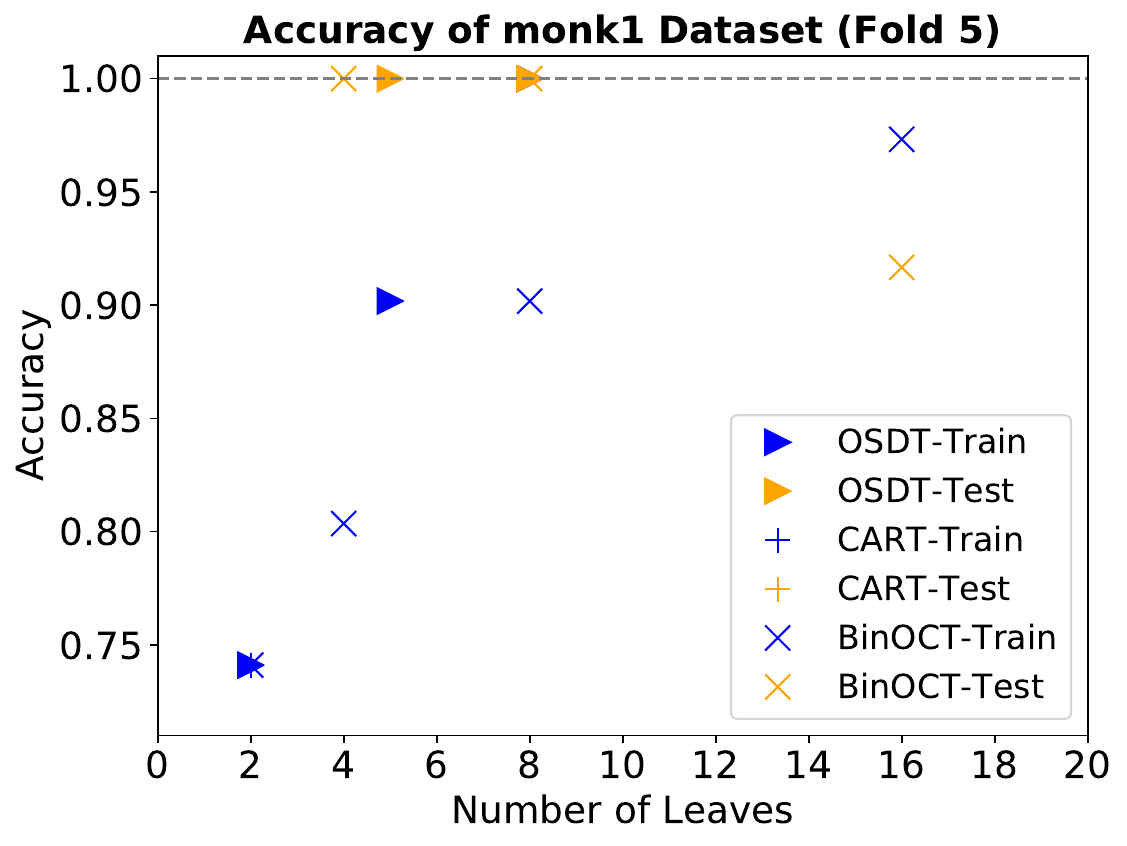}
        \label{fig:monk15}
    \end{subfigure}
    \begin{subfigure}[b]{0.4\textwidth}
        \centering
        \includegraphics[trim={0mm 12mm 0mm 15mm}, width=\textwidth]{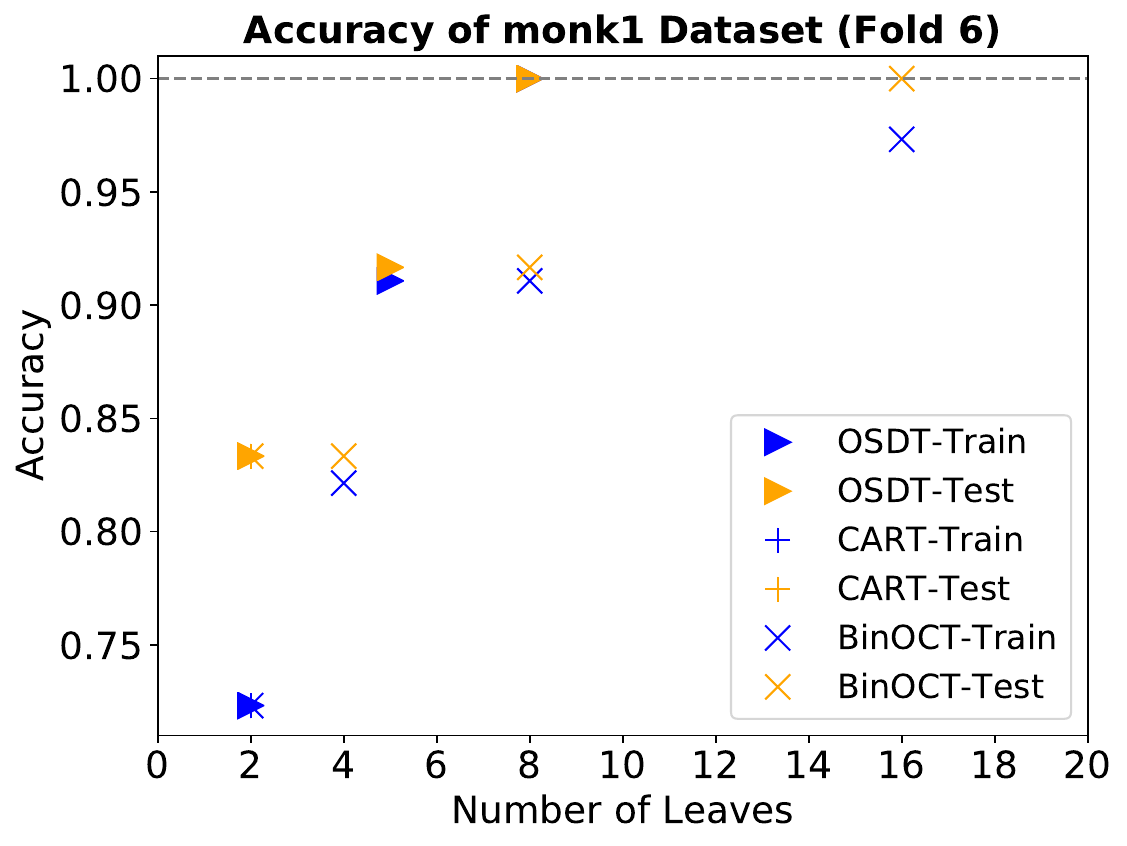}
        \label{fig:monk16}
    \end{subfigure}
    \vskip\baselineskip
    \begin{subfigure}[b]{0.4\textwidth}
        \centering
        \includegraphics[trim={0mm 12mm 0mm 15mm}, width=\textwidth]{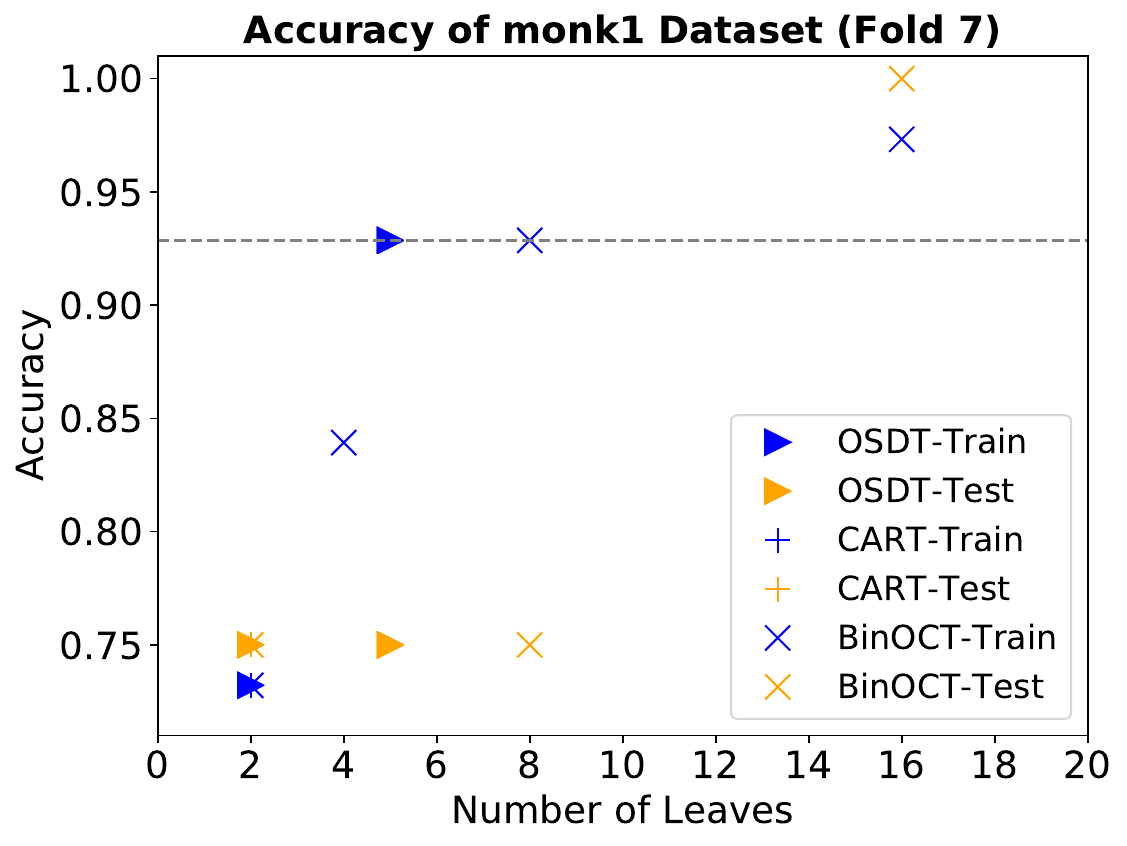}
        \label{fig:monk17}
    \end{subfigure}
    \begin{subfigure}[b]{0.4\textwidth}
        \centering
        \includegraphics[trim={0mm 12mm 0mm 15mm}, width=\textwidth]{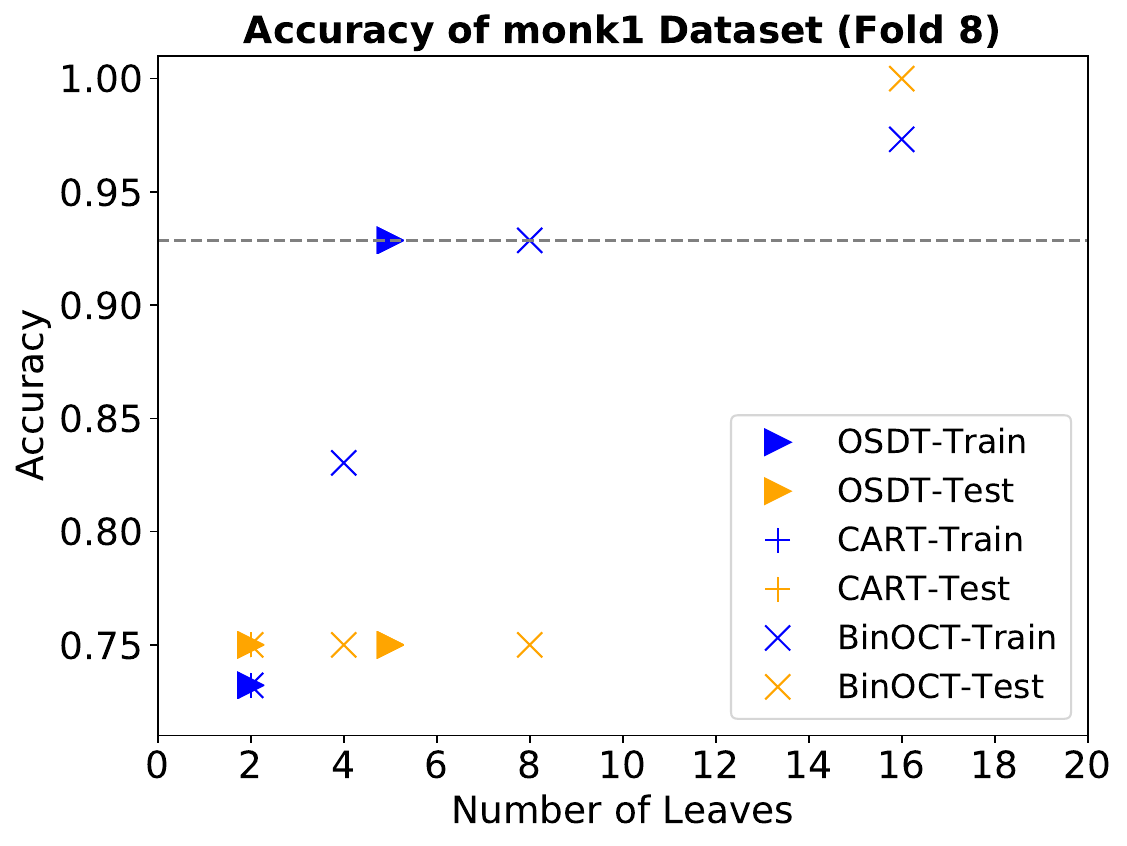}
        \label{fig:monk18}
    \end{subfigure}
    \vskip\baselineskip
    \begin{subfigure}[b]{0.4\textwidth}
        \centering
        \includegraphics[trim={0mm 12mm 0mm 15mm}, width=\textwidth]{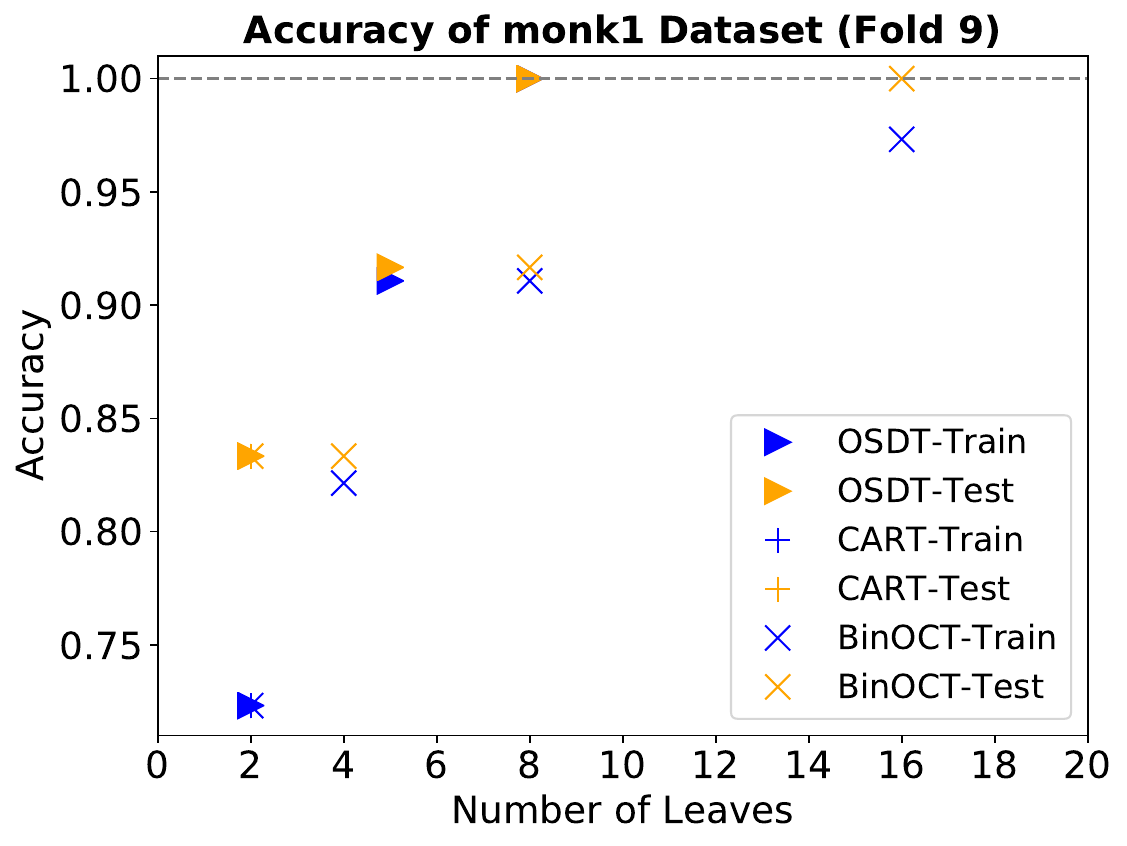}
        \label{fig:monk19}
    \end{subfigure}
    \begin{subfigure}[b]{0.4\textwidth}
        \centering
        \includegraphics[trim={0mm 12mm 0mm 15mm}, width=\textwidth]{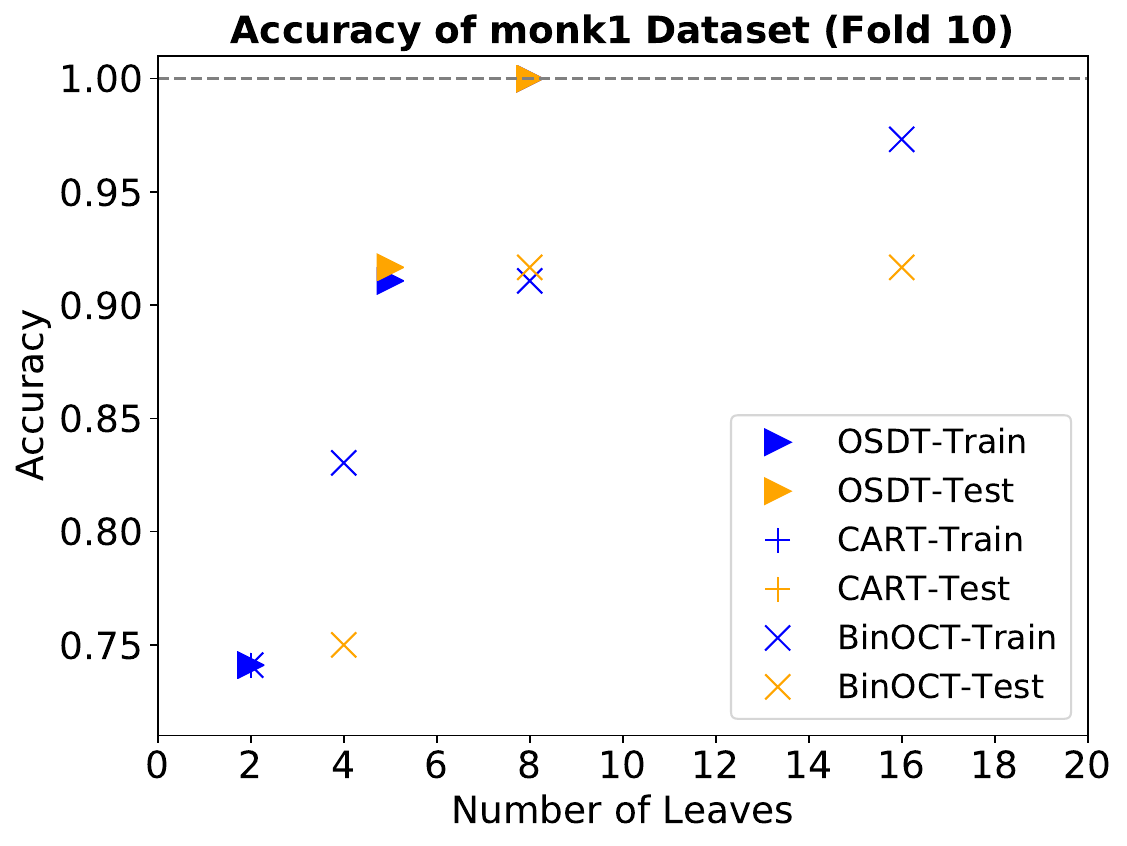}
        \label{fig:monk110}
    \end{subfigure}
    \vskip\baselineskip
    \caption{10-fold cross-validation experiment results of OSDT, CART, BinOCT on Monk1 dataset. Horizontal lines indicate the accuracy of the best OSDT tree in training. %
    }
    \label{fig:cv-monk1}
\end{figure*}

\begin{figure*}[t!]
    \centering
    \begin{subfigure}[b]{0.4\textwidth}
        \centering
        \includegraphics[trim={0mm 12mm 0mm 15mm}, width=\textwidth]{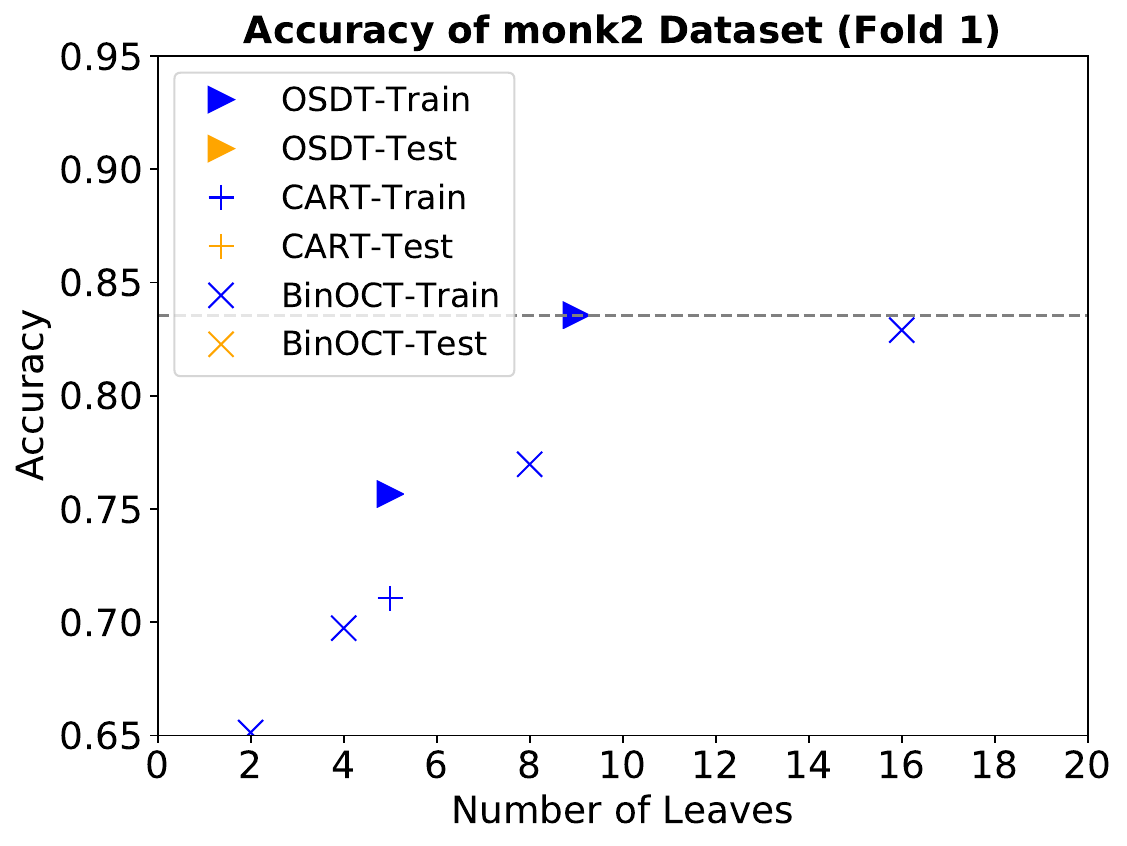}
        \label{fig:monk21}
    \end{subfigure}
    \begin{subfigure}[b]{0.4\textwidth}
        \centering
        \includegraphics[trim={0mm 12mm 0mm 15mm}, width=\textwidth]{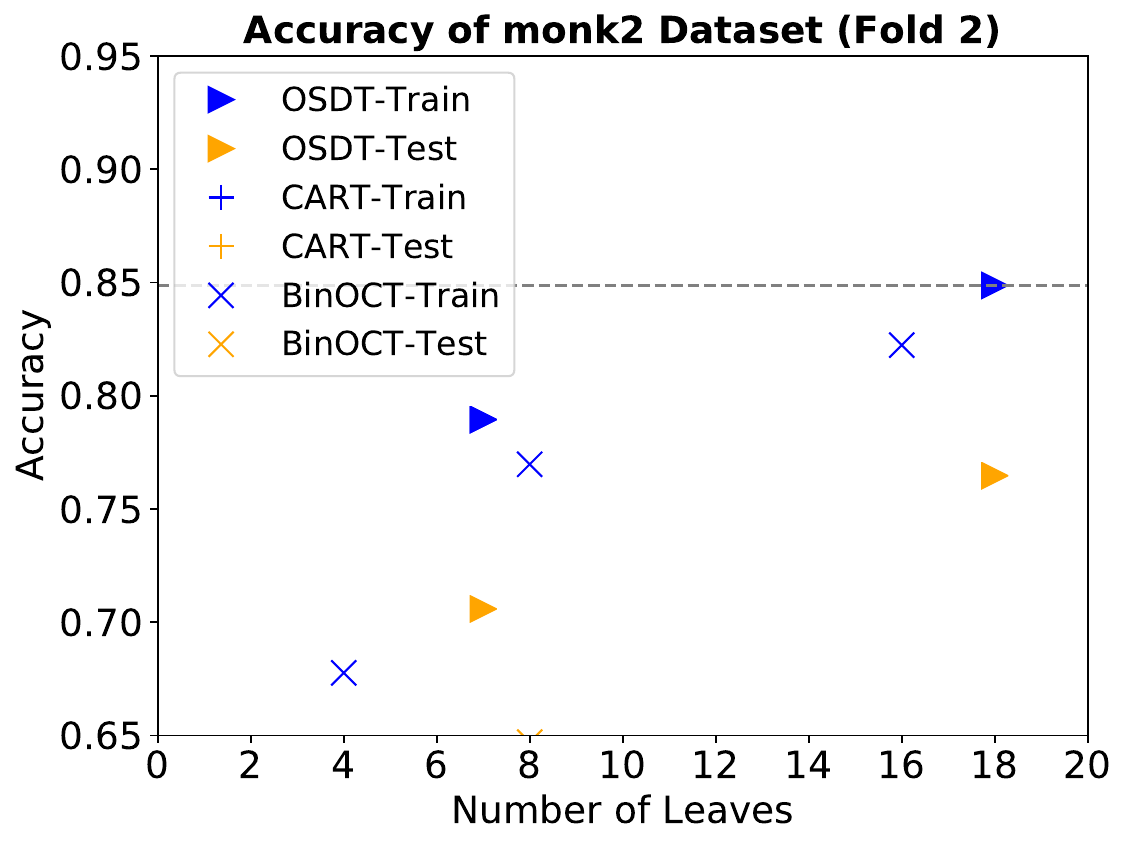}
        \label{fig:monk22}
    \end{subfigure}
    \vskip\baselineskip
    \begin{subfigure}[b]{0.4\textwidth}
        \centering
        \includegraphics[trim={0mm 12mm 0mm 15mm}, width=\textwidth]{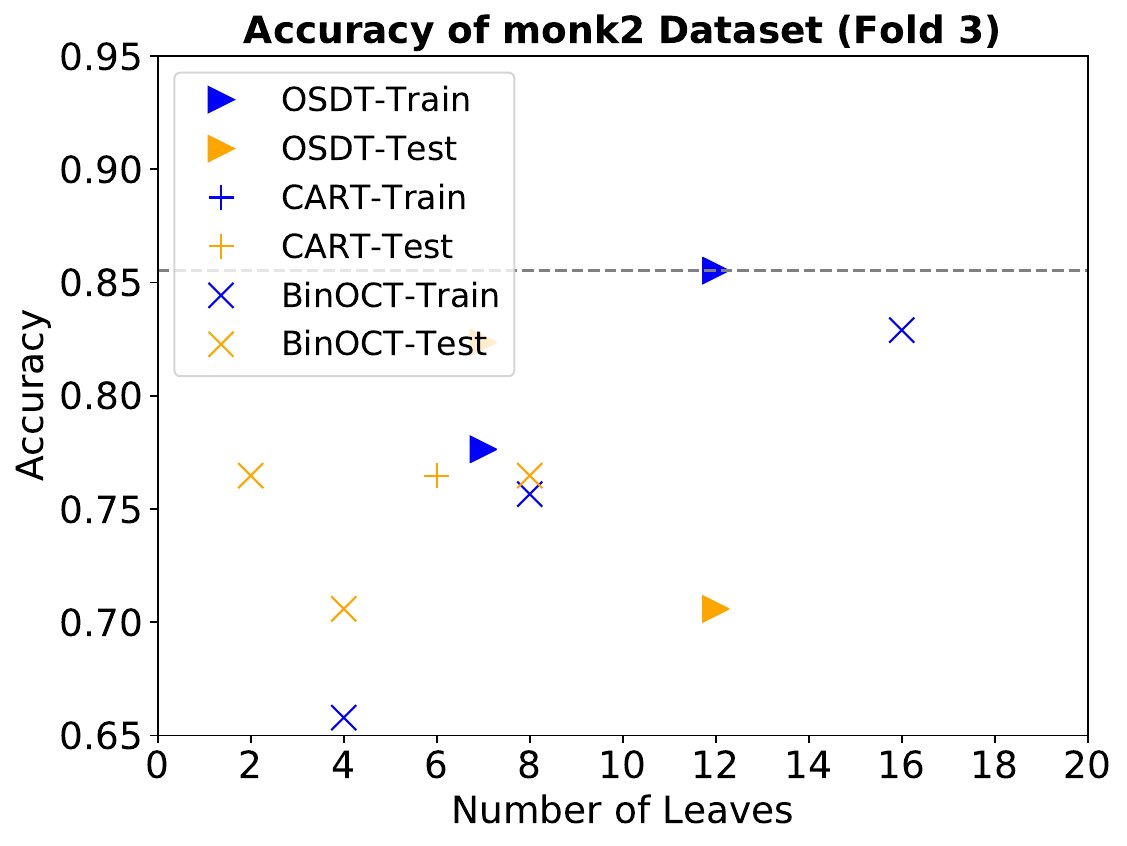}
        \label{fig:monk23}
    \end{subfigure}
    \begin{subfigure}[b]{0.4\textwidth}
        \centering
        \includegraphics[trim={0mm 12mm 0mm 15mm}, width=\textwidth]{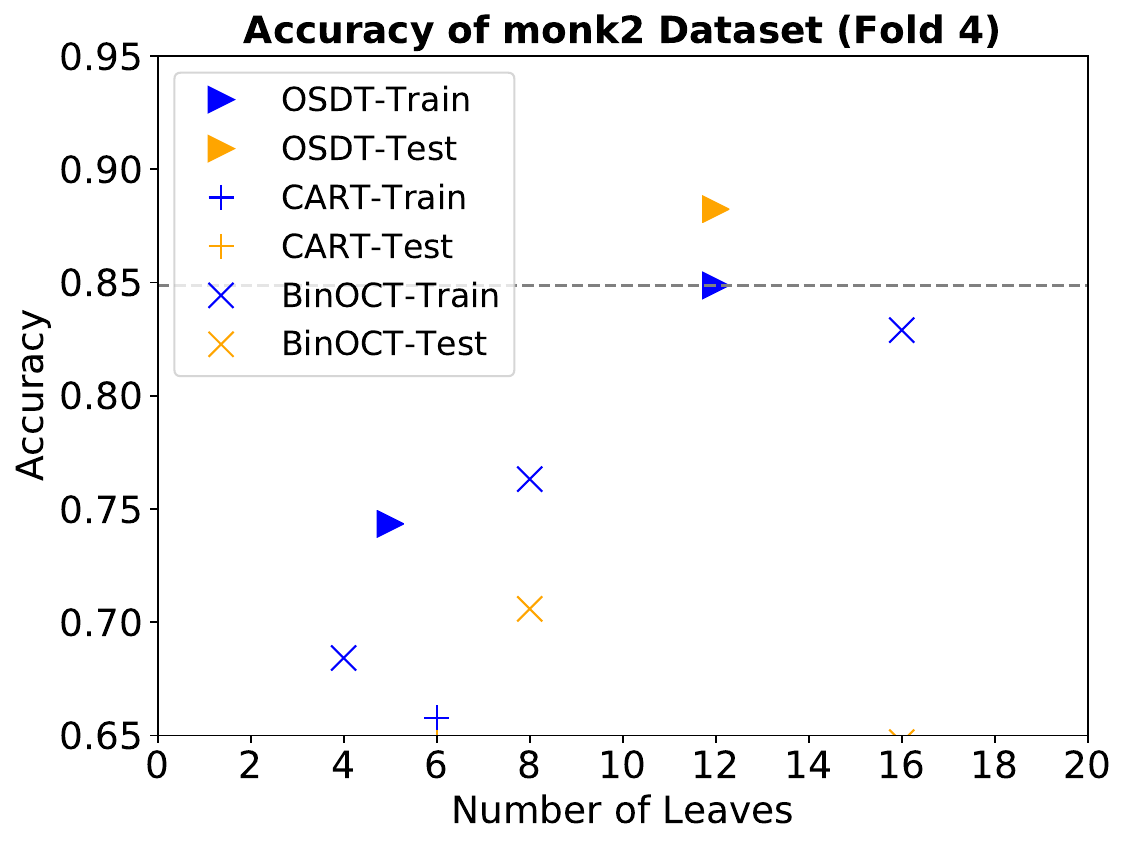}
        \label{fig:monk24}
    \end{subfigure}
    \vskip\baselineskip
    \begin{subfigure}[b]{0.4\textwidth}
        \centering
        \includegraphics[trim={0mm 12mm 0mm 15mm}, width=\textwidth]{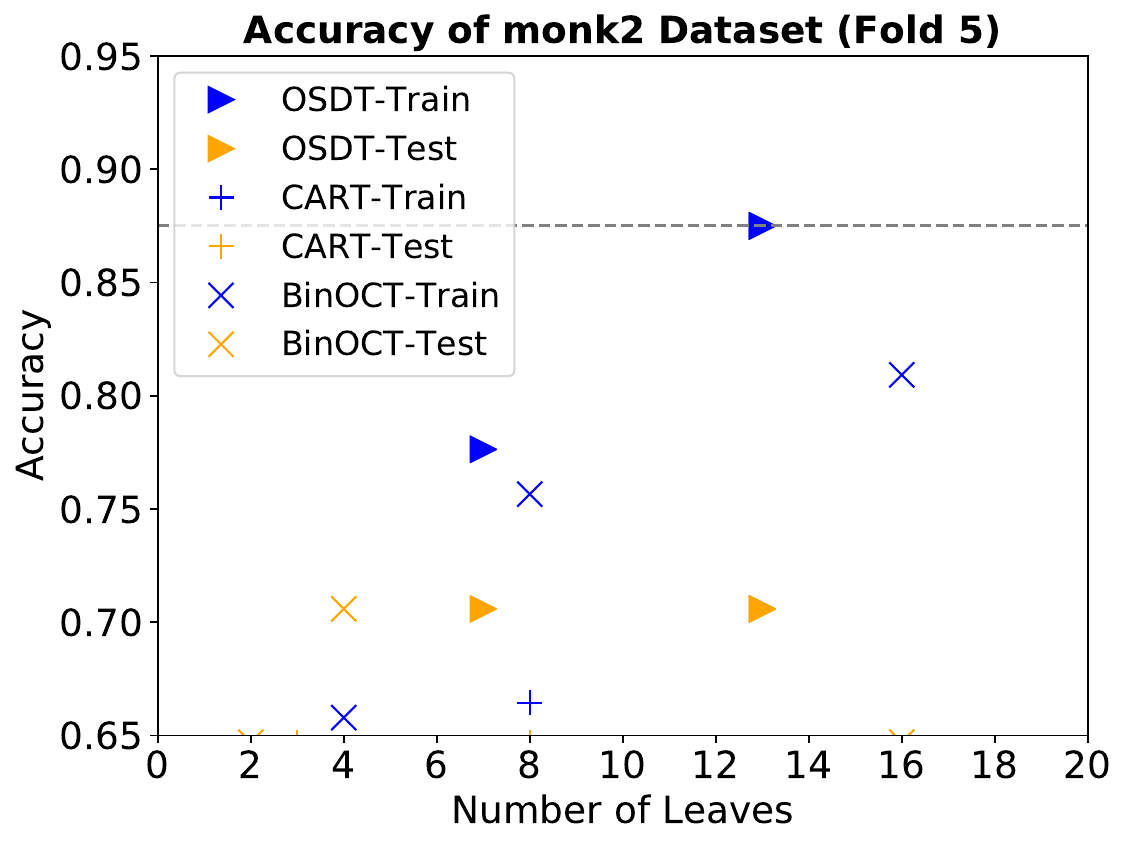}
        \label{fig:monk25}
    \end{subfigure}
    \begin{subfigure}[b]{0.4\textwidth}
        \centering
        \includegraphics[trim={0mm 12mm 0mm 15mm}, width=\textwidth]{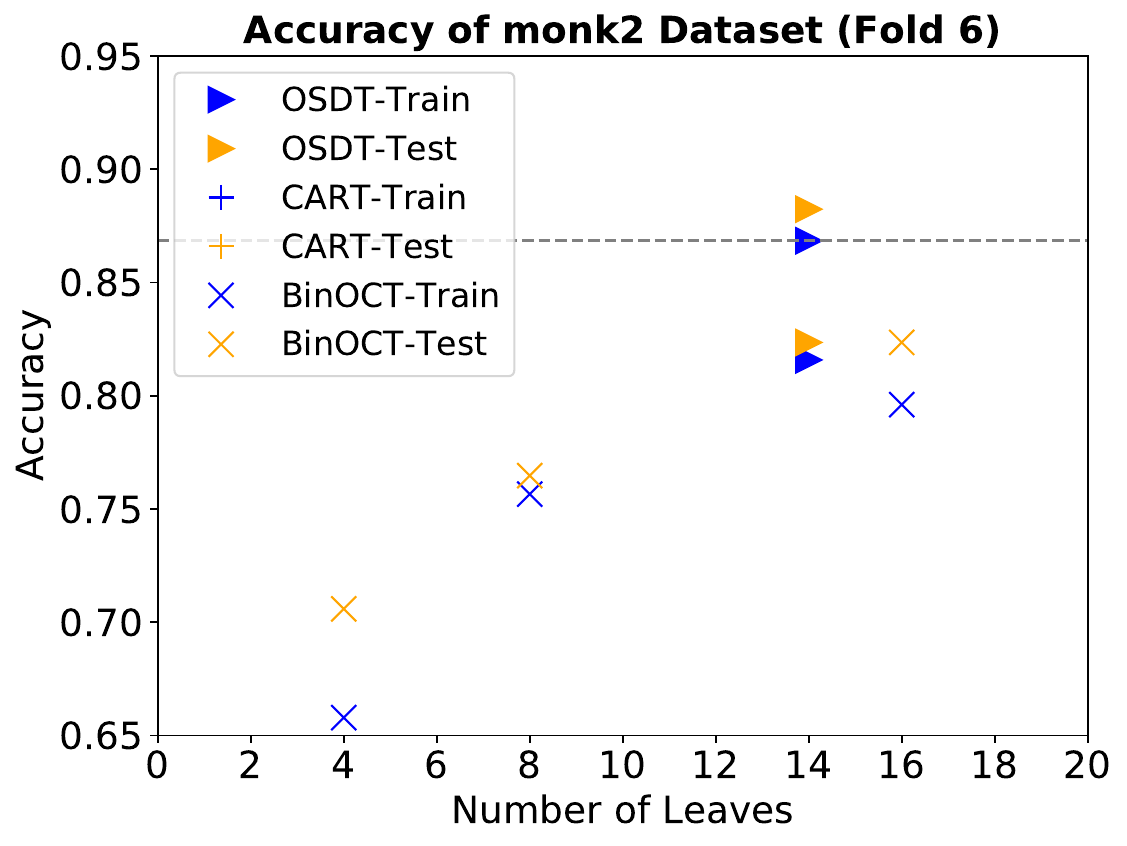}
        \label{fig:monk26}
    \end{subfigure}
    \vskip\baselineskip
    \begin{subfigure}[b]{0.4\textwidth}
        \centering
        \includegraphics[trim={0mm 12mm 0mm 15mm}, width=\textwidth]{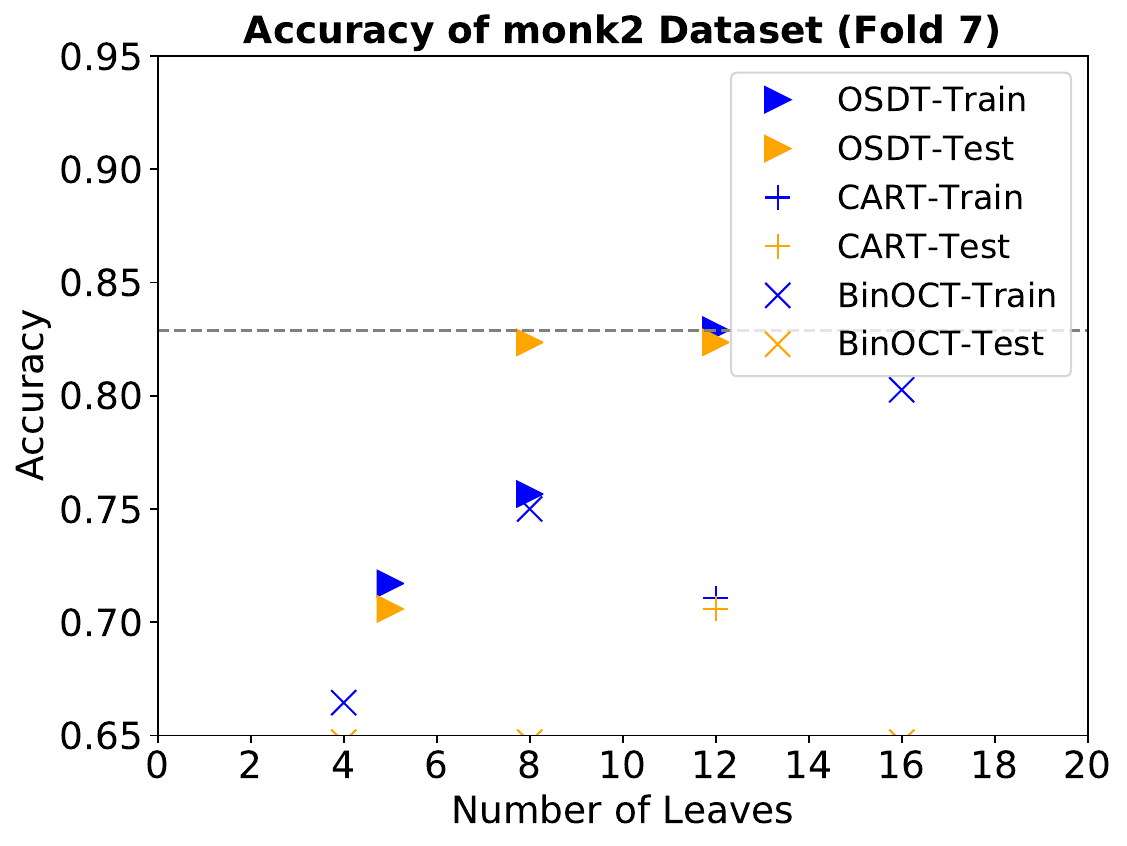}
        \label{fig:monk27}
    \end{subfigure}
    \begin{subfigure}[b]{0.4\textwidth}
        \centering
        \includegraphics[trim={0mm 12mm 0mm 15mm}, width=\textwidth]{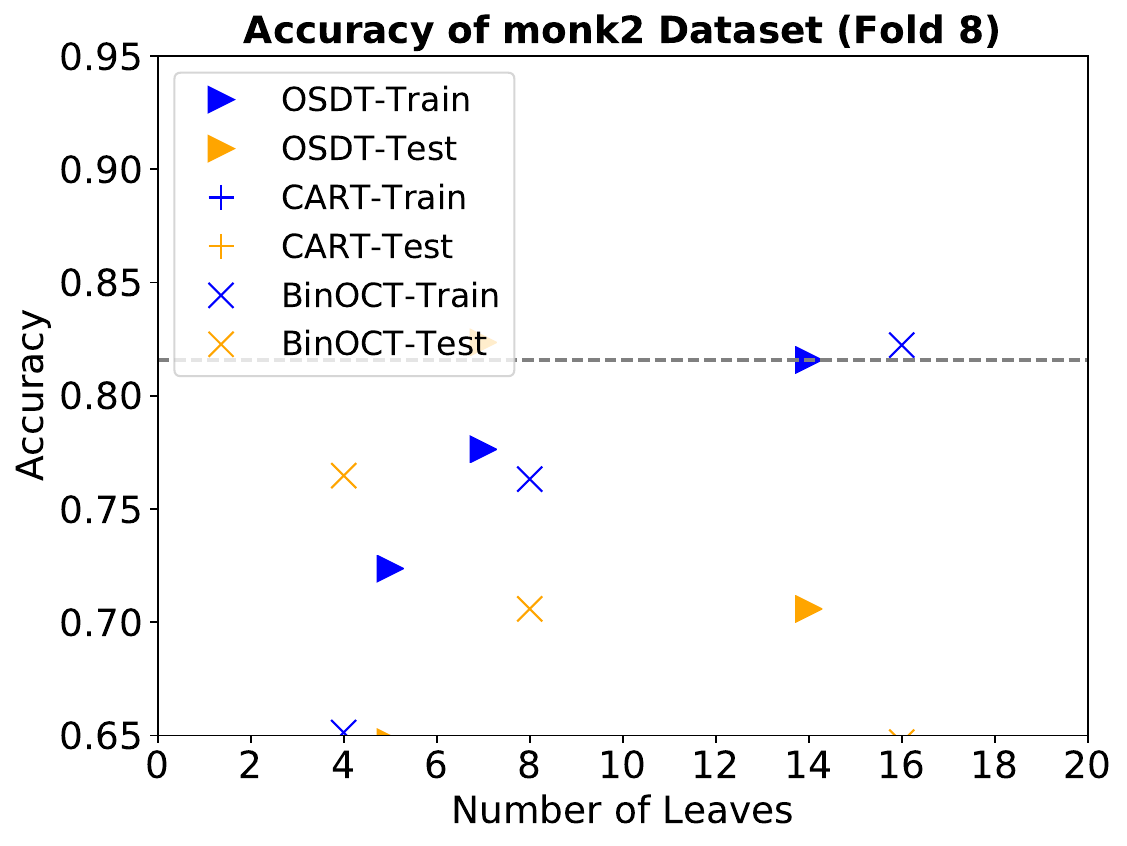}
        \label{fig:monk28}
    \end{subfigure}
    \vskip\baselineskip
    \begin{subfigure}[b]{0.4\textwidth}
        \centering
        \includegraphics[trim={0mm 12mm 0mm 15mm}, width=\textwidth]{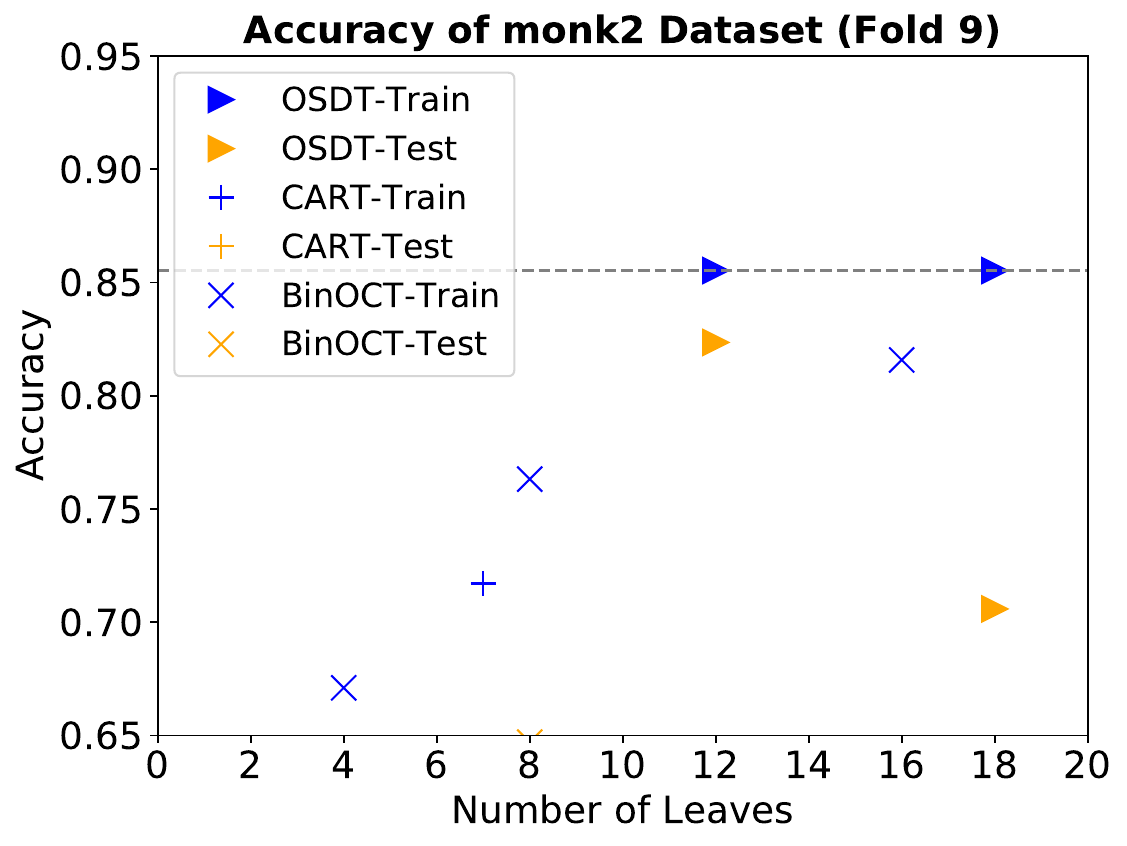}
        \label{fig:monk29}
    \end{subfigure}
    \begin{subfigure}[b]{0.4\textwidth}
        \centering
        \includegraphics[trim={0mm 12mm 0mm 15mm}, width=\textwidth]{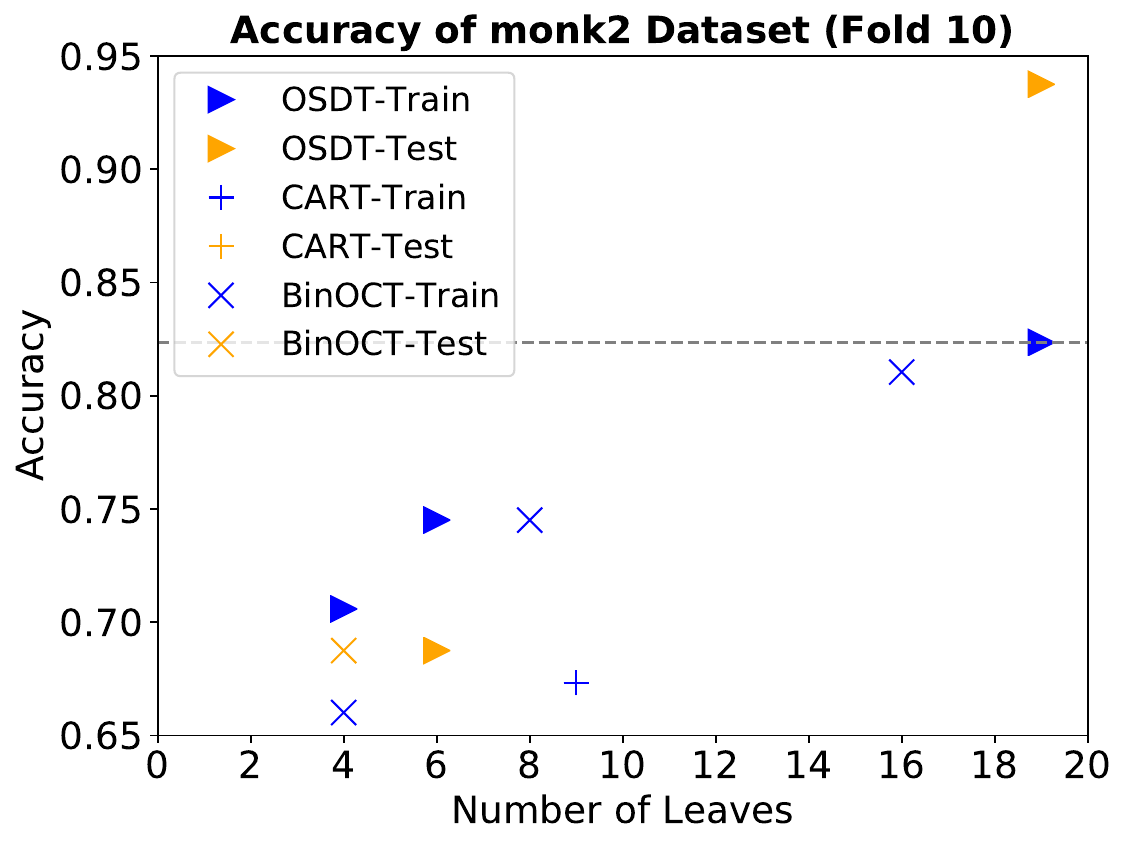}
        \label{fig:monk210}
    \end{subfigure}
    \vskip\baselineskip
    \caption{10-fold cross-validation experiment results of OSDT, CART, BinOCT on Monk2 dataset. Horizontal lines indicate the accuracy of the best OSDT tree in training. %
    }
    \label{fig:cv-monk2}
\end{figure*}

\begin{figure*}[t!]
    \centering
    \begin{subfigure}[b]{0.4\textwidth}
        \centering
        \includegraphics[trim={0mm 12mm 0mm 15mm}, width=\textwidth]{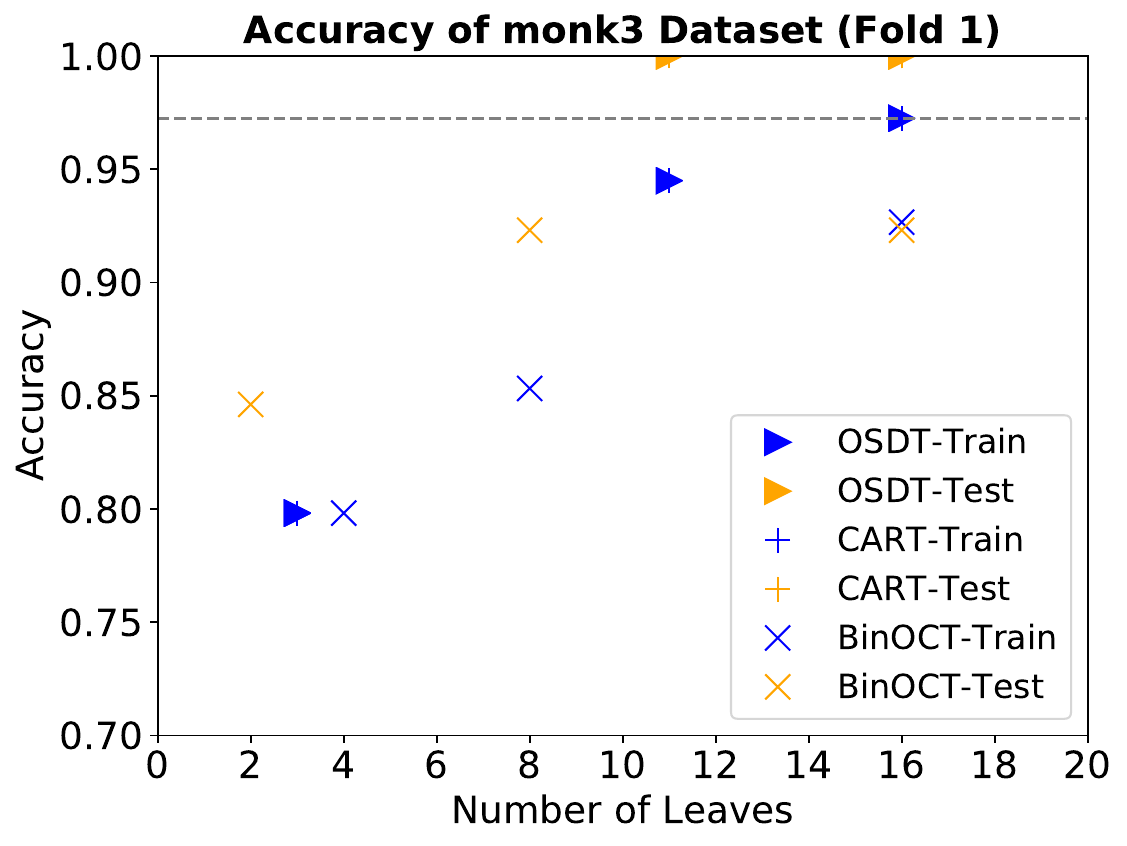}
        \label{fig:monk31}
    \end{subfigure}
    \begin{subfigure}[b]{0.4\textwidth}
        \centering
        \includegraphics[trim={0mm 12mm 0mm 15mm}, width=\textwidth]{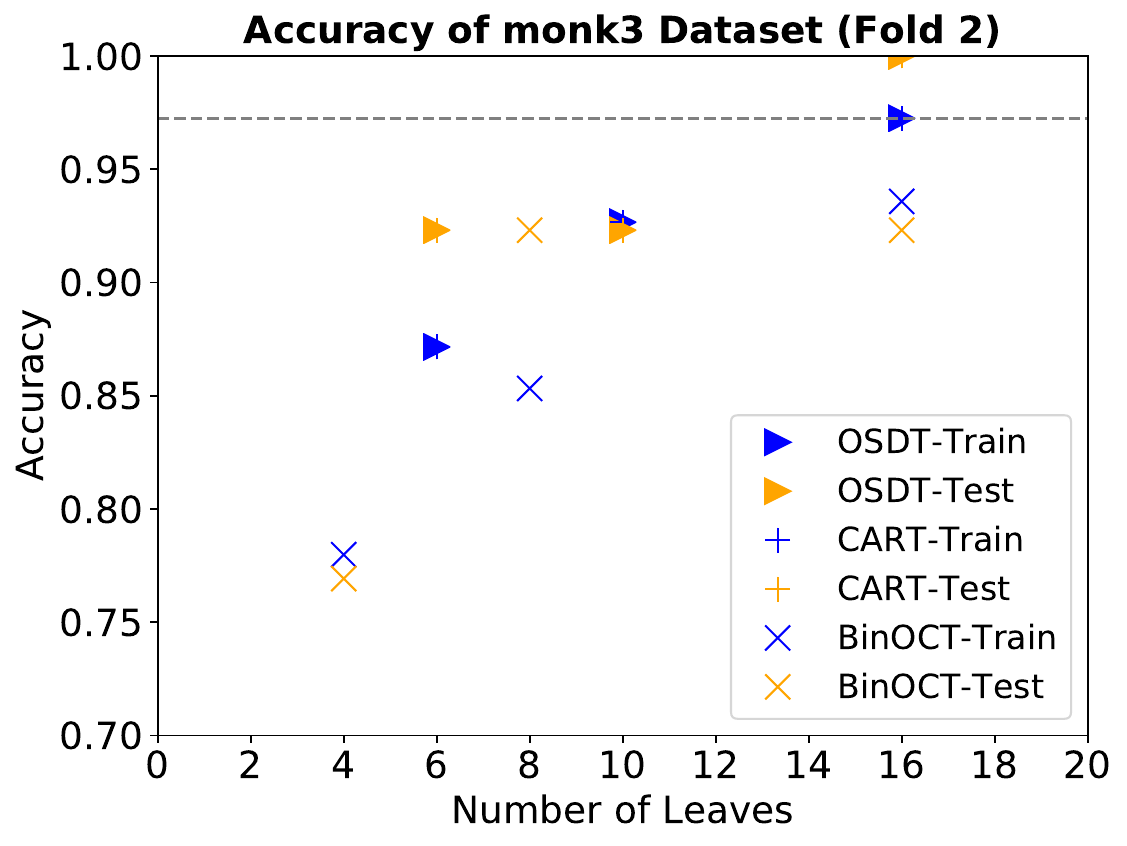}
        \label{fig:monk32}
    \end{subfigure}
    \vskip\baselineskip
    \begin{subfigure}[b]{0.4\textwidth}
        \centering
        \includegraphics[trim={0mm 12mm 0mm 15mm}, width=\textwidth]{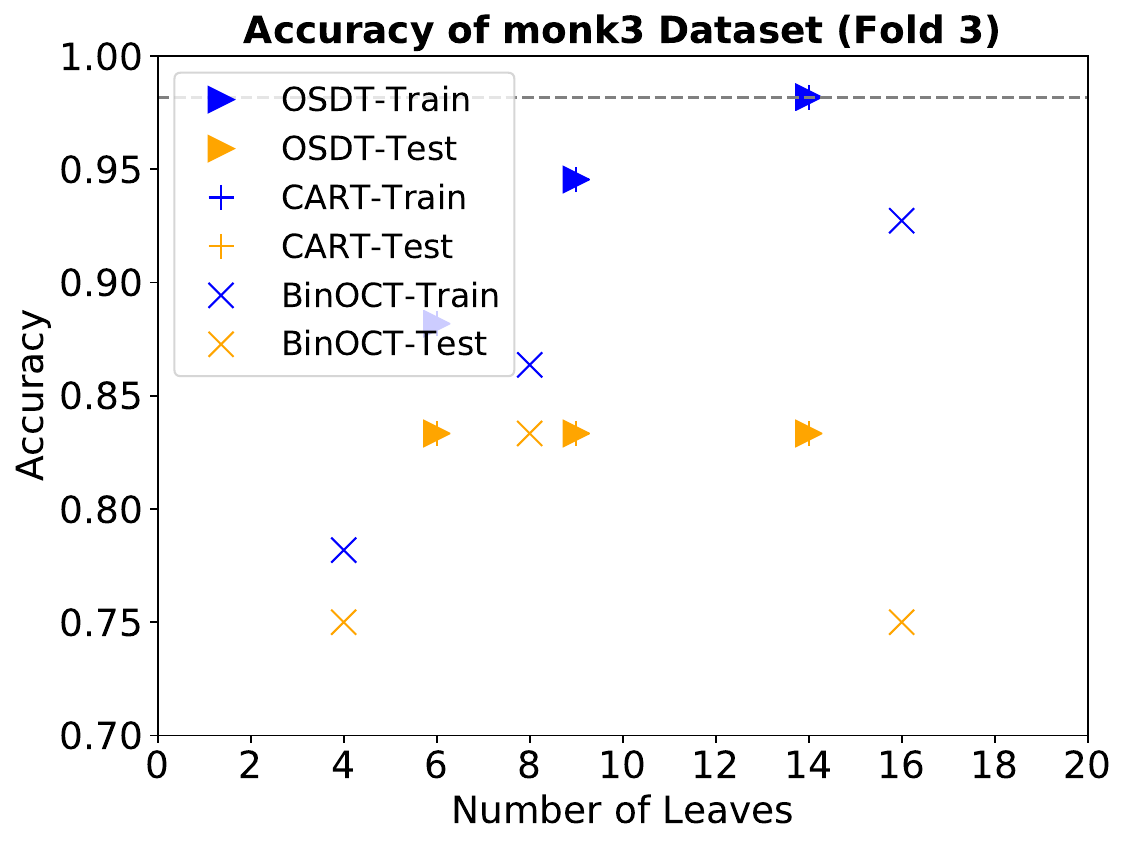}
        \label{fig:monk33}
    \end{subfigure}
    \begin{subfigure}[b]{0.4\textwidth}
        \centering
        \includegraphics[trim={0mm 12mm 0mm 15mm}, width=\textwidth]{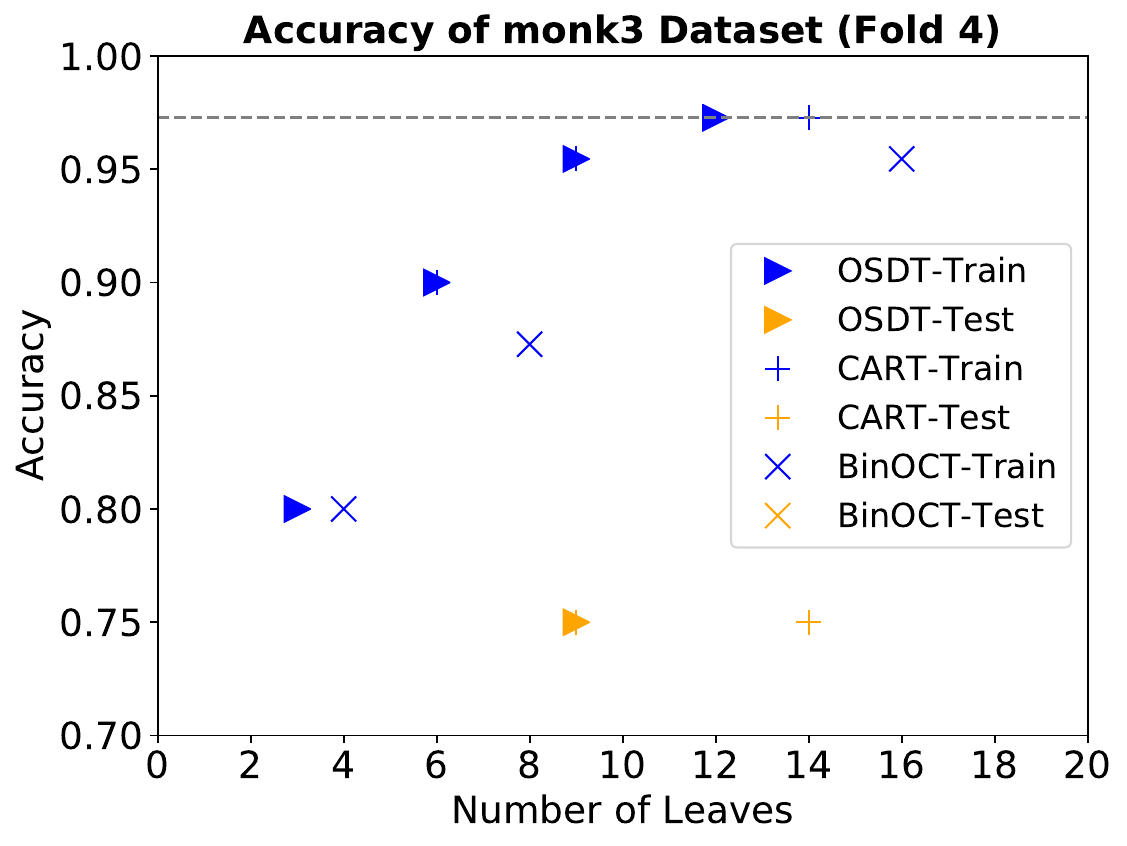}
        \label{fig:monk34}
    \end{subfigure}
    \vskip\baselineskip
    \begin{subfigure}[b]{0.4\textwidth}
        \centering
        \includegraphics[trim={0mm 12mm 0mm 15mm}, width=\textwidth]{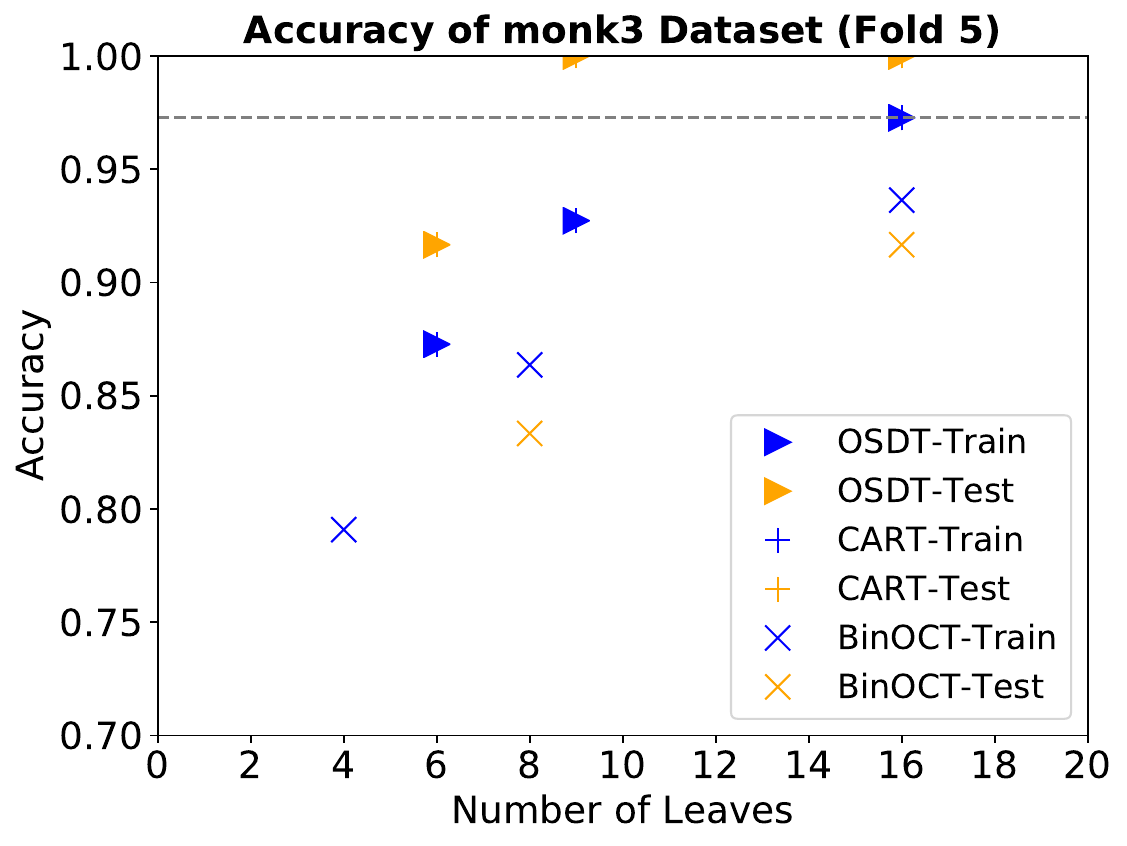}
        \label{fig:monk35}
    \end{subfigure}
    \begin{subfigure}[b]{0.4\textwidth}
        \centering
        \includegraphics[trim={0mm 12mm 0mm 15mm}, width=\textwidth]{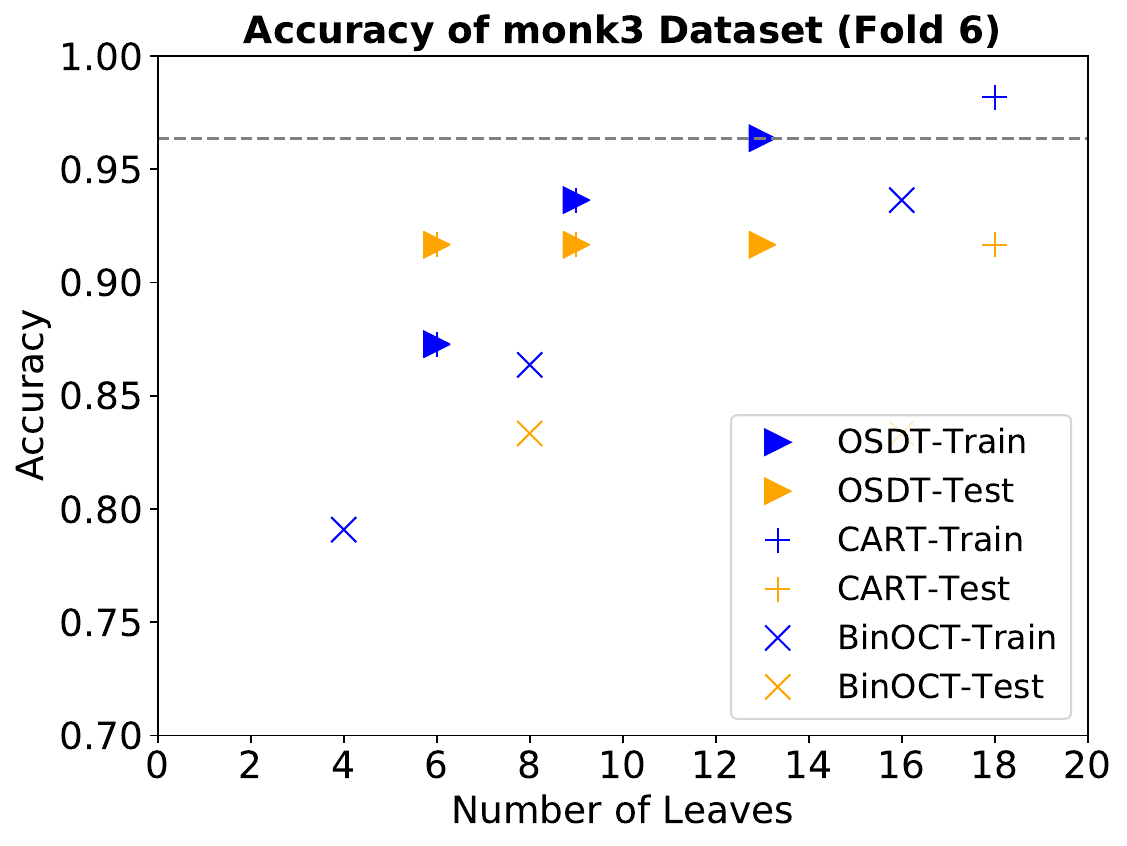}
        \label{fig:monk36}
    \end{subfigure}
    \vskip\baselineskip
    \begin{subfigure}[b]{0.4\textwidth}
        \centering
        \includegraphics[trim={0mm 12mm 0mm 15mm}, width=\textwidth]{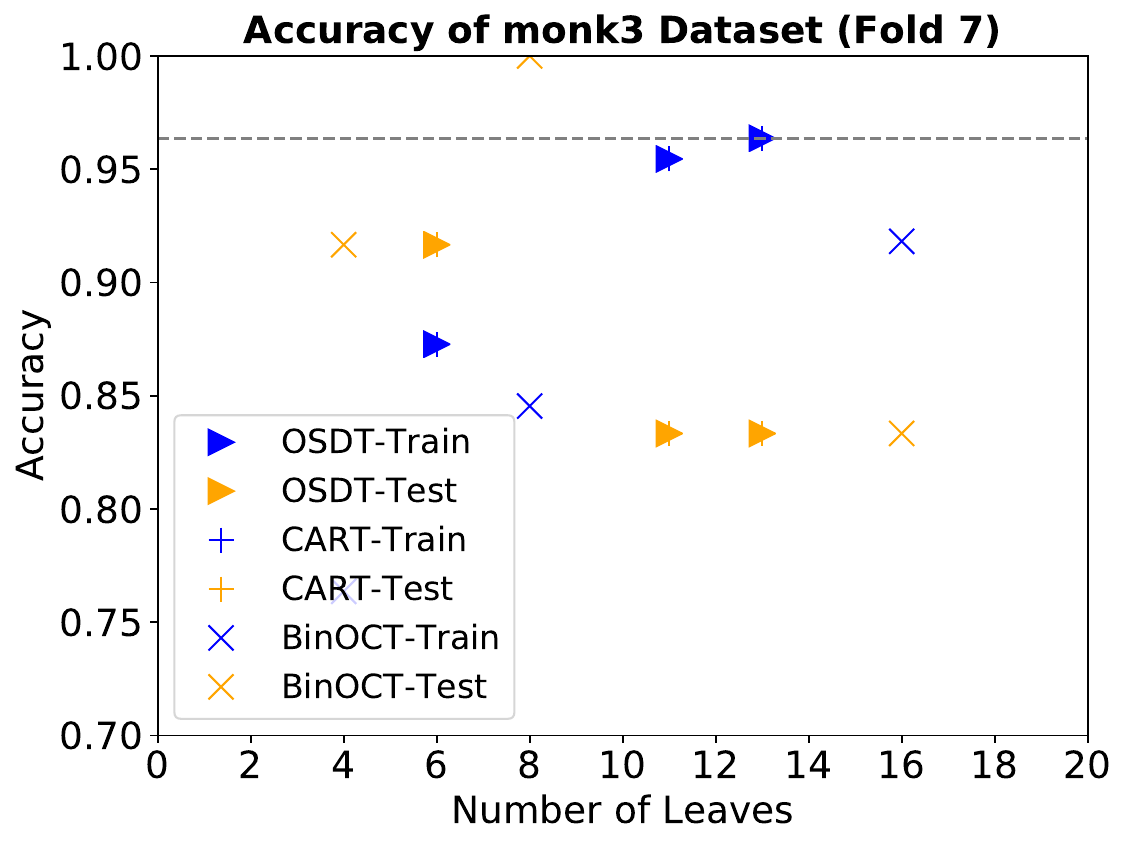}
        \label{fig:monk37}
    \end{subfigure}
    \begin{subfigure}[b]{0.4\textwidth}
        \centering
        \includegraphics[trim={0mm 12mm 0mm 15mm}, width=\textwidth]{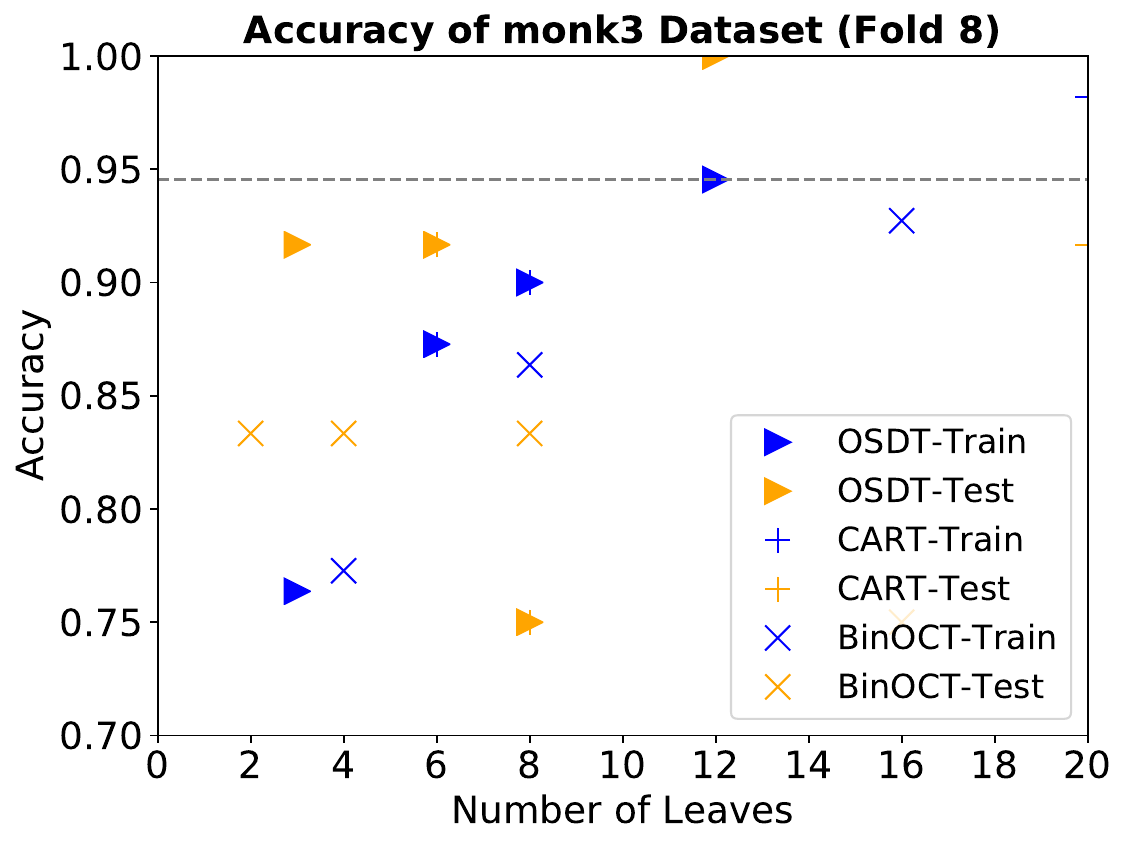}
        \label{fig:monk38}
    \end{subfigure}
    \vskip\baselineskip
    \begin{subfigure}[b]{0.4\textwidth}
        \centering
        \includegraphics[trim={0mm 12mm 0mm 15mm}, width=\textwidth]{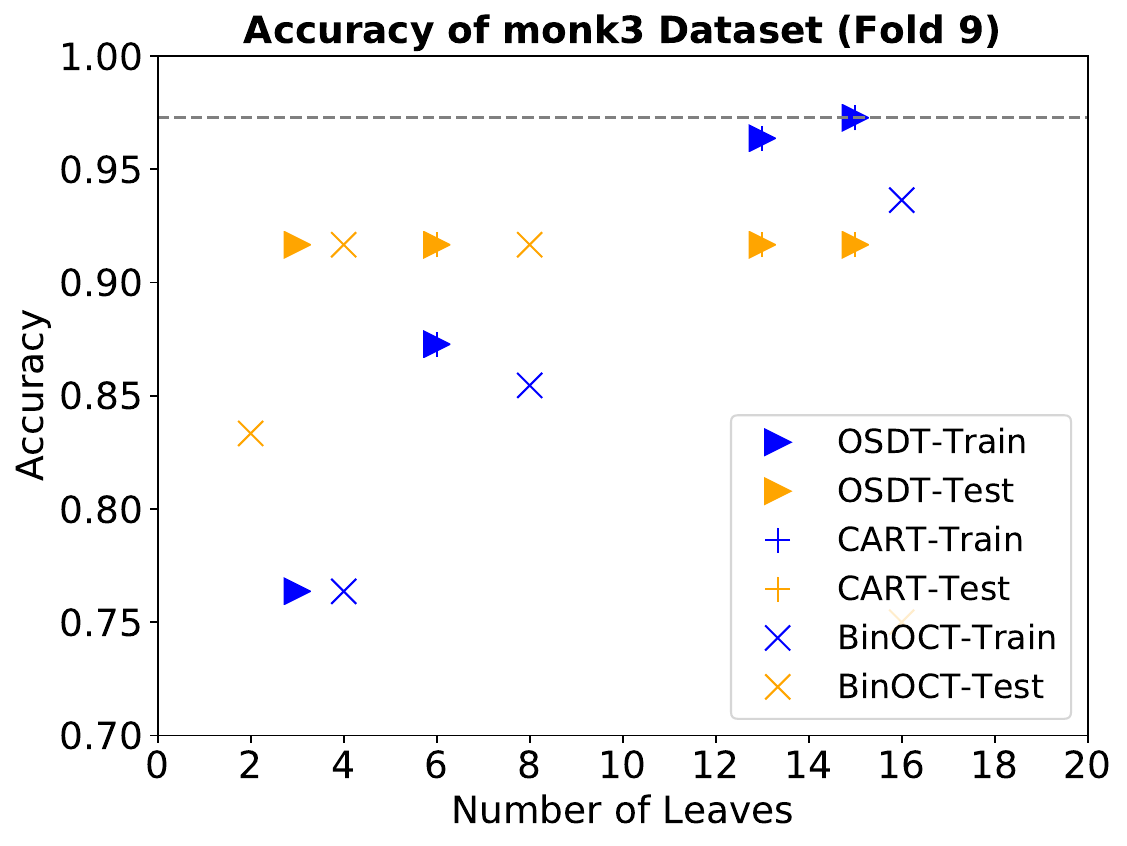}
        \label{fig:monk39}
    \end{subfigure}
    \begin{subfigure}[b]{0.4\textwidth}
        \centering
        \includegraphics[trim={0mm 12mm 0mm 15mm}, width=\textwidth]{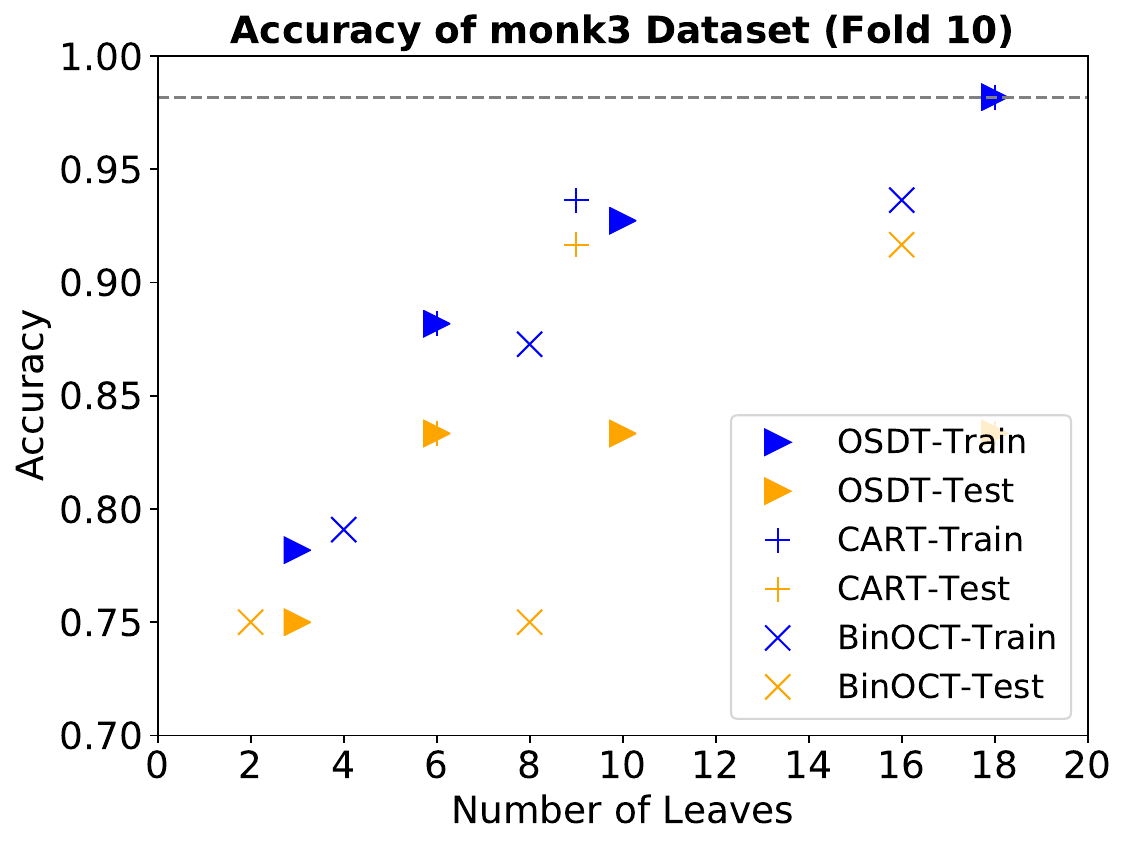}
        \label{fig:monk310}
    \end{subfigure}
    \vskip\baselineskip
    \caption{10-fold cross-validation experiment results of OSDT, CART, BinOCT on Monk3 dataset. Horizontal lines indicate the accuracy of the best OSDT tree in training. %
    }
    \label{fig:cv-monk3}
\end{figure*}

\end{document}